%% file: main.tex
\documentclass[10pt]{article} 
\usepackage[preprint]{tmlr}
\usepackage{xspace}
\usepackage{wrapfig}
\usepackage{booktabs}

\input{math_commands.tex}

\usepackage{hyperref}
\usepackage{url}

\usepackage{graphicx}
\usepackage{subcaption}
\graphicspath{ {./Figures/} }

\usepackage{amsthm}
\usepackage{longtable}
\usepackage{multirow}

\title{Robust High-Dimensional Mean Estimation With Low Data Size\update{, an Empirical Study}}

\author{\name Cullen Anderson \email cyanderson@umass.edu \\
      \addr University of Massachusetts Amherst
      \AND
      \name Jeff M. Phillips \email jeffp@cs.utah.edu \\
      \addr University of Utah
}

\newtheorem{theorem}{Theorem}

\newtheorem{corollary}{Corollary}[theorem]

\newcommand{\N}{\mathcal{N}}
\renewcommand{\eps}{\varepsilon}

\newcommand{\update}[1]{#1}

\newcommand{\estname}[1]{{\text{\sffamily #1}}\xspace}

\newcommand{\sample}{\estname{sample\_mean}}
\newcommand{\gsample}{\estname{good\_sample\_mean}}
\newcommand{\coordmed}{\estname{coord\_median}}
\newcommand{\coordprune}{\estname{coord\_trimmed\_mean}}
\newcommand{\medmean}{\estname{median\_of\_means}}

\newcommand{\geomed}{\estname{geometric\_median}}
\newcommand{\lvog}{\estname{lee\_valiant}}
\newcommand{\lvsim}{\estname{lee\_valiant\_simple}}
\newcommand{\lrv}{\estname{LRV}}
\newcommand{\ev}{\estname{ev\_filtering}}

\newcommand{\evln}{\estname{ev\_filtering\_low\_n}}

\newcommand{\pgd}{\estname{PGD}}
\newcommand{\que}{\estname{QUE}}
\newcommand{\queln}{\estname{QUE\_low\_n}}
\newcommand{\lmin}
{\estname{$\ell_p$\_min}}
\newcommand{\lminln}{\estname{$\ell_p$\_min\_low\_n}}
\newcommand{\etal}{\emph{et.al. }}

\def\month{02}  
\def\year{2025} 
\def\openreview{\url{https://openreview.net/forum?id=1QeI99nH9k}}

\begin{document}

\maketitle

\begin{abstract}

Robust statistics aims to compute quantities to represent data where a fraction of it may be arbitrarily corrupted.   The most essential statistic is the mean, and in recent years, there has been a flurry of theoretical advancement for efficiently estimating the mean in high dimensions on corrupted data.  While several algorithms have been proposed that achieve near-optimal error, they all rely on large data size requirements as a function of dimension. In this paper, we perform an extensive experimentation over various mean estimation techniques where data size might not meet this requirement due to the high-dimensional setting.  

For data with inliers generated from a Gaussian with known covariance, we find experimentally that several robust mean estimation techniques can practically improve upon the sample mean, with the \emph{quantum entropy scaling} approach from Dong \etal (NeurIPS 2019) performing consistently the best.  However, this consistent improvement is conditioned on a couple of simple modifications to how the steps to prune outliers work in the high-dimension low-data setting, and when the inliers deviate significantly from Gaussianity. In fact, with these modifications, they are typically able to achieve roughly the same error as taking the sample mean of the uncorrupted inlier data, even with very low data size. In addition to controlled experiments on synthetic data, we also explore these methods on large language models, deep pretrained image models, and non-contextual word embedding models that do not necessarily have an inherent Gaussian distribution.  Yet, in these settings, a mean point of a set of embedded objects is a desirable quantity to learn, and the data exhibits the high-dimension low-data setting studied in this paper.  We show both the challenges of achieving this goal, and that our updated robust mean estimation methods can provide significant improvement over using just the sample mean. We additionally publish a library of Python implementations of robust mean estimation algorithms, allowing practitioners and researchers to apply these techniques and to perform further experimentation.
\end{abstract}

\section{Introduction}
\label{sec:intro}

Given samples from an unknown distribution, mean estimation is perhaps the most-fundamental and oldest problems in data analysis.  And it is even more relevant in modern analysis for learning and AI tasks where data \update{is very high dimensional}, and there is little else one can reliably compute -- at least not without first grappling with the mean.  

In the past several years, there has been a flurry of theoretical advancement on this topic, including improved asymptotic bounds~\citep{lee2022optimal,gupta2023finite, catoni2011challengingempiricalmeanempirical, gupta2024catoni, lugosi2017subgaussianestimatorsmeanrandom}, 
and the development of more robust methods for dealing with adversarially corrupted data distributions~\citep{lai2016agnostic,diakonikolas2017being,diakonikolas2019robust,cheng2019fast,dong2019quantumentropyscoring,deshmukh2022robustmean}.  

This paper supports this development in two key ways:
\begin{enumerate}
\item We provide a large experimental study of many new methods, which had not been thoroughly compared.  In the non-corrupted case, with moderate data size we do not see substantial improvement over the classic sample mean approach.  However, in the corrupted setting, we find that some \update{methods can significantly improve upon the sample mean.  In some cases consistent improvement on the sample mean requires adjustments that we develop.} \update{As a summary,} the quantum entropy scaling approach of \cite{dong2019quantumentropyscoring} \update{(using an adjustment we describe)} consistently performs the best \update{as long as inliers are reasonable similar to Gaussian}, and often basically matches the mean of the (unknown) inlier data.  

\item We bring to the fore the $d > n$ setting, where there are more dimensions $d$ than data points $n$, or at least we do not have $n$ as substantially larger than $d$.  This setting is becoming more common as dimensionality grows, but has not typically been considered because the theoretical advancements did not provide exciting new bounds here.  In this setting, we revisit some algorithmic derivations and empirically explore what is possible.  
\update{In particular, we revise a common and critical outlier pruning step, and the key adjustment is ultimately simple: a $\sqrt{d/n}$ term, which vanishes when $n \gg d$, needs to be included in a key threshold.  This is detailed in Section \ref{sec:new-algo}. }

\end{enumerate}

Our experimental study considers mean estimation in a variety of settings, focusing on when $n < d$ or $n$ is not much larger than $d$. 
While other experimental studies have been done, many like in \cite{diakonikolas2017being} provided a comparison in the $n \gg d$ case.  And while \cite{deshmukh2022robustmean} has some experiments with $n$ not much larger than $d$, these are not nearly as comprehensive as our study.  
First, we consider standard Gaussian data with known covariance, and no corruption.  
Then we extend this to the setting with various types of adversarial corruption.  
\update{We also consider some limited cases with unknown covariance.  However, because straight-forward adaptations of mean-estimation approaches towards estimating covariance (mapping to a ${d \choose 2}$-dimensional problem) further stresses the need for data size $n$ as a function of $d$, we defer a thorough exploration of this challenge to future work.}  
Finally, we consider real world data scenarios where data is generated via embeddings resulting from large language models, deep pretrained image models, and word embedding models; here we do not have direct enforcement of Gaussianity of the data, but desire a high-dimensional mean nonetheless.  
In all cases, we consider a wide variety of efficient mean estimation approaches, including both classical ones and modern ones with stronger guarantees in the large $n$ setting.  
We provide an anonymous link to our code for easy reproducability here:
\url{https://github.com/cullena20/RobustMeanEstimation}.

\section{Background}
\label{sec:background}

\begin{wrapfigure}{r}{0.255\linewidth}
\vspace{-4mm}
\begin{tabular}{cl}
    \toprule
    & key notation \\ \midrule 
    $n$ & \# data samples \\
    $d$ & \# dimensions \\
    $\eps$ & error bound \\
    $\eta$ & true corruption \\ 
    $\tau$ & expected corruption \\
    \bottomrule
\end{tabular}
\vspace{-4mm}
\end{wrapfigure}

We consider as input a set $X \subset \R^d$ of $n$ samples from an unknown distribution, and the goal is to estimate the mean of that distribution.  Consider first the case where the distribution is the Gaussian $\N_d(0,I)$ where $I$ is the identity matrix representing an \update{isotropic} covariance. 
\update{For $x \sim \N_d(0,I)$ we have $\E[\|x\|^2] = d$.  For the sample mean $\bar x \in \R^d$ from $n$ points drawn iid from $\N_d(0,I)$ we have $\E[\|\bar x\|^2] = d/n$ and more importantly it strongly concentrates as $\Pr[| \|\bar x\|^2 - d/n | > t] \leq 2 \exp(- C t^2)$ for a constant $C$~\citep{vershynin2011randommatrices}. 
This implies for $n = \Omega(d /\eps^2)$ we have $\|\bar x\| < \eps$ with high probability; but for $d > n$ we do not get useful concentration results.  
The Gaussian is the most studied and used distribution for many reasons including that it has Normal marginals for any dimension, is easy to sample from, models an $\ell_2$ loss, and is the limiting distribution of the central limit theorem.  As such, it is our main object of study.  
However, we note that other distributions have distinct behavior for the large $d$ setting.  For instance, for $n$ samples from a distribution with mean $\mu$ and covariance $\Sigma$, the expected squared deviation from the mean in $d$ dimensions can be bounded by $\Tr(\Sigma)/n$ (c.f., \citep{lee2022optimal}).  This implies for instance if $X$ is drawn uniformly from a unit ball (so $\Tr(\Sigma) = 1$) or other distributions with bounded $\Tr(\Sigma)$, then the behavior for $d > n$ can still be well-concentrated. 
  On the other hand, other unbounded and heavy-tailed distributions where, like Gaussians, $\Tr(\Sigma) = \Theta(d)$}
\footnote{\update{We use standard asymptotic notation so for some constants $C_1, C_2, C_3$ and functions $f,g$ then 
$g(x) = O(f(x))$ implies 
$\forall x > C_3$ then $g(x) \leq C_1 f(x) + C_2$; 
$g(x) = \Omega(f(x))$ implies 
$\forall x > C_3$ then $g(x) \geq C_1 f(x) + C_2$, with possibly different constants; and 
$g(x) = \Theta(f(x))$ implies $g(x) = O(f)$ and $g(x) = \Omega(f(x))$.}}\update{, present similar challenges in the $d > n$ setting. }

\paragraph{Corrupted data models.}
Another setting considers some fraction $\eta \in (0, \frac{1}{2})$ of the data to be adversarially corrupted from $X$~\citep{huber1964robust,diakonikolas2023algorithmic}.  Under the \emph{Huber model}, we draw data $X \sim (1-\eta) P + \eta Q$ where $P$ is the set of inliers with mean $\mu$ (we consider $P = \mathcal{N}_d(\mu, I)$ as identity covariance Gaussian data), and $Q$ is any adversarial outlier distribution.  The stronger \emph{total variation} corruption model first draws $X' \sim P$ (with mean $\mu$), and then creates $X$ by adversarially changing any $\eta$-fraction of $X'$ to a new location. That is, it can also adversarially subtract data from the inlier data in addition to adding outliers.  
How accurately can we recover the mean $\mu$ under these settings?  We mostly focus on the Huber model, and observe that subtractive corruption (a component of the stronger total variation model) can induce a consistent and hard to avoid error, and does not seem to expose significant differences between approaches.

As the mean minimizes the sum of squared deviations, the sample mean is very susceptible to outliers.  A single point of corruption can arbitrarily affect the sample mean.  On the other hand, such corruption can be easily detected by filtering out the furthest points from the sample mean, and recomputing the sample mean on the remainder of the data.  A more challenging setting relocates points to roughly $\sqrt{d}$ from the mean, where the inliers are, but all in a tight cluster; then no individual points can be so easily filtered, but the sample mean can be given a non-trivial bias of as much as $\Omega(\eta\sqrt{d})$.  We will empirically consider a variety of challenging $\eta$-corruption situations.

For many years\update{, when dealing with} high dimensions, practitioners were faced with either potentially large error (e.g., on order of $\eta \sqrt{d}$) in using the sample mean or other generalizations of the median~\citep{small1990survey}, or one could spend time exponential in $d$ and return an estimator that is guaranteed to be close to the true mean~\citep{tukey1975mathematics} \update{(or c.f., }\citep{chen2015robustcovariance,zhu2020does}). Around 2016, two papers broke this barrier \citep{lai2016agnostic} and \citep{diakonikolas2019robust}.  They considered $X \sim \N_d(\mu,I)$, and allowed an $\eta$ fraction of the data to be corrupted and return an estimate of the mean $\hat \mu$ so that $\|\mu - \hat \mu\| \leq O(\eta \sqrt{\log 1/\eta})$ or $\leq O(\eta \sqrt{\log d})$.  These works however assume $n = \Omega(d/\eta^2)$; otherwise one runs into the roadblock that even the sample mean of the inliers (the uncorrupted points) has more than $\eta$ error.  
Since then, much follow-up work has furthered our understanding.  Some work~\citep{dong2019quantumentropyscoring, cheng2019fast, depersin2019nearlylinear} improved the time complexity of robust mean estimation algorithms, and our understanding of the problem's hardness~\citep{diakonikolas2017statisticalquerylowerbounds, hopkins2019hardrobustmeanestimation}.  Others provide formulations where gradient descent can be used despite non-convexity~\citep{cheng2020graddescent, zhu2020generalizedquasigradients}.  There has also been effort to improve other robust statistics tasks such as covariance estimation \citep{chen2015robustcovariance, chen2017robustcovariance, cheng2019fastrobustcovariance}, sparse estimation \citep{balakrishnan2017robustsparse, diakonikolas2019sparseestimation, cheng2022outlierrobustsparseestimationnonconvex, diakonikolas2022sparseestimation, diakonikolas2024robustsparse}, list decodable learning \citep{charikar2017learninguntrusteddata, diakonikolas2017listdecodable}, robustly learning mixtures of Gaussians \citep{bakshi2022robustgausmix}, robust optimization \citep{diakonikolas2019sever, prasad2018robustgradient}, robust regression \citep{diakonikolas2018robustregression, klivans2020robustregression}, or in the context of adversarial machine learning~\citep{tran2018backdoorattacks}.  \update{Importantly, robust statistics are more amenable to differential privacy, in particular to privacy through noise addition, and privacy mechanisms are naturally robust \citep{dwork2009differential, liu2021robustdifferentiallyprivate, hopkins2023robustnessprivacy, asi2023robustness}.}  Recent work has also expanded methods for different corruptions models \citep{liu2021coordcorr, zhu2020resilience}. 
For a more thorough review see the recent textbook by \cite{diakonikolas2023algorithmic}.

There has also been significant complementary work in mean estimation under heavy-tailed distributions~\citep{lugosi2021robust, lugosi2022mean, gupta2024catoni, catoni2011challengingempiricalmeanempirical, lugosi2017subgaussianestimatorsmeanrandom, devroye2015subgaussianmeanestimators, lee2022optimal}; see the recent survey by \cite{lugosi2019heavytailsurvey}. Recent work has also developed connections between optimality under heavy-tailed distributions, and optimality in the Huber corruption setting \citep{prasad2019unifiedrobustheavy}.

\section{Mean Estimation Algorithms}
\label{sec:algos}

Here we will document the mean estimation algorithms considered in this paper.  Some are classic, and we also include several ones from the recent literature designed to be potentially practical and algorithmically efficient. 
Some include asymptotic theoretical bounds which use astronomical constants; we make a best effort to replace them with reasonable values so they remain practical.  
Some use \update{an expected corruption} parameter $\tau$, meant to be an upper bound true corruption, $\eta$.
The ones we consider are as follows:

\textbf{\sample}: The \emph{sample mean} simply returns $\hat \mu = \frac{1}{|X|} \sum_{x \in X} x$.  

\textbf{\coordmed}: The \emph{coordinate-wise median} computes the median of each coordinate individually so $\hat \mu_j = \mathsf{median}(\{x_{i,j} \mid x_i \in X \})$.  

\textbf{\coordprune}: First compute a \emph{trimmed mean estimator} for each coordinate individually, parameterized by a value $\tau \in (0,1)$.  That is, in one dimension, it sorts the data, and removes $\tau |X|$ points which have the smallest values, and also removes $\tau |X|$ with largest values.  Then it computes the mean of the remaining $(1-2\tau)|X|$ points.  The \emph{coordinate-wise trimmed mean} applies this estimator separately for each coordinate; which points are removed in coordinate $j$ have no bearing on which points are removed from coordinate $j'$~\citep{lugosi2021robust}.

\textbf{\medmean}: Split the data into $k$ chunks, find the mean of each chunk, take the coordinate wise median of these $k$ means~\citep{lugosi2019heavytailsurvey, minsker2023ustatisticsgrowingordersubgaussian, minsker2023efficientmedianmeansestimator}. As a default, we set $k=10$; this hyperparameter is explored in Appendix \ref{app:hp_tuning}.

\textbf{\geomed}: The \emph{geometric median} is the point which minimizes the sum of distances to all sample points. This is iteratively approximated using the Weiszfeld algorithm~\citep{small1990survey,vardi2001modified}.

\textbf{\lvog}: (\cite{lee2022optimal}) The Lee and Valiant algorithm first estimates the mean $\mu'$ on a $\gamma$ percentage of data points $X_\gamma$ using a mean estimator.  It then centers all points to $X' = \{x' = x - \mu' \mid x \in X\}$.  Let $X_t$ be the $t$ points in $X$ so their corresponding $x'$ have the largest norm.  Let $X'_*$ be the subset consisting of $x' \in X'$ with their corresponding points \emph{not} in $X_\gamma$ or in $X_t$.  
Then return $\mu' + \frac{1}{|X|} \sum_{x' \in X'_*} x'$. 
Rather than the extremely large constants in the original paper, we set $\gamma = 0.5$ and $t=\tau |X|$.  As default, we use $\medmean_k$ estimator with $k=10$ to obtain the initial mean estimator $\mu'$.

\textbf{\lrv}: (\cite{lai2016agnostic})
The \lrv method recursively reduces the dimension by half, until $1$ or $2$ dimensions remain.  Following the original author's code\footnote{\url{https://github.com/kevinalai/AgnosticMeanAndCovarianceCode}}, in the $(\leq 2)$-dimensional base case, it returns \coordmed. The recursive step has three components.  
First, it calculates a weight $w_i$ for each point $x_i$ as $w_i = \exp(-\|x_i - a\|^2/(C s^2))$ where $s^2$ is a robust sample estimate of the trace of the true covariance matrix, $a$ is a rough estimator of the mean chosen as \coordmed, and $C$ is a hyperparameter. We use $C=1$; this hyper parameter is explored in Appendix \ref{app:hp_tuning}. 
Second, it computes $\mu_{w} = \frac{1}{|X|}\sum_{x_i \in X} w_i x_i$, which is the weighted mean of the input, and $\Sigma_{w} = \frac{1}{|X|} \sum_{x_i \in X} w_i (x_i - \mu_{w}) (x_i - \mu_{w})^T$, which is the weighted covariance of the input.  Let $V$ by the span of the top $\lfloor d/2 \rfloor$ singular vectors of $\Sigma_w$; let $V_\perp$ be the span of the bottom $\lceil d/2 \rceil$ singular vectors of $\Sigma_w$.  
Third, recurse on data projected onto $V$, and return an estimate $\mu_1$.  We also build an estimator $\mu_2$ of the data projected onto the $\lceil d/2 \rceil$-dimensional remainder space $V_\perp$ using the weighted sample mean projected onto $V_\perp$: that is $\mu_2 = \frac{1}{|X_\perp|}\sum_{x_i^\perp \in X_\perp} w_i x_i^\perp$ where $X_\perp$ is the data projected onto $V_\perp$. Finally return $\mu_1 + \mu_2$.

\textbf{\ev}: (\cite{diakonikolas2019robust,diakonikolas2019sever}) 
This method observes that when inliers are from a standard Gaussian, then a set of corrupted data which substantially affects the mean estimate must result in a sufficiently large top eigenvalue after centering (i.e., of the sample covariance matrix), and this can be remedied by pruning points which are far along the top eigenvector.  
In this method, if after centering by the sample mean $\hat \mu$, the top eigenvalue exceeds $O(\tau \log {1/\tau})$ (\cite{diakonikolas2017being}\footnote{\url{https://github.com/hoonose/robust-filter}} implements this as $1 + 3 \tau \log(1/\tau)$), then this data is considered additively corrupted along the direction of the top eigenvector. We call this the corruption detection step. Then they consider all points projected onto the associated top eigenvector and sorted $P = \langle p_1, \ldots, p_n\rangle$; and then a set of points furthest from the median $\mathsf{med}(P)$ are pruned.  We call this the pruning step.  The determination of which points to prune is based on those which exceed a Gaussian concentration inequality. Specifically, it finds the smallest index $i$ so $T_i = p_i - \mathsf{med}(P) - 2\tau$ satisfies $\frac{n-i}{n} > \gamma (\mathsf{erfc}(T_i /\sqrt{2})/2 + \tau/(d \log(d \tau / 0.1))$, where $\mathsf{erfc}$ is the complementary error function ($1 -$ the cdf of the Normal) and prunes all points $i$ or larger.  Intuitively, the centered projected data is expected to be a standard Normal distribution, and this bound compares the true percentage of points that exceed a threshold, $T_i$, with the probability that points will exceed that threshold, given by $\mathsf{erfc}$ with some slack terms added. Then the algorithm is recursively called with all points not-yet pruned until the top eigenvalue threshold is not violated. This algorithm critically assumes identity covariance and $n = \Omega(d/\tau^2) \gg d$.

\textbf{\que}:
(\cite{dong2019quantumentropyscoring})
Quantum Entropy Scoring, \que for short, scores outliers based on quantum entropy regularization, and returns a mean using the same structure as \ev, but with a modified pruning procedure. Rather than pruning points based on their projection onto the top eigenvalue, points are given outlier scores relevant to all directions. First, calculate the normalized matrix exponential $U=\mathsf{exp}(\alpha \Sigma) / \mathsf{tr}(\mathsf{exp}(\alpha \Sigma))$ where $\alpha \geq 0$ is a hyperparamater and $\Sigma$ is the sample covariance. Then, calculate a vector of quantum entropy scores, $w$, with $w_i$ = $(x_i - \mu')^T U (x_i - \mu')$, where $x_i$ is the $i$th data point and $\mu'$ is the sample mean. This is implemented efficiently using a Chebyshev expansion of the matrix exponential and Johnson-Lindenstrauss approximations. Points with the largest scores are pruned, and the algorithm continues recursively with the remaining points until the top eigenvalue threshold is not violated. Following the original author's code \citep{dong2019quantumentropyscoring}\footnote{\url{https://github.com/twistedcubic/que-outlier-detection}}, we prune $\tau/2$ percentage of points during every iteration. Additionally, while the author's provide a theoretical threshold on the top eigenvalue, the constants are not given. Rather than tuning this threshold, we implement it using the same threshold as \ev; that is $1 + 3 \tau \log{1 / \tau}$. Because of this threshold, the algorithm critically assumes identity covariance and $n = \Omega(d/\tau^2) \gg d$.  We set $\alpha=4$ as in the author code; simple experiments show little variation with $\alpha$ between $0.5$ and $200$.

\textbf{\pgd}:
(\cite{cheng2020graddescent})
Projected Gradient Descent, \pgd for short, frames robust mean estimation as a non-convex optimization problem, and despite non-convexity, directly solves this using gradient descent. \pgd finds a vector, $w$, of outlier scores, which can then be used to return a mean estimate $\mu' = \frac{1}{|X|}\sum_{x_i \in X} w_i x_i$. $w$ is found to minimize the spectral norm of the standard weighted covariance matrix, $\Sigma_w$, subject to the constraint that the weights represent at least a $(1-\tau)$-density fractional subset of the dataset. The vector $w$ is found as an approximate stationary point to this objective by first performing gradient descent on  the spectral norm of the weighted covariance matrix, and then projecting onto the simplex of feasible weight vectors. First, define a function $F(u, w) = u^T \Sigma_w u$. Then, repeat the following for $\gamma$ iterations, where $\gamma$, following the conventions of a code implementation by the same author as the original paper \citep{cheng2021robustlearningfixedstructurebayesian}\footnote{\url{https://github.com/chycharlie/robust-bn-faster}}, is a hyperparameter. Calculate the top eigenvector, $u_t$, of $\Sigma_w$, which corresponds to finding the unit vector $u_t$ such that $F(w, u_t) \geq (1-\tau) \mathsf{max}_{u}F(w, u)$. 
Then, update $w$ as $w = P(w - \alpha \nabla_w F(w, u_t))$ where $P$ projects onto $\Delta_{n, 2\tau} = \{w \in \mathbb{R}^n : \|w\|_1=1$ and $0 \leq w_i \leq \frac{1}{(1-2\tau)n}\}$, and $\nabla_w F(w, u_t)) = X u_t \odot  X u_t - 2 (w^T X u_t) X u_t$ where $\odot$ indicates element-wise multiplication, and $\alpha$ is the learning rate, initialized as $1/n$ and updated dynamically through learning. We set the number of iterations $\gamma=15$; this hyperparameter is explored in Appendix \ref{app:hp_tuning}.

\textbf{\lmin}:
(\cite{deshmukh2022robustmean})
This method frames robust mean estimation as a semi-definite program (SDP). Similar to \pgd, a vector, $w$, of outlier scores is found, and the weighted mean $\mu' = \frac{1}{|X|}\sum_{x_i \in X} w_i x_i$ is returned. The $\ell_p$ norm for hyperparameter $0 \leq p \leq 1$ is maximized with respect to $w$, under the constraint that the top eigenvalue of the weighted covariance matrix is less than a constant. The weight vector $w$ is iteratively updated by solving a SDP until the number of iterations is less than a bound determined by $\tau$, in which case $\hat{\mu}$ defined above is returned. Update $w$ by approximately solving an SDP to maximize $w$ in $\|w\|_1$ over $\Delta_{n, \tau}$. Each step of the optimization problem is convex and can be solved as the following packing SDP:
\[
\mathsf{max}_w \quad \text{s.t.} \quad w_i \geq 0 \; \forall i, \quad \sum_{i=1}^{n} w_i 
\begin{bmatrix}
e_i e_i^T &  \\
 & (x_i - \mu_w)(x_i - \mu_w)^T
\end{bmatrix}
\preceq
\begin{bmatrix}
I_{n \times n} &  \\
 & c_\tau n I_{d \times d}
\end{bmatrix},
\]

where, $c_\tau$ is a function of $\tau$.  
This analysis of this algorithm critically assumes identity covariance and $n = \Omega(d/\tau^2) \gg d$.

\subsection{New Algorithms and Variants}
\label{sec:new-algo}
We also consider a few new methods, with subtle but important extensions of these existing ones.  

The primary insight needed to adapt methods to the $d \geq n$ case is found by revisiting how we identify outliers with respect to a $d$-dimensional Gaussian distribution.  The bounds used in the $n \gg d$ case have enough data in each direction $d$ to concentrate, whereas in the $d \geq n$ case we need to account for this additional variance.  The key result leverages a theorem of \cite{vershynin2011randommatrices} to understand the concentration of the top eigenvalue of the sample covariance matrix.

\begin{theorem}
\label{thm:Sigma2-bound-main}
Let $X$ be a $n \times d$ matrix whose entries are independently drawn from $\mathcal{N}(\mu, I)$. Let $\Sigma = \frac{1}{n}(X-\bar{\mu})^T(X-\bar{\mu})$ be the sample covariance matrix of $X$, where $\bar{\mu} = \frac{1}{n} \sum_i X_i$ and $X_i$ is the $i$th row of $X$. Then for every $t > 0$, with probability of at least $1 - 3 \exp(-t^2/2)$, one has 
\[
\  \|\Sigma\|_2 \leq \left(1 + \sqrt{d/n} + t/\sqrt{n} + \frac{\sqrt{d + \sqrt{2d}t + t^2}}{n} \right)^2.  
\]
\end{theorem}

The proof is deferred to Appendix \ref{app:ev_theory}.  A more convenient form shows that the fourth term is lower-order and can be absorbed into the probability of failure.  

\begin{corollary}
Under the same setting as Theorem \ref{thm:Sigma2-bound-main}, if one assumes $d/n \leq 16, n \geq 16, t \geq 5$, then with probability of at least $1 - 3 \exp(-t^2/8)$, one has 
\[
\  \|\Sigma\|_2 \leq \left(1 +  \sqrt{d/n} + t/\sqrt{n} \right)^2.  
\]
\label{cor:prune-2t-main}
\end{corollary}

\textbf{\evln}:  The \ev algorithm assumes the sample size is $n = \Omega(d/\tau^2)$.  This assumption is used in several parts of the analysis, and it allows the filtering bound to be simplified to $1 + \tau \log(1/\tau)$; however, when $n = o(d/\tau^2)$, this simplification does not hold, and the filtering bound needs to depend on $d$.  
We instead filter points if the top eigenvalue $\lambda_{\max} > (1 + \sqrt{d/n} + t/\sqrt{n})^2$ using Corollary \ref{cor:prune-2t-main}.  We set $t=10$ to achieve almost $100\%$ ($\approx 0.999$) 
success. All other steps of the algorithm remain the same.

\textbf{\queln}: The \queln algorithm extends the same filtering bound as \ev. Although their paper mentions a $O(\sqrt{d/n})$ factor in the error, the code seems to assume $n = \Omega(d/\eps^2)$ and is implemented very similar to \ev.  As this approach does not work under low data size, in our newly proposed variant, we instead filter points if the top eigenvalue $\lambda_{\max} > (1 + \sqrt{d/n} + t/\sqrt{n})^2$ using Corollary \ref{cor:prune-2t-main}.

\textbf{\lminln}: The \lmin algorithm uses the condition that 
the top eigenvalue of the weighted covariance matrix is bounded by $c_\tau n$ where $c_\tau$ is a hyperparameter suggested to be set at  $1 + \tau \log(1/\tau)$.  As previously observed, this threshold does not hold when $n = o(d/\tau^2)$ and to account for this, in our newly proposed variant we set $c_\tau = (1 + \sqrt{d/n} + t/\sqrt{n})^2$, using Corollary \ref{cor:prune-2t-main}.

\textbf{\lvsim}: We use a simplified version of the Lee and Valiant algorithm~\citep{lee2022optimal}, which aligns with an informal description in their abstract.  It completely removes the $\tau$ percentage of points classified as outliers rather than simply downweighting them. That is, it returns $\hat \mu =  \frac{1}{|X'_*|} \sum_{x' \in X'_*} x$; the average of all points $X'_*$ which were not in the original estimate, nor from the pruned set furthest from $\mu'$.

\section{Experiments}
\label{sec:expers}

We generally evaluate the performance of these mean estimation algorithms as data size $n$, dimension $d$, and corruption $\eta$ are varied. Error is measured as the Euclidean distance $\|\mu - \hat \mu\|$ between the true mean $\mu$ and the estimate $\hat \mu$ returned by a mean estimation algorithm. We set the default values as $n=500$, $d=500$, and $\eta=0.1$. We examine the performance as we fix one of these variables and vary the others under various distributions for both the uncorrupted and corrupted data. We first examine uncorrupted standard normal Gaussian data, demonstrating that nothing really improves upon \sample, and observing the robustness of mean estimation techniques when applied to uncorrupted data. We then examine corrupted Gaussian data over various covariances and noise distributions.  The example distributions are chosen among challenging examples in the literature meant to distinguish various models. Experiments were run on a 2022 Macbook Air with Apple M2 Chip, 16GB memory, running MacOS 12. 

At one point in our experiment, the values of $n$, $d$, and $\eta$ are used to generate data according to a supplied data generation function and noise scheme (both of which will vary depending on the experiment). A mean estimate is made on this data using each of the mean estimators being tested. For each mean estimator, error is then stored as the Euclidean distance between the true mean of the data and the returned mean estimate. These errors are accumulated over 5 runs and averaged. We additionally plot the error incurred by the sample mean of the original uncorrupted data, which we call the \gsample error.  This serves as a valuable baseline for comparison. In practice, we can not expect to achieve error better than the sample mean of the inliers. Therefore, a reasonable goal for a robust estimator is to closely match the performance of \gsample, thereby removing the effects of corrupted data points. Could a robust mean estimator somehow improve upon this?  We do not observe this; but we will observe methods that basically match \gsample, even without $n = \Omega(d/\eps^2)$.

\paragraph{Fraction of corrupted data.}
Some algorithms are designed for data where a $\eta$-fraction of the data has been corrupted.  And in some cases, this fraction is taken as a parameter $\tau$ used within the algorithms (\coordprune, \lvsim, \lvog, \ev, \evln, \que, \queln, \pgd, \lmin, \lminln).

In theory, these algorithms work best using their parameter $\tau$ set to the true fraction of corrupted data $\eta$, and may even result in arbitrary error if the parameter $\tau$ is not set to at least an upper bound for the true fraction.  However, increasing the value of $\tau$ in the algorithms also theoretically increases the error incurred by algorithms. 
Recent work by \cite{jain2022robustestimationalgorithmsdont} showed a meta algorithm that allows robust estimation algorithms to perform asymptotically optimal without knowing true corruption $\eta$. We investigate robustness to expected corruption, $\tau$, empirically, in Appendix \ref{app:expected_corruption}, as we fix $\tau$ and vary true corruption $\eta$. We observe that the best algorithms do not show a strong dependence on this relationship, so long as $\tau$ is an upper bound on $\eta$. Hence, for all other experiments, we simply set the parameter $\tau$ according to the true corrupted fraction $\eta$ or to $\tau=0.1$ if the data is not corrupted.

\paragraph{Selecting algorithmic variants.}
There are many algorithms to be considered, and plots can become cluttered.  To reduce this, we perform some comparison among variants. We summarize key findings here, with further details deferred to Section \ref{sec:variants}.  

\begin{table}[h!]
    \centering
    \begin{tabular}{lcccc}
    \hline
    \multirow{2}{*}{\textbf{Algorithm}} & \multicolumn{2}{c}{\textbf{$n=500$, $d=500$}} & \multicolumn{2}{c}{\textbf{$n=200$, $d=500$}} \\
    \cline{2-5}
     & \textbf{Error} & \textbf{Time (s)} & \textbf{Error} & \textbf{Time (s)} \\
    \hline
        \sample  & 2.47 ± 0.04     & 0.00019 ± 0.000002 & 2.74 ± 0.05    & 0.00019 ± 0.000001 \\
        \lrv     & 1.14 ± 0.04     & 0.81 ± 0.12        & 1.76 ± 0.10    & 0.64 ± 0.03 \\
        \pgd     & 1.08 ± 0.02     & 82.4 ± 8.8         & 1.68 ± 0.05    & 72.5 ± 3.2 \\
        \evln    & 1.07 ± 0.02     & 0.20 ± 0.02        & 1.69 ± 0.04    & 0.08 ± 0.03 \\
        \ev      & 13.49 ± 3.56    & 0.48 ± 0.15        & 17.06 ± 5.92   & 0.05 ± 0.02 \\
        \queln   & 1.04 ± 0.03     & 0.71 ± 0.05        & 1.70 ± 0.049    & 0.35 ± 0.03 \\
        \que     & 20.81 ± 0.40    & 2.70 ± 0.08        & 20.88 ± 0.38   & 1.99 ± 0.04 \\
        \lminln  & 1.17 ± 0.04     & 1182.6 ± 35.3      & 1.67 ± 0.03   & 265.9 ± 15.2 \\
        \lmin    & 1.62 ± 0.04    & 1076.8 ± 43.9      &   5.62 ± 0.40  & 250.07 ± 17.2 \\
        \hline
    \end{tabular}
    \caption{Error and Runtime Across Simple Corrupted Identity Covariance Gaussian}
    \label{tab:time}
\end{table}

First, we do not consider \lmin and \lminln in our plots due their exceptionally large runtimes. We report runtimes and errors (defined as the Euclidean distance from the estimated mean to the true mean) of selected algorithms under $n=500$ and $d=500$  and under $n=200$ and $d=500$ over a simple corrupted Gaussian distribution in Table \ref{tab:time}. We report results as the mean $\pm$ the standard deviation, averaged over 5 runs. With the notable exceptions of \lminln, \lmin, and \pgd, most estimators are efficient and took under 3 seconds to run with $n=500$ and $d=500$ for a simple corrupted Gaussian distribution. \lmin and \lminln rely on an SDP solver, which we implement with the cvxpy \citep{diamond2016cvxpy, agrawal2018rewriting} package and the mosek solver. Although this is theoretically efficient, it is slow in practice for the data scale and dimesionality we consider in this paper.  For $n=500$ and $d=500$, both algorithms took about 1100 seconds, or about 18 minutes, to return a mean estimate over a simple corrupted data scheme. 
For that reason, and since we run many trials of each input size and error level, we do not consider these algorithms in our plots. However, we note that employing Corollary \ref{cor:prune-2t-main} for \lminln achieves a noticeable performance increase over \lmin; showing gains from $1.62$ error to $1.17$ error in the $n=500, d=500$ case and from $5.62$ to $1.67$ error in the $n=200, d=500$ case. \pgd is also much slower than other robust estimators, taking approximately 80 seconds to run with $n=500$ and $d=500$. While this significant slow down is relevant when considering a practical algorithm, it is not as prohibitive as \lmin. As a result, we include it in all of our plots.

Second, we observe that when the data does not satisfy that $n \gg d$, then both \ev and \que can have catastrophic failure.  Our variants \evln and \queln avoid this issue in the $d > n$ and $d \approx n$ settings, while basically matching the effectiveness of their original versions when they do not have catastrophic failure.  This result is highlighted in Table \ref{tab:time}, where we observe that both \que and \ev achieve significantly worse error than \queln and \evln respectively. As a result, we use \evln and \queln in all comparisons.

Thirdly, we find that \lvsim performs slightly better than the original \lvog; however the difference is fairly small.  We also do not notice any meaningful advantages from using \lvsim or \lvog with different choices of initial mean estimators. As such, we only use \lvsim in all comparisons.

\subsection{Uncorrupted Gaussian Data with Identity Covariance}
\label{sec:uncorr}

We first evaluate the performance of mean estimation algorithms over uncorrupted Gaussian data with identity covariance. In particular, we draw uncorrupted data $X \sim \mathcal{N}_d(\mu, I)$, where $\mu$ is an arbitrary mean and $I$ is identity covariance. For these experiments, we set $\mu$ to be the all-fives vector, but did not find performance to depend on $\mu$. For algorithms that utilize $\tau$, expected corruption, as input, we use the default value of $\tau=0.1$.

\begin{figure}[h]
\centering
\includegraphics[width=\linewidth]{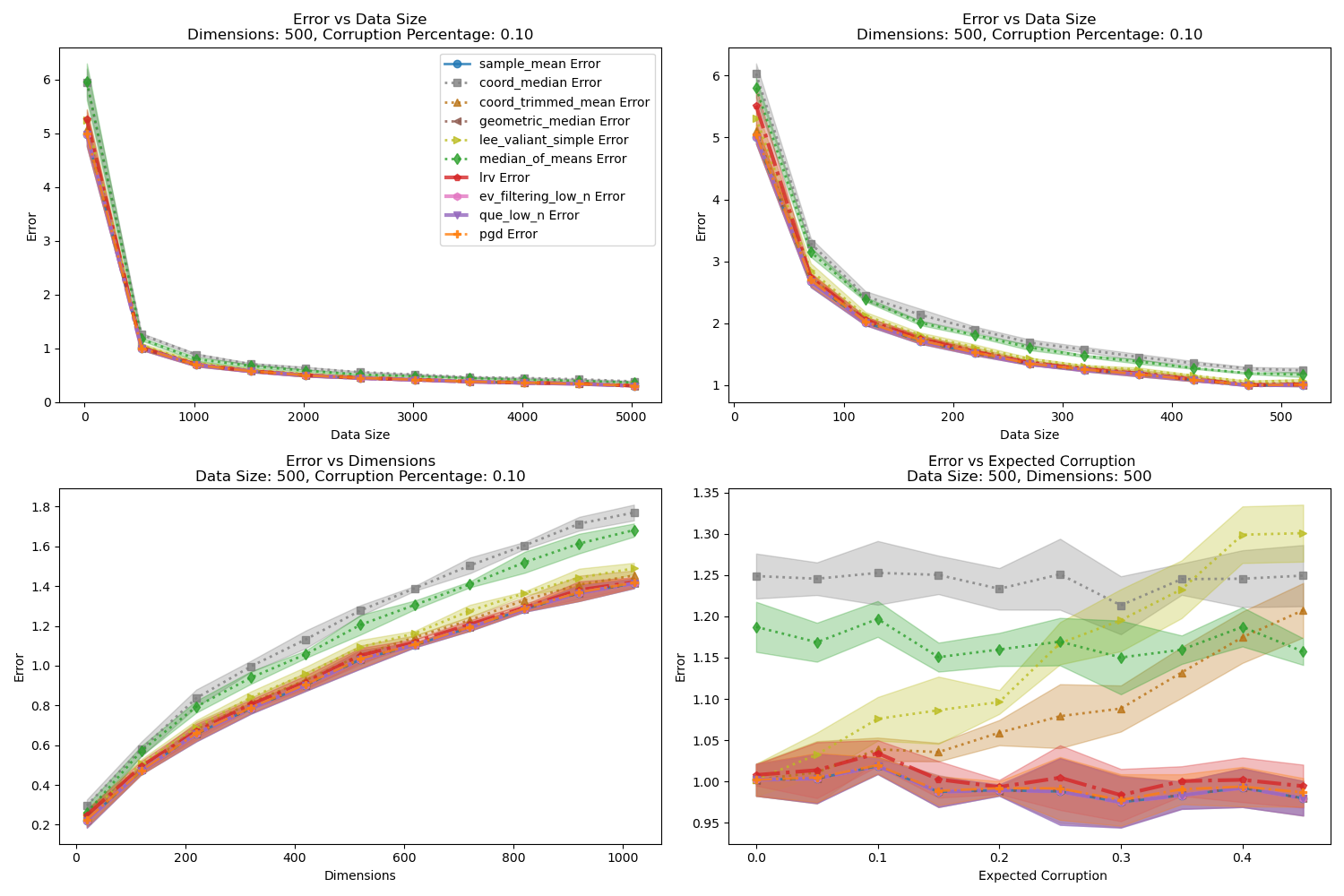} 
\caption{Uncorrupted Gaussian Identity Covariance}
\label{fig:uncorr}
\end{figure}

We provide our first experimental plots in Figure \ref{fig:uncorr}; most further experiments will follow this same set-up, consisting of a set of $4$ charts, each measuring the Error $\|\mu - \hat \mu\|$ on the $y$-axis.  The top two charts vary the data size $n$ along the $x$-axis, but on different scales.  The top left shows a large scale from $n = 20$ to $n=5020$, focusing on the $n > d = 500$ paradigm.  The top right shows $n=20$ to $n=520$, focusing on the $n < d$ paradigm.  
The bottom left plot show the effect of varying the dimension from $d=20$ to $d=1020$ while fixing $n=500$.  
The bottom right shows varying the algorithm's parameter, $\tau$, for the expected noise from $0$ to $0.45$ with fixed $n=500$, $d=500$.  
Each algorithm is shown as a curve, with the average error of $5$ independent data generations at regular intervals on the $x$-axis.  A shaded area is shown at a radius of $1$ standard deviation from that average error value. 

The plots are a bit cluttered because most algorithms perform about the same, including \sample.  No algorithm can be seen to noticeably outperform \sample,
\update{which, as the MLE for this data, and by the Gauss-Markov theorem, is not surprising.    
Methods } \lrv, \evln, \queln, \pgd,  \coordprune, \geomed, and \lvsim have about the same error in most cases.  However, \medmean, and \coordmed perform slightly worse, with the gap becoming more apparent in high dimensions.  Moreover, \lvsim and \coordprune do significantly worse with a higher expected corruption parameter $\tau$. This is a result of expected corruption, $\tau$, being a hyperparameter that directly controls the percentage of points to prune.
Finally, as predicted by basic theory, with $n$ fixed as the dimension $d$ increases, the measured error increases at a rate roughly $\sqrt{d}$.

\begin{figure}[h]
\centering
\includegraphics[width=\linewidth]{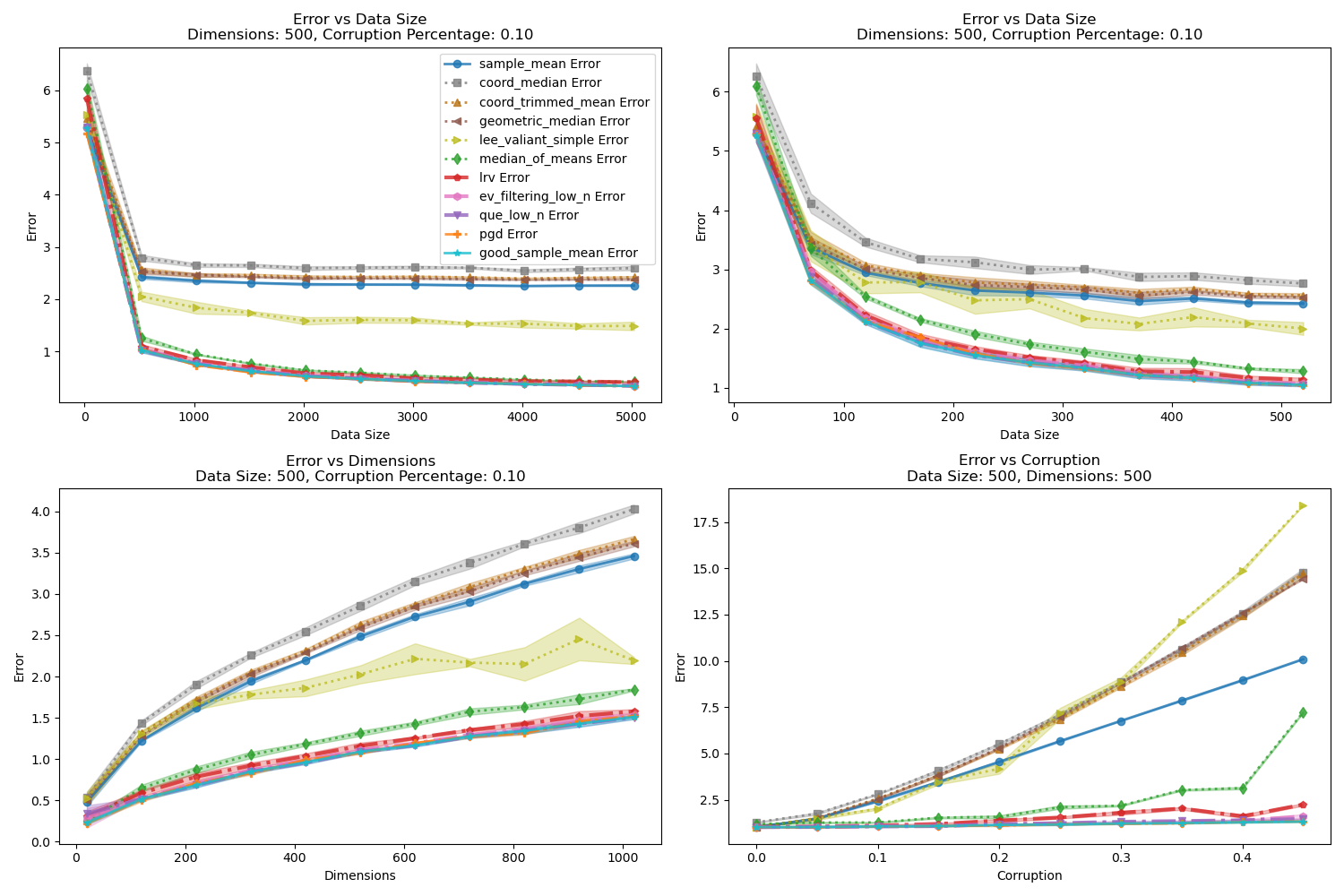} 
\caption{Corrupted Gaussian Identity Covariance: Additive Variance Shell Noise}
\label{fig:sq-d-corrupt}
\end{figure}

\subsection{Corrupted Gaussian Data with Identity Covariance}
\label{sec:corrid}

We evaluate corrupting noise added to Gaussian data with identity covariance.  In particular, we draw $X \sim (1-\eta) P + \eta Q$ where $P = \mathcal{N}_d(\mu,I)$ and $Q$ describes the corrupted data distribution. This is equivalent to the more general case where any covariance $\Sigma$ is known, as we could simply scale the data to have identity covariance, apply these methods, and scale the mean estimate back. We provide a wrapper in our implementation to perform this operation.

\paragraph{Gaussian noise shifted to variance shell.}
We first consider corrupted data distribution 
$Q = \mathcal{N}_d(\mu', \frac{1}{10} I)$ so $\|\mu - \mu'\| = \sqrt{d}$.  
Since $\E_{x \sim P} [\|x - \mu\|^2] = d$, corrupted data from $Q$ is not easily identified.  The location of this cluster is determined by a random rotation at every generation to ensure that no coordinate-axis specific bias is introduced.
This is shown in Figure \ref{fig:sq-d-corrupt} in the same 4 experiments as with uncorrupted data, except now the bottom right figure varies $\eta$, the fraction of corrupted data from $Q$, along the $x$-axis. We set the expected corruption hyperparameter equal to true corruption, that is $\tau = \eta$.  
\update{In Appendix \ref{app:expected_corruption} we explore the relation between expected corruption $\tau$ versus actual corruption $\eta$; for the most part  as long as $\tau > \eta$.}

There is now more clear separation between the algorithms designed for adversarial corruption, and those not. 
Here \evln, \queln, and \pgd do the best among all settings, with \lrv  \update{, perhaps doing the best, even appearing better than \gsample for large dimensions, although within 1 standard deviation error margin}.  Due to the high dimensionality, $d$, \gsample, the sample mean of points from the uncorrupted part of the distribution $P$, does not have error approaching $0$ until $n$ is very large. \evln, \queln, and \pgd work so well that they are nearly overlapping this best possible standard.  
Also, perhaps surprisingly, \medmean also does nearly as well, especially under larger $n$, though it degrades much worse with larger $\eta$. 

In contrast, \coordmed, \sample, \coordprune, \geomed, and \lvsim all do considerably worse, even with large data size.  With large corruption levels, these all even seem to do worse than just \sample, indicating that the algorithms prune the wrong data points or face some other similar issue.

\begin{figure}[h]
\centering
\includegraphics[width=\linewidth]{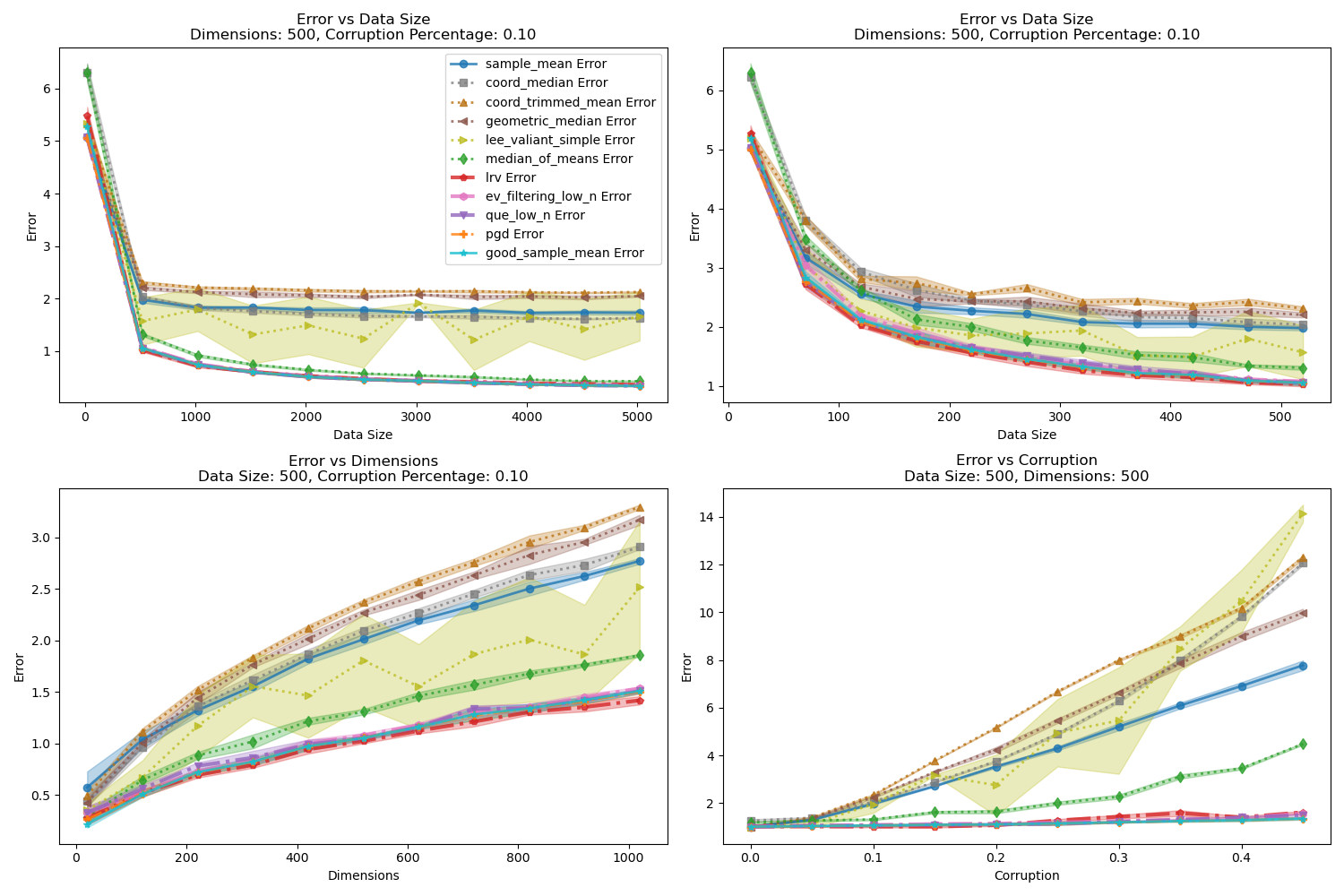} 
\caption{Corrupted Gaussian Identity Covariance: DKK Noise}
\label{fig:DKK-noise}
\end{figure}

\paragraph{Large + Subtle outliers: DKK Noise.} 

We now recreate the noise distribution from \cite{diakonikolas2017being}, which utilizes a more sophisticated corruption scheme that includes both easier and harder to detect outliers. Half of the noise is drawn from the product distribution over the hypercube where every coordinate is -1 or 0 away from the true mean at that coordinate with equal probability. The other half is drawn from the product distribution where the first coordinate is either 11 or -1 away from the true mean at that coordinate with equal probability, the second coordinate is -3 or -1 away from the corresponding true mean coordinate with equal probability, and all remaining coordinates are -1 away from the true mean.
We call this corruption scheme DKK Noise. 
This is shown in Figure \ref{fig:DKK-noise}, with similar results.  \evln, \queln, \pgd, and \lrv achieve performance nearly matching \gsample, with \medmean also doing almost as well -- at least while the dimension $d$ and rate of corruption $\tau$ are on the smaller side.  
Other than \medmean, all classic methods perform noticeably worse than \gsample and achieve similar error to \sample. The only difference of note here is that \lvsim exhibits far larger error bars, suggesting that its performance may vary significantly depending on random initializations made within the algorithm.  Also, \lrv may even outperform \gsample for very large dimensions.

\paragraph{Subtractive noise.}
We additionally consider subtractive noise in Figure \ref{fig:subtractive}. Here, an adversary is able to remove a $\eta$ percentage of points from the data distribution. We implement this by removing the $\eta$-percentage of points which are most extreme in some direction.  Unlike in the additive corruption case, there is a strict upper bound on the error under subtractive corruption from a standard Gaussian distribution; the error induced is bounded as $O(\eta)$ even using \sample, and clustering the subtracted points as most extreme in some direction ensures their effect is $\Omega(\eta)$ under \sample. In general, we wish to consider noise distributions that may add outliers and also remove inliers through such subtractive noise.  However, we do not find any surprising capabilities among methods in this scenario.  As a result, for the remainder of this paper, we focus on additive corruption.

\begin{figure}[h]
\centering
\includegraphics[width=\linewidth]{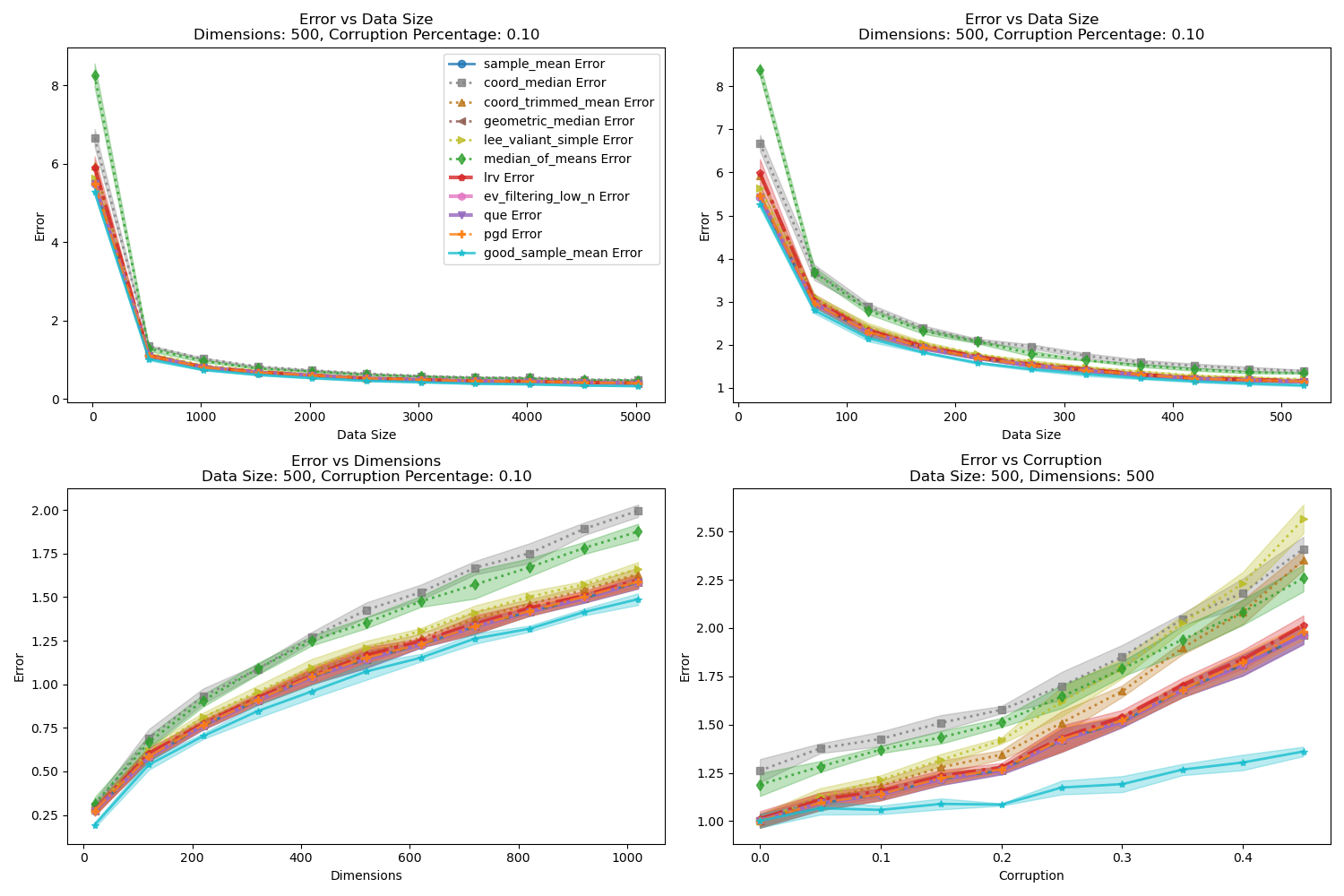} 
\caption{Corrupted Gaussian Identity Covariance: Subtractive Noise}
\label{fig:subtractive}
\end{figure}

Under subtractive corruption, nothing outperforms \sample; and now nothing can match \gsample in error in most settings. However, \evln, \queln, \pgd, and \lrv all nearly match the performance of \sample. Unlike in the previous additive corruption schemes, \medmean performs significantly worse under subtractive corruption, always achieving error notably worse than \sample. 
Among other estimators, \geomed nearly matches \sample error across all settings, \lvsim and \coordprune perform similarly but degrade much more under larger corruption, while \coordmed performs significantly worse.

We find similar results across several other noise distributions. In addition to the hard-to-detect distributions, we also show that \evln, \queln, \pgd, and \lrv are generally robust to arbitrary outliers. These details are deferred to Appendix \ref{app:idcov_morenoise}.

\subsection{Corrupted Gaussian Data with Unknown Covariance}
\label{sec:corrsph}

We now evaluate corrupted Gaussian data for general unknown covariance. Since \evln and \queln rely on the identity covariance assumption, we employ a simple heuristic to adapt these algorithms to the unknown covariance case. We estimate the trace of the covariance as $\Tr(\hat{\Sigma}) = \frac{1}{n-1} \sum_{i = 1}^n ||x_i - \hat{\mu}||^2$, where $\hat{\mu}$ is the sample mean and $\hat{\Sigma}$ is the sample covariance. We rescale the data to $X’ = \{x’_i = x_i/\sqrt{\frac{\Tr(\hat{\Sigma})}{d}} \mid x_i \in X\}$. We then estimate the mean of $X'$, rescale this estimate by $\sqrt{\frac{\Tr(\hat{\Sigma})}{d}}$, and report the results. This heuristic is used for \evln and \queln across all unknown covariance experiments, and not used for any  other algorithms.  The other standard algorithms are either invariant to this linear rescaling, or themselves account for it; this was supported by our own observations.

Another possible method is to utilize a robust covariance estimate instead of a sample trace estimate, as discussed in \cite{diakonikolas2023algorithmic}. We do not evaluate such methods, as this would involve a thorough study into robust covariance estimation methods over low data size, which goes beyond the scope of this work. Naively treating robust covariance estimation as robust mean estimation in $d^2$ dimensions further exasperates issues related to low data size. We also choose to use a simple sample trace estimate rather than the robust approach proposed by \cite{lai2016agnostic}. We find that the approach proposed often results in significant underestimates across difficult noise distributions, causing \evln and \queln to fail catastrophically. We note that these underestimates are potentially more harmful than overestimates as through them, even the inlier data may not pass the threshold, causing continuous pruning. While a sample trace estimate approach is more prone to overestimates, this can be remedied by naively pruning large outliers. \cite{diakonikolas2017being} also provides an algorithm for unknown covariance mean estimation similar to \ev, but the corruption detection threshold is not easily adapted to the low data size case.

\subsubsection{Unknown Spherical Covariance}

We evaluate corrupting noise added to Gaussian data with spherical covariance. We draw $X \sim (1-\eta) P + \eta Q$ where $P = \mathcal{N}_d(\mu,\sigma^2I)$ and $Q$ describes the additive corrupted data distribution. We consider $\mu$ to be the all-fives vector and $\sigma=5$.

\begin{figure}[h]
\centering
\includegraphics[width=\linewidth]{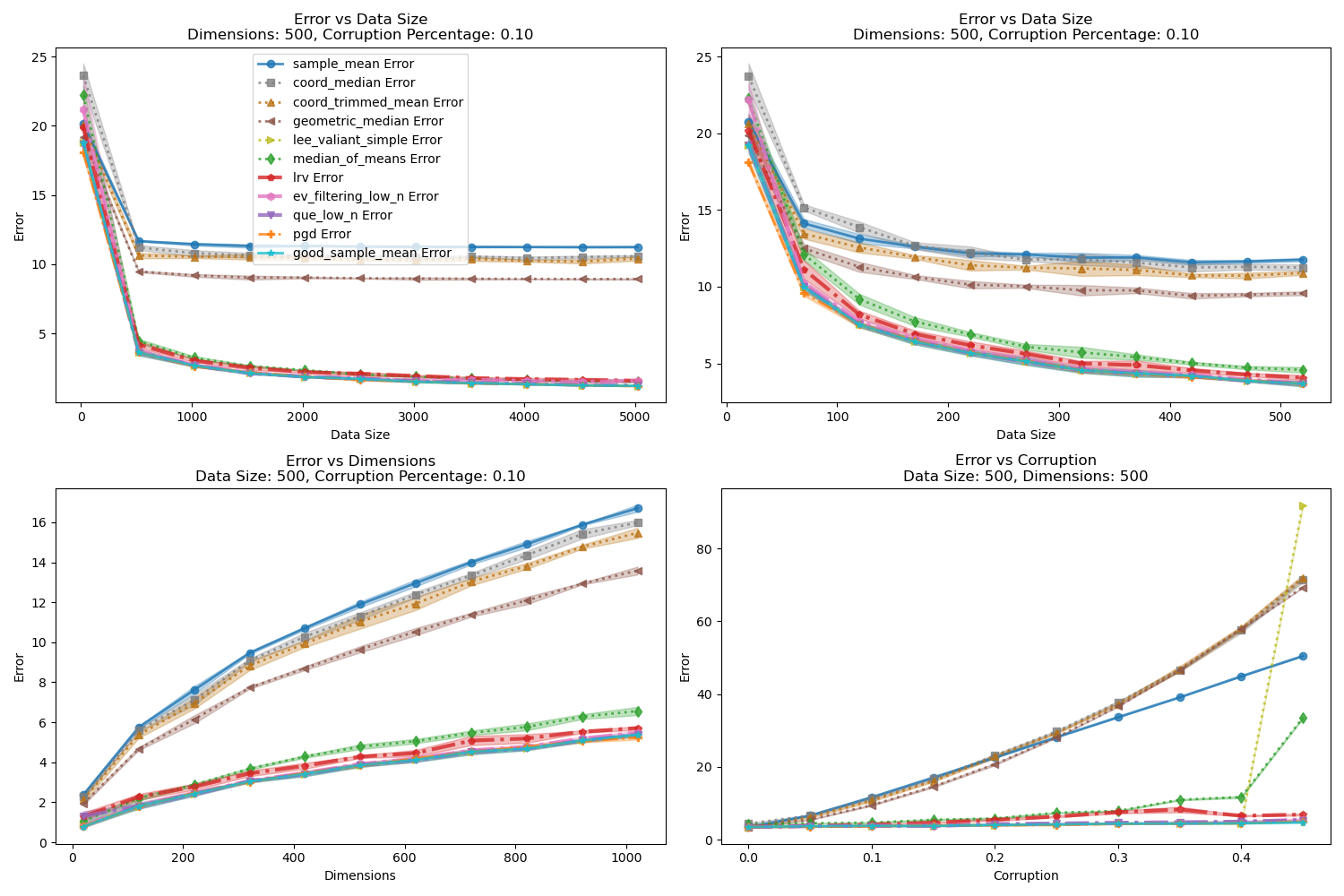} 
\caption{Corrupted Gaussian Large Spherical Covariance: Additive Variance Shell Noise}
\label{fig:largesp_varshell}
\end{figure}

\paragraph{Gaussian noise shifted to scaled variance shell}
We adapt the identity covariance noise distribution models by appropriately scaling coordinates by $\sigma$. We first consider the corrupted data distribution 
$Q = \mathcal{N}(\mu', \frac{1}{10} I)$ so $\|\mu - \mu'\| = \sigma \sqrt{d}$.  
With $P$ now having covariance $\sigma^2I$, $\E_{x \sim P} [\|x - \mu\|^2] = \sigma^2 d$, and corrupted data from $Q$ is not easily identified. Results are show in Figure \ref{fig:largesp_varshell}. 

While the overall error here is higher, matching the theory that even for uncorrupted data, the sample mean is expected to have error of $O(\sigma \sqrt{d/n})$, the relative performance of algorithms is nearly identical to the identity covariance case. \evln, \queln, and \pgd nearly match \gsample error throughout, with \lrv performing only slightly worse. \medmean lags behind both estimators but still performs noticeably better than \sample. However, \lvsim performs much better in this scenario, nearly exactly matching \gsample except with large enough corruption \update{-- where with $\eta = 0.45$, it and \medmean probably confuse which points are inliers and have much worse error}. Other methods perform similarly to \sample or worse.

As in the identity covariance case, we find similar results across noise distributions. The only notable exception is for \evln, which sometimes performs slightly worse and doesn't always converge to \gsample error as $n$ increases, probably due to instabilities in the trace scaling heuristic. We additionally show that relative performance of algorithms is mostly independent of the choice of $\sigma$. The exception to this is \lrv, which notably outperforms all other methods, including \gsample, with large enough $\sigma$ across noise distributions. These details are deferred to Appendix \ref{app:unknownsp_morenoise}.

\paragraph{Unknown Non-Spherical Covariance}
In Appendix \ref{app:unknown_non_sp} we also explore the unknown, non-spherical covariance case.  This is even more sensitive to the covariance estimate, and so is further outside the primary scope of this study. Nonetheless, we continue to observe that the best robust estimators, including \queln, continue to perform well. 

\section{Large Language Model Experiment}
\label{sec:realworld}

To evaluate whether robust mean estimation methods are overly sensitive to distributional assumptions, we evaluate performance over real world data. We first study the problem of estimating the mean of \update{vectors from language models}. \update{Such ``word embeddings'' have had an enormous impact on natural language processing, starting from simple sparse term-frequency vectors~\citep{robertson2009probabilistic}.  Second generation word embeddings (e.g., GloVE~\citep{pennington-etal-2014-glove} and word2vec~\citep{mikolov2013distributed}) made the advancement of creating a ``low" dimensional vector (about 300 dimensions) for each word, where Euclidean (and cosine) distances could be used as a proxy for the similarity between words based on how they are used.  As a serendipitous side-effect, structure emerged where dot-products, means, and linear classifiers made sense in this embedding space~\citep{bolukbasi2016man,dev2019attenuating}.  Third generation embeddings created representations for each word in the context of the nearby words; that is, each use of a word had a different embedding.  These were a first main use of transformer architectures, and implicitly capture more meaning and context with progressive layers of a neural network.  As is most common, we use the last layer of the embedding network as the representation of a word. } \update{Our first study, shown next, uses these third generation word embeddings.}  
Further experiments over deep pretrained image model embeddings and \update{second generation} word embeddings are deferred to Appendix \ref{app:image_experiments} and Appendix \ref{app:word_experiments},  respectively. 

We first examine performance over \update{third generation} embeddings of a \update{homonym word where all instances} correspond to the same \update{meaning}. We then examine performance where embeddings corresponding to one \update{meaning} of a word are corrupted by embeddings corresponding to another \update{meaning} of the same word. Effectively calculating this mean may be important to many downstream tasks (e.g., topic modeling~\citep{griffiths2004finding,blei2009topic}, bias estimation and attenuation~\citep{bolukbasi2016man,dev2019attenuating}).   This models a realistic form of corruption that may arise within LLM embedded data.

We build a dataset of 400 sentences that use the word "field" corresponding to the following definition: "an area of open land, especially one planted with crops or pasture, typically bounded by hedges or fences". We build another dataset of 400 sentences that use the word "field" corresponding to the following, alternate definition: "a particular branch of study or sphere of activity or interest". From now on, we refer to these as "fields of land" and "fields of study". We generated these sentences using ChatGPT-4o.  For more details on how we generate this dataset, and the exact sentences used, see Appendix \ref{app:llm_dataset}.  We embed these sentences and extract the in-context embeddings for the word "field" using 4 LLMs of varying embedding dimensions: MiniLM~\citep{wang2020minilm}, T5~\citep{raffel2023t5}, BERT~\citep{devlin2019bert}, and ALBERT~\citep{lan2020albert}. MiniLM has an embedding dimension of 384, T5 has one of 512, BERT and ALBERT have embedding dimensions of 768. We choose these 4 LLMs to sample a variety of models across different dimensionalities.

\subsection{Common Definition Embeddings}

\begin{figure}[t]
    \begin{subfigure}{0.5\linewidth}
        \centering
        \includegraphics[width=\linewidth]{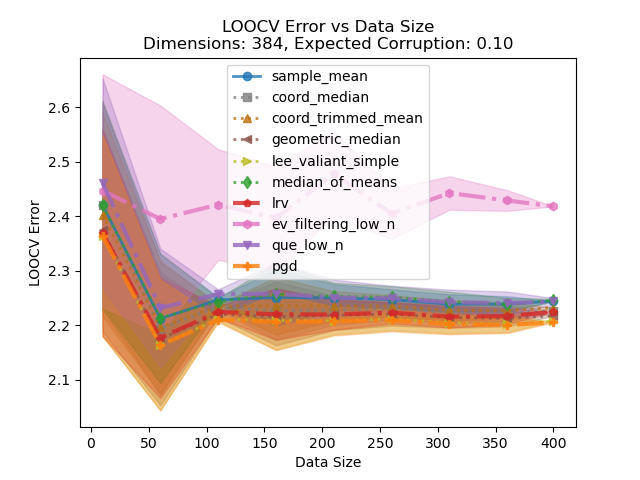}
        \caption{MiniLM}
    \end{subfigure}
        \begin{subfigure}{0.5\linewidth}
        \centering
        \includegraphics[width=\linewidth]{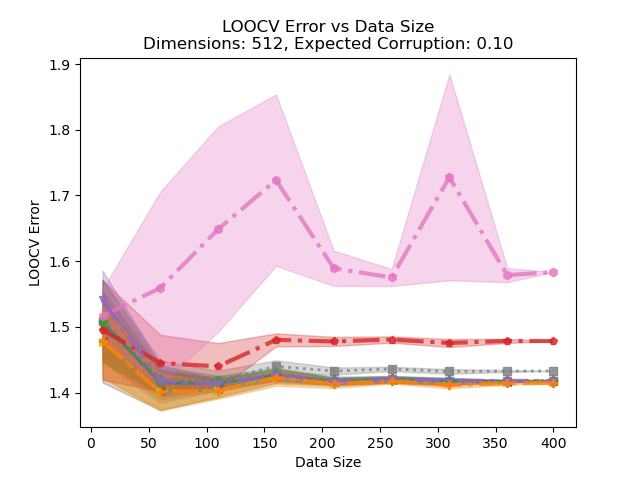}
        \caption{T5}
    \end{subfigure}
    \\
    \begin{subfigure}{0.5\linewidth}
        \centering
        \includegraphics[width=\linewidth]{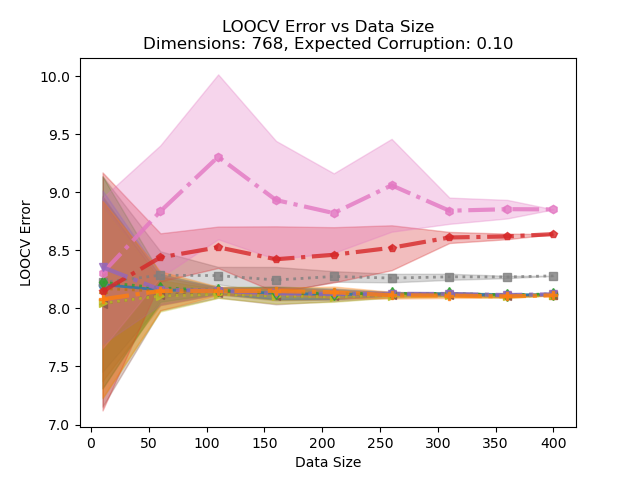}
        \caption{BERT}
    \end{subfigure}
    \begin{subfigure}{0.5\linewidth}
        \centering
        \includegraphics[width=\linewidth]{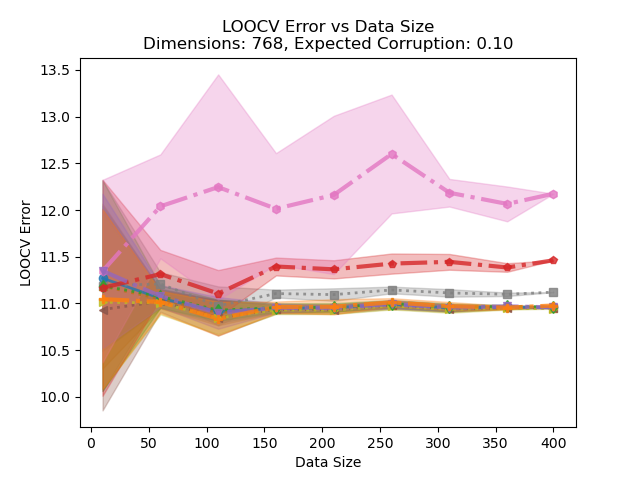}
        \caption{ALBERT}
    \end{subfigure}
    \caption{LOOCV Error on "Field of Land" Embeddings}
    \label{fig:loocv_field_land}
\end{figure}

We first consider performance over embeddings corresponding to the same definition. This is analogous to the uncorrupted data case.  As an error metric, we use Leave One Out Cross Validaton (LOOCV).  We only average over the bottom 90\% of errors to account for potential bias introduced by words that less clearly belong to a specific category. 
LOOCV error is defined here as $\frac{1}{n'} \sum_{i=1}^{n'} \|\mathsf{estimator}(X_{-i}) - x_i\|$ where $n' = 0.9 n$ is $90\%$ of the number of data points in the dataset $X$ (those with smallest errors), $x_i$ is the $i$th data point in $X$, and $X_{-i}$ is $X$ excluding $x_i$. LOOCV under the sample mean represents a valuable baseline for comparison as it demonstrates the minimum error to be expected across this data set \update{under the 10\% expected corruption ($\tau = 0.1$) modeled by the algorithms}. We take the dataset of 400 sentences corresponding to the "field of land" definition. We vary data size from $n=10$ to $n=400$, and, as in prior experiments, average results over $5$ runs and report shaded regions to denote 1 standard deviation of error.  For algorithms that utilize $\tau$, expected corruption, as input, we use the default value of $\tau=0.1$.  We employ the sample trace scaling heuristic for \evln and \queln. We additionally halt \queln whenever more than $2\tau$ percentage of the data has been pruned, regardless of whether or not the threshold is passed. We note that the early halting heuristic is necessary for \queln to perform well under this setting and that it does not meaningfully improve \evln; this is further explored in Section \ref{subsec:early_halt}. We show results in Figure \ref{fig:loocv_field_land}.

Our results do not match our synthetic experiments, suggesting that some robust mean estimation algorithms are sensitive to (Gaussian) distributional assumptions, at least under small data size.  Across LLMs, no algorithm significantly beats the error of \sample. Moreover, \evln performs significantly worse than \sample over all embeddings. This is unsurprising due to the sensitivity of \evln to knowledge of the true covariance. Despite having a similar dependency to knowledge of the true covariance, \queln achieves performance nearly matching \sample error throughout. We also find, that \lrv performs meaningfully worse than \sample across all LLMs except MiniLM, though not quite as catostrophically as \evln.  Aside from \coordmed, which performs noticeably worse than \sample over all LLMs besides MiniLM, all other estimators perform similarly to \sample.

\subsection{Corrupted Embeddings}

\begin{figure}[t]
    \begin{subfigure}{0.5\linewidth}
        \centering
        \includegraphics[width=\linewidth]{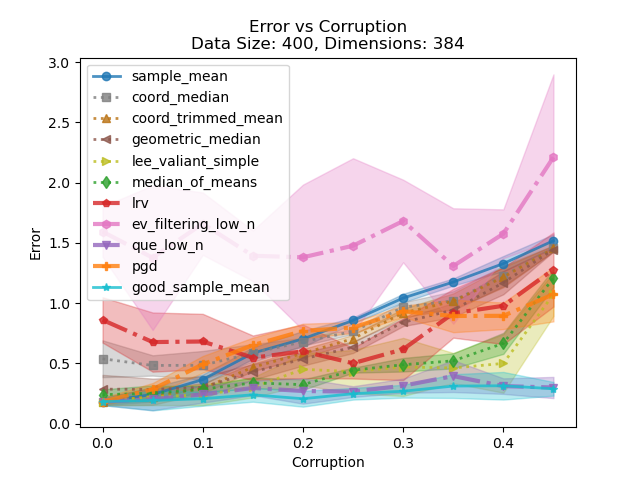}
        \caption{MiniLM}
    \end{subfigure}
        \begin{subfigure}{0.5\linewidth}
        \centering
        \includegraphics[width=\linewidth]{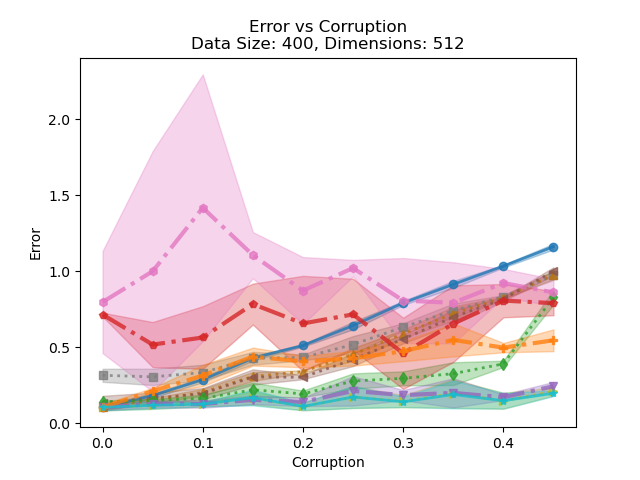}
        \caption{T5}
    \end{subfigure}
    \\
    \begin{subfigure}{0.5\linewidth}
        \centering
        \includegraphics[width=\linewidth]{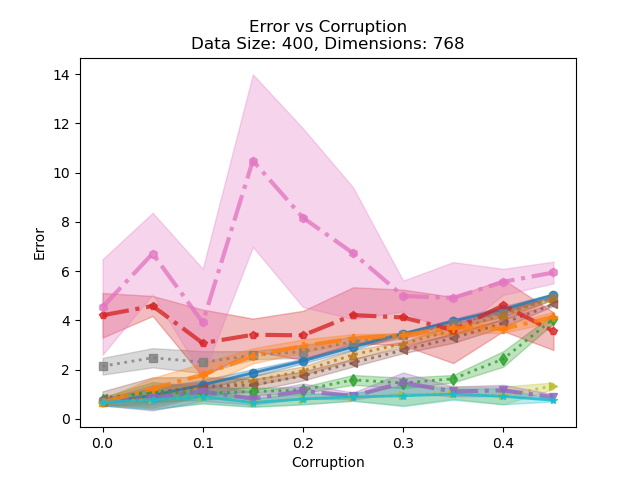}
        \caption{BERT}
    \end{subfigure}
    \begin{subfigure}{0.5\linewidth}
        \centering
        \includegraphics[width=\linewidth]{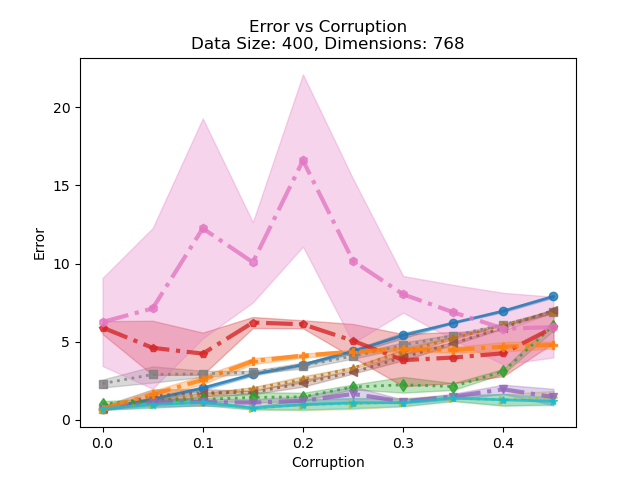}
        \caption{ALBERT}
    \end{subfigure}
    \caption{Error on "Field of Land" Embeddings Corrupted with "Field of Study" Embeddings}
    \label{fig:loocv_corrupted}
\end{figure}

We examine performance over corrupted embeddings of the word "field". We draw corrupted data $X \sim (1 - \eta) P + \eta Q$, where the inlier distribution, $P$, consists of embeddings of the word "field" corresponding to the "field of land" definition, and the outlier distribution, $Q$, consists of embeddings of the word "field" corresponding to the "field of study" definition. As with previous experiments, we measure the Error $\|\mu - \hat{\mu}\|$ on the $y$-axis, taking $\mu$ as the mean of all 400 "field of land" embeddings, and $\hat{\mu}$ as the estimate returned by a mean estimation algorithm. We measure Error vs $\eta$, vary $\eta$ from $0$ to $0.45$, and always have $n = 400$. We average results over $5$ runs and report shaded regions to represent 1 standard deviation of error.  \gsample is plotted to represent the mean of the data before corruption. These results are shown in Figure \ref{fig:loocv_corrupted}. 

We find that mean estimation algorithms can indeed significantly improve performance on this real-world task, but do not observe the same trends as in our synthetic data experiments. \queln and \lvsim are the best estimators, with both significantly outperforming \sample. In fact, \queln performs nearly identical to \gsample across all LLMs, and \lvsim performs similarly, except on MiniLM, where it degrades worse with larger $\eta$ but still outperforms \sample. The performance of \lvsim supports the observation that it is a more effective naive pruning method, which happens to work among the best in these experiments. 
Moreover, \medmean performs very effectively here, always significantly outperforming \sample.  However, neither \evln nor \lrv perform effectively. \lrv outperforms \sample with large enough $\eta$, but these results are not consistent across LLMs suggesting sensitivity to distributional assumptions. Additionally, it almost never outperforms \medmean, \lvsim, or \queln and achieves much worse error with lower $\eta$. This may suggest that \lrv finds irregularities in the uncorrupted data, causing it to return a mean significantly different from \gsample. While this could be beneficial, suggesting that \gsample isn't the best error metric, this is not supported by the "uncorrupted" LOOCV error results, where \lrv also generally results in degraded LOOCV error.
\evln fails catastrophically across LLMs. This matches the "uncorrupted" LOOCV error results, supporting the observation that \evln may fail catostrophically without sufficient knowledge of the true covariance matrix. As in the "uncorrupted" LOOCV error results, there is the somewhat surprising observation that despite seemingly having a similar dependency on knowledge of the true covariance matrix to \evln, \queln performs nearly optimally here.
This suggests the superiority of the quantum entropy based scoring method over naively ranking outliers based on the top eigenvalue of the sample covariance. 

\subsection{Effect Of Early Halting and Ablation}
\label{subsec:early_halt}

\begin{figure}[t]
    \begin{subfigure}{0.5\linewidth}
        \centering
        \includegraphics[width=\linewidth]{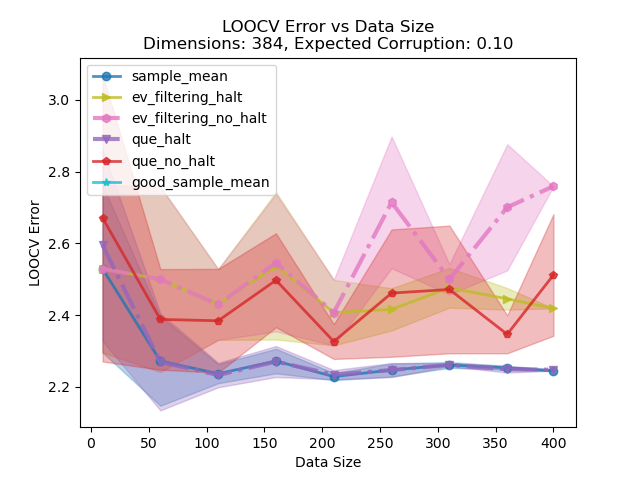}
        \caption{LOOCV - MiniLM}
    \end{subfigure}
    \begin{subfigure}{0.5\linewidth}
        \centering
        \includegraphics[width=\linewidth]{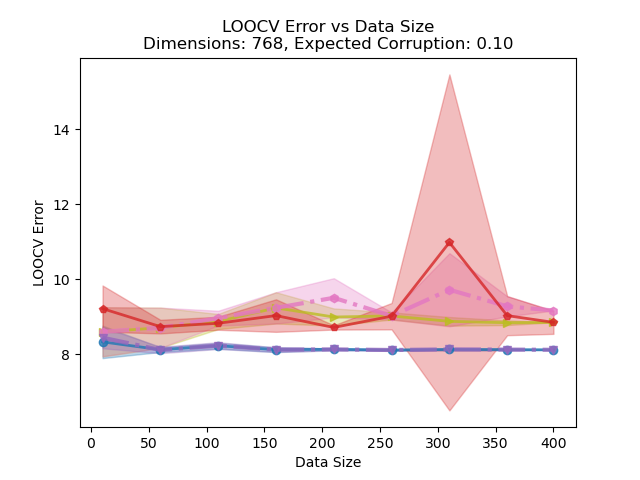}
        \caption{LOOCV - BERT}
    \end{subfigure}
    \\
    \begin{subfigure}{0.5\linewidth}
        \centering
        \includegraphics[width=\linewidth]{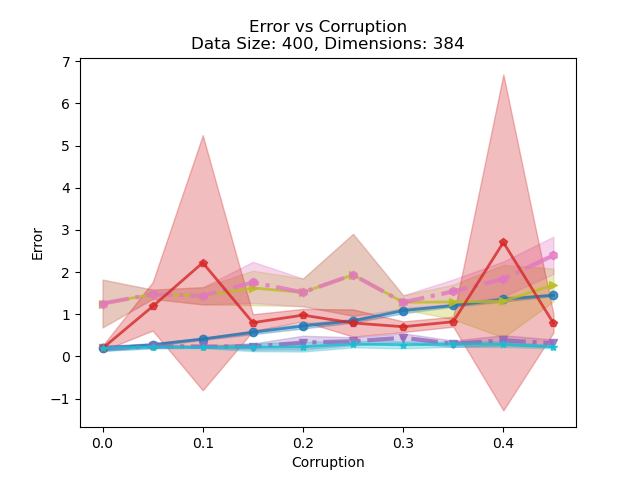}
        \caption{Corruption - MiniLM}
    \end{subfigure}
        \begin{subfigure}{0.5\linewidth}
        \centering
        \includegraphics[width=\linewidth]{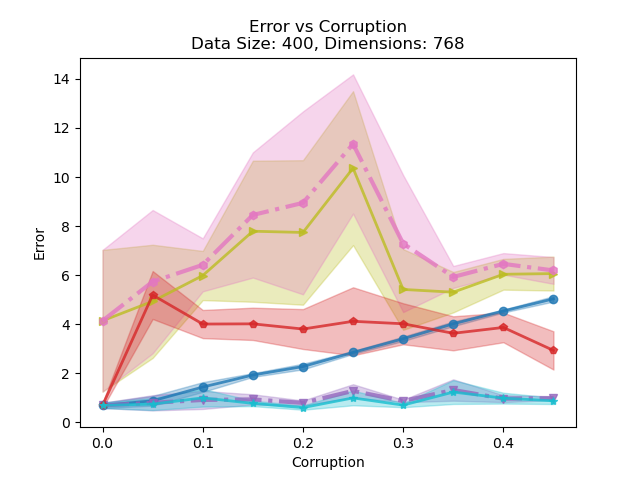}
        \caption{Corruption - BERT}
    \end{subfigure}
    \caption{LLM Comparison - With and without \emph{early halting}}
    \label{fig:llm_no_halt}
\end{figure}

\emph{Early halting} is the following strategy with respect to a given threshold $\tau$:  If more than $2\tau$ points have been pruned by an algorithm, then this halts the pruning process (independent of other criteria) and returns the sample mean of remaining data.  

Here we examine the effect of early halting on \queln and \evln in the context of these real world data where inliers are not generated directly from a prescribed Gaussian distribution.  In other settings explored in this paper, this strategy is almost never invoked, so has no visible effect.  
We compare the performance of both of these algorithm with and without enforcing early halting. We examine LOOCV and Corruption Error over MiniLM and BERT embeddings. Results are shown in Figure \ref{fig:llm_no_halt}. Across all 4 experiments the performance of \queln shows significant degradation without early halting, going from nearly matching \sample in LOOCV error and nearly matching \gsample in corrupted error, to yielding error significantly worse than \sample and \gsample without early halting. Additionally, \queln without halting yields far larger variance results. Meanwhile, \evln performs only slightly better with early halting, and still fails catastrophically across experiments.

 We perform further ablations on these LLM experiments in Appendix \ref{app:llm_ablations}. We explore the effect of different pruning methods on \evln and different weighting methods on \lrv. In both cases, we find that LOOCV error can be improved using non-Gaussian pruning and weighting methods, whereas corrupted error is not meaningfully improved. The failure of \evln across pruning methods suggests the fundamental sensitivity of the outlier scoring method used in \evln to distributional assumptions, which is not seen in \queln (with early halting). We additionally examine performance over these same experiments with the roles of "field of study" and "field of land" embeddings flipped, finding nearly identical results despite differences in distribution. 

\begin{figure}[h]
\centering
\begin{subfigure}{0.48\linewidth}
    \centering
    \includegraphics[width=\linewidth]{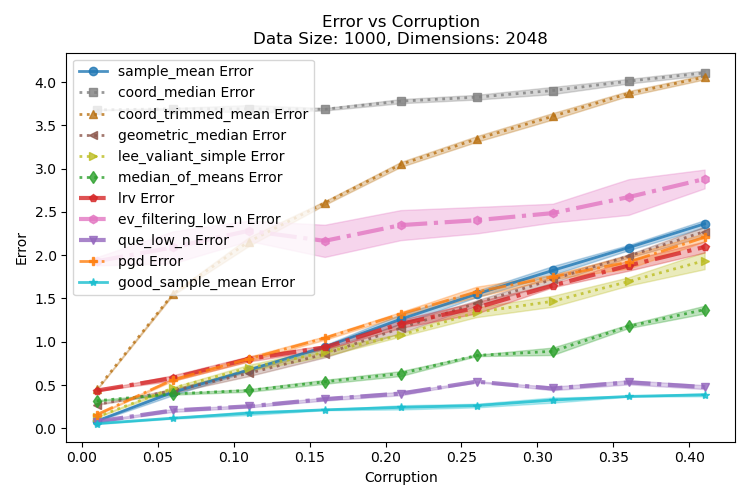}
    \caption{ResNet-50 2048 Dimensional Image Embeddings: \\
    Cat Images Corrupted With Dog Images}
\end{subfigure}
\begin{subfigure}{0.48\linewidth}
    \centering
    \includegraphics[width=\linewidth]{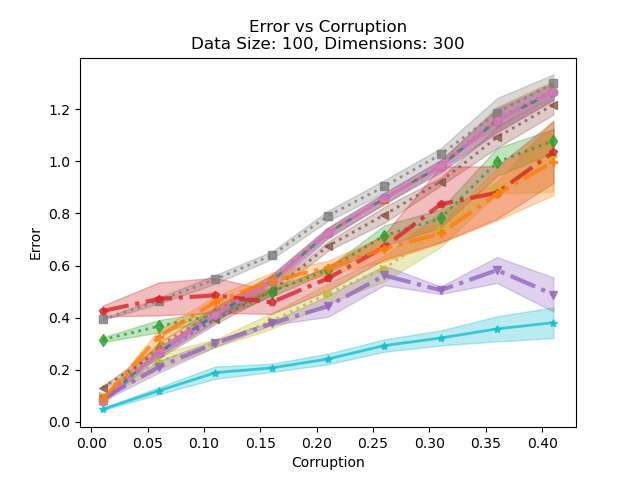}
    \caption{GloVe 300 Dimensional  Word Embeddings: \\
    Pleasant Words Corrupted With Unpleasant Words}
\end{subfigure}
\caption{Additional Real World Corrupted Experiments}
\label{fig:more_real_world}
\end{figure}

\subsection{Additional Real World Experiments}

We additionally examine the performance of robust mean estimation algorithms on corrupted embeddings from deep pretrained image models and non-contextual word embedding models. For the image embedding experiment, we utilize a set of images of cats and dogs from the CIFAR10 dataset \citep{Krizhevsky2009LearningML} with 2048 dimensional embeddings generated from a pretrained ResNet-50 model \citep{he2015deepresiduallearningimage}. For the word embedding experiment, we utilize a dataset of pleasant and unpleasant words from \citep{aboagye2023interpretable} with 300 dimensional embeddings generated from a pretrained GloVe model \citep{pennington-etal-2014-glove}. Experiments are run analogously to the LLM experiments with identical settings for the mean estimators. For the image embedding experiment, inlier data is defined as embeddings of cat images, outlier data is defined as embeddings of dog images, and data size is fixed at $n=1000$. For the word embedding experiment, inlier data is defined as embeddings of "pleasant" words, outlier data is defined as embeddings of "unpleasant" words, and data size is fixed at $n=100$. Results are shown in Figure \ref{fig:more_real_world}.

We find similar results to the LLM experiments, with \queln using early halting noticeably outperforming \sample across both settings, and nearly matching \gsample over image embeddings. Other robust estimators tend to perform similarly or worse to \sample, with \evln again demonstrating significant degradation
 without knowledge of distributional assumptions. Additionally, \lvsim, which performs strongly in the LLM experiments, does not perform as well, demonstrating its sensitivity to distributional assumptions. Similarly, \medmean, does not perform as well across word embeddings as it does across LLM and image embeddings, no longer noticeably outperforming other estimators.
 
 We find similar results across varying dimensionalities of image embedding and GloVe models. We examine additional image embeddings models of varying dimensionalities under LOOCV and corrupted error in Appendix \ref{app:image_experiments}. We additionally recreate the corrupted data experiment, but vary data size instead of corruption, finding that even with $n \gg d$, only \queln and \medmean significantly outperform \sample error, with both estimators nearly converging to \gsample error with corruption $\eta = 0.1$.  We also examine GloVe models of varying dimensionalities under LOOCV and corrupted error in Appendix \ref{app:word_experiments}. 

\update{
\section{Non-Gaussian Synthetic Data Experiments}
\label{app:nongauss}
We examine the performance of robust mean estimators across a few non-Gaussian synthetic data experiments. As before, we draw $X \sim (1 - \eta)P + \eta Q$, where $P$ is an inlier data distribution and $Q$ is the corrupted data distribution. Similar to real world experiments, we employ the trace scaling heuristic on \evln and \queln and enforce early halting on \queln if more than a $2\tau$ percentage of the data has been pruned.

\begin{figure}[h]
    \centering
    \includegraphics[width=\linewidth]{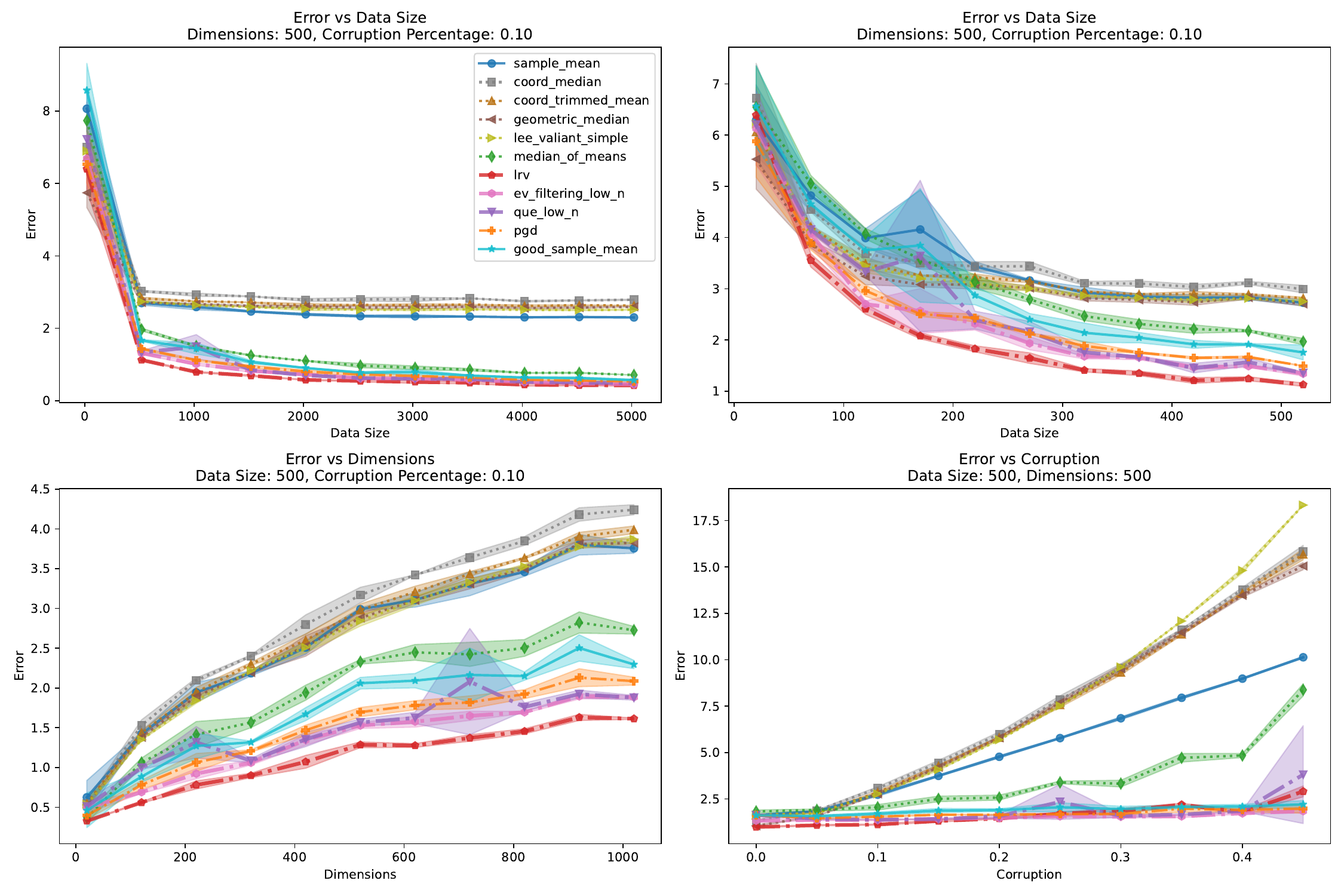}
    \caption{Corrupted Multivariate t-distribution}
    \label{fig:multi_t}
\end{figure}

\paragraph{Multivariate t-distribution} 
Define $P$ as the multivariate t-distribution parametrized by $\mu$ as the all-fives vector, $\Sigma$ as the identity matrix, and degrees of freedom $\nu=3$. Observe that the covariance is $\frac{\nu}{\nu - 2}\Sigma$, not $\Sigma$. This is a heavy-tailed distribution with polynomial tail decay. As in other experiments, consider corrupted data distribution $Q = \mathcal{N}_d(\mu', \frac{1}{10}I)$ where $\|\mu - \mu'\| = \sqrt{d}$. Results are shown in Figure \ref{fig:multi_t}. As with the other experiments, \queln, \evln, \pgd, \lrv, and \medmean all notably outperform \sample here. However here, \queln, \evln, \pgd, and \lrv all consistently outperform \gsample. This trend is particularly notable under lower data size and high dimensions. This is explainable given that the sample mean is known to be a sub-optimal estimator for heavy tailed distributions. These results suggest that robust estimators designed for the Huber contamination model have practical application in the heavy-tailed distribution setting. We leave further investigation into this connection to future work, building on \cite{prasad2019unifiedrobustheavy}. We also note that there is a more notable separation between the best performing methods in this settings than in the inlier Gaussian data scenario. In particular, \evln and \lrv consistently outperform \queln, and also \queln exhibits areas of high variance in error.

\paragraph{Laplace Distribution}
Define $P$ as the product distribution of $d$ independent Laplace distributions, each with mean $5$ and scale $1$. Then, the true mean $\mu$ is defined as the all-fives vector. Again, consider corrupted data distribution $Q = \mathcal{N}_d(\mu', \frac{1}{10}I)$ where $\|\mu - \mu'\| = \sqrt{d}$. Results are shown in Figure \ref{fig:laplace}. Results among the best estimators are nearly identical to the Gaussian inlier data scenarios, with \queln, \pgd, \evln all nearly matching \gsample error and with \lrv performing marginally worse. We note that although the Laplacian distribution has a heavier tail than the Gaussian, we do not observe the same trends as in the multivariate t-distribution. This can be explained by the fact that Laplacian tails still decay exponentially, whereas the tails in the multivariate t-distribution decay polynomially.

\begin{figure}[h]
    \centering
    \includegraphics[width=\linewidth]{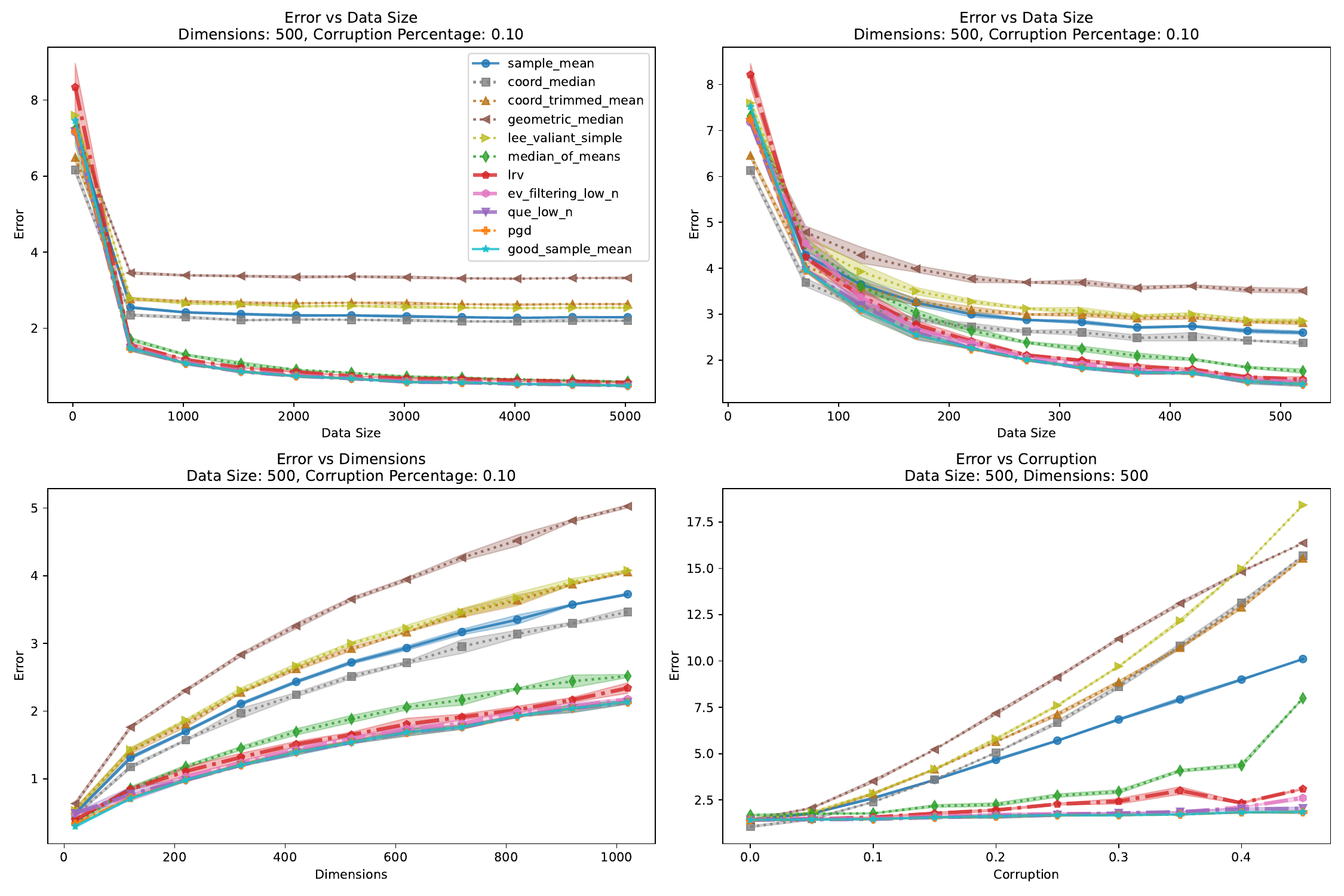}
    \caption{Corrupted Laplace Distribution}
    \label{fig:laplace}
\end{figure}

\paragraph{Poisson Distribution}

Define $P$ as the product distribution of $d$ independent Poisson distributions, each with mean $5$. Then, the true mean $\mu$ is defined as the all-fives vector. Again consider corrupted data distribution $Q = \mathcal{N}_d(\mu', \frac{1}{10}I)$ where $\|\mu - \mu'\| = \sqrt{d}$. Results are shown in Figure \ref{fig:poisson}. Results among the best estimators are again nearly identical to the Gaussian inlier data scenarios.

\begin{figure}[h]
    \centering
    \includegraphics[width=\linewidth]{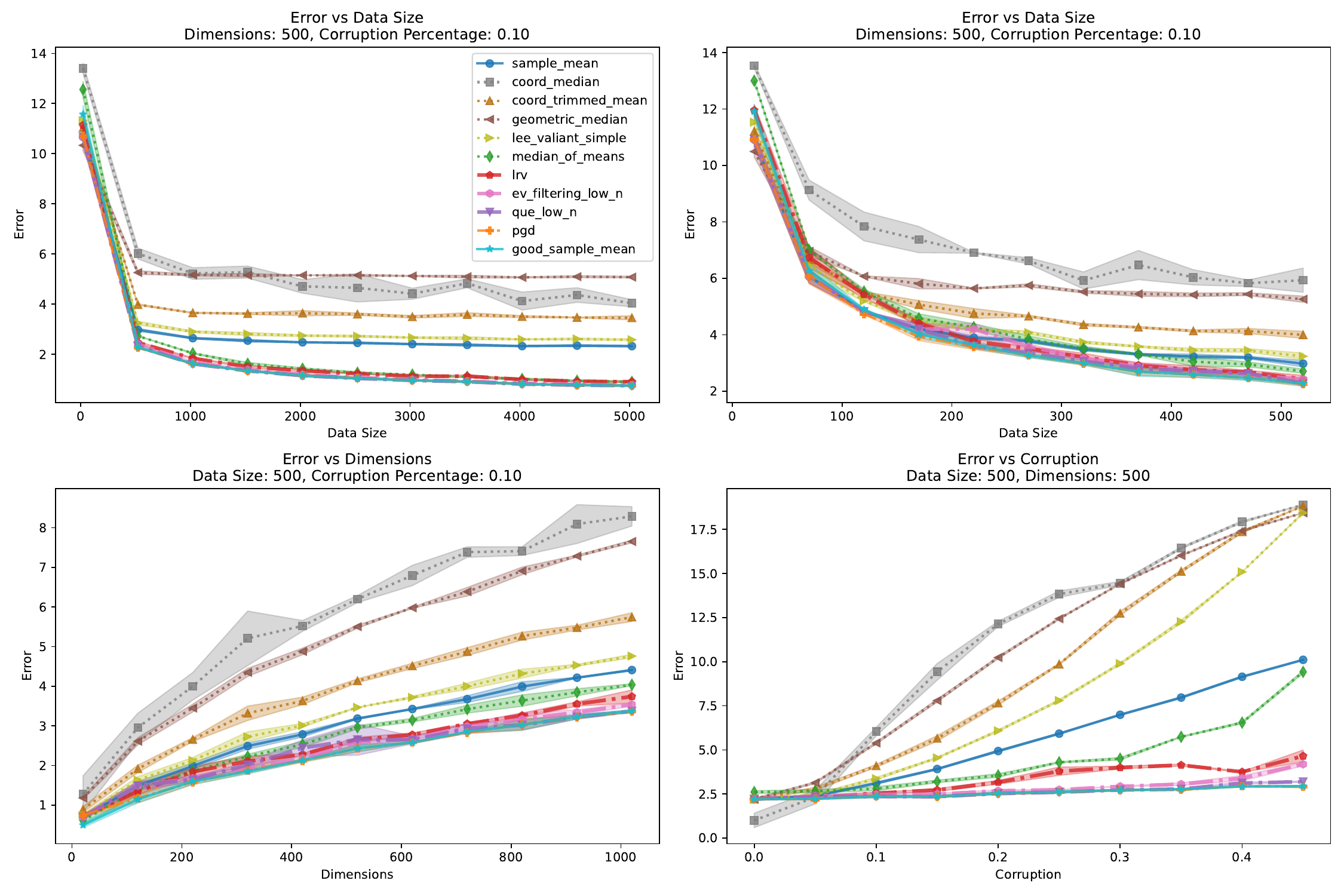}
    \caption{Corrupted Poisson Distribution}
    \label{fig:poisson}
\end{figure}

\paragraph{Mixture of Gaussians}

Define $P$ as a mixture of Gaussians with three components, each equally weighted. The components have means $\mu_1 = \vec{1}$, $\mu_2 = \vec{0}$, and $\mu_3 = -\vec{1}$, where $\vec{1}$ and $\vec{0}$ are the $d$-dimensional vectors of all ones and all zeros, respectively. Each component has identity covariance. Define $Q$ as $\mathcal{N}({\vec{2}, \frac{1}{10}I})$ where $\vec{2}$ is the all-twos vector. Results are shown in Figure \ref{fig:mix_gaus}. While robust methods tend to outperform \sample, relative performance here differs from the Gaussian inlier data setting. Firstly, note that \queln does not perform the best and shows irregular areas of high error under lower dimensionality and moderate corruption levels.  The observation that \queln, \evln, and \pgd do not perform as well is likely because they rely significantly on the Gaussianity assumption for the inliers.  Methods such as \medmean, on the other hand which do not rely as directly on this model are not as effected. Furthermore, \evln, \lrv, and \pgd all show significant degradation as corruption levels increase, which is not observed in the Gaussian inlier settings.  This may be occurring if they completely filter one of the three ``inlier" modes as outliers.  
Interestingly, \lrv notably outperforms all other estimators and even \gsample except with corruption $\tau > 0.2$ and especially for $n < d$. The strong performance of \lrv in this setting is interesting, and may be a consequence of the three inlier distributions means lying on a 1-dimensional subspace.  

\begin{figure}[h]
    \centering
    \includegraphics[width=\linewidth]{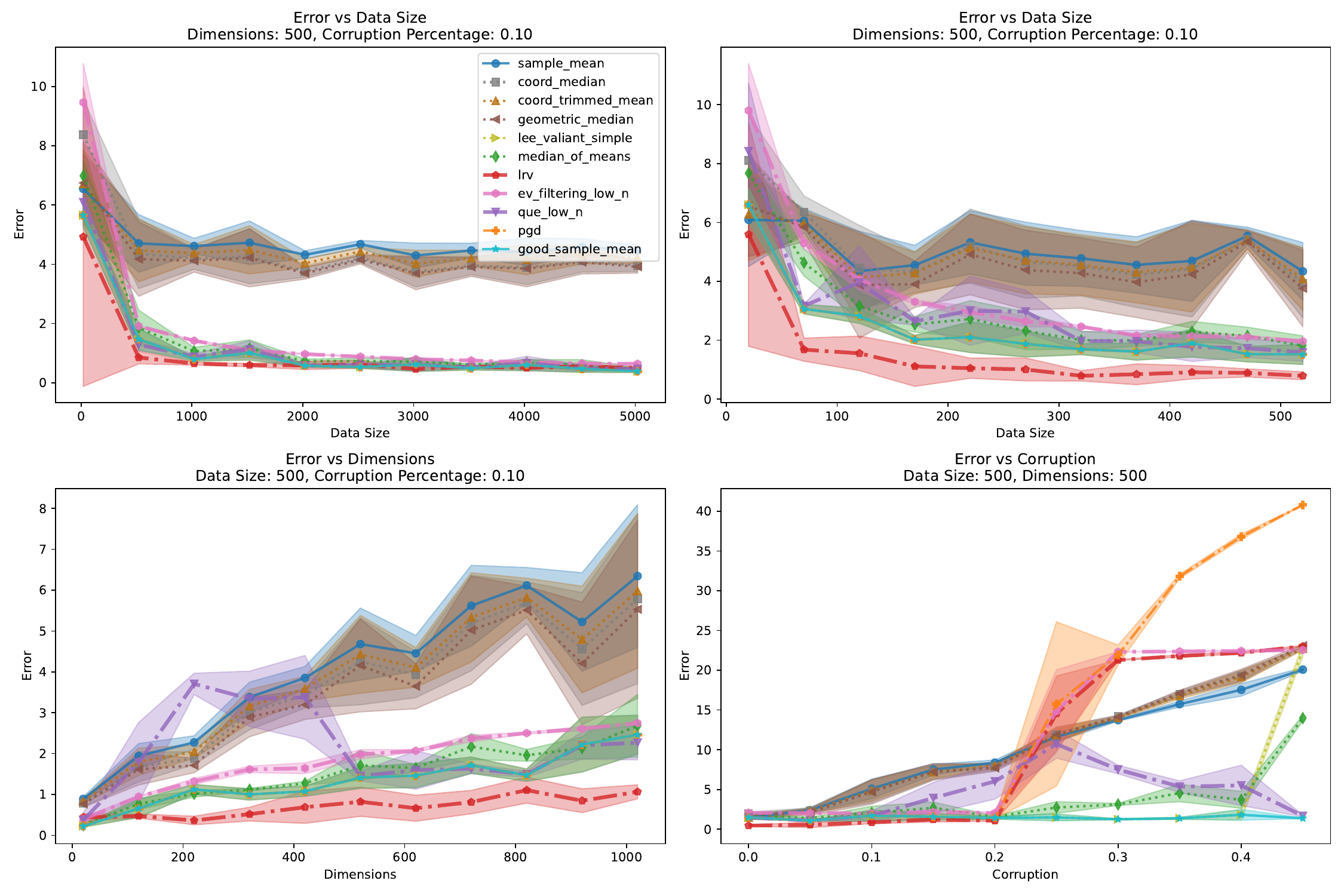}
    \caption{Corrupted Mixture Of Gaussians}
    \label{fig:mix_gaus}
\end{figure}
}

\section{Comparing Algorithm Variants and Ablation}
\label{sec:variants}

In this section, we justify and explore our adaptations to \ev, \que, and \lvog. We use these adaptations for the remainder of our experiments.

\paragraph{Eigenvalue-based Threshold}

Here we compare \ev and \evln, along with \que and \queln. We observe that when we do not have $n$ very large compared to $d$, then \ev and \que can \update{dramatically shift from low error to abysmal error rates}.

We recreate the experiment over corrupted Gaussian data with identity covariance and DKK Noise from Section \ref{sec:corrid}. This is shown in Figure \ref{subfig:evpruning}. We find that \ev and \que fail catastrophically with insufficient data, performing far worse than \sample. However, \evln and \queln never perform worse than \sample and achieve near optimal performance regardless of data size. With sufficient data, \ev and \que abruptly begin to work, and achieve near identical performance to their adjusted threshold counterparts.

\begin{figure}[h]
\centering
\begin{subfigure}{0.48\linewidth}
    \centering
    \includegraphics[width=\linewidth]{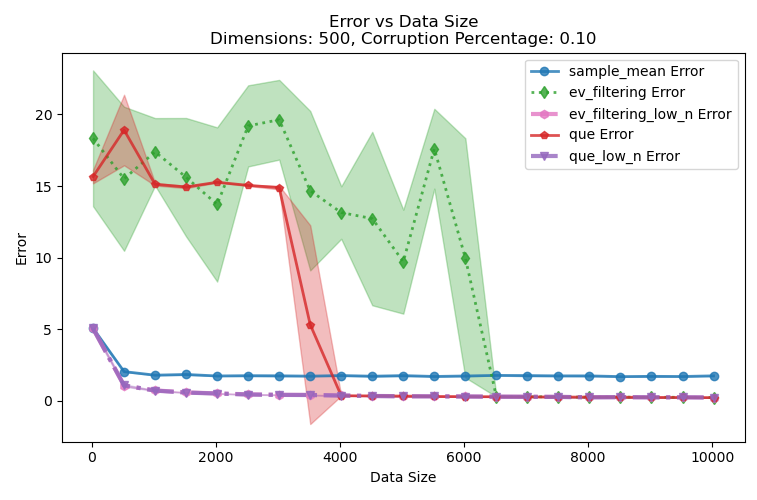}
    \caption{Corrupted Identity Covariance - DKK Noise}
    \label{subfig:evpruning}
\end{subfigure}
\begin{subfigure}{0.48\linewidth}
    \centering
    \includegraphics[width=\linewidth]{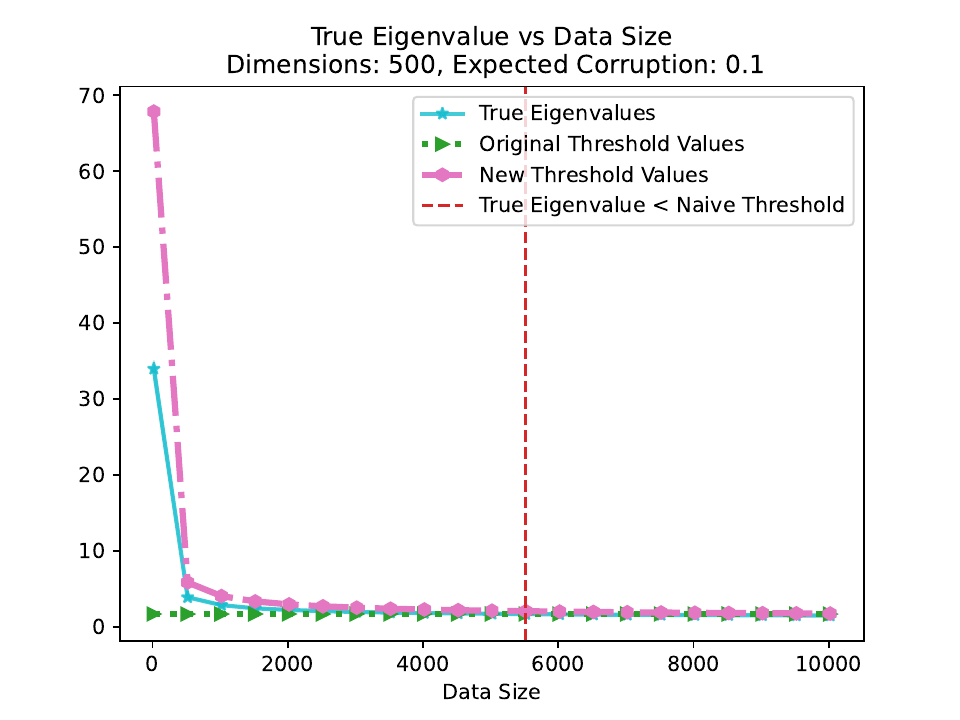}
    \caption{Uncorrupted Identity Covariance}
    \label{subfig:ev_vs_thresholds}
\end{subfigure}
\caption{Eigenvalue Based Filtering Comparison}
\label{fig:comparison}
\end{figure}

The failure of \ev and \que occurs in the corruption detection step. This corruption detection threshold on the top eigenvalue is initialized as $1+ 3 \tau \log(1/\tau)$ in \ev and \que. This constant threshold uses the fact that with large enough data size, the top eigenvalue of an identity covariance matrix approaches 1 and the $3 \tau \log(1/\tau)$ term can account for tolerable noise. However, this does not account for corruption due to low data size, in which the top eigenvalue of the uncorrupted data will necessarily have a larger expectation as data size decreases. Therefore, \ev and \que can never be expected to work since even without any corruption, the top eigenvalue will exceed the threshold. As a result, we find that \ev and \que keep on pruning until there are only very few data points left, resulting in the catastrophic error exhibited. \evln and \queln remedy this problem by simply incorporating our new result (Corollary \ref{cor:prune-2t-main}) as a threshold on the top eigenvalue of the covariance matrix in the threshold. This is empirically shown in Figure \ref{subfig:ev_vs_thresholds}. The threshold in \ev and \que does not become a true upper bound on the top eigenvalue of the uncorrupted data until the vertical red line, which roughly corresponds to the point that \evln begins to perform better. \que begins to perform better with much less data size than \update{\ev}, but this is unsurprising given the rapid convergence of the inlier top eigenvalue to $1$. Meanwhile, our new threshold used in \evln and \queln is always an upper bound on the top eigenvalue, and is near-optimal in practice.

\paragraph{Lee and Valiant variants.}

\begin{figure}[h]
\centering
\includegraphics[width=\linewidth]{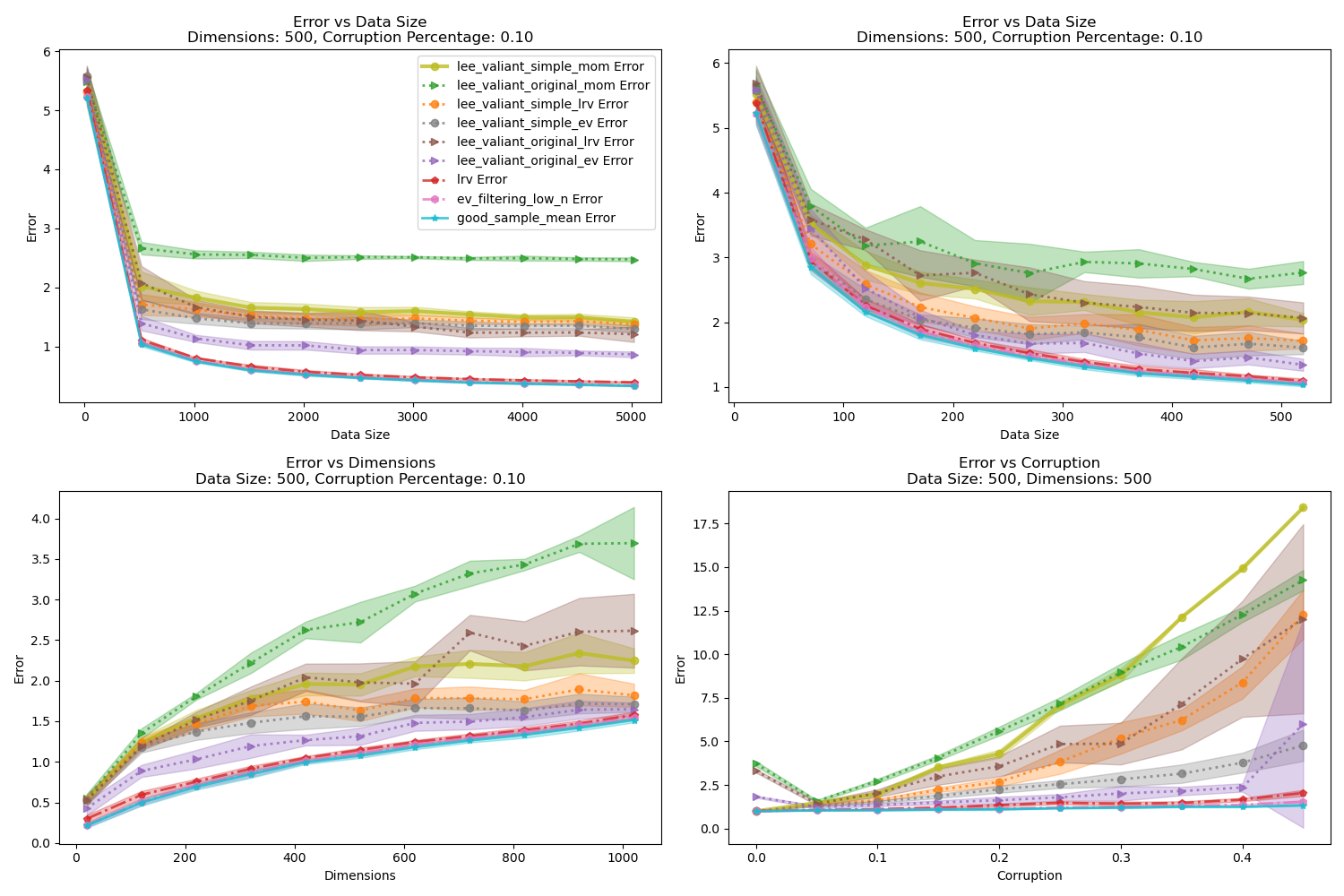} 
\caption{Lee Valiant Variants: Identity Covariance - Additive Variance Shell Noise}
\label{fig:lv_variants}
\end{figure}

Here we compare different variants of the Lee and Valiant algorithm. We observe that \lvsim performs a bit better than \lvog, and \medmean is an illustrative choice for initial estimator.  
We recreate the experiment over Gaussian data with identity covariance and additive variance shell noise from Section \ref{sec:corrid} across different variants of the Lee and Valiant algorithm. We test \lvsim and \lvog using \medmean, \lrv, and \evln as initial mean estimators. We additionally plot \lrv and \evln as baselines. This is shown in Figure \ref{fig:lv_variants}.

\lvsim differs from \lvog in two ways: 
 (1) it removes outliers completely instead of downweighting them and 
 (2) it does not use initial mean estimate in the final result.
We see that \lvsim performs better than \lvog in practice, especially with larger $n$ or $d$. Both \lvsim and \lvog see benefit in the use of improved initial estimators, with this improvement being more significant in \lvog. The difference in the relative improvement in performance between the algorithms is explained by the fact that \lvog additively incorporates the initial estimate directly into its final estimate. However, there is no benefit gained from combining \lvog or \lvsim with an improved initial estimator compared to using the initial estimator alone. As a result of these findings, we only evaluate \lvsim in our experiments.

\section{Conclusion}
\label{sec:conclusion}

We perform the first wide-scale experimental study of robust mean estimation techniques in high dimensions and relatively-low data size. We showed that under Gaussian data with bounded covariances, robust mean estimation techniques can significantly outperform sample mean, nearly matching the optimal error obtainable, regardless of data size, dimensionality, or corruption level. We provide an updated eigenvalue filtering bound that is useful in this high-dimensional setting, and use it to devise a small but novel and meaningful modification to two existing robust mean estimation algorithms; eigenvalue pruning from \cite{diakonikolas2019robust} and quantum entropy scoring from \cite{dong2019quantumentropyscoring}. This enables these methods to almost exactly match optimal error regardless of data size -- that is almost matching the error of the so-called \emph{good sample mean}, which is the mean of the inliers.   \update{It seems \queln works so well because it can identify all ways input distributions deviate from Gaussianity, whereas other methods may require more iterations, which may ultimately prune too many points in the $n \leq d$ setting before it is able to filter outliers in each direction that has them. }

However, all methods perform significantly worse than the mean of all inliers under \emph{subtractive corruption} where an $\eta$-fraction of data points can be removed adversarially.  This suggests that in this Gaussian modeled data regime, practical improvements may be possible in considering the effect of subtractive corruption.  

We also provide a novel evaluation on realistic settings based on the embeddings generated from large language models, deep pretrained image models, and word embedding models.  These are representative of real world settings where the data size $n$ may be smaller or not much larger than the dimension $d$.  
In these settings, quantum entropy scoring with early halting 
tends to perform near optimally, suggesting its potential application to real world data distributions regardless of data size. However, other robust mean estimation algorithms do not work as well as when the inlier data is not Gaussian, as perhaps foreshadowed by theoretical results leveraging this assumption. This suggests that further valuable results may be derived by moving away from the assumption that inliers are precisely Gaussian.  
\update{Our initial explorations for corrupted data with non-Gaussian inliers shows the same techniques mostly perform well, but which method performs the best can vary based on the inlier distribution.  Notably, some methods can even have less error than \gsample when the sample mean is not the MLE.}

Overall, our work demonstrates that there is value in applying robust mean estimation techniques to data, even with insufficient data size for the classic theoretical bounds. We hope that our work inspires researchers to further consider, both experimentally and theoretically, the crucial case of high-dimensional robust statistics under low data size.

\subsubsection*{Acknowledgments}
The authors thank Meysam Alishahi for early discussions, especially towards the analysis in Theorem 2.  
We also thank the NSF REU project 224492, under which CA was a participant and initiated this work.  JP also thanks funding from NSF 2115677 and 2421782, and Simons Foundation  MPS-AI-00010515.

\input{main.bbl}
\bibliographystyle{tmlr}

\appendix
\section{Appendix}

\subsection{Updated Eigenvalue Threshold}
\label{app:ev_theory}

Here we leverage a theorem in \cite{vershynin2011randommatrices} to bound the complexity of aggregated high-dimensional Gaussian random variables.  We will use it in a couple ways.

\begin{theorem}[\citet{vershynin2011randommatrices} Thm. 5.35]
Let $A$ be a $n \times d$ matrix whose entries are independent standard normal random variables. Let $\|A\|_2$ denote the spectral norm of $A$. Then for every $t \geq 0$, with probability of at least $1 - 2 \exp(-t^2/2)$, one has
\[
\sqrt{n} - \sqrt{d} - t \leq s_{\min}(A) \leq \|A\|_2 \leq \sqrt{n} + \sqrt{d} + t.
\]
Where $s_{\min}(A)$ is the smallest singular value of $A$.  The lower bound assumes $n > d$, if not the roles are reversed.  
\label{thm:verGnd}
\end{theorem}

We prove the following implication.

\begin{theorem}[restatement of Theorem \ref{thm:Sigma2-bound-main}]
\label{thm:Sigma2-bound}
Let $X$ be a $n \times d$ matrix whose entries are independently drawn from $\mathcal{N}(\mu, I)$. Let $\Sigma = \frac{1}{n}(X-\bar{\mu})^T(X-\bar{\mu})$ be the sample covariance matrix of $X$, where $\bar{\mu} = \frac{1}{n} \sum_i X_i$ and $X_i$ is the $i$th row of $X$. Then for every $t > 0$, with probability of at least $1 - 3 \exp(-t^2/2)$, one has 
\[
\  \|\Sigma\|_2 \leq \left(1 + \sqrt{d/n} + t/\sqrt{n} + \frac{\sqrt{d + \sqrt{2d}t + t^2}}{n} \right)^2.  
\]
\end{theorem}

\begin{proof}
Let $\bar{X} = X - \mu$ be the centered matrix, equivalent to each entry being drawn from $\mathcal{N}(0, I)$. Let $Z = X - \bar{\mu}$ be the matrix centered by the sample mean. Then $\|Z\|_2 = \|\bar{X} + [\mu - \bar{\mu}]\|_2 \leq \|\bar{X}\|_2 + \|\mu - \bar{\mu}\|_2$ by triangle inequality.

First, by Theorem \ref{thm:verGnd}, we have that $\|\bar X\|_2 \leq \sqrt{n} + \sqrt{d} + t$, with probability at least $1-2\exp(-t^2)$.  

Second, to bound $\|\mu - \bar{\mu}\|_2$ we first decompose by coordinate $\|\mu - \bar{\mu}\|_2^2 = \sum_{j=1}^d (\mu_j - \bar{\mu}_j)^2$.  Now consider $d$ random variables $B_j = \mu_j - \bar{\mu_j}$ for $j = 1 \ldots d$, and further write $B_j = \frac{1}{n} \sum_{i=1}^n F_i$ where $F_i \sim \mathcal{N}(0,1)$.  As a result $B_j \sim \mathcal{N}(0, 1/n) = \frac{1}{\sqrt{n}} \mathcal{N}(0,1)$, since the average of $n$ normals is still normal with variance reduced by factor $n$. 
As a result $B_j$ is a squared normal distribution, and $B = n \|\mu - \bar{\mu}\|^2 = n \sum_{j=1}^d B_j^2$ is a chi-squared distribution $\chi^2(d)$.  

Hence we have~\citep{laurent2000adaptive}
\[
\mathsf{Pr}[B \geq d + 2\sqrt{d} t + 2s^2] 
\leq 
\exp(-s^2).  
\]
Inside the probability expression, using $\sqrt{B/n} = \|\mu -  \bar\mu\|$, and letting $t = \sqrt{2}s$, we can rewrite this as
\[
\mathsf{Pr}\left[\|\mu - \bar{\mu}\| < \sqrt{(d + \sqrt{2} \sqrt{d} t + t^2)/n}\right]  
\geq 
1-\exp(-t^2/2).
\]

So now if both of these events hold, which by union bound occurs with probability at least $1-3\exp(-t^2/2)$, we have that
\[
\|Z\|_2 
  \leq 
\| \bar X\|_2 + \|\mu - \bar \mu\| 
  \leq 
(\sqrt{n} + \sqrt{d} + t) + \sqrt{\frac{d+ \sqrt{2 d} t + t^2}{n}}
\]

Notice that $\|\Sigma\|_2 = \|\frac{1}{n}Z^TZ\|_2 = \frac{\|Z\|_2^2}{n}$.
Thus we have 
\[
\| \Sigma \|_2 = \frac{\|Z\|_2^2}{n} \leq \left( 1 + \sqrt{d}/\sqrt{n} + t/\sqrt{n} + \frac{\sqrt{d + \sqrt{2d}t + t^2}}{n} \right)^2
\]
\end{proof}

Note that the fourth term in this bound, coming from the error in $\|\mu - \bar  \mu\|$, is a lower order effect.  This is captured in the following corollary.  

\begin{corollary}[restatement of Corollary \ref{cor:prune-2t-main}]
Let $X$ be a $n \times d$ matrix whose entries are independently drawn from $\mathcal{N}(\mu, I)$. Let $\Sigma = \frac{1}{n}(X-\bar{\mu})^T(X-\bar{\mu})$ be the sample covariance matrix of $X$, where $\bar{\mu} = \frac{1}{n} \sum_i X_i$ and $X_i$ is the $i$th row of $X$. If one assumes $d/n \leq 16, n \geq 16, t \geq 5$, then with probability of at least $1 - 3 \exp(-t^2/8)$, one has 
\[
\  \|\Sigma\|_2 \leq \left(1 +  \sqrt{d/n} + t/\sqrt{n} \right)^2.  
\]
\label{cor:prune-2t}
\end{corollary}
\begin{proof}
Starting with the bound in Theorem \ref{thm:Sigma2-bound} we have
\begin{align*}
\|\Sigma\|_2 
  &\leq 
\left( 1 + \sqrt{d/n} + t/\sqrt{n} + \frac{\sqrt{d + \sqrt{2d}t + t^2}}{n} \right)^2
 \\ & =
\left( 1 + \sqrt{d/n} + t/\sqrt{n} + \frac{t}{\sqrt{n}} \cdot \frac{\sqrt{d/t^2 + \sqrt{2d}/t + 1}}{\sqrt{n}} \right)^2
 \\ & =
\left( 1 + \sqrt{d/n} + \frac{t}{\sqrt{n}}\left( 1+  \sqrt{(d/n)/t^2 + \sqrt{2}\sqrt{d/n}/t/\sqrt{n} + 1/n} \right) \right)^2
\\ & \leq
\left( 1 + \sqrt{d/n} + \frac{t}{\sqrt{n}}\left( 1+  \sqrt{(16)/t^2 + \sqrt{2}\sqrt{16}/t/\sqrt{n} + 1/n} \right) \right)^2
\\ & \leq
\left( 1 + \sqrt{d/n} + \frac{t}{\sqrt{n}}\left( 1+  \sqrt{16/25 + \sqrt{32}/(5 \sqrt{n}) + 1/n} \right) \right)^2
\\ & \leq
\left( 1 + \sqrt{d/n} + \frac{t}{\sqrt{n}}\left( 1+  \sqrt{16/25 + \sqrt{32}/(5 \cdot 4) + 1/16} \right) \right)^2
\\ & <
\left( 1 + \sqrt{d/n} + 2t/\sqrt{n} \right)^2
\end{align*}
Adjusting $t$ to $2t$ in the probability of failure, so it is $3\exp(-t^2/8)$ instead of $3 \exp(-t^2/2)$, completes the proof.  
\end{proof}

\subsection{Corrupted Gaussian Data Identity Covariance: Additional Noise Schemes}
\label{app:idcov_morenoise}

We examine the performance of robust mean estimators across additional corruption schemes. We still draw $X \sim (1-\eta) P + \eta Q$ where $P = \mathcal{N}_d(\mu,I)$ and $Q$ describes the corrupted data distribution, where $\mu$ is the all-fives vector. We utilize the following additional corruption schemes:

\paragraph{Two Gaussian clusters shifted to variance shell. }
Consider corrupted data distribution $Q = 0.7\mathcal{N}_d(\mu^0, \frac{1}{10} I) \cup 0.3\mathcal{N}_d(\mu^1, \frac{1}{10} I)$ where $\|\mu - \mu^0\| = \sqrt{d}$, $\|\mu - \mu^1\| = \sqrt{d}$, and $\theta = 75^\circ$ where $\theta$ is the angle between $\mu^0$ and $\mu^1$.  The location of $\mu^0$ is determined by a random rotation matrix to prevent any coordinate-axis specific biases. Results over this noise distribution are shown in Figure \ref{fig:id_cov_gaus_two}.

\paragraph{In Distribution Noise. }
Consider corrupted data distribution, $Q$, where for each corrupted data point $q_i \in Q$, each coordinate $j$ of $q_i$ is drawn from $\mathsf{Uniform}(\mu_j, \mu_j+2)$. Here $\mu_j$ represents the $j$th coordinate of the true mean $\mu$. Results over this noise distribution are shown in Figure \ref{fig:id_cov_unif_top}.

\paragraph{Large Outlier Noise. }
Consider corrupted data distribution $Q = 0.7\mathcal{N}_d(\mu^0, \frac{1}{10} I) \cup 0.3\mathcal{N}_d(\mu^1, \frac{1}{10} I)$ where $\|\mu - \mu^0\| = 10\sqrt{d}$, $\|\mu - \mu^1\| = 20\sqrt{d}$, and $\theta = 75^\circ$ where $\theta$ is the angle between $\mu^0$ and $\mu^1$. The location of $\mu^0$ is determined by a random rotation matrix to prevent any coordinate-axis specific biases. Results over this noise distribution are shown in Figure \ref{fig:id_cov_obvious}.

\paragraph{Large Outlier Noise Mixes. } 
Consider corrupted data distribution $Q = 0.5L \cup 0.5 Q'$ where $L$ is the large outlier corruption scheme previously described and $Q'$ is a subtle corruption scheme. We examine two settings for $Q'$: additive variance shell corruption with one cluster, shown in Figure \ref{fig:id_cov_large_obvious_gaus}, and DKK corruption, shown in Figure \ref{fig:id_cov_large_obvious_dkk}.

\begin{figure}[h]
    \centering
    \includegraphics[width=0.85\linewidth]{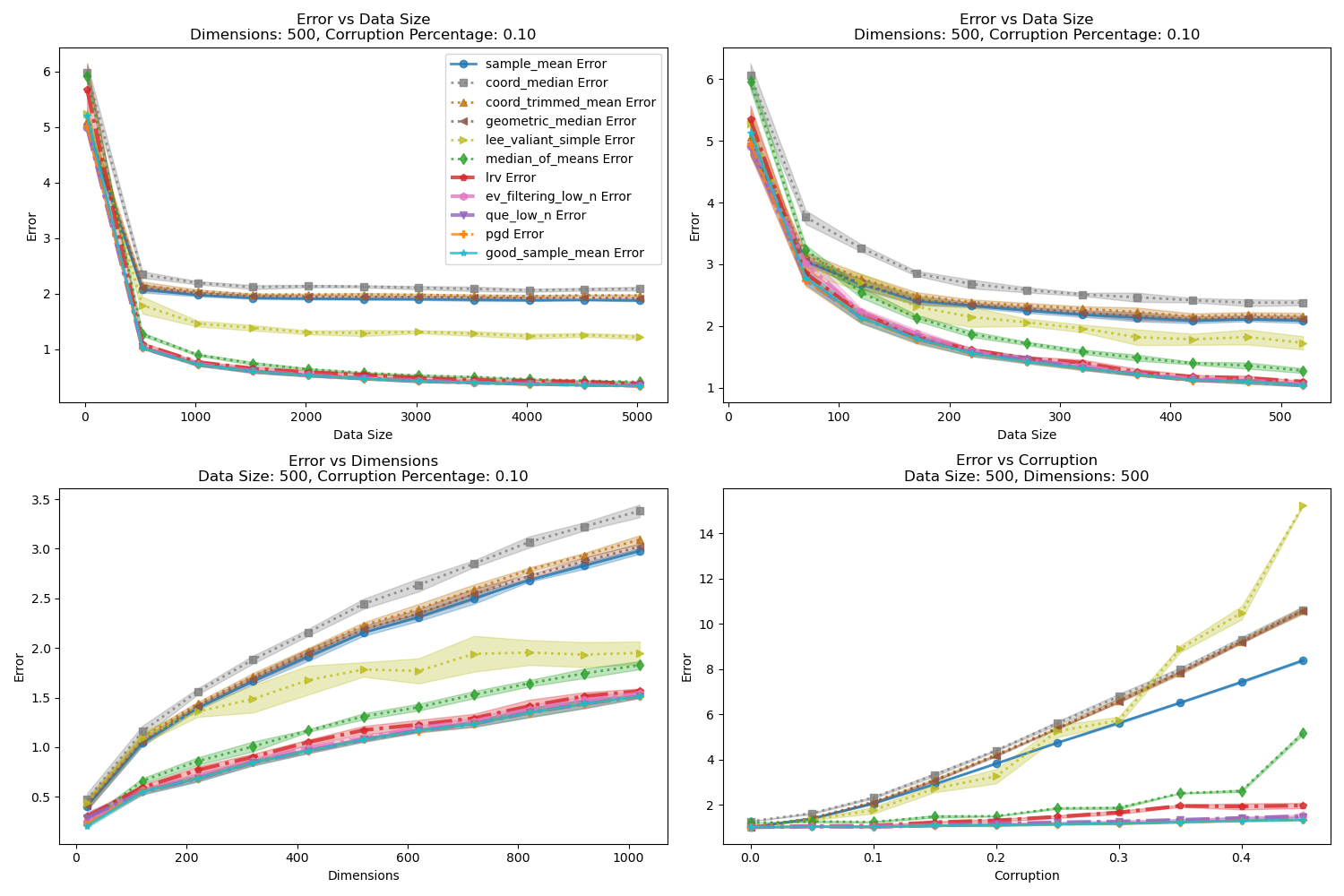}
    \caption{Corrupted Gaussian Identity Covariance: Two Variance Shell Clusters}
    \label{fig:id_cov_gaus_two}
\end{figure}

\begin{figure}[h]
    \centering
    \includegraphics[width=0.85\linewidth]{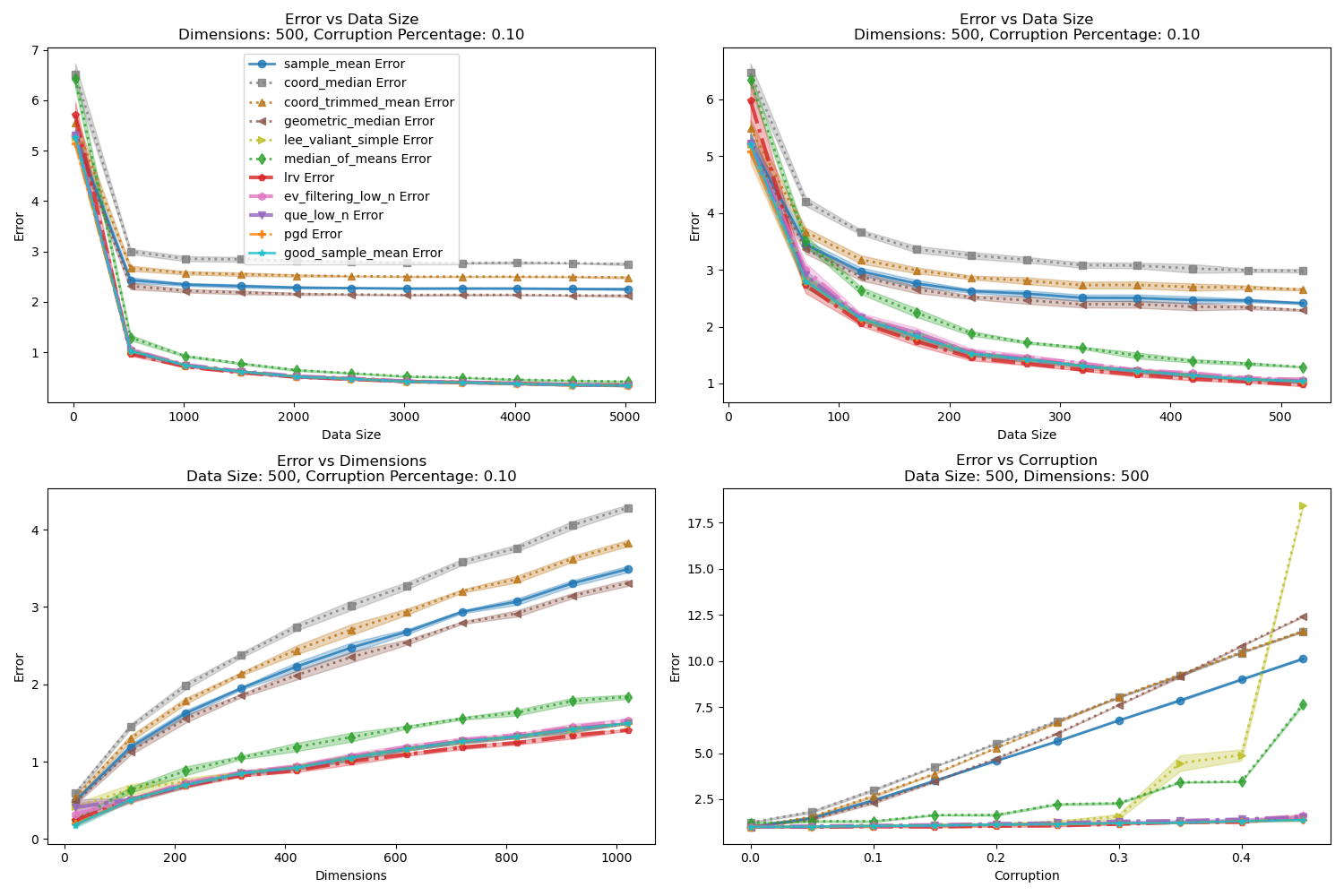}
    \caption{Corrupted Gaussian Identity Covariance: In Distribution Noise}
    \label{fig:id_cov_unif_top}
\end{figure}

\begin{figure}[h]
    \centering
    \includegraphics[width=0.85\linewidth]{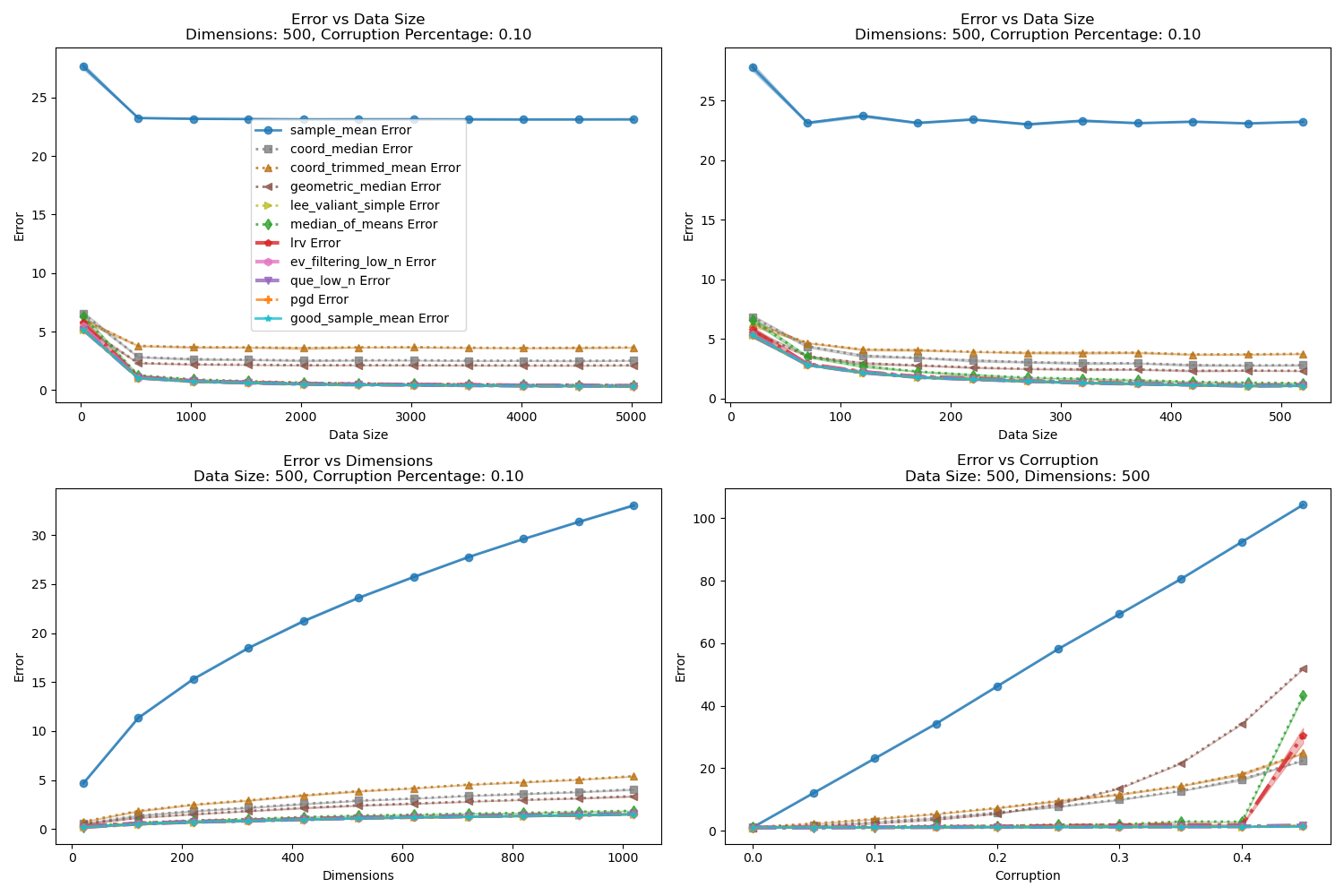}
    \caption{Corrupted Gaussian Identity Covariance: Large Outliers}
    \label{fig:id_cov_obvious}
\end{figure}

\begin{figure}[h]
    \centering
    \includegraphics[width=0.85\linewidth]{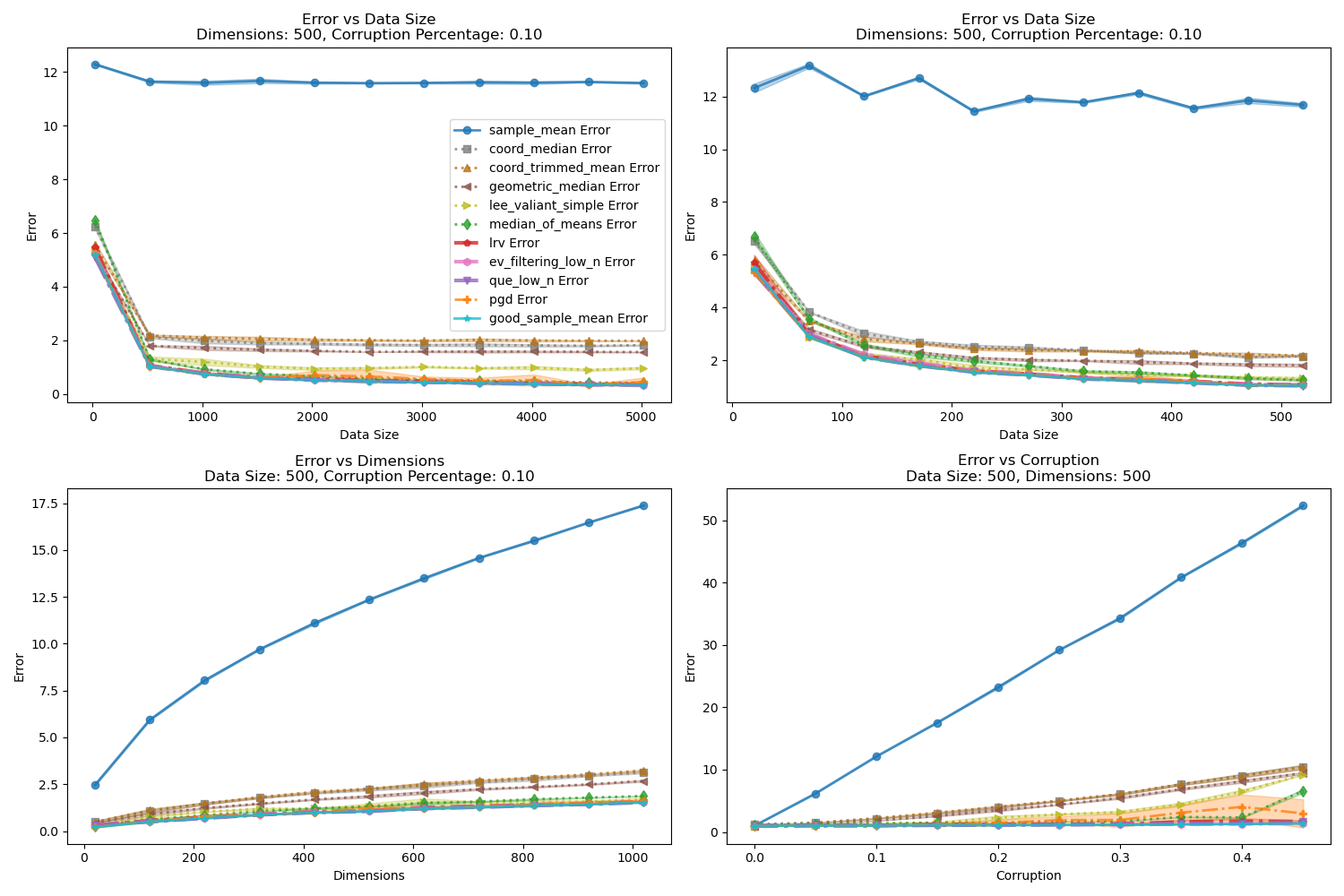}
    \caption{Corrupted Gaussian Identity Covariance: Large Outliers w/ Additive Variance Shell Noise}
    \label{fig:id_cov_large_obvious_gaus}
\end{figure}

\begin{figure}[h]
    \centering
    \includegraphics[width=0.85\linewidth]{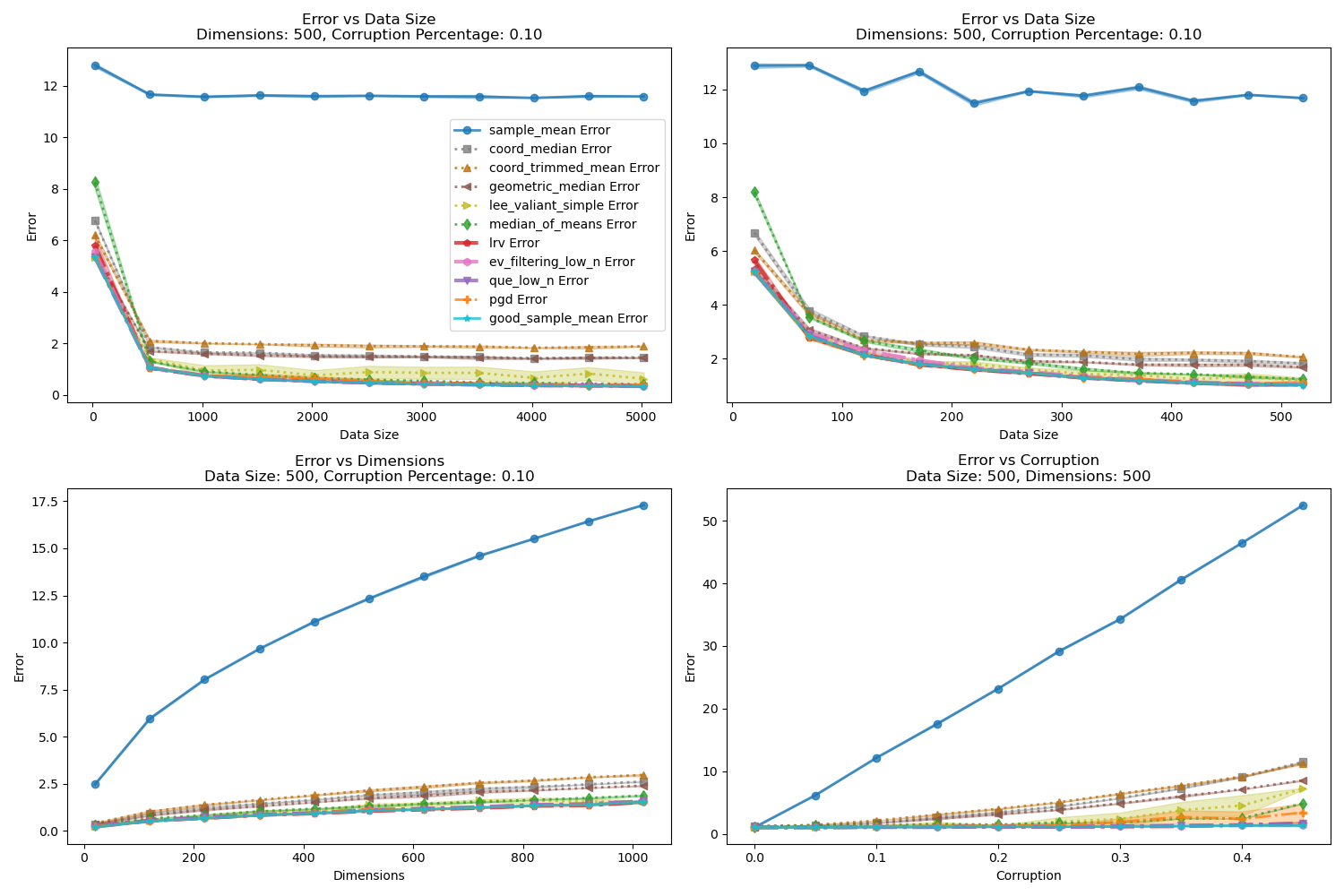}
    \caption{Corrupted Gaussian Identity Covariance: Large Outliers w/ DKK Noise}
    \label{fig:id_cov_large_obvious_dkk}
\end{figure}

Across all of these distributions, including those with large outliers, \evln, \queln, \pgd, and \lrv perform the best, suggesting that their performance is not overly sensitive to the noise distribution. However, we note that across schemes with large outliers, \pgd sees areas of higher variance and slightly worse performance, and \lrv degrades worse as $\eta$ increases with large outliers. This downgrade in performance can be remedied by first preprocessing data by removing large outliers through a naive pruning method, but this step doesn't appear necessary for other methods.  We remark that \lrv requires outlier weights to be clipped to avoid numerical instability issues under large outlier schemes. Otherwise, it will degrade poorly over large outliers and large $\eta$ as predicted outliers will be assigned near-zero weights. We also see again that \medmean significantly outperforms other simple estimators, especially under large data size, although its performance degrades poorly under certain conditions, such as with larger $\eta$. With large outliers, \lvsim nearly matches \gsample error across conditions, achieving much better performance than it does across subtle noise distributions. As it additionally outperforms \coordprune, this suggests that \lvsim may operate as a more effective naive pruning method, as seen in the LLM experiments.

\paragraph{Dependence on true mean}
We additionally verify that performance does not depend on the choice of true mean, $\mu$. We recreate experiments over Additive Variance Shell Noise and DKK Noise over different choices of $\mu$. We replicate the same experimental setup as before, but draw every coordinate of $\mu$ from $\mathcal{N}(0, 50)$ at every iteration in the experiment rather than fixing $\mu$ as the all-fives vector. As a reminder, this occurs for every choice of the independent variable over every run. If performance depended on $\mu$, we would expect to achieve high variance results. Instead, we find nearly identical results to the original experiments across both distributions. These results are shown in Figure \ref{fig:mean_dependence_gaus_one} and Figure \ref{fig:mean_dependence_dkk}.

\begin{figure}[h]
    \centering
    \includegraphics[width=0.85\linewidth]{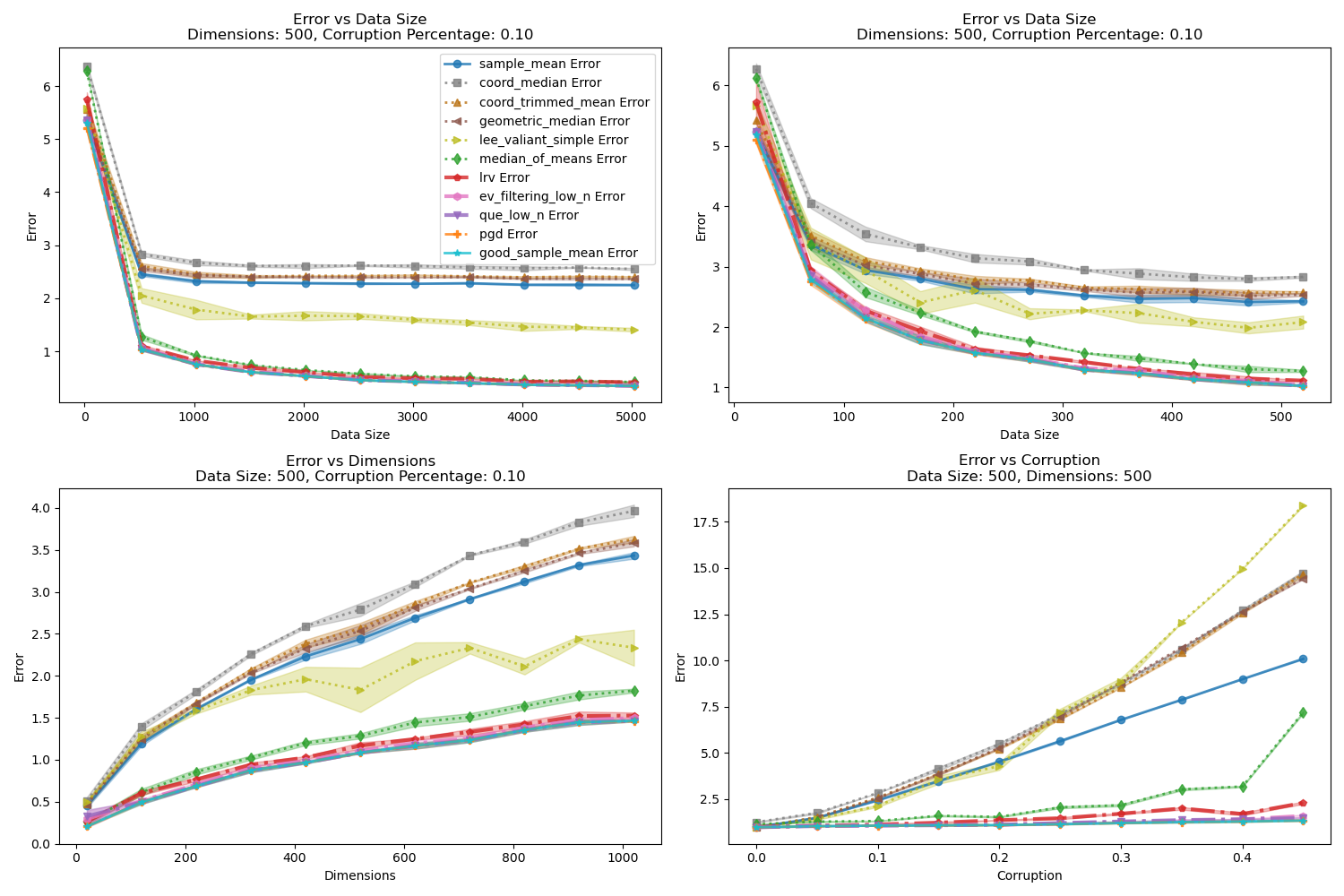}
    \caption{Dependence On True Mean: Identity Covariance, Additive Variance Shell Noise}
    \label{fig:mean_dependence_gaus_one}
\end{figure}

\begin{figure}[h]
    \centering
    \includegraphics[width=0.85\linewidth]{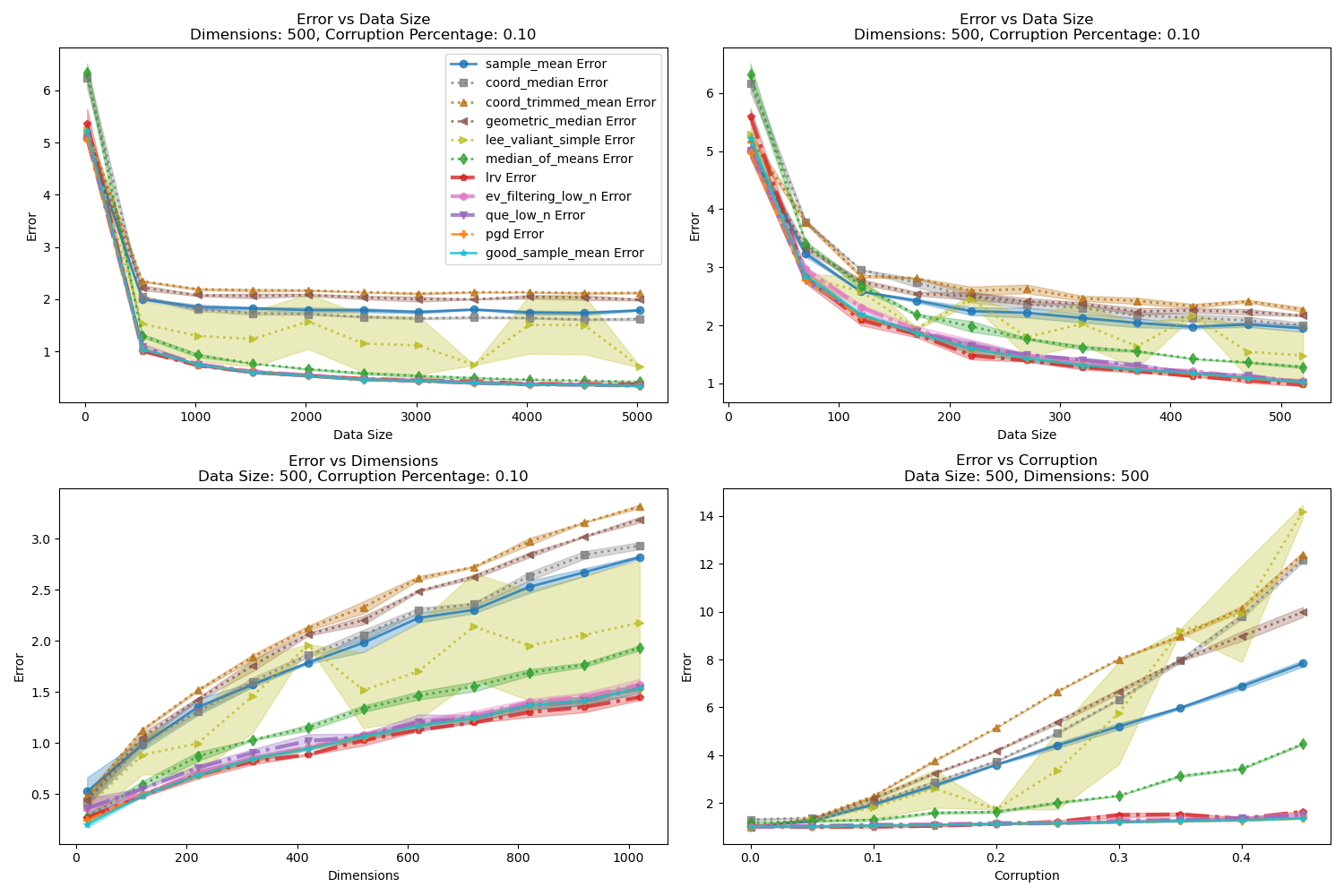}
    \caption{Dependence On True Mean: Identity Covariance, DKK Noise}
    \label{fig:mean_dependence_dkk}
\end{figure}

\clearpage

\subsection{Corrupted Gaussian Data Unknown Spherical Covariance: Additional Corruption Schemes}
\label{app:unknownsp_morenoise}

We examine the unknown spherical covariance case across additional corruption schemes. We utilize the same uncorrupted distribution as before, $P = \mathcal{N}_d(\mu, \sigma^2 I)$ where $\mu$ is the all-fives vector and $\sigma = 5$. We find similar performance across the distributions we test.

\paragraph{Adapting noise distributions to spherical covariance}

As in the Gaussian noise shifted to variance shell case, we utilize the well know Gaussian concentration inequality that for data $X \sim \mathcal{N}_d(\mu, \sigma I)$, $\E_{x \sim X} [\|x - \mu\|^2] = \sigma^2 d$. This observation is used to adapt noise distributions in the identity covariance case to this case. For two additive variance shell clusters, each cluster has mean $\mu^i$ for $i \in [0, 1]$ where $\|\mu - \mu^i\| = \sigma \sqrt{d}$, with other conditions remaining the same; results are shown in Figure \ref{fig:large_sp_gaus_two}. For DKK noise, half the noise is drawn over the hypercube where every coordinate is $-\sigma$ or 0 away from the true mean at that coordinate with equal probability. The other half is drawn from the product distribution where the first coordinate is either $11\sigma$ or $-\sigma$ away from the true mean at that coordinate with equal probability, the second coordinate is $-3\sigma$ or $-\sigma$ away from the corresponding true mean coordinate with equal probability, and all remaining
coordinates are $-\sigma$ away from the true mean. Results are shown in Figure \ref{fig:large_sp_dkk}. For in distribution corruption, we draw each coordinate $j$ of a corrupted data point from $\mathsf{Uniform}(\mu_j, \mu_j + 2\sigma)$; results are shown in Figure \ref{fig:large_sp_unif_top}. We also perform subtractive corruption, using the same scheme as in the identity covariance case; results are shown in Figure \ref{fig:large_sp_subtractive_corruption}. 

\begin{figure}[h]
    \centering
    \includegraphics[width=0.85\linewidth]{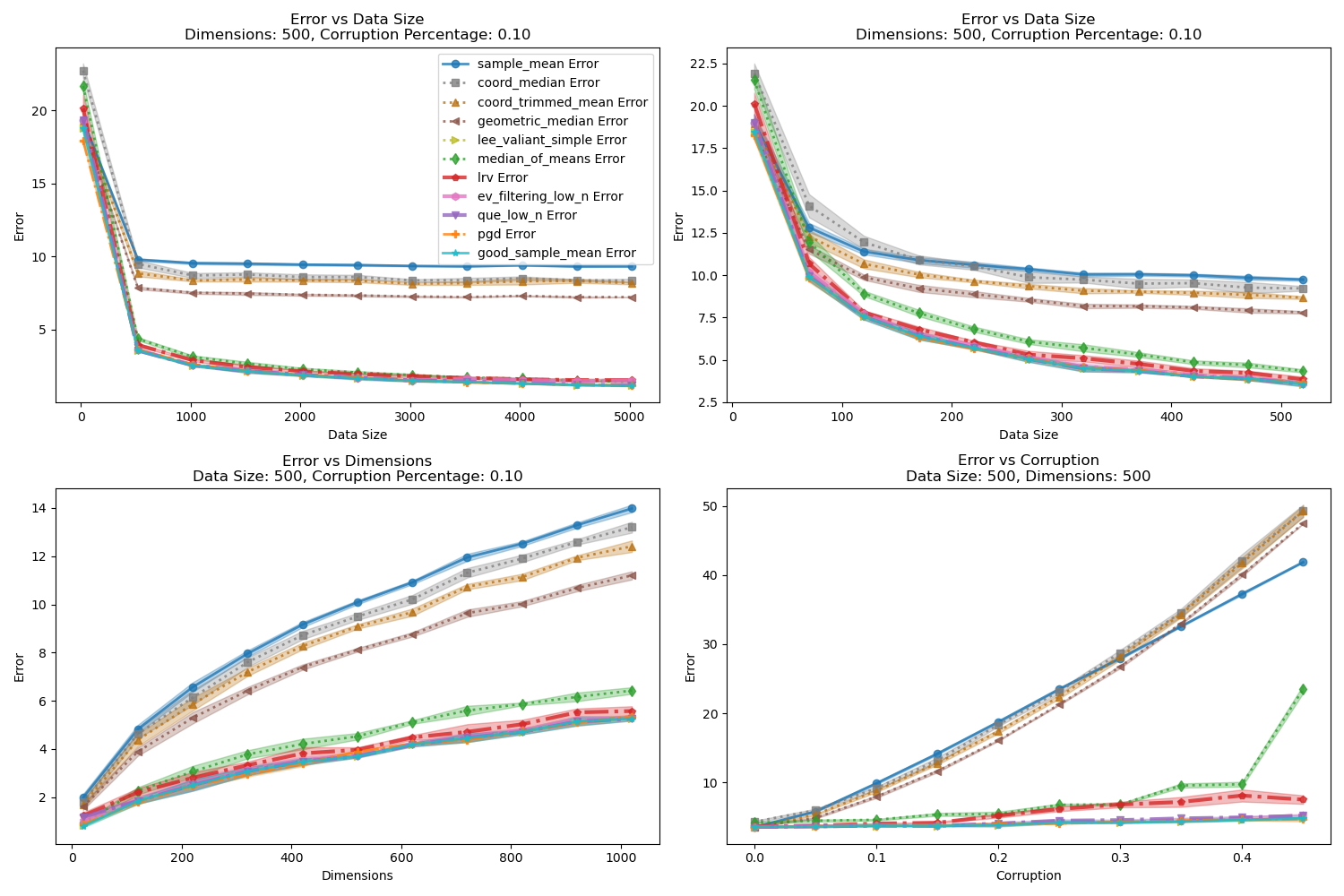}
    \caption{Corrupted Gaussian Large Spherical Covariance: Two Variance Shell Clusters}
    \label{fig:large_sp_gaus_two}
\end{figure}

\begin{figure}[h]
    \centering
    \includegraphics[width=0.85\linewidth]{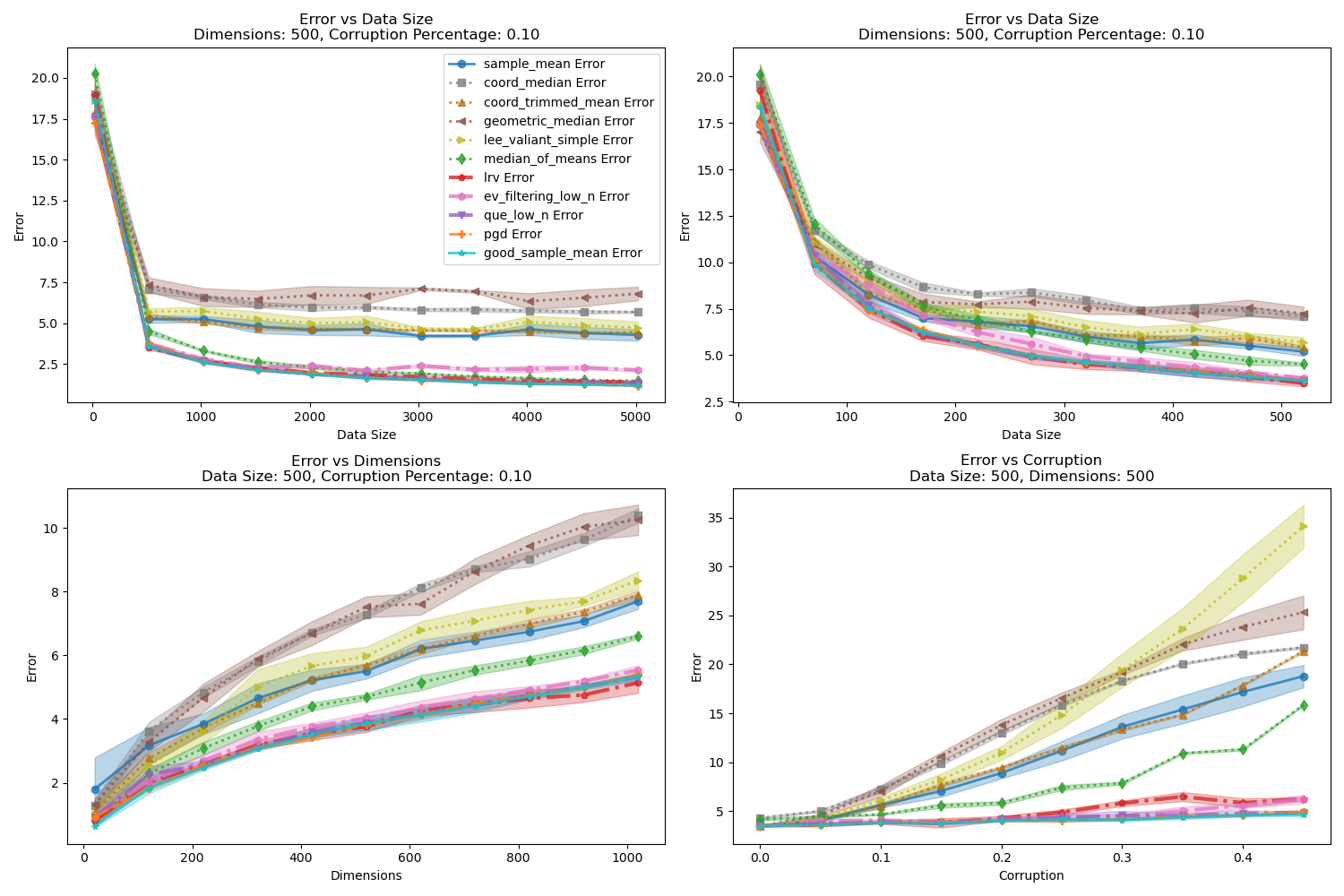}
    \caption{Corrupted Gaussian Large Spherical Covariance: DKK Noise}
    \label{fig:large_sp_dkk}
\end{figure}

\begin{figure}[h]
    \centering
    \includegraphics[width=0.85\linewidth]{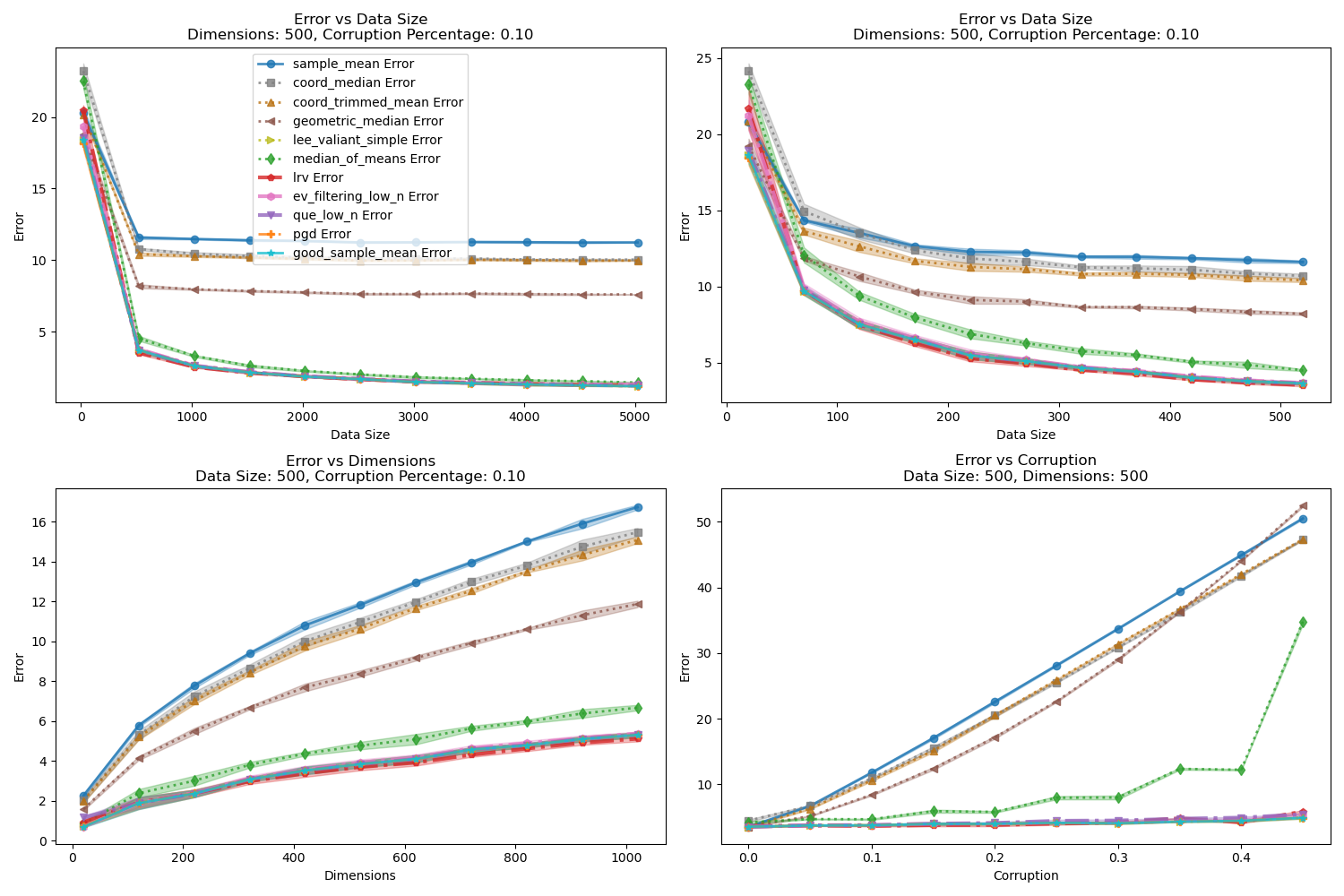}
    \caption{Corrupted Gaussian Large Spherical Covariance: In Distribution Noise}
    \label{fig:large_sp_unif_top}
\end{figure}

\begin{figure}[h]
    \centering
    \includegraphics[width=0.85\linewidth]{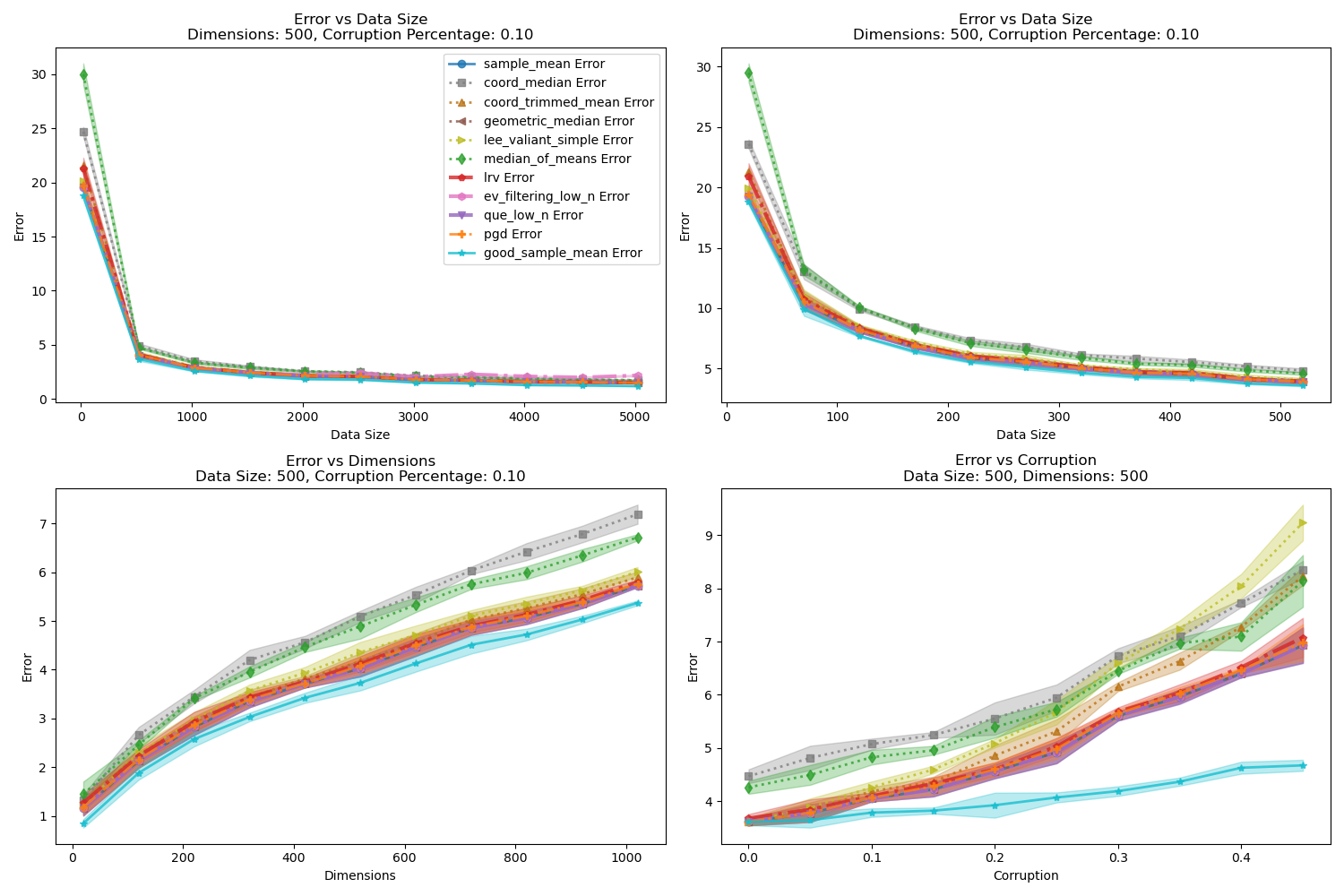}
    \caption{Corrupted Gaussian Large Spherical Covariance: Subtractive Noise}
    \label{fig:large_sp_subtractive_corruption}
\end{figure}

We find similar results to the identity covariance case across the best estimators, again observing the near optimal performance of \queln and \pgd, along with the slightly worse but still near optimal performance of \lrv. As a result of using the scaling data heuristic, \evln degrades slightly. This is especially noticeable with DKK noise, where it does not converge to \gsample error as data size increases. \queln appears to be less sensitive to the trace scaling heuristic, retaining its performance in these experiments. There is some variance among other estimators. Notably, \lvsim performs near optimally across two variance shell corruption and in distribution noise with spherical covariance (its performance in these plots is hidden amongst the best estimators which approximately match \gsample error), whereas it performs comparably  worse across analogous noise distributions for identity covariance data. Still, \lvsim does not generally perform better in the spherical covariance case compared to the (known) identity covariance case; performing consistently worse than \sample compared to outperforming \sample except with large $\eta$ under identity covariance.

\paragraph{Varying $\sigma$}
We rerun several experiments as we vary $\sigma$ from $\sigma=0.1$ to $\sigma=200$ -- the coordinate wise standard deviation of the true covariance matrix -- and fix other variables as their default values. In particular, we examine Additive Variance Shell Noise, DKK Noise, In Distribution Uniform Noise, and Two Variance Shell Clusters Noise. Results are shown in Figure \ref{fig:vary_var}. As expected, error tends to increase linearly with $\sigma$. Generally, relative performance of the algorithms remains the same, with \evln, \queln, and \pgd nearly identically matching \gsample error throughout. Surprisingly, \lrv error does not grow linearly with $\sigma$, consistently outperforming even \gsample with large enough choices of $\sigma$. A similar trend is also seen for \coordmed, but only across DKK Noise, in which it is noticeably the best estimator with larger values of $\sigma$. 

\begin{figure}[h]
    \begin{subfigure}{0.5\linewidth}
        \centering
        \includegraphics[width=\linewidth]{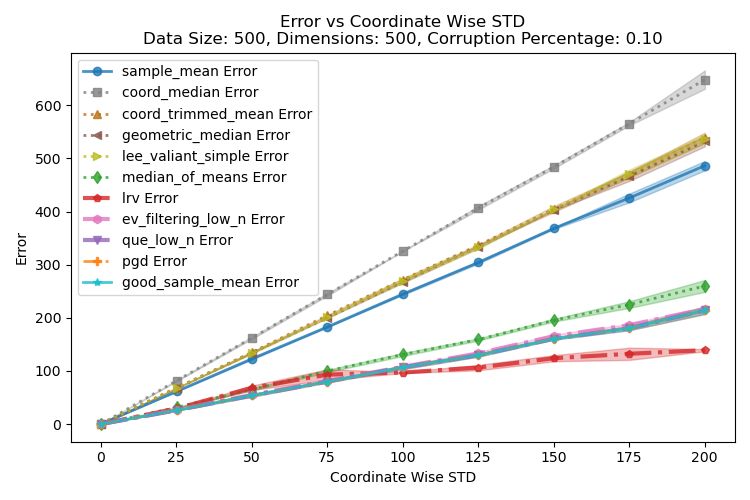}
        \caption{Additive Variance Shell Noise}
    \end{subfigure}
        \begin{subfigure}{0.5\linewidth}
        \centering
        \includegraphics[width=\linewidth]{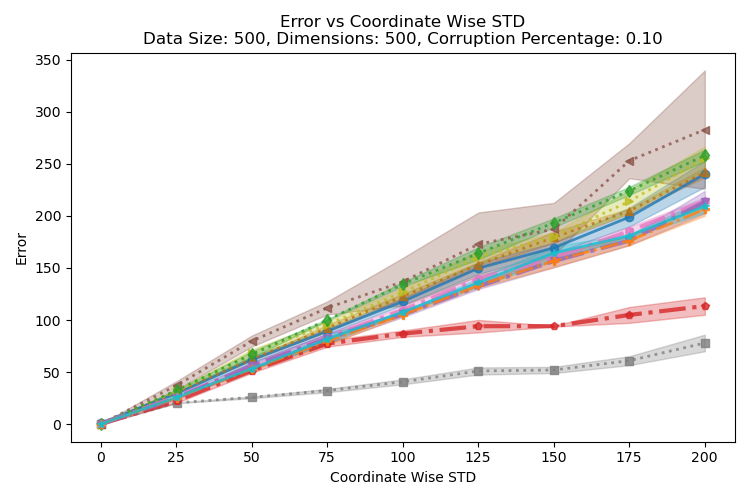}
        \caption{DKK Noise}
    \end{subfigure}
    \\
    \begin{subfigure}{0.5\linewidth}
        \centering
        \includegraphics[width=\linewidth]{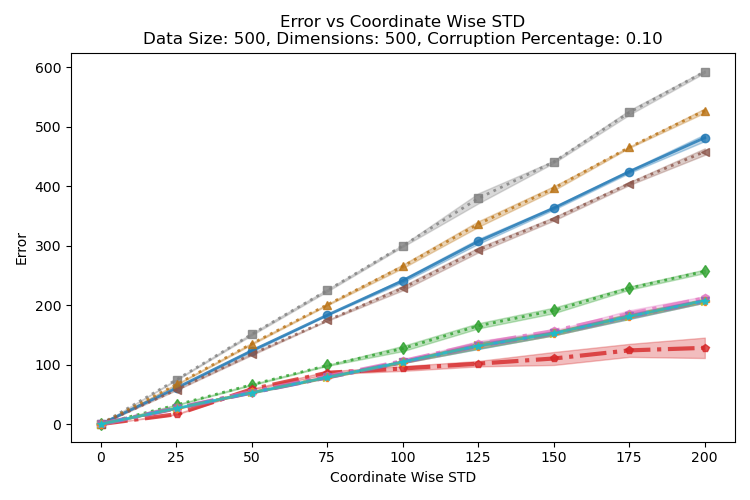}
        \caption{In Distribution Noise}
    \end{subfigure}
    \begin{subfigure}{0.5\linewidth}
        \centering
        \includegraphics[width=\linewidth]{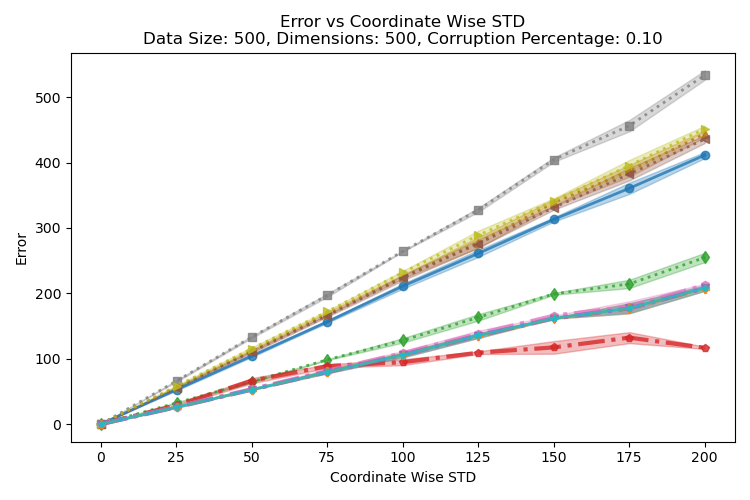}
        \caption{Two Variance Shell Clusters}
    \end{subfigure}
    \caption{Corrupted Gaussian Large Spherical Covariance: Varying Coordinate-Wise Standard Deviation $\sigma$}
    \label{fig:vary_var}
\end{figure}

\clearpage

\subsection{Corrupted Gaussian Data Unknown Non Spherical Covariance}
\label{app:unknown_non_sp}

\subsubsection{Unknown Diagonal Covariance}

Here we consider the performance of mean estimators on corrupted Gaussian data with unknown diagonal non-spherical covariance. We draw uncorrupted data from $\mathcal{N}_d(\mu, \Sigma)$ where $\mu$ is the all fives-vector and $\Sigma$ has large diminishing covariance. In particular, the diagonal elements uniformly decrease from 25 to 0.1.

 \paragraph{Noise Distributions}

 We adapt the variance shell additive noise distribution to cluster outliers to be a standard deviation away from the true mean along every coordinate axis. That is, consider corrupted data distribution 
 $Q = \mathcal{N}_d(\mu', \frac{1}{10}I)$ with $|\mu'_j - \mu_j| = \Sigma_{j}$, 
  where $\Sigma_j$ is the $j$th diagonal element in $\Sigma$, $\mu'_j$ is the $j$th coordinate of $\mu'$, and $\mu_j$ is the $j$th coordinate of the true mean $\mu$; results are shown in Figure \ref{fig:large_dim_diag_gaus_one}. We adapt in distribution uniform noise to draw each coordinate $j$ of a corrupted data point from 
  $\mathsf{Uniform}(\mu_j, \mu_j + \Sigma_{j})$; 
  results are shown in Figure \ref{fig:large_dim_diag_unif_top}. For large outlier noise we weight the distance of clusters from $\mu$ by $\sqrt{\frac{\Tr(\Sigma)}{d}}$. That is, consider corrupted data distribution $Q = 0.7\mathcal{N}_d(\mu^0, \frac{1}{10} I) \cup 0.3\mathcal{N}_d(\mu^1, \frac{1}{10} I)$ where $\|\mu - \mu^0\| = 10\sqrt{\frac{\Tr(\Sigma)}{d}} \sqrt{d}$, $\|\mu - \mu^1\| = 20\sqrt{\frac{\Tr(\Sigma)}{d}} \sqrt{d}$, and $\theta = 75^\circ$ where $\theta$ is the angle between $\mu^0$ and $\mu^1$.; 
  results are shown in Figure \ref{fig:large_dim_diag_obvious}. We also utilize subtractive noise which already works in this case; results are shown in Figure \ref{fig:large_dim_diag_subtractive_corruption}. 
 
\begin{figure}[h]
    \centering
    \includegraphics[width=0.85\linewidth]{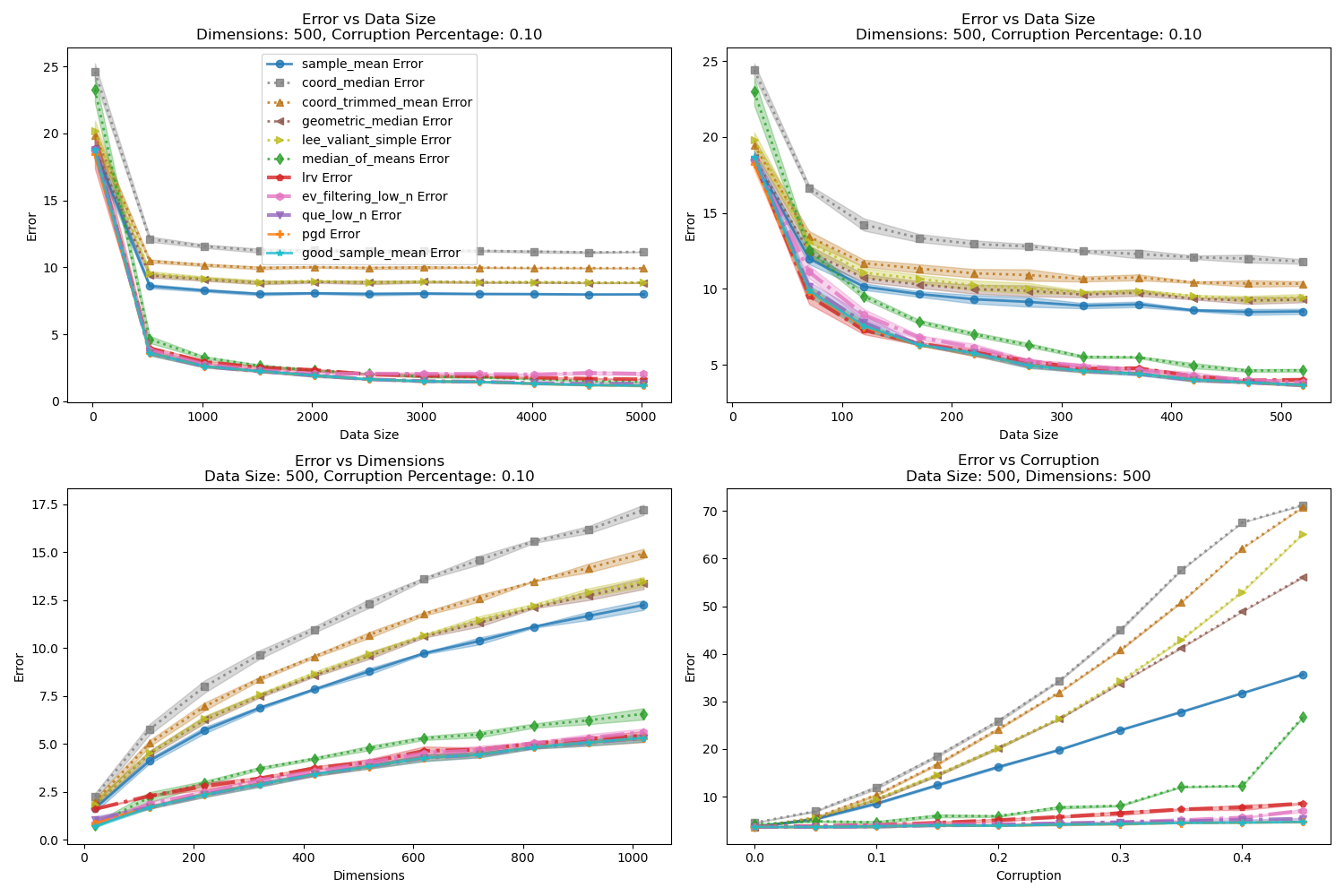}
    \caption{Corrupted Gaussian Large Diminishing Diagonal Covariance: Additive Variance Shell Noise}
    \label{fig:large_dim_diag_gaus_one}
\end{figure}

\begin{figure}[h]
    \centering
    \includegraphics[width=0.85\linewidth]{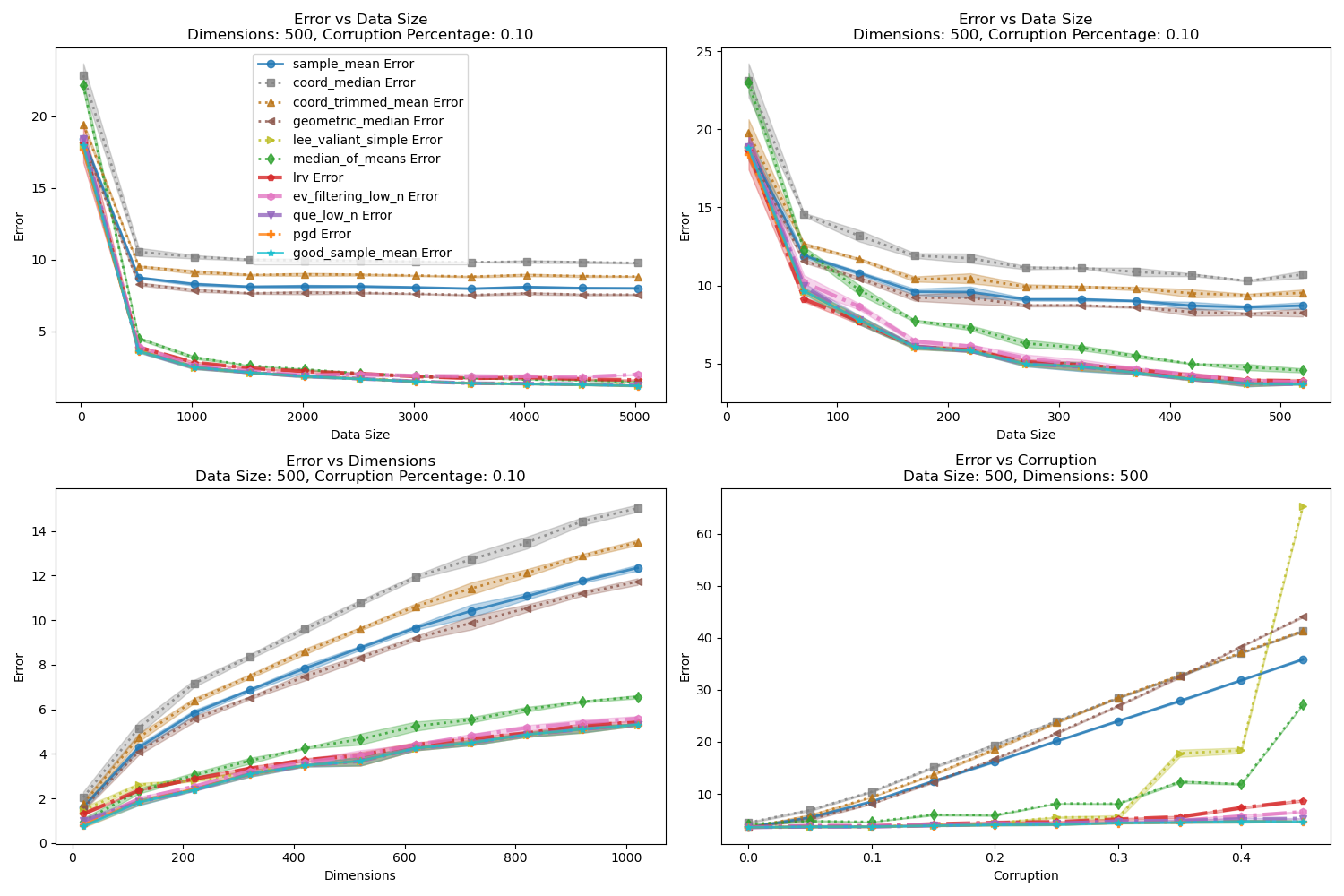}
    \caption{Corrupted Gaussian Large Diminishing Diagonal Covariance: In Distribution Noise}
    \label{fig:large_dim_diag_unif_top}
\end{figure}

\begin{figure}[h]
    \centering
    \includegraphics[width=0.85\linewidth]{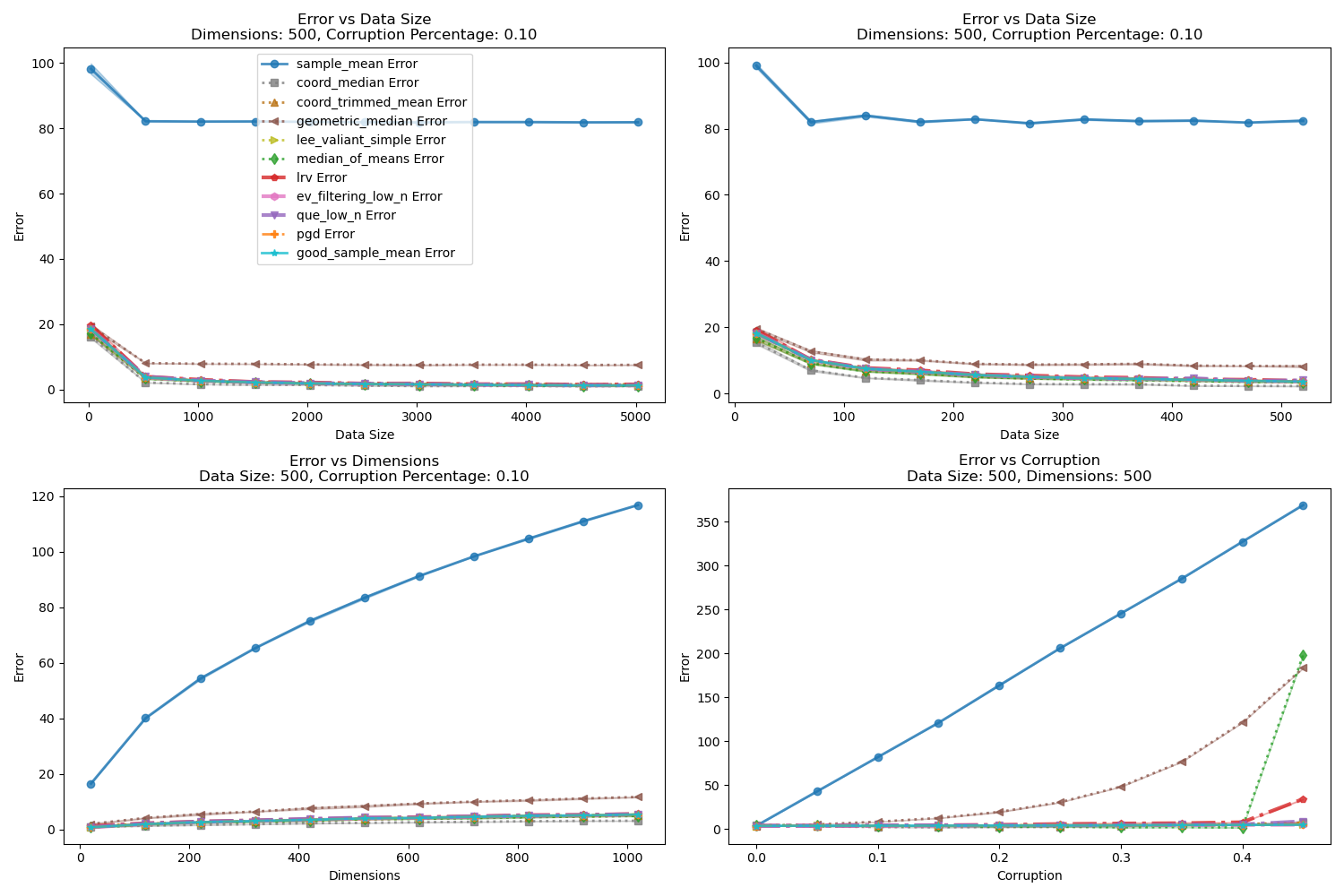}
    \caption{Corrupted Gaussian Large Diminishing Diagonal Covariance: Large Outliers}
    \label{fig:large_dim_diag_obvious}
\end{figure}

\begin{figure}[h]
    \centering
    \includegraphics[width=0.85\linewidth]{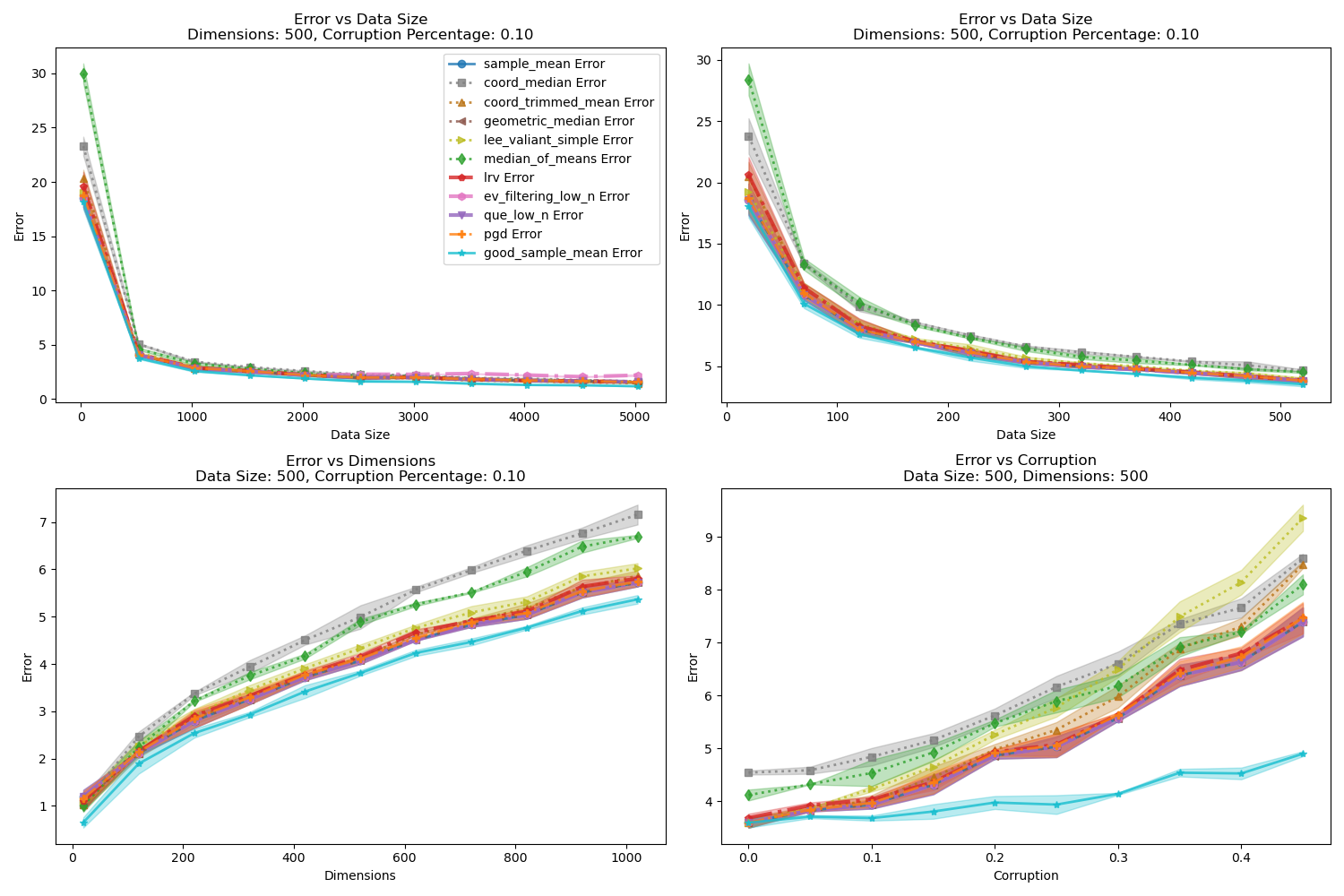}
    \caption{Corrupted Gaussian Large Diminishing Diagonal Covariance: Subtractive Noise}
    \label{fig:large_dim_diag_subtractive_corruption}
\end{figure}

Again, we find that \queln and \pgd nearly match \gsample error across distributions, with 
\lrv doing slightly worse but still significantly outperforming \sample. \evln still does among the best here, but, as in the large spherical covariance case, sees slight degradation as a result of the scaling data heuristic. In particular, error does not clearly converge to \gsample error as $n$ increases over Additive Variance Shell Noise and In Distribution Noise. \queln does not encounter this issue despite employing the same heuristic to generalize to non-identity covariance data, suggesting that it is more robust to distributional assumptions. Additionally, \medmean performs best among simpler estimators, outperforming the \sample with $n \approx d$ and performing similarly to \gsample with sufficiently large $n$.

\paragraph{Varying top Eigenvalue}

We rerun Additive Variance Shell Noise and In Distribution Noise as we vary the squareroot of the top eigenvalue of the true covariance matrix, labeled as $\sigma$, from $\sigma=0.1$ to $\sigma=200$. In particular, the diagonal of the covariance will uniformly decrease from $\sigma^2$ to $0.1$. For every choice of $\sigma$, the noise is scaled as described previously.  These results are shown in Figure \ref{fig:vary_var_non_sp}. Like in the spherical case, we find that the relative performance of algorithms remains nearly identical throughout choices of $\sigma$.

\begin{figure}[h]
\centering
\begin{subfigure}[t]{0.48\linewidth}
    \centering
    \includegraphics[width=\linewidth]{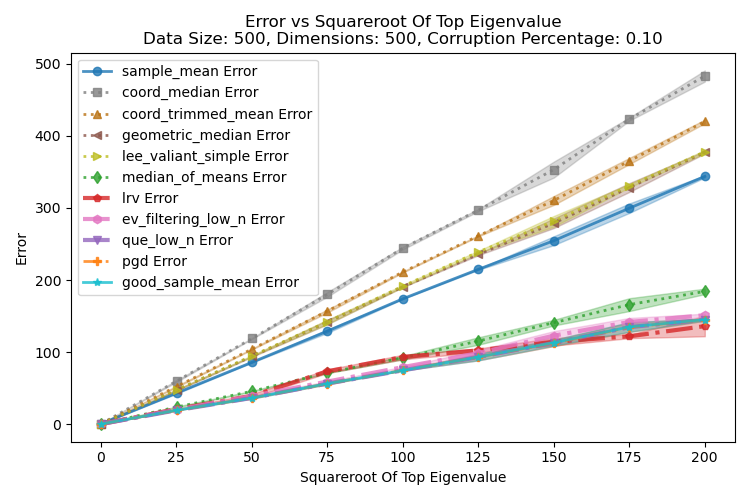} 
    \caption{Additive Variance Shell Noise}
\end{subfigure}
\hfill
\begin{subfigure}[t]{0.48\linewidth}
    \centering
    \includegraphics[width=\linewidth]{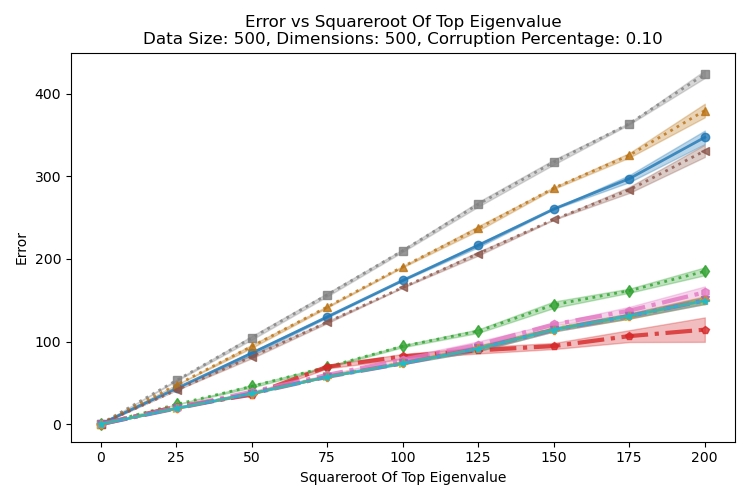} 
    \caption{In Distribution Noise}
\end{subfigure}
\caption{Corrupted Gaussian large diminishing diagonal covariance: Varying the square root of the top eigenvalue: $\sigma$}
\label{fig:vary_var_non_sp}
\end{figure}

\clearpage

\subsubsection{Unconstrained Covariance}
\label{sec:corrsph}

So far, we have only examined inlier data with diagonal covariance matrices. However, in line with the intuition that there is nothing inherently special about the standard orthonormal basis, we hope for a robust estimator to work well regardless of the choice of coordinate axis. Since the covariance matrix is always symmetric, it is also diagonalizable by taking the eigenvectors as the orthonormal basis. Then, any possible data distribution over unconstrained covariance can be framed as a data distribution over a diagonal matrix by using these eigenvectors as the orthonormal basis. As a result, any robust estimator that does not leverage the standard orthonormal basis should perform equally well on unconstrained covariance. However, this does not necessarily hold for the estimators that we examine. We employ a trace estimate to adapt \evln and \queln to the unknown covariance case. \coordmed, \coordprune, and \medmean all directly utilize coordinate wise calculations. \lrv utilizes a trace estimate when downweighting points. In this section, we evaluate the performance of robust mean estimators over data with non-diagonal covariance matrices.

\paragraph{Rotated Data Noise}

Because the covariance matrix is always symmetric, it is diagonalizable, and experiments over unconstrained covariance can be framed as an ablation on noise distributions over inliers with diagonal covariances. We reuse data and noise distributions, but randomly rotate everything before estimation, resulting in unconstrained true covariances and appropriately difficult noise distributions. Random rotation is implemented by generating a standard normal matrix and utilizing its QR decomposition. 

We examine the performance on Rotated Identity Covariance with DKK Noise in Figure \ref{fig:rotate_id_dkk}; Rotated Identity Covariance with Subtractive Noise in Figure \ref{fig:rotate_id_sub}; Rotated Large Spherical Covariance with Additive Variance Shell Noise (with coordinate-wise standard deviation $\sigma=5$) in Figure \ref{fig:rotate_sp_gaus_one}; and Rotated Large Diminishing Covariance with Additive Variance Shell Noise (with squareroot of the top eigenvalue $\sigma=5$) in Figure \ref{fig:rotate_nonsp_gaus_one}. As in the original experiments, we set the true mean, $\mu$, to be the all-fives vector. 

\begin{figure}[h]
    \centering
    \includegraphics[width=0.85\linewidth]{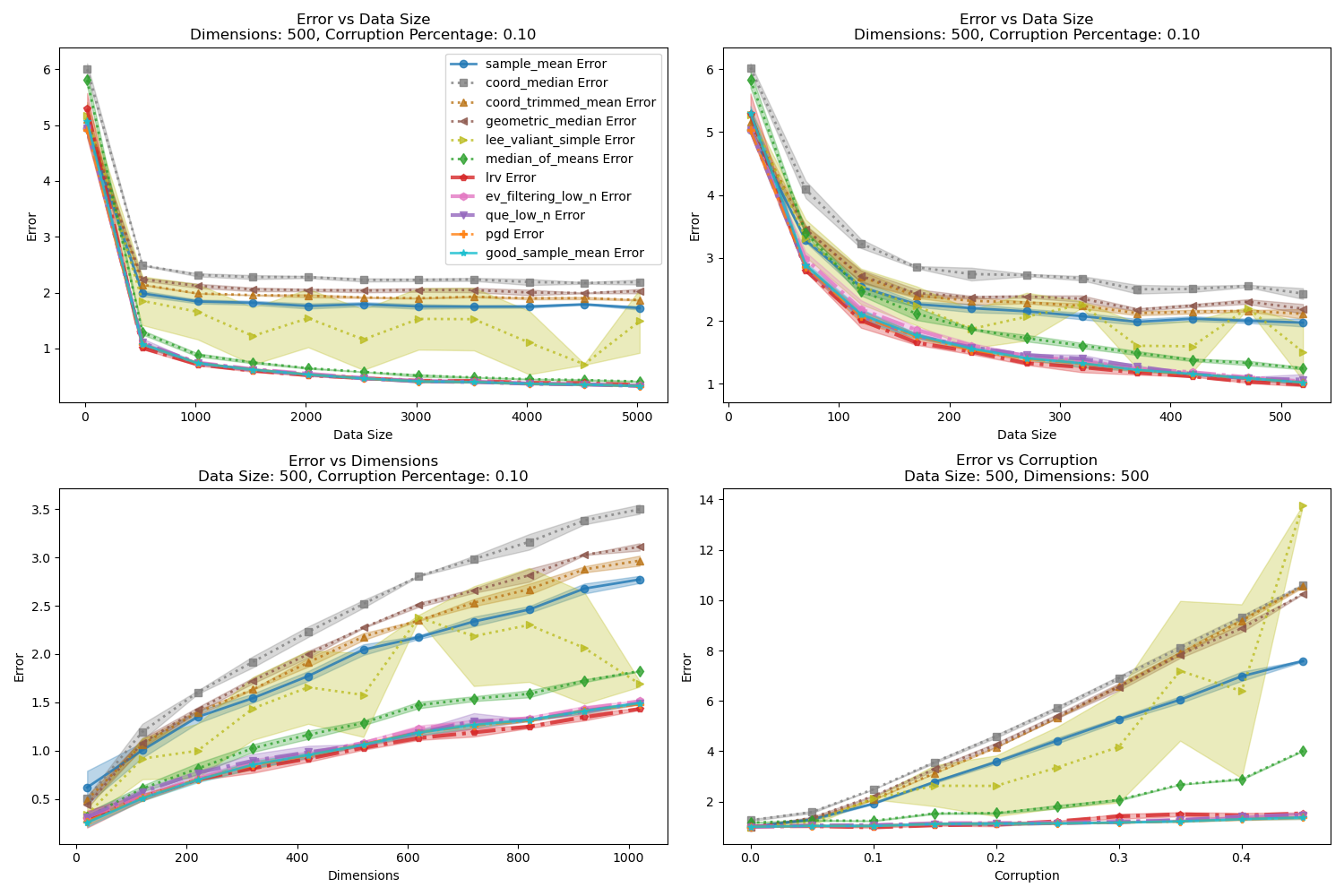}
    \caption{Corrupted Rotated Identity Covariance - DKK Noise}
    \label{fig:rotate_id_dkk}
\end{figure}

\begin{figure}[h]
    \centering
    \includegraphics[width=0.85\linewidth]{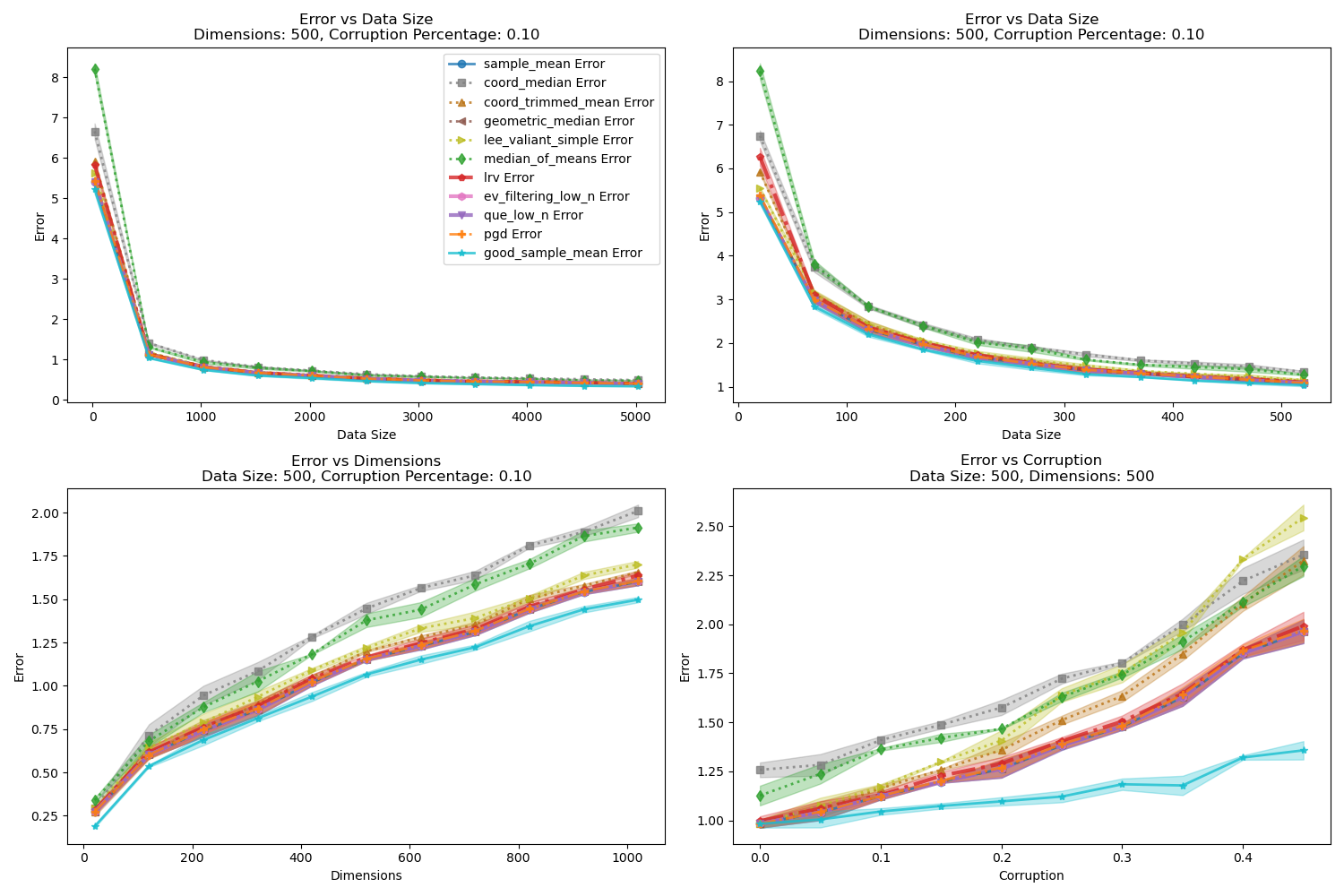}
    \caption{Corrupted Rotated Identity Covariance - Subtractive Noise}
    \label{fig:rotate_id_sub}
\end{figure}

\begin{figure}[h]
    \centering
    \includegraphics[width=0.85\linewidth]{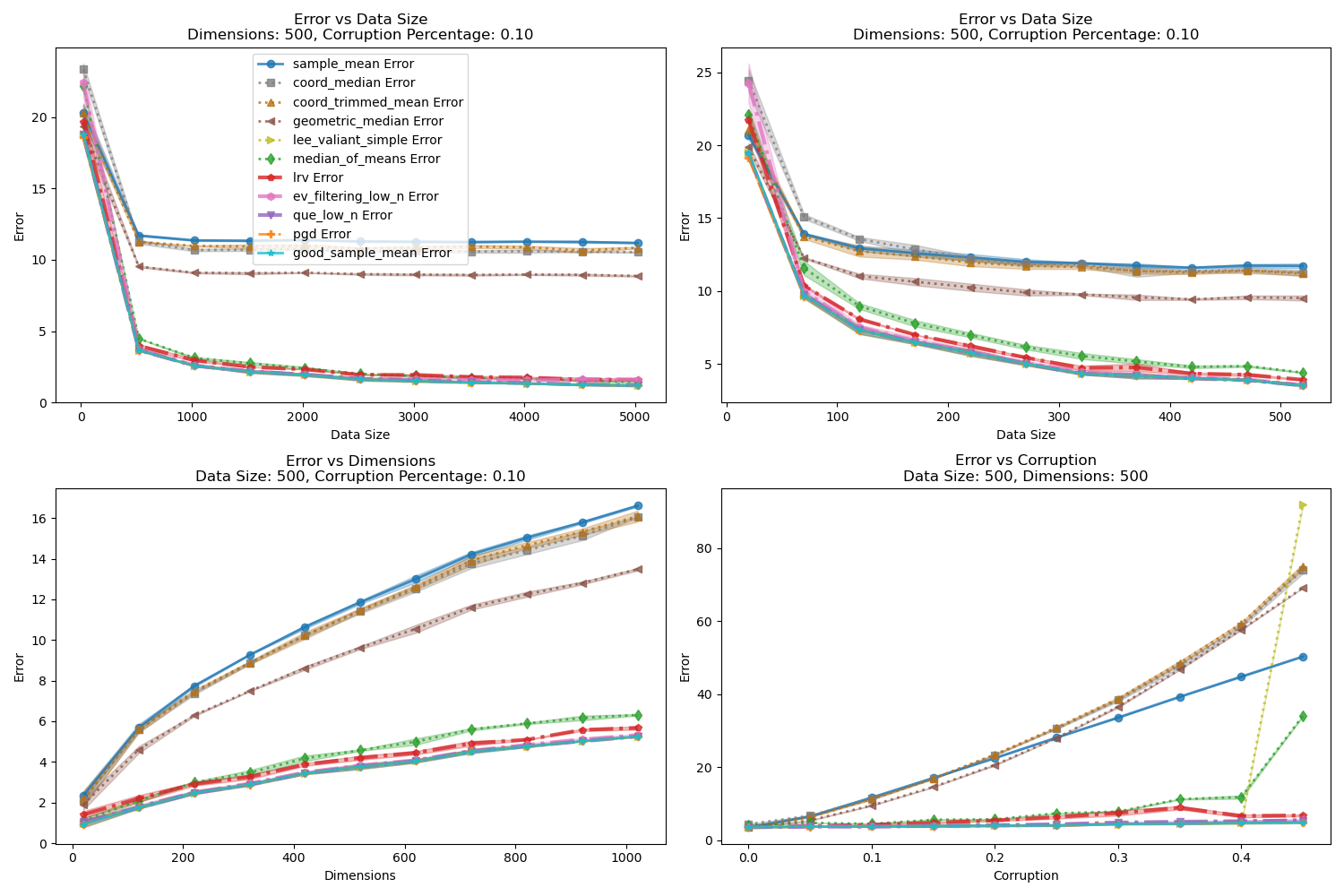}
    \caption{Corrupted Rotated Large Spherical Covariance - Additive Variance Shell Noise}
    \label{fig:rotate_sp_gaus_one}
\end{figure}

\begin{figure}[h]
    \centering
    \includegraphics[width=0.85\linewidth]{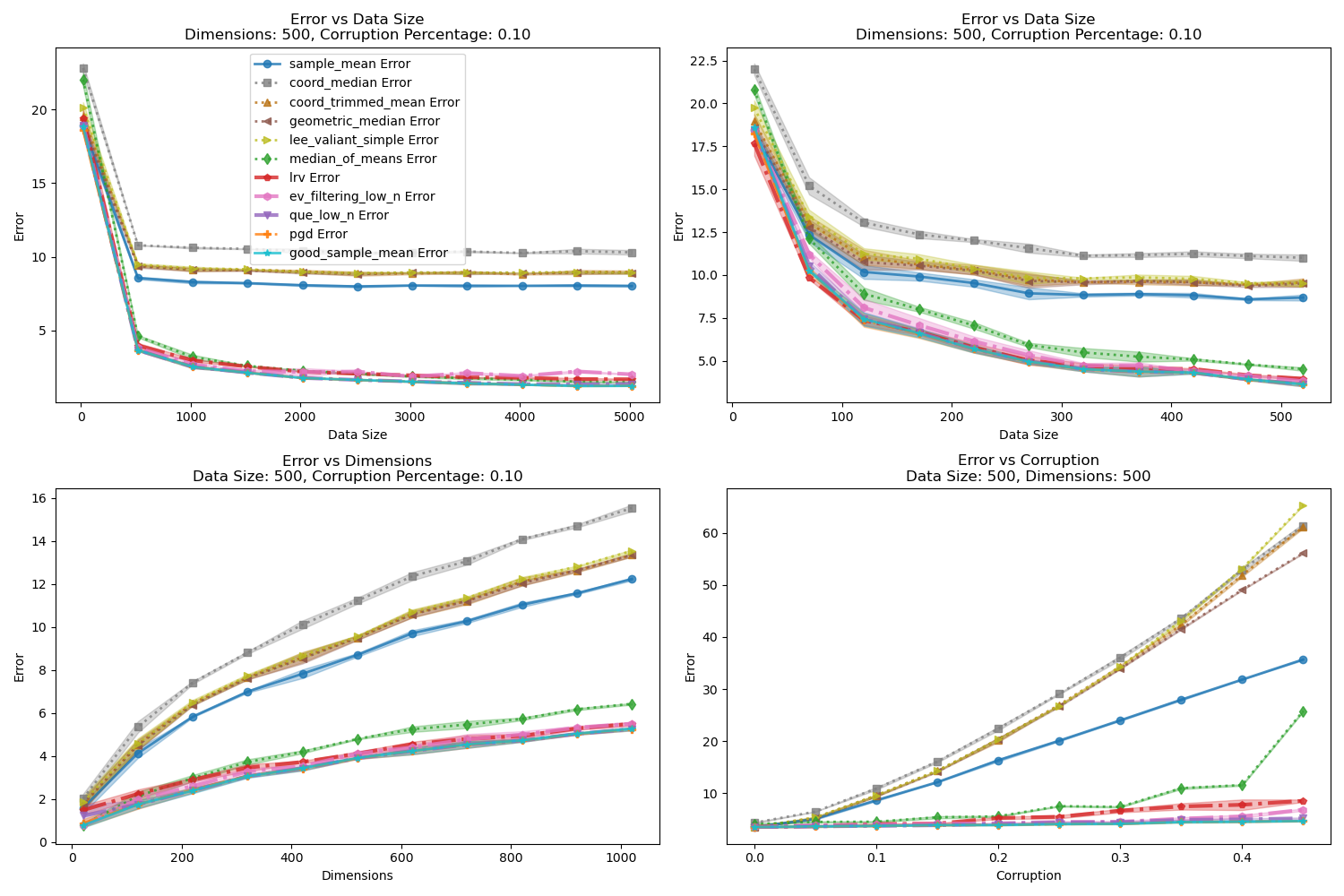}
    \caption{Corrupted Rotated Large Diminishing Covariance - Additive Variance Shell Noise}
    \label{fig:rotate_nonsp_gaus_one}
\end{figure}

We find nearly identical results among the best estimators to the corresponding non-rotated data experiment. While many of these algorithms induce a bias to the coordinate axis, they are not enough to significantly skew results in the distributions that we examine. There is some variation between \coordmed and \coordprune with the corresponding non-rotated data experiments, but no major changes in their trends. There is no such variation for \queln, \evln, or \medmean among the settings that we test.

\clearpage

\subsection{Hyperparameter Tuning}
\label{app:hp_tuning}
In this section we tune the hyperparameters of some of the most interesting algorithms, \medmean, \lrv, \evln, and \pgd. We utilize the best hyperparameters in this section across all other experiments. In general, we find that none of these algorithms are overly sensitive to choices of hyperparamaters, as long as they lay within a reasonable range.

We evaluate performance over a subset of 4 corrupted data distributions previously discussed: Identity Covariance with DKK Noise; Identity Covariance with In Distribution Noise; Large Spherical Covariance with Variance Shell Additive Noise; Large Diminishing Covariance with Variance Shell Additive Noise. We manually pick the hyperparameters that achieve the best performance across these distributions, or when similar use default ones from the corresponding paper.  

\paragraph{Median of how many means?}

Here we explore the parameter $k$ in \medmean algorithm. This parameter controls the number of chunks that we split the data into; then we take the median of $k$ means determined by these chunks. We vary $k$ in the set $[3, 5, 10, 15, 20, 30]$. For the case where $n < k$, we simply set $k=n$. These results are shown in Figures \ref{fig:med_mean_dkk}, \ref{fig:med_mean_unif_top}, \ref{fig:med_mean_sp_gaus_one}, \ref{fig:med_mean_nonsp_gaus_one}. We find that although there is not always an obvious choice for $k$, that $k=10$ tends to perform well throughout most settings. However, we find that this and larger choices of $k$ are more prone to error as $\eta$ increases than smaller choices of $k$. Approximately when $\eta>0.15$, $k=3$ becomes the best choice of $k$. However, with smaller corruption, such as $\eta=0.10$ which we generally test, $k=3$ performs notably worse, making $k=10$ a better choice. Since we utilize $\eta=0.1$ as a default value, we set $k=10$ throughout our experiments. 

\begin{figure}[h]
    \centering
    \includegraphics[width=0.85\linewidth]{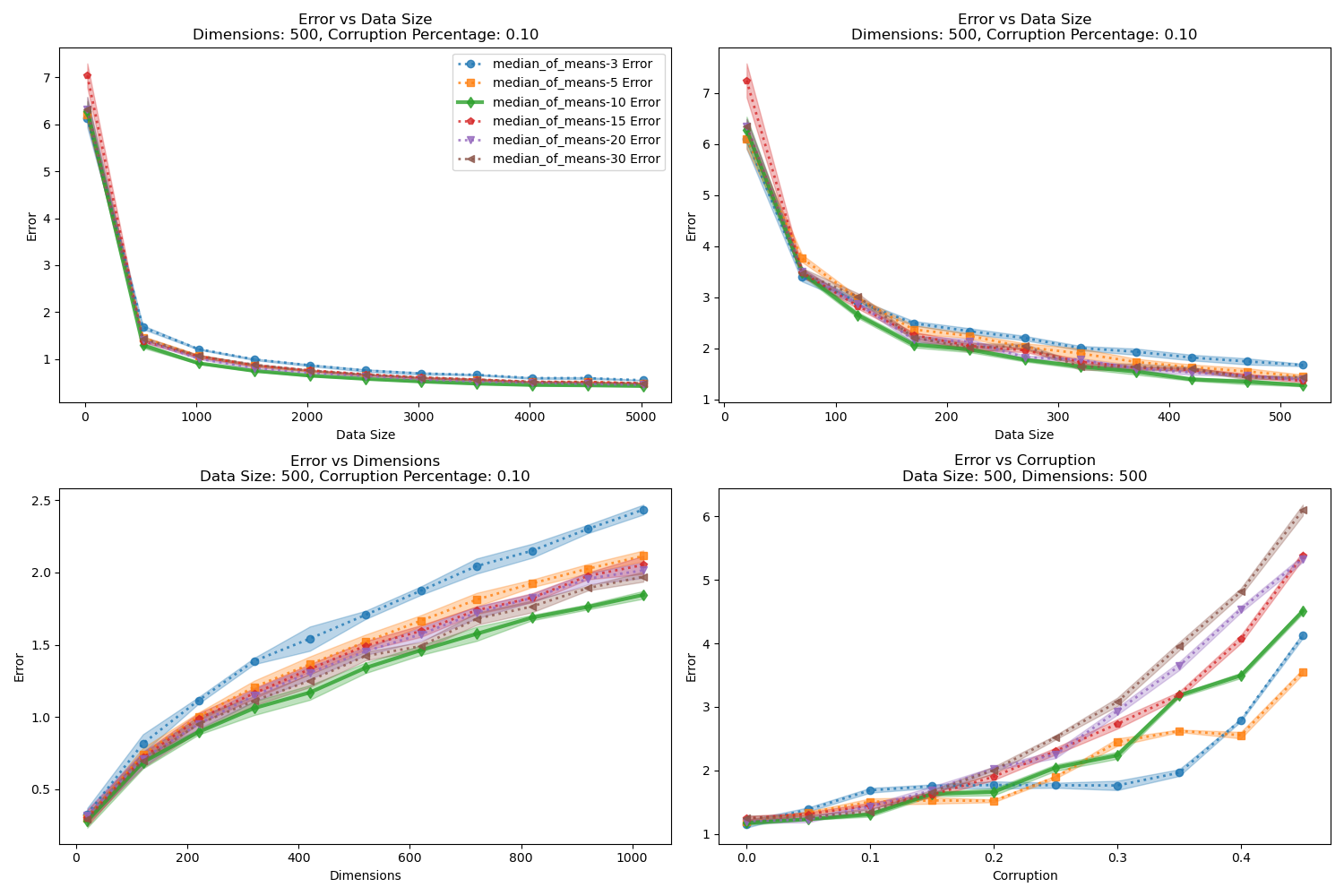}
    \caption{Median Of Means - Number of Chunks $k$: Identity Covariance, DKK Noise}
    \label{fig:med_mean_dkk}
\end{figure}

\begin{figure}[h]
    \centering
    \includegraphics[width=0.85\linewidth]{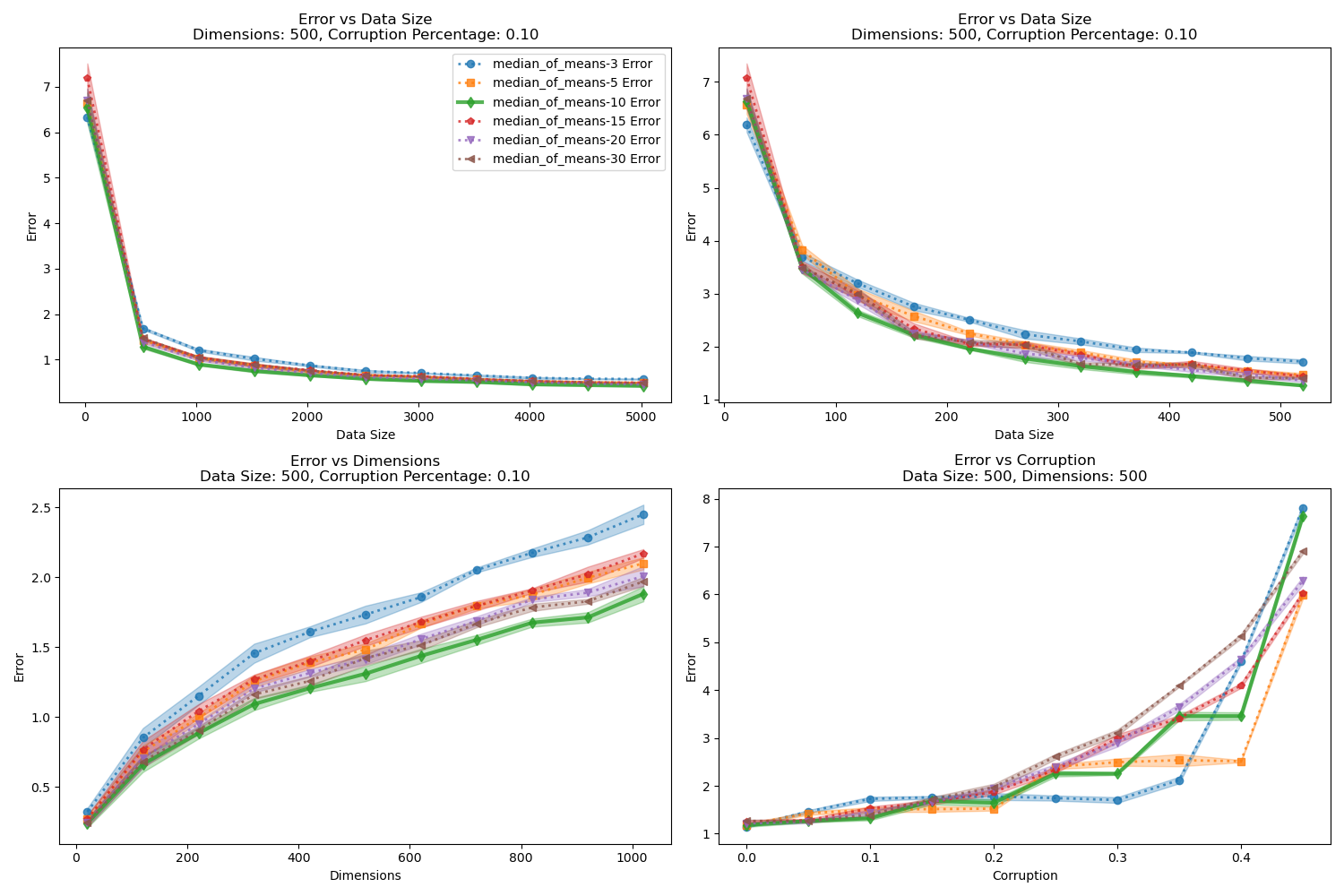}
    \caption{Median Of Means - Number of Chunks $k$: Identity Covariance, In Distribution Noise}
    \label{fig:med_mean_unif_top}
\end{figure}

\begin{figure}[h]
    \centering
    \includegraphics[width=0.85\linewidth]{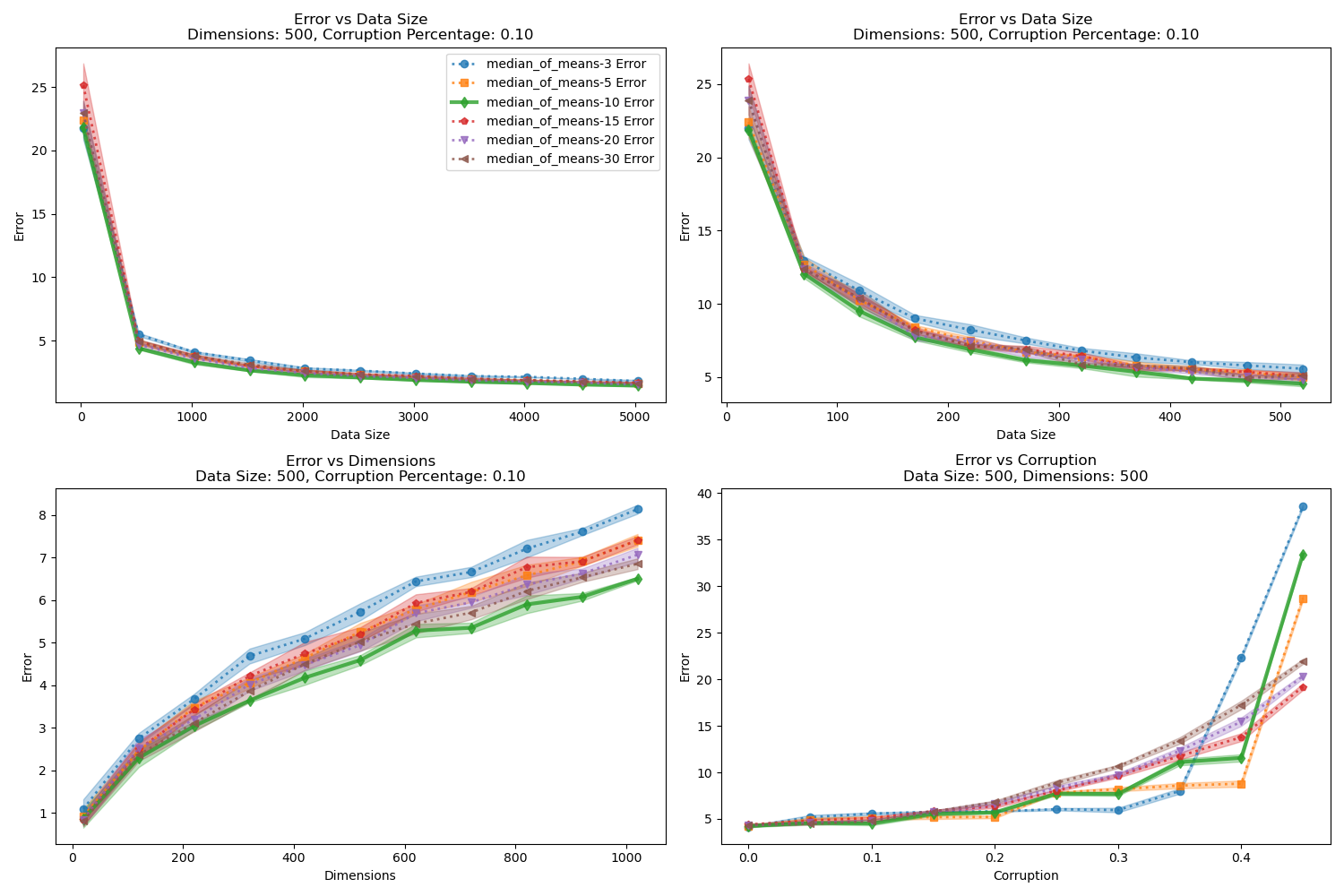}
    \caption{Median Of Means - Number of Chunks $k$: Large Spherical Covariance, Additive Variance Shell Noise}
    \label{fig:med_mean_sp_gaus_one}
\end{figure}

\begin{figure}[h]
    \centering
    \includegraphics[width=0.85\linewidth]{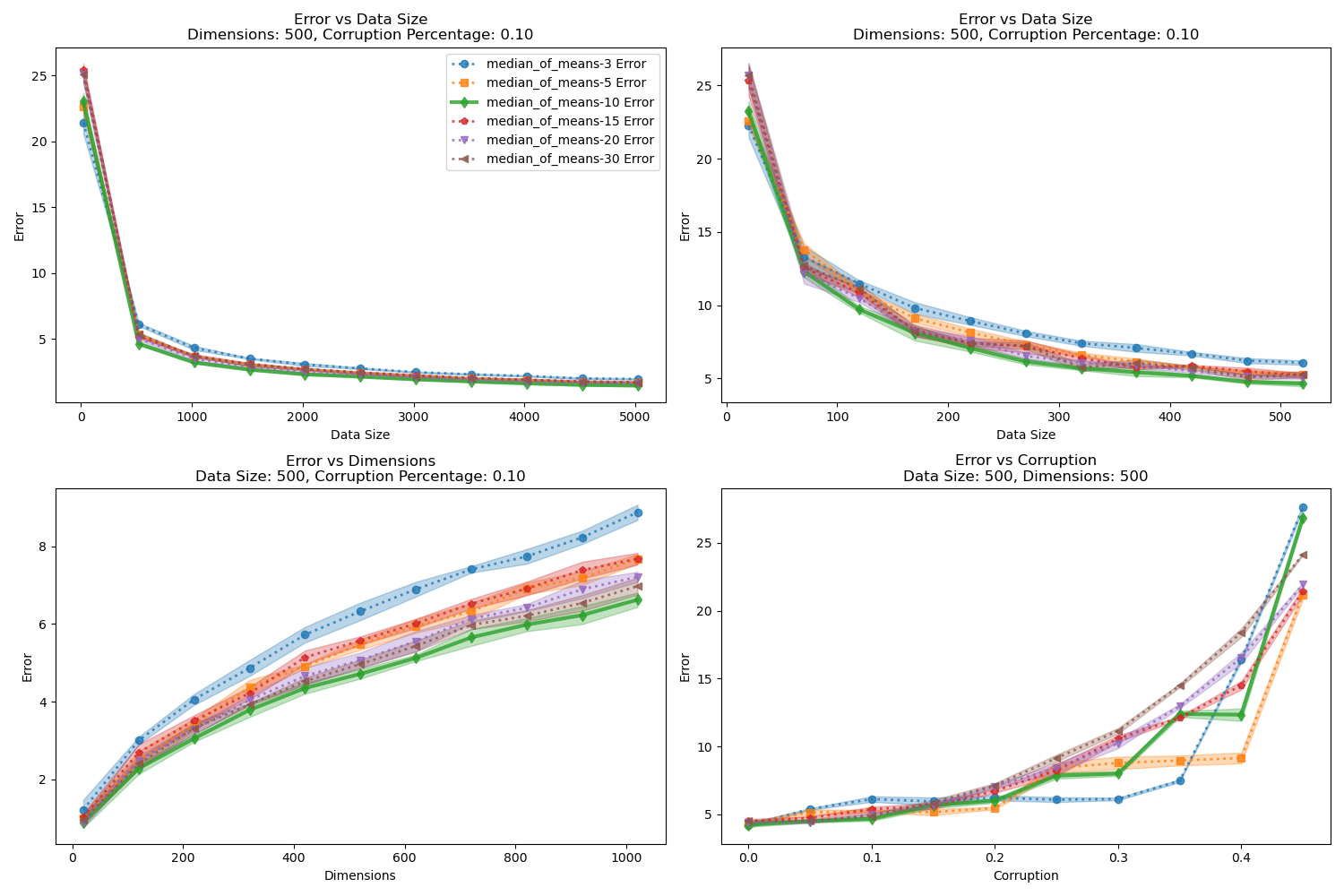}
    \caption{Median Of Means - Number of Chunks $k$: Large Diminishing Covariance, Additive Variance Shell Noise}
    \label{fig:med_mean_nonsp_gaus_one}
\end{figure}

\clearpage

\paragraph{LRV Weighting Procedure}

 Here we explore the weighting procedure in \lrv. First, we examine the parameter $C$ in the weighting procedure for \lrv. This parameter is used when we calculate weights for each point, $x_i$ as $w_i = \exp(-\|x_i - a\|^2/(C * s^2))$. 

We vary $C$ in the set $[0.1, 0.5, 1, 5, 10, 20, 50]$. Results are shown in Figures \ref{fig:lrvc_id_dkk}, \ref{fig:lrvc_id_unif_top}, \ref{fig:lrvc_sp_gaus_one}, \ref{fig:lrvc_nonsp_gaus_one}. We notice that performance may degrade with choices of $C$ that are too high or too low, such as with $C=0.5$ and $50$. We also notice that smaller choices of $C$ tend to degrade worse with greater corruption. To strike a balance, we select $C=1$ throughout our experiments, which consistently performs among the best throughout the hyperparameter trials that we test, and is the default value used the original author's implementation of \lrv. Although there are cases where larger choices of $C$ noticeably outperform $C=1$, this is not robust as such choices may perform meaningfully worse over different noise distributions. For example, $C=20$ noticeably outperforms $C=1$ over Large Diminishing Covariance with Additive Variance Shell Noise, especially with larger $\eta$, but performs much worse over Identity Covariance with DKK Noise and Large Spherical Covariance with Additive Variance Shell Noise.

\begin{figure}[h]
    \centering
    \includegraphics[width=0.85\linewidth]{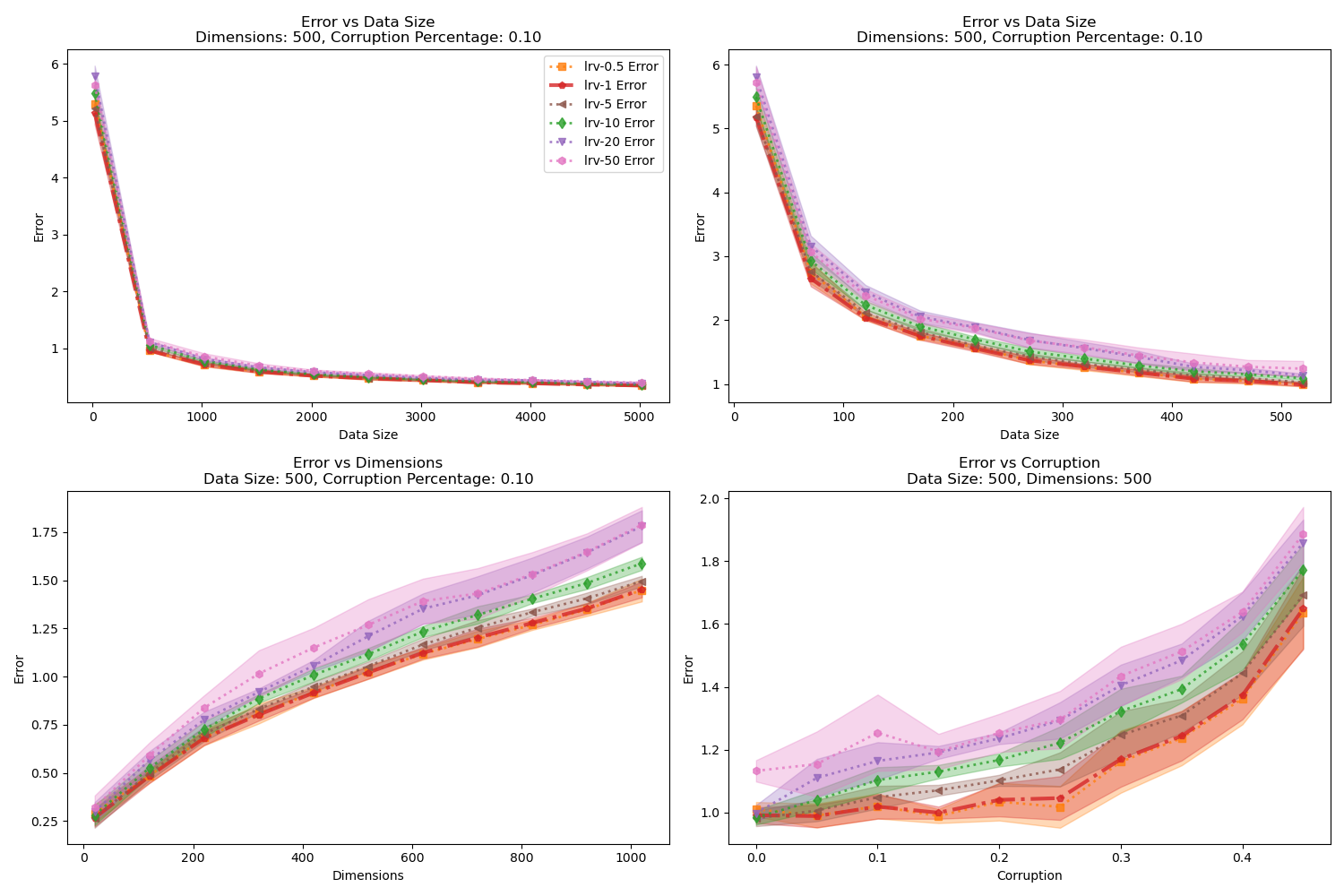}
    \caption{LRV - Choice Of $C$: Identity Covariance, DKK Noise}
    \label{fig:lrvc_id_dkk}
\end{figure}

\begin{figure}[h]
    \centering
    \includegraphics[width=0.85\linewidth]{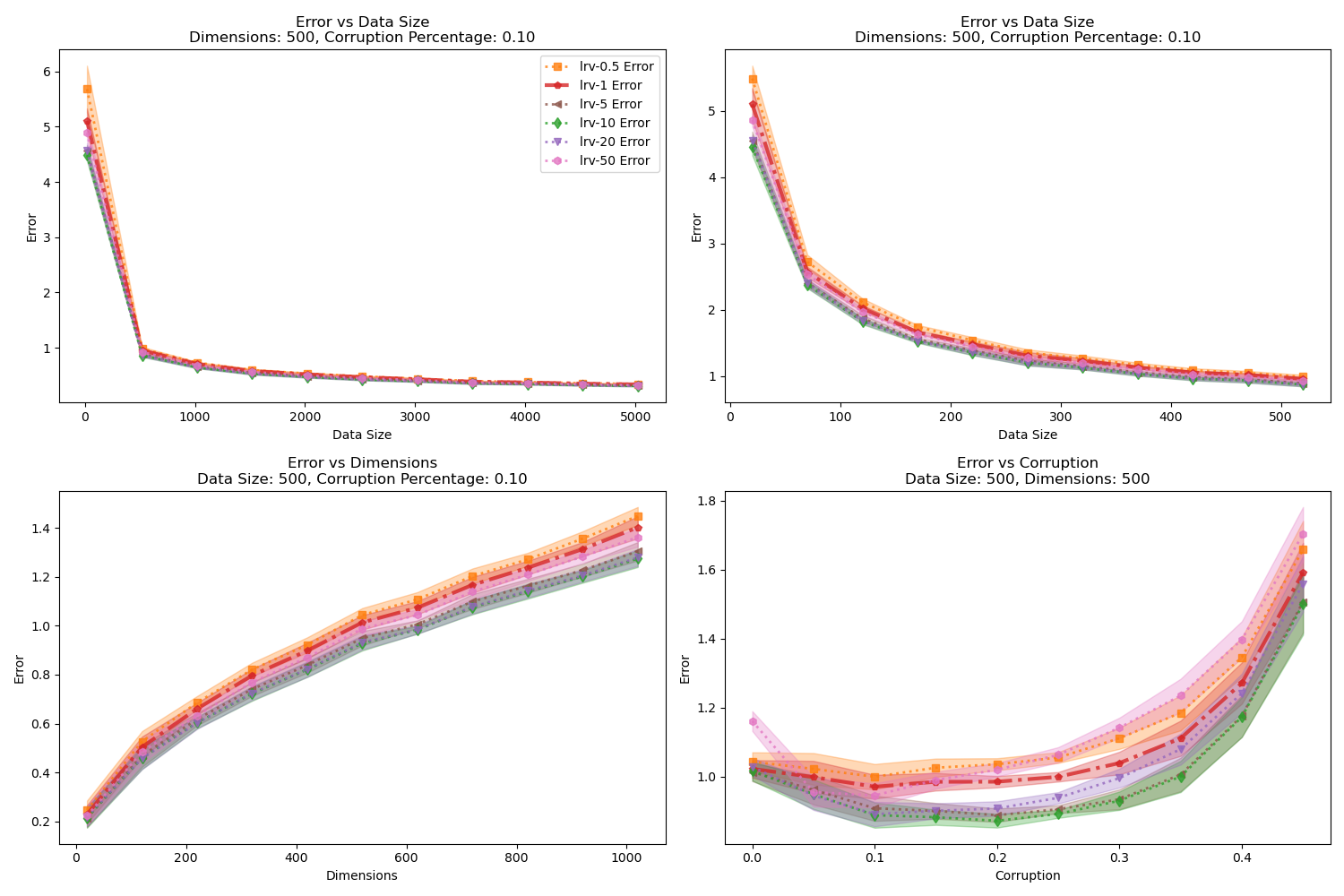}
    \caption{LRV - Choice Of $C$: Identity Covariance, In Distribution Noise}
    \label{fig:lrvc_id_unif_top}
\end{figure}

\begin{figure}[h]
    \centering
    \includegraphics[width=0.85\linewidth]{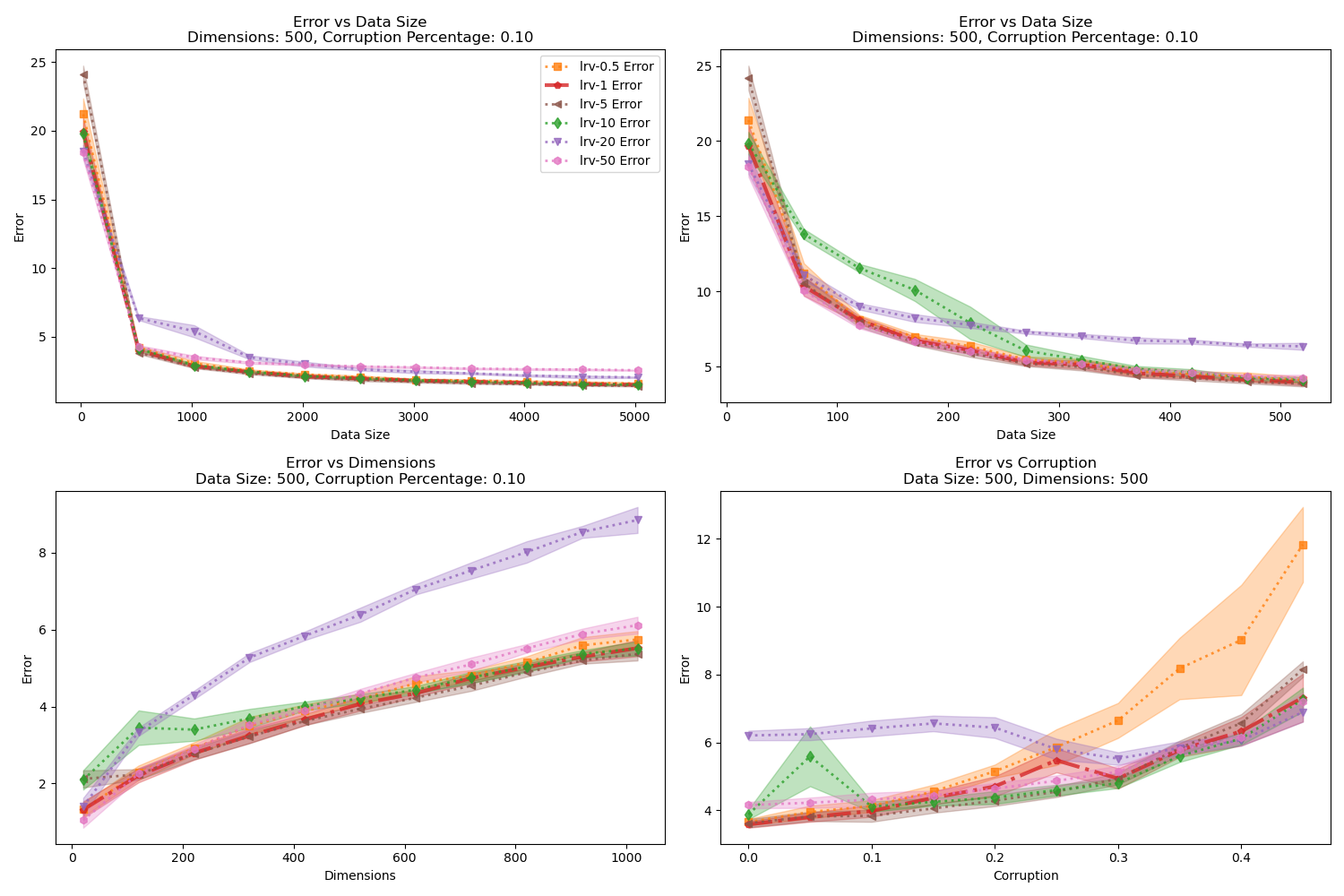}
    \caption{LRV - Choice Of $C$: Large Spherical Covariance, Additive Variance Shell Noise}
    \label{fig:lrvc_sp_gaus_one}
\end{figure}

\begin{figure}[h]
    \centering
    \includegraphics[width=0.85\linewidth]{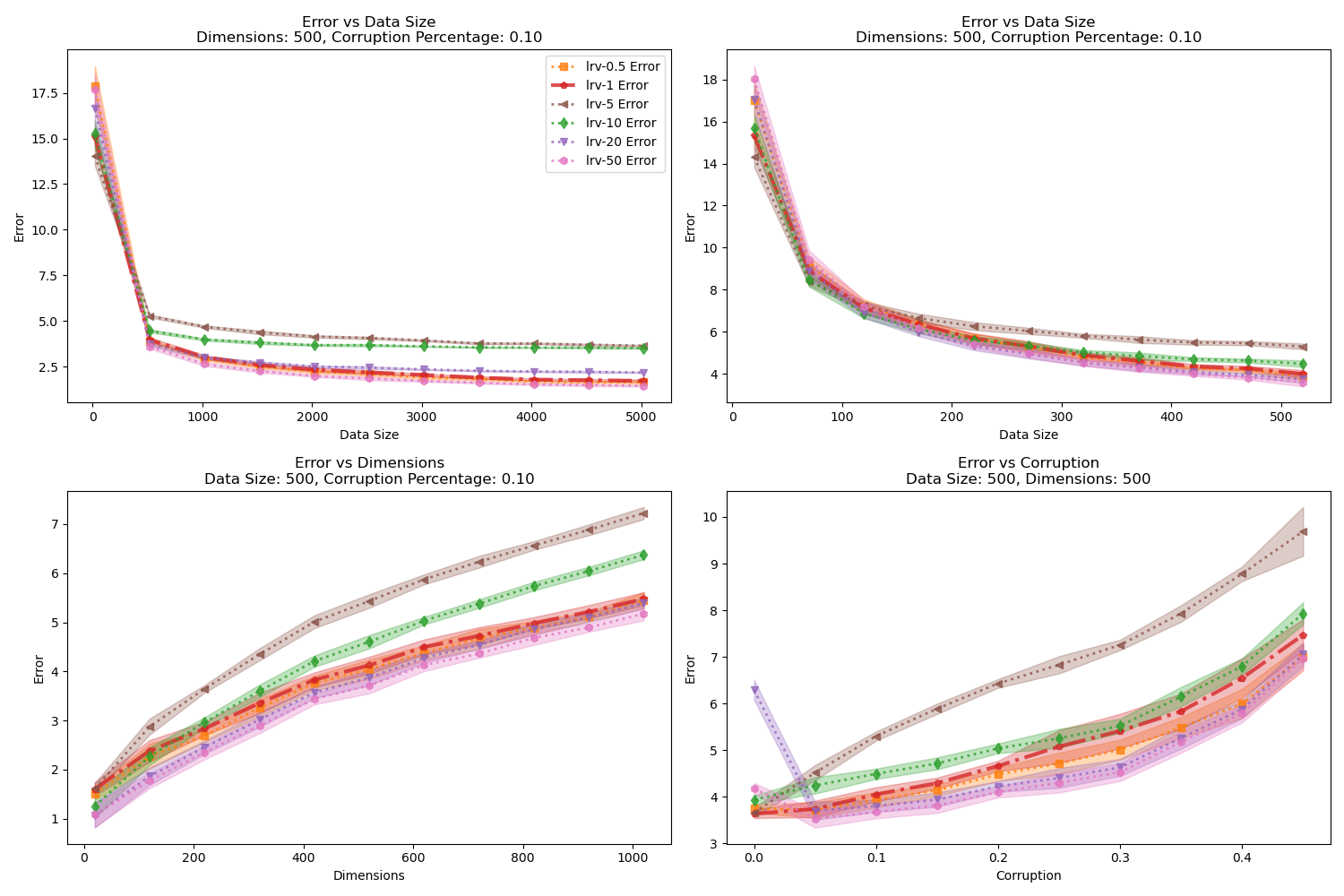}
    \caption{LRV - Choice Of $C$: Large Diminishing Covariance, Additive Variance Shell Noise}
    \label{fig:lrvc_nonsp_gaus_one}
\end{figure}

\clearpage

We additionally compare the weighting procedure of \lrv that we consider with one meant for more general distributions discussed by \cite{lai2016agnostic}. Rather than downweighting outliers, this alternate procedure completely prunes outliers by calculating a point, $\mu'$, analogous to the coordinate wise median, finding a ball centered at $\mu'$ that contains $1-\tau$ percentage of data points, and throwing away all points outside of this ball. Results are shown in Figures \ref{fig:lrvg_dkk}, \ref{fig:lrvg_unif_top}, \ref{fig:lrvg_sp_gaus_one}, \ref{fig:lrvg_nonsp_gaus_one}. This general weighting procedure performs meaningfully worse than the Gaussian weighting procedure and degrades significantly worse with larger corruption across all of the distributions considered. However, it achieves similar results as data size increases. Since we focus on the low data size regime and synthetic data with Gaussian inliers, we only evaluate \lrv with Gaussian-based outlier downweighting.

\begin{figure}[h]
    \centering
    \includegraphics[width=0.85\linewidth]{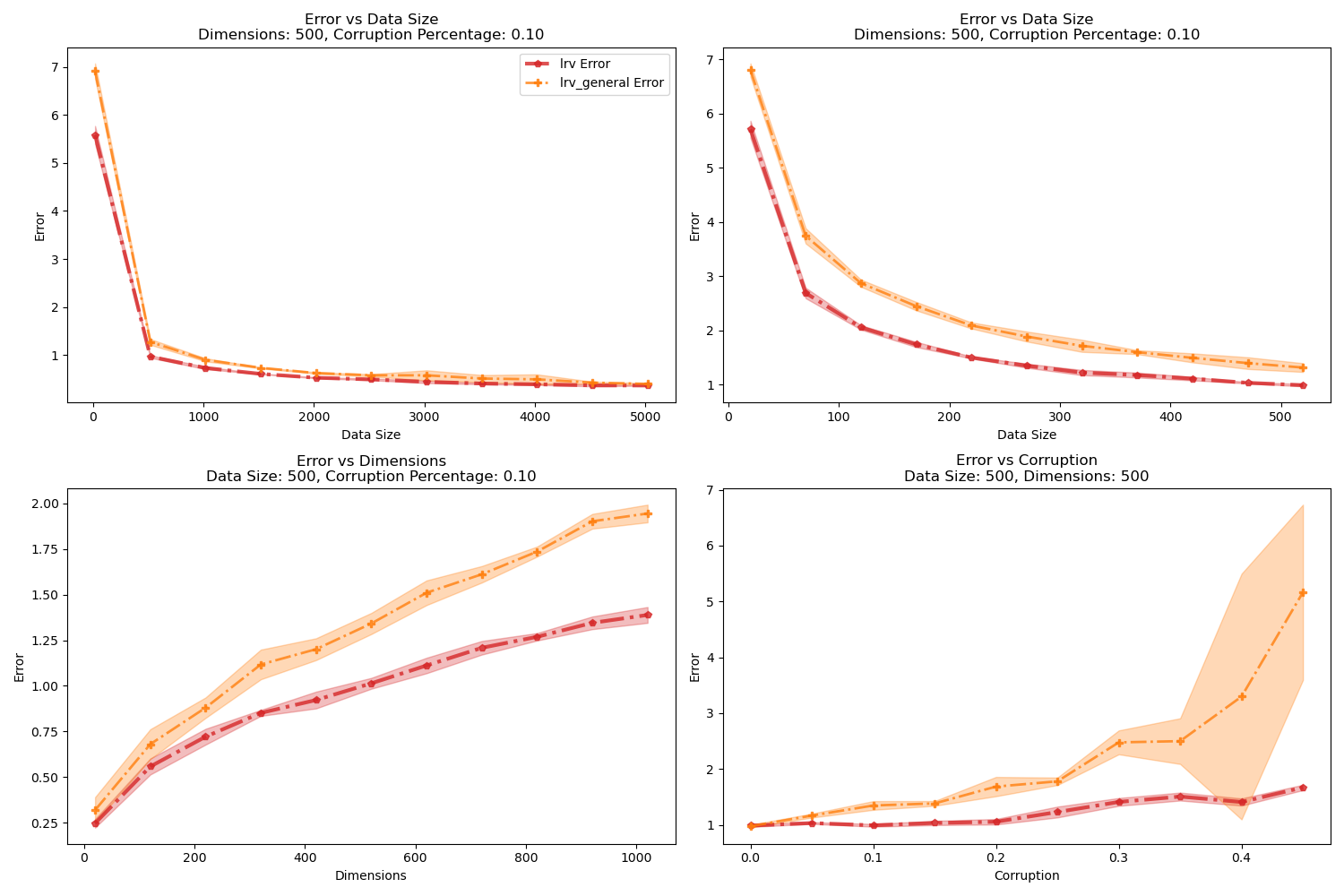}
    \caption{LRV - Gaussian Vs General Weighting: Identity Covariance, DKK Noise}
    \label{fig:lrvg_dkk}
\end{figure}

\begin{figure}[h]
    \centering
    \includegraphics[width=0.85\linewidth]{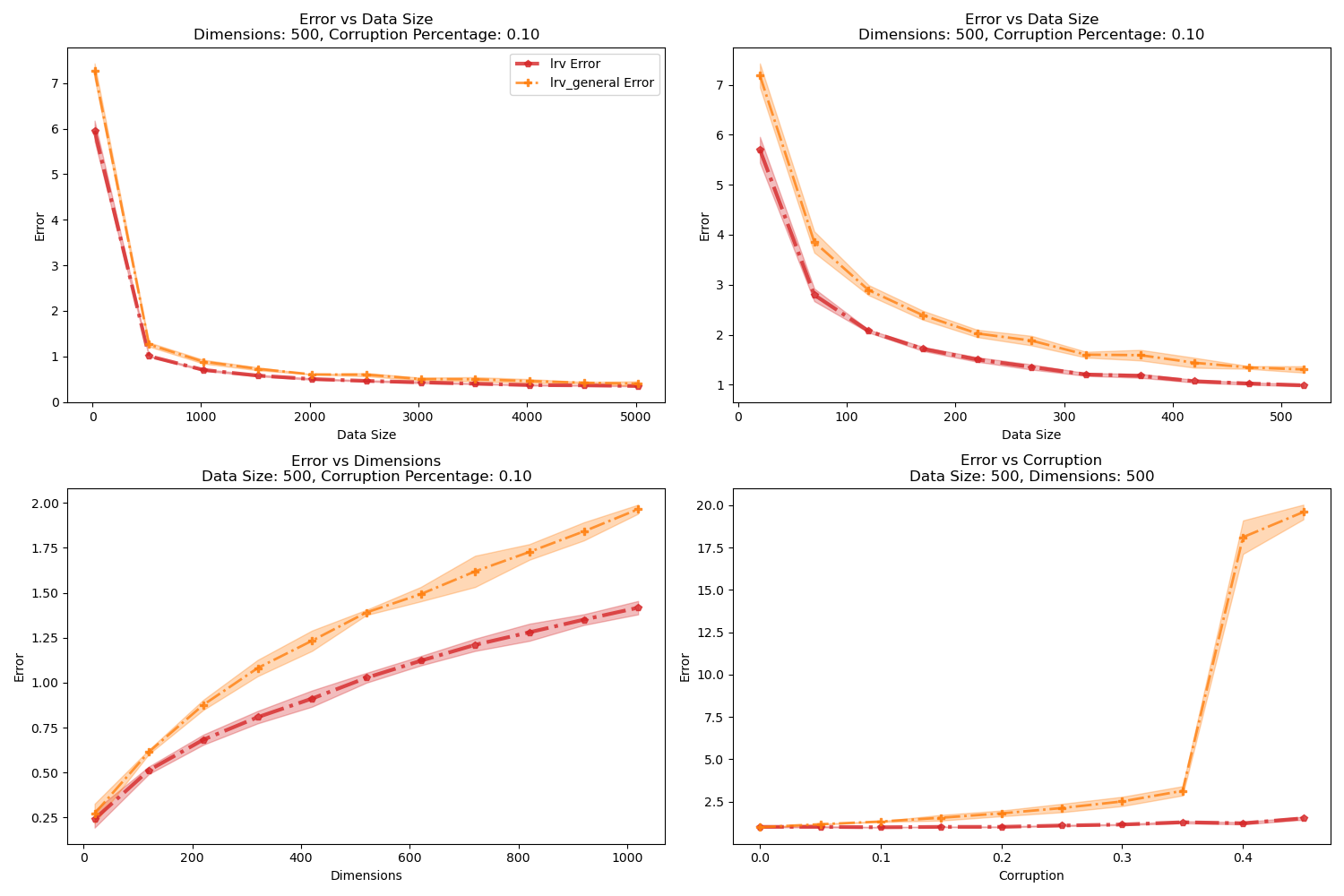}
    \caption{LRV - Gaussian Vs General Weighting: Identity Covariance, In Distribution Noise}
    \label{fig:lrvg_unif_top}
\end{figure}

\begin{figure}[h]
    \centering
    \includegraphics[width=0.85\linewidth]{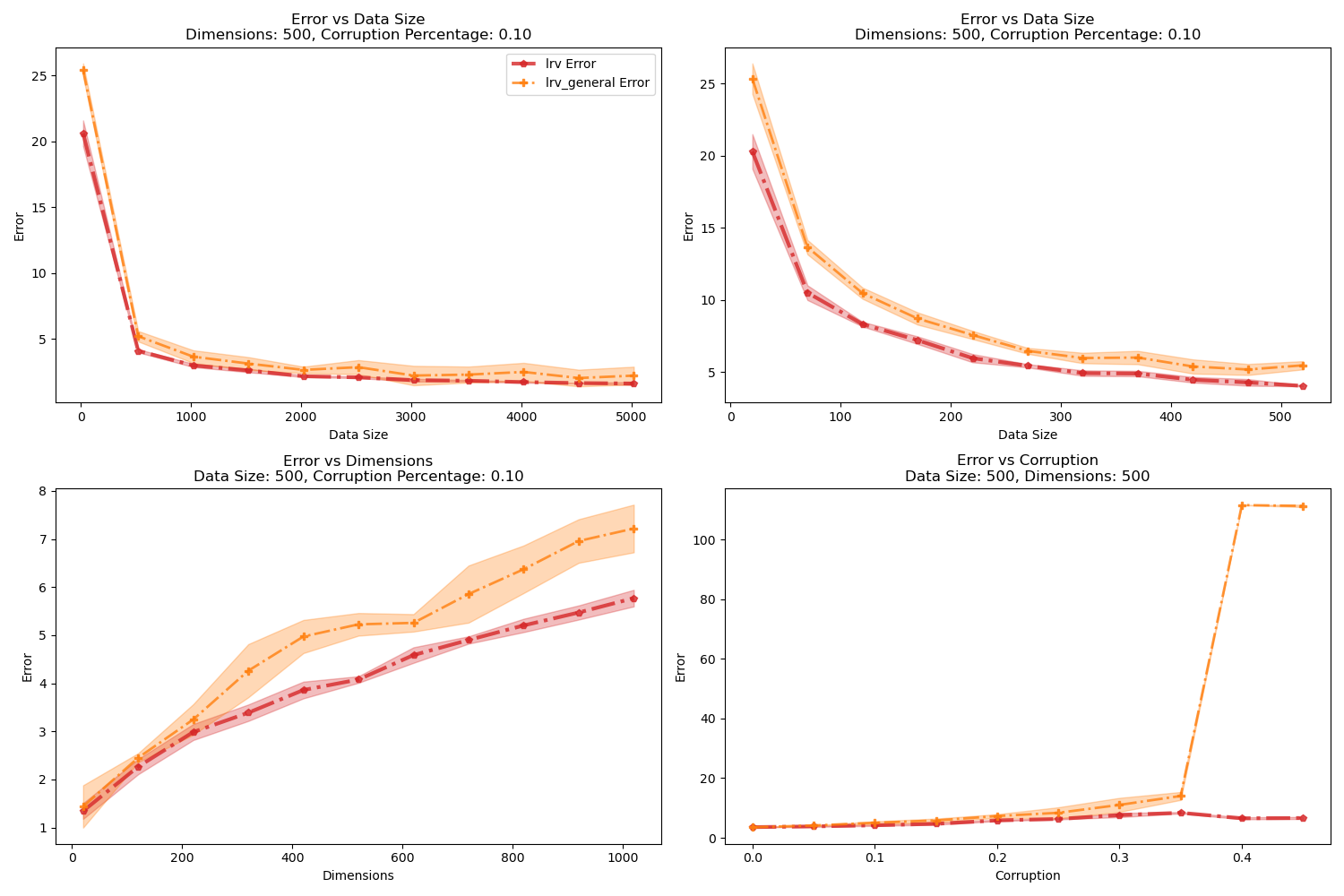}
    \caption{LRV - Gaussian Vs General Weighting: Large Spherical Covariance, Additive Variance Shell Noise}
    \label{fig:lrvg_sp_gaus_one}
\end{figure}

\begin{figure}[h]
    \centering
    \includegraphics[width=0.85\linewidth]{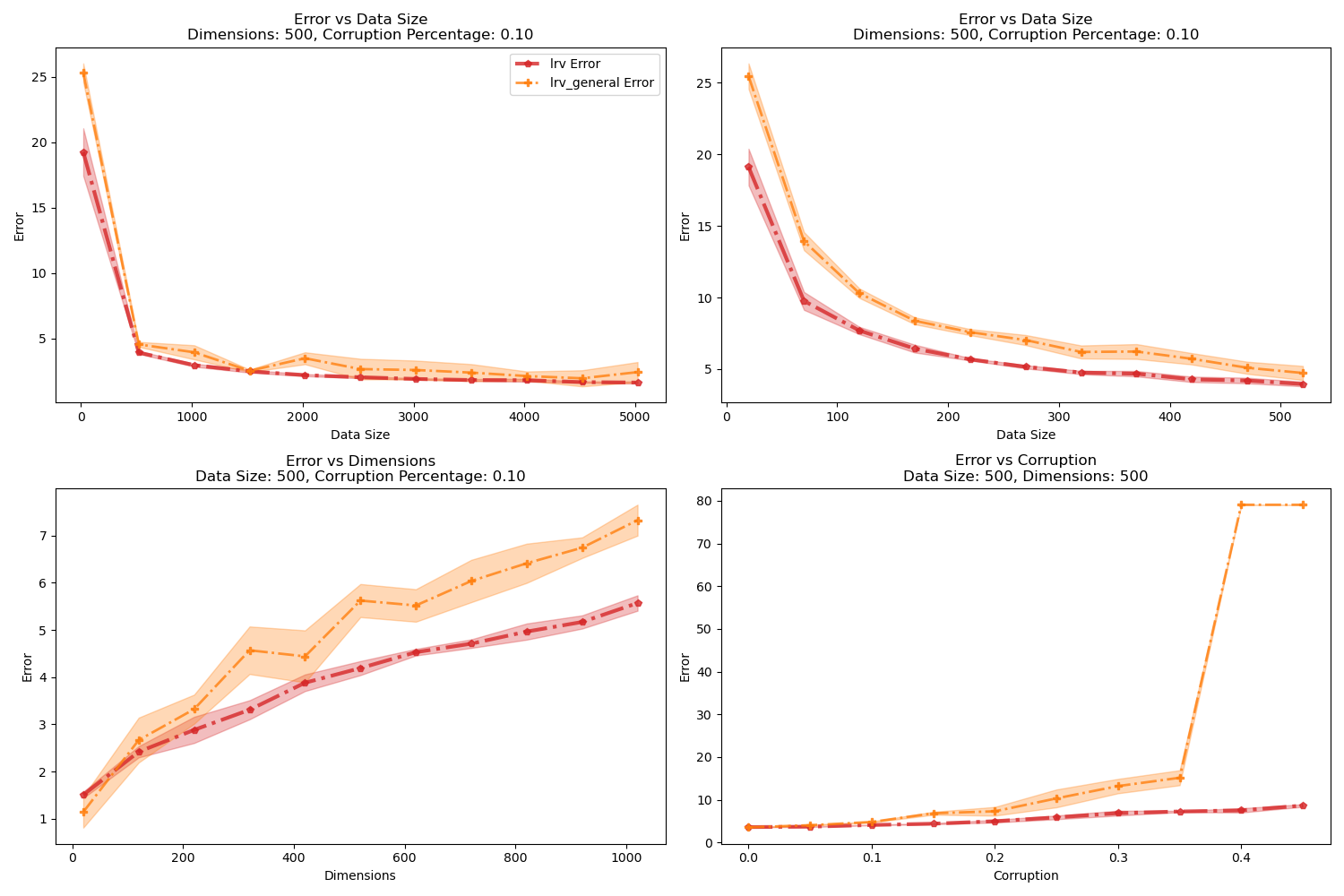}
    \caption{LRV - Gaussian Vs General Weighting: Large Diminishing Covariance, Additive Variance Shell Noise}
    \label{fig:lrvg_nonsp_gaus_one}
\end{figure}

\clearpage

\paragraph{Eigenvalue Pruning Tail Threshold}

Here we explore the pruning routine in \evln. First, we examine the parameter $\gamma$ in the pruning threshold for \evln. $\gamma$ weights the expectation that the Gaussian concentration inequality gives for how many points will surpass a certain value; larger values correspond to less aggressive pruning. We vary $\gamma$ in the set $[0.5, 1, 2.5, 5, 10, 20, 50]$. Results are shown in Figures \ref{fig:ev_id_dkk}, \ref{fig:ev_id_unif_top}, \ref{fig:ev_sp_gaus_one}, \ref{fig:ev_nonsp_gaus_one}. We find that using values of $\gamma$ that are too small result in significantly worse error. Setting $\gamma=0.5$ or $\gamma=1$ both achieve performance identical to \sample because the pruning threshold is too sensitive, performing significantly worse than all other choices of $\gamma$, as it determines all data to be outliers. We note that when all data is determined to be outliers, we simply return \sample. However, using reasonably sized $\gamma$ results in mostly similar performance across distributions. Notably, larger values of $\gamma$ tend to perform better over large diminishing covariance with additive variance shell noise, especially with larger $n$. We select $\gamma=5$ throughout our experiments.

\begin{figure}[h]
    \centering    \includegraphics[width=0.85\linewidth]{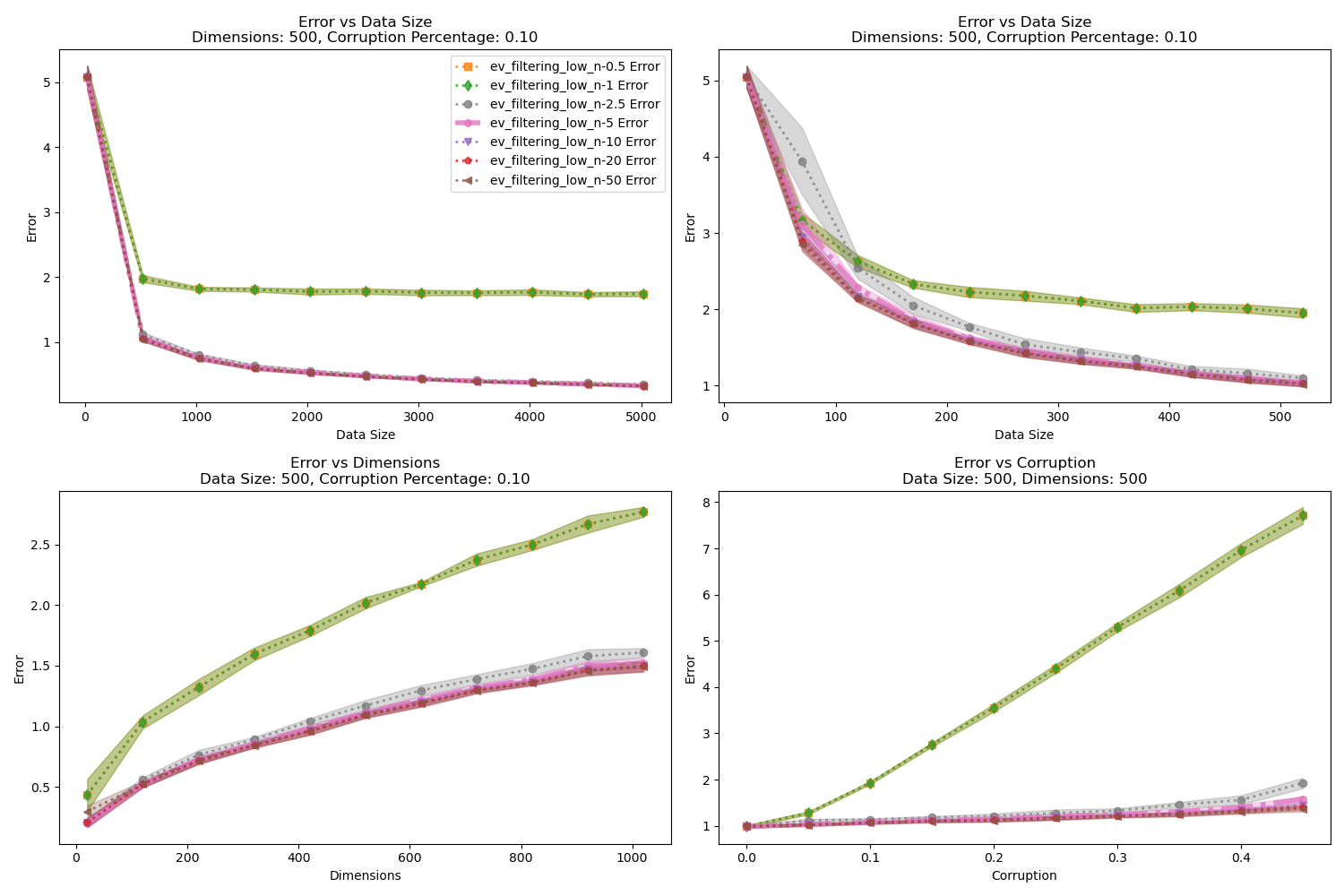}
    \caption{Eigenvalue Pruning - Choice Of $\gamma$: Identity Covariance, DKK Noise}
    \label{fig:ev_id_dkk}
\end{figure}

\begin{figure}[h]
    \centering
    \includegraphics[width=0.85\linewidth]{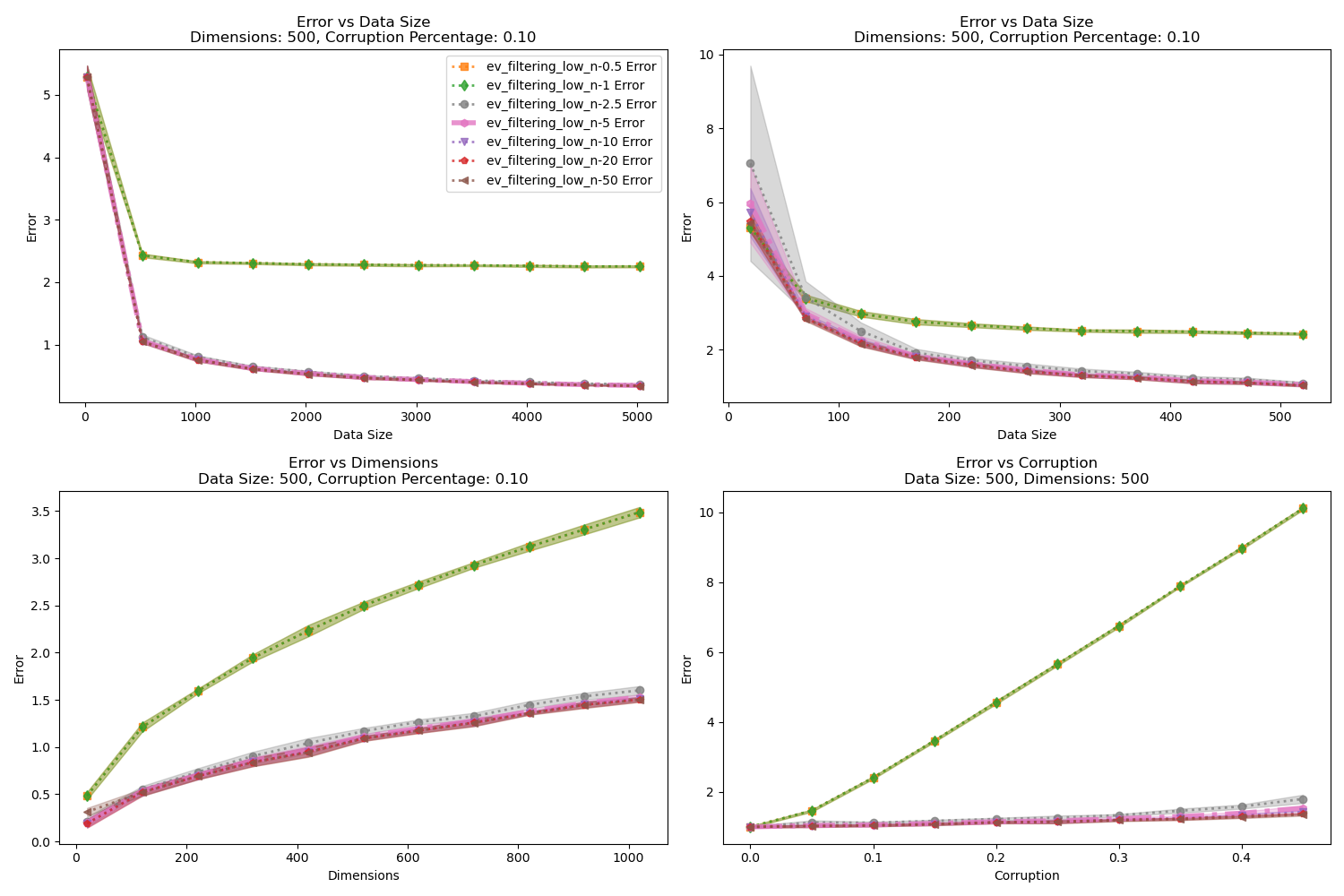}
    \caption{Eigenvalue Pruning - Choice Of $\gamma$: Identity Covariance, In Distribution Noise}
    \label{fig:ev_id_unif_top}
\end{figure}

\begin{figure}[h]
    \centering
    \includegraphics[width=0.85\linewidth]{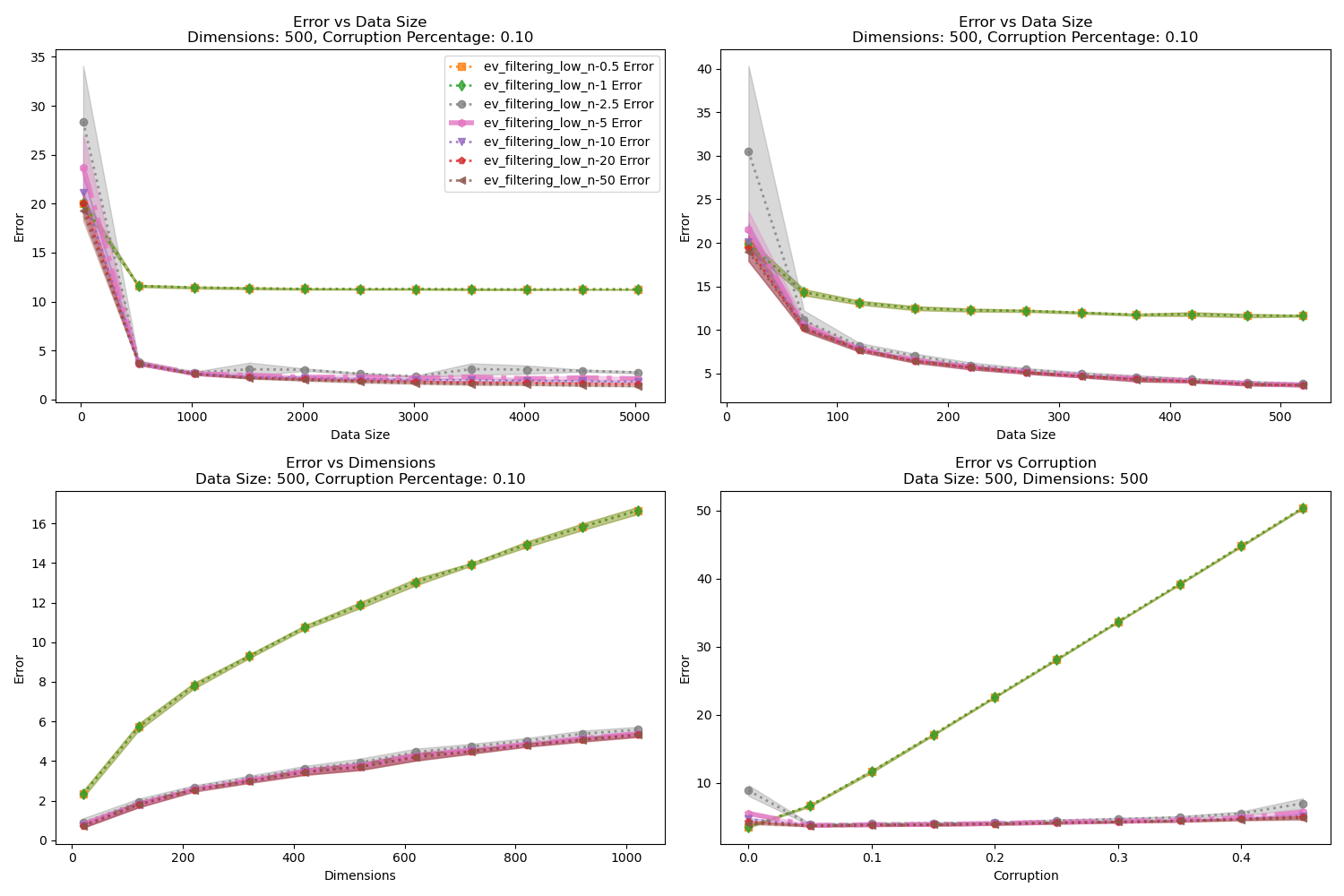}
    \caption{Eigenvalue Pruning - Choice Of $\gamma$: Large Spherical Covariance, Additive Variance Shell Noise}
    \label{fig:ev_sp_gaus_one}
\end{figure}

\begin{figure}[h]
    \centering
    \includegraphics[width=0.85\linewidth]{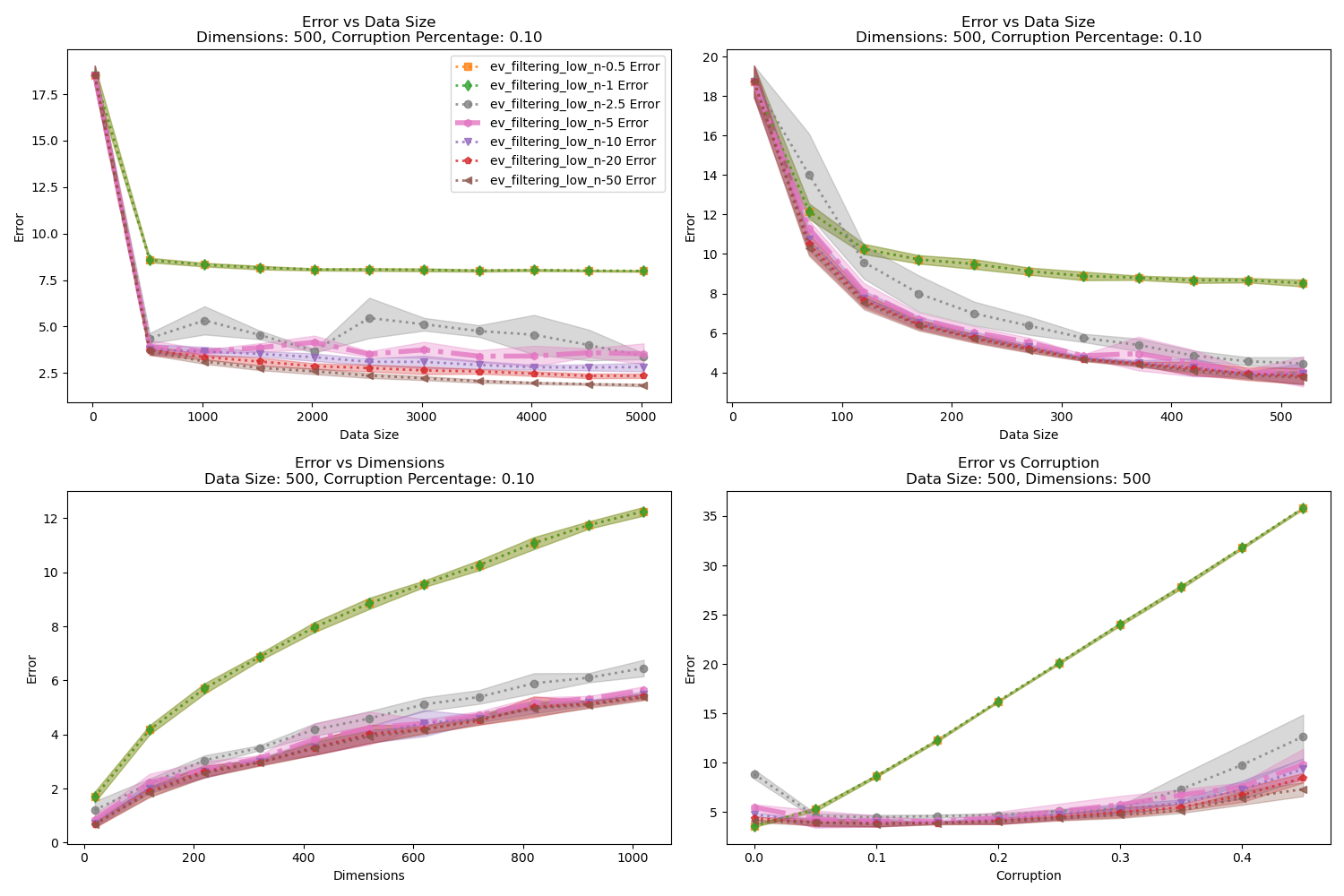}
    \caption{Eigenvalue Pruning - Choice Of $\gamma$: Large Diminishing Covariance, Additive Variance Shell Noise}
    \label{fig:ev_nonsp_gaus_one}
\end{figure}

\clearpage

We also explore \evln using two alternate pruning methods not explicitly based on the Gaussian assumption as the current one is: randomized pruning and fixed pruning,. Randomized Pruning removes points based on a random scaling of the largest deviation in the dataset. Define $T$ as the largest deviation of a point projected onto the top eigenvector from the median of the projected points. Draw $Z$ from the distribution on $[0, 1]$ with probability density function $2x$. Then, prune all points whose projected distance onto the top eigenvector is at least $TZ$. This randomized pruning method is derived from the mean estimation algorithm for unknown covariance distributions by \cite{diakonikolas2017being}.
Fixed Pruning simply prunes the $0.5\tau$ percentage of points whose projection onto the top eigenvector is furthest from the median of the projected points at every iteration. This is identical to the pruning method in \queln, with projected deviations being used as "outlier scores", instead of the quantum entropy scores used in \queln. Results are shown in Figures \ref{fig:ev_pruning_id_dkk}, \ref{fig:ev_pruning_id_unif_top}, \ref{fig:ev_pruning_sp_gaus_one}, \ref{fig:ev_pruning_nonsp_gaus_one}. We find that both randomized and fixed pruning are able to match or slightly outperform the standard Gaussian pruning method. However, we note that unlike in \queln, fixed pruning could potentially result in catastrophic error. In particular, if corruption is uniformly distributed across $O(d)$ orthogonal clusters, then \evln may take $O(d)$ runs to return an outlier, since it can only prune in one direction at once. But with fixed pruning, each iteration will prune too many outliers in each direction. We only evaluate Gaussian pruning to follow the conventions of \cite{diakonikolas2017being}. 

\begin{figure}[h]
    \centering
    \includegraphics[width=0.85\linewidth]{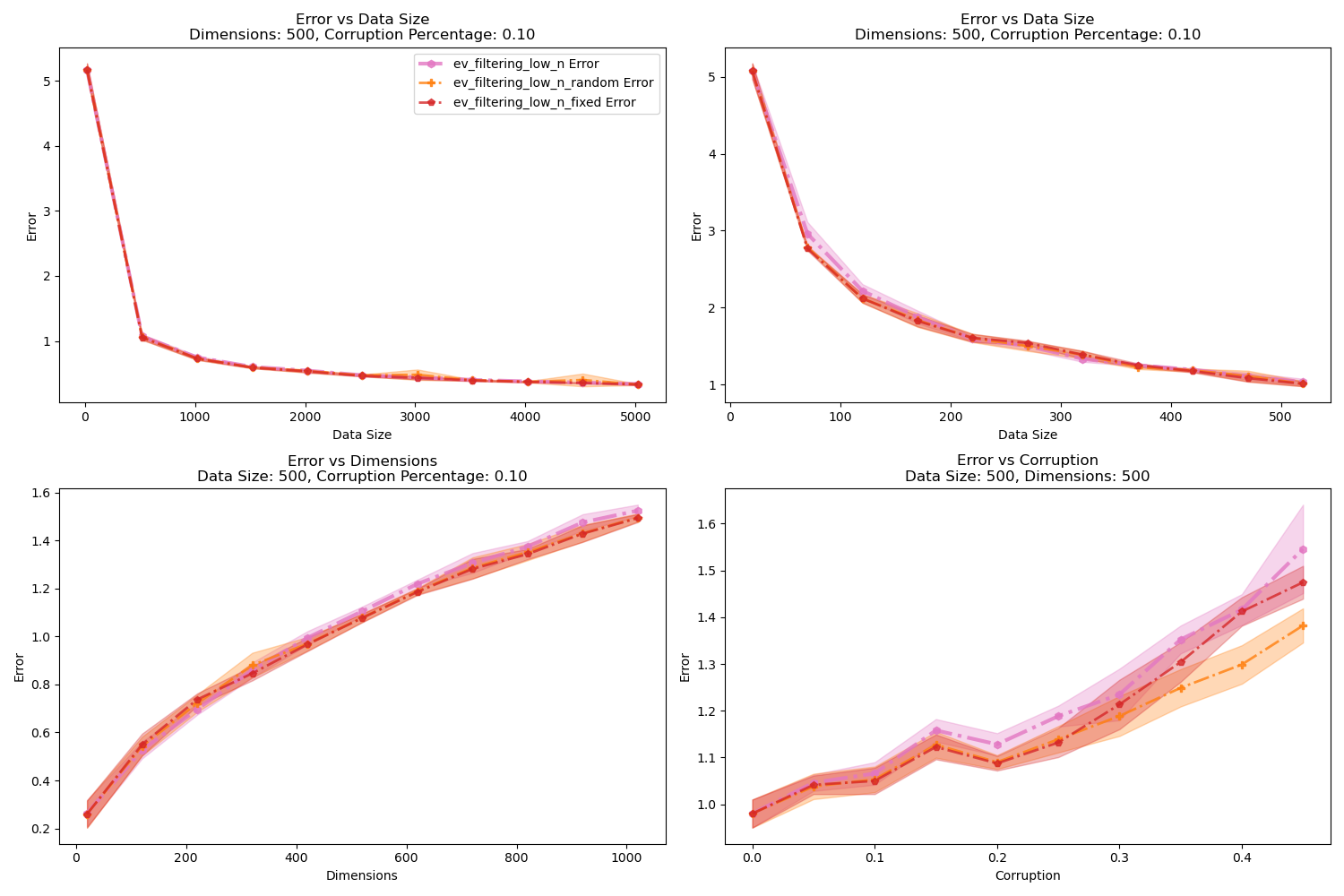}
    \caption{Eigenvalue Pruning - Pruning Method: Identity Covariance, DKK Noise}
    \label{fig:ev_pruning_id_dkk}
\end{figure}

\begin{figure}[h]
    \centering
    \includegraphics[width=0.85\linewidth]{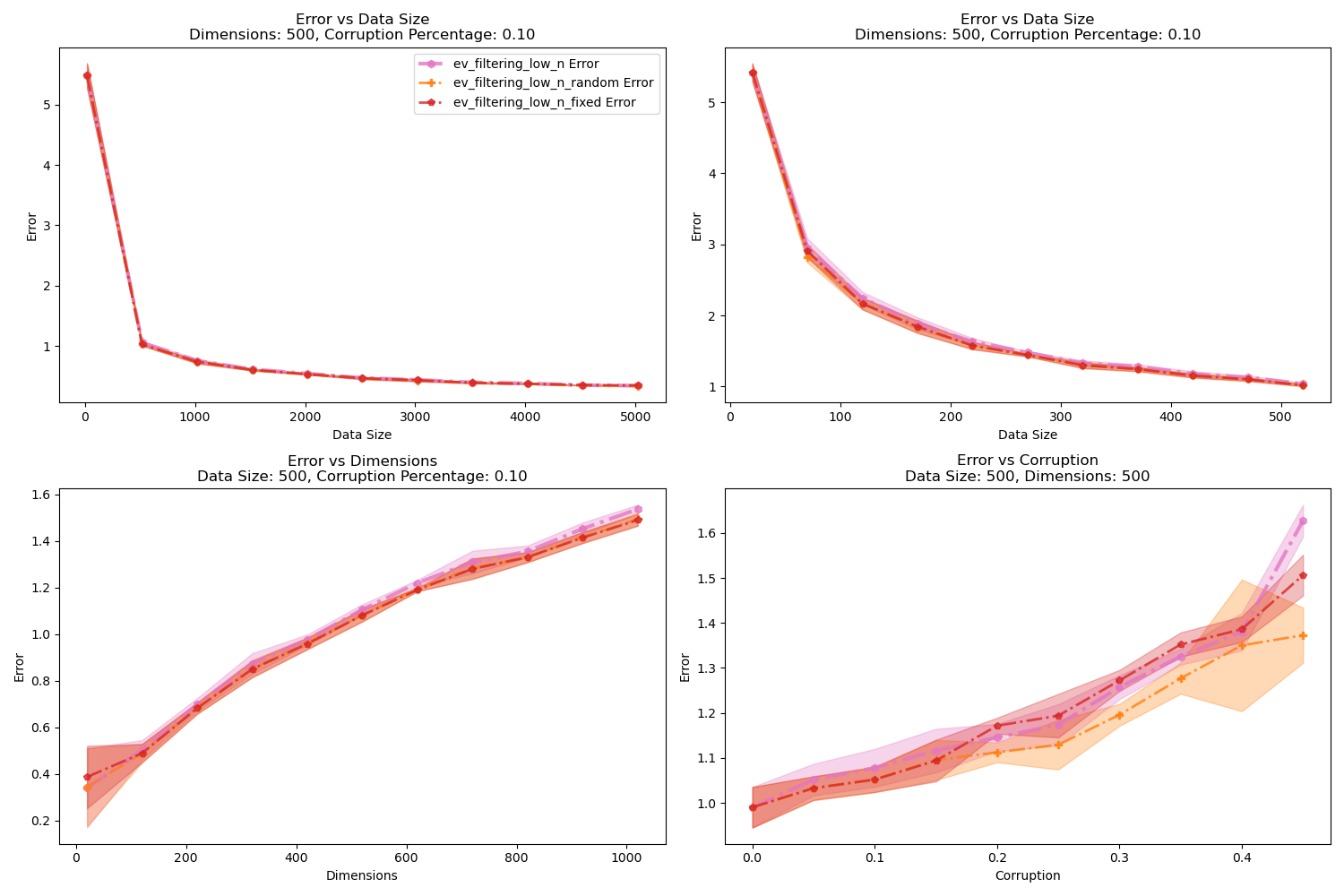}
    \caption{Eigenvalue Pruning - Pruning Method: Identity Covariance, In Distribution Noise}
    \label{fig:ev_pruning_id_unif_top}
\end{figure}

\begin{figure}[h]
    \centering
    \includegraphics[width=0.85\linewidth]{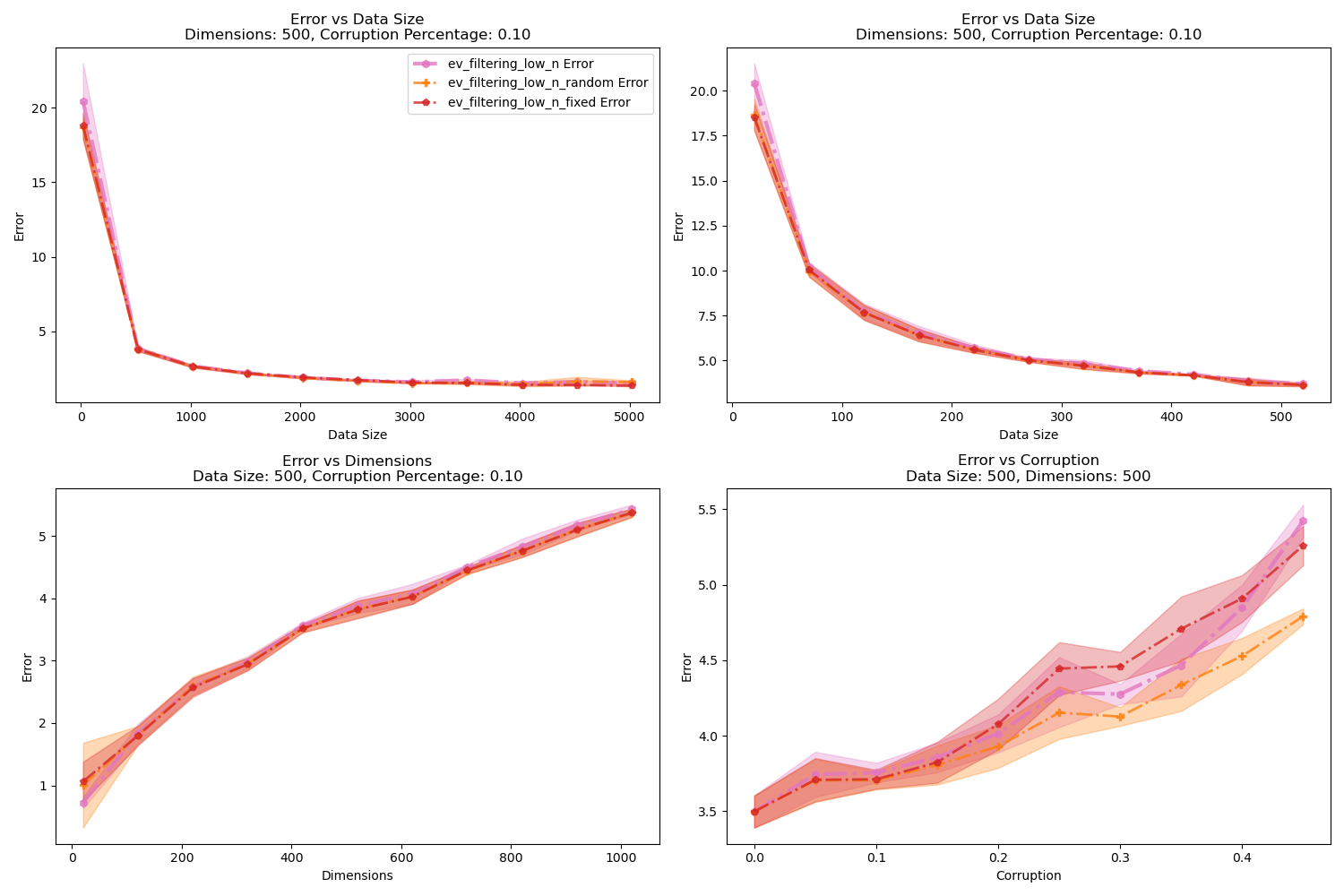}
    \caption{Eigenvalue Pruning - Pruning Method: Large Spherical Covariance, Additive Variance Shell Noise}
    \label{fig:ev_pruning_sp_gaus_one}
\end{figure}

\begin{figure}[h]
    \centering
    \includegraphics[width=0.85\linewidth]{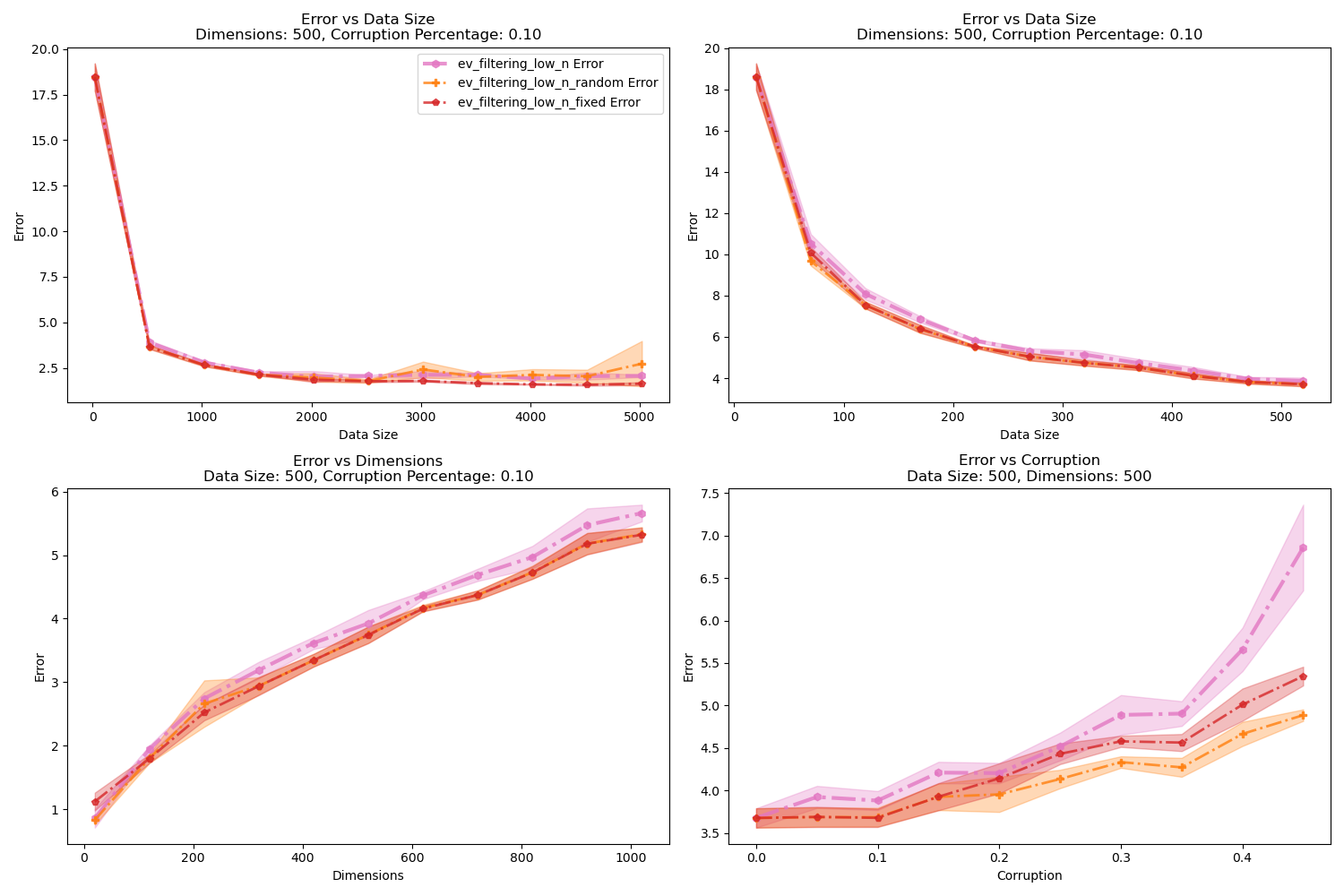}
    \caption{Eigenvalue Pruning - Pruning Method: Large Diminishing Covariance, Additive Variance Shell Noise}
    \label{fig:ev_pruning_nonsp_gaus_one}
\end{figure}

\clearpage

\paragraph{Projected Gradient Descent Iterations}

 Here we explore the number of iterations parameter, $\gamma$ for \pgd. We choose values for $\gamma$ in the set $[1, 5, 10, 15, 20]$. Results are shown in Figures \ref{fig:pgd_id_dkk}, \ref{fig:pgd_id_unif_top}, \ref{fig:pgd_sp_gaus_one}, \ref{fig:pgd_nonsp_gaus_one}. We find that low choices of $\gamma$ result in significantly worse performance, while higher choices perform roughly equally. Notably, when $\gamma$ is set equal to $10$, \pgd performs much worse over Identity Covariance with DKK Noise, especially under large $n$, despite this choice of $\gamma$ performing among the best across other distributions. This suggests that larger choices of $\gamma$ may be necessary for \pgd to be robust across different corruption schemes. We note that the runtime of \pgd increases approximately linearly with respect to $\gamma$, so there is a meaningful tradeoff when using larger values of $\gamma$. We set $\gamma=15$ across our experiments because it is the lowest $\gamma$ that performs among the best across the distributions tested.

\begin{figure}[h]
    \centering
    \includegraphics[width=0.85\linewidth]{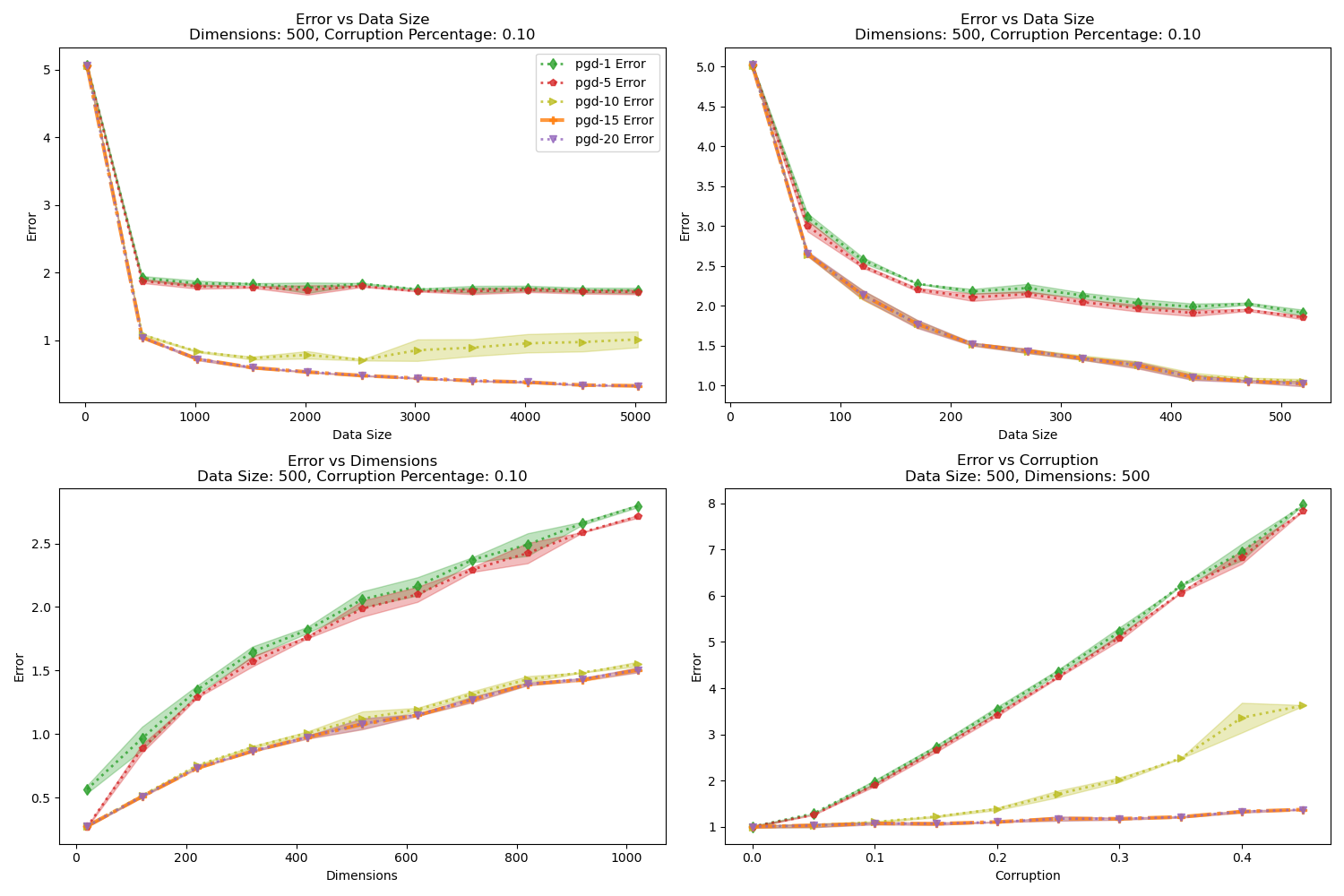}
    \caption{Projected Gradient Descent - Number Of Iterations $\gamma$: Identity Covariance, DKK Noise}
    \label{fig:pgd_id_dkk}
\end{figure}

\begin{figure}[h]
    \centering
    \includegraphics[width=0.85\linewidth]{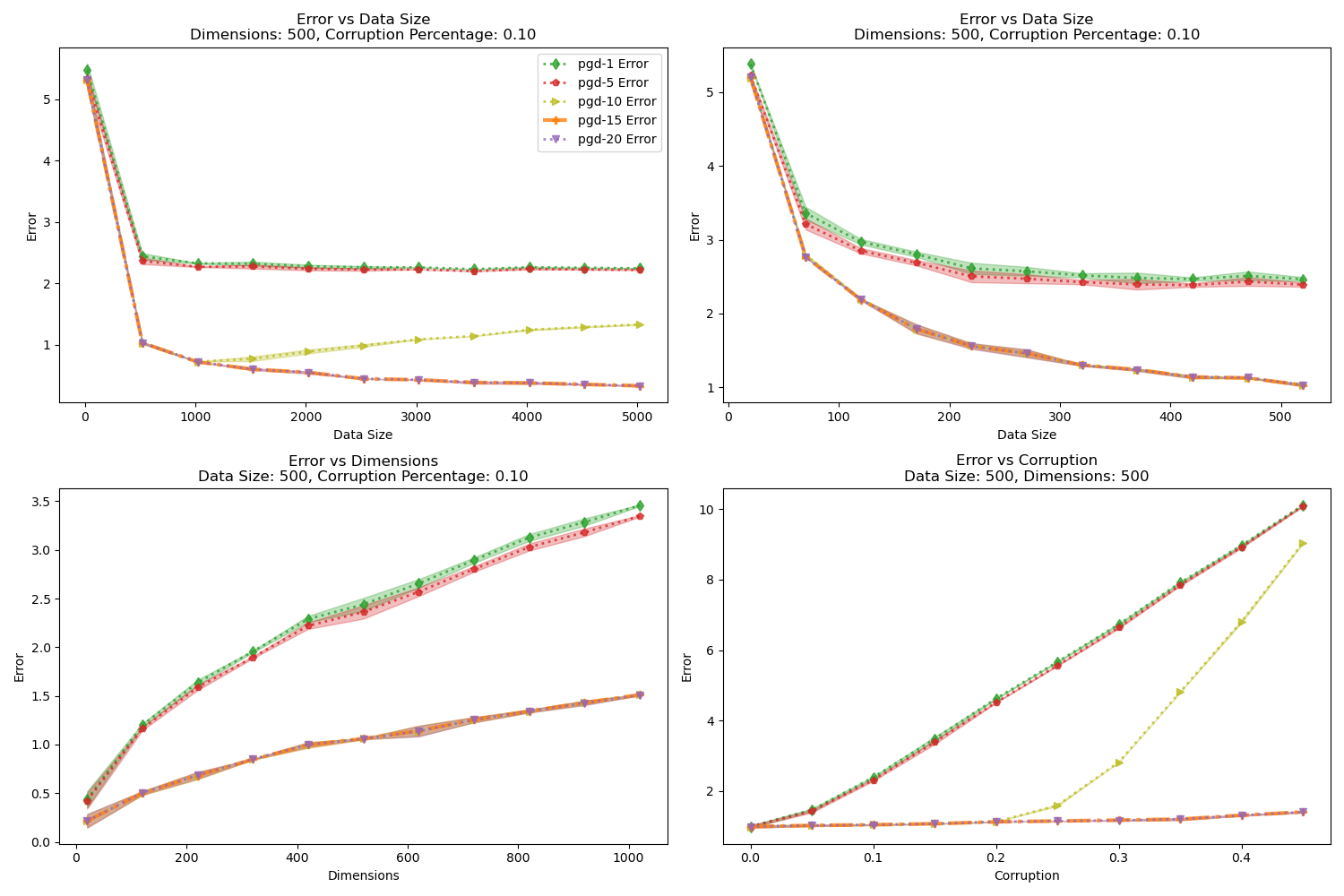}
    \caption{Projected Gradient Descent - Number Of Iterations $\gamma$: Identity Covariance, In Distribution Noise}
    \label{fig:pgd_id_unif_top}
\end{figure}

\begin{figure}[h]
    \centering
    \includegraphics[width=0.85\linewidth]{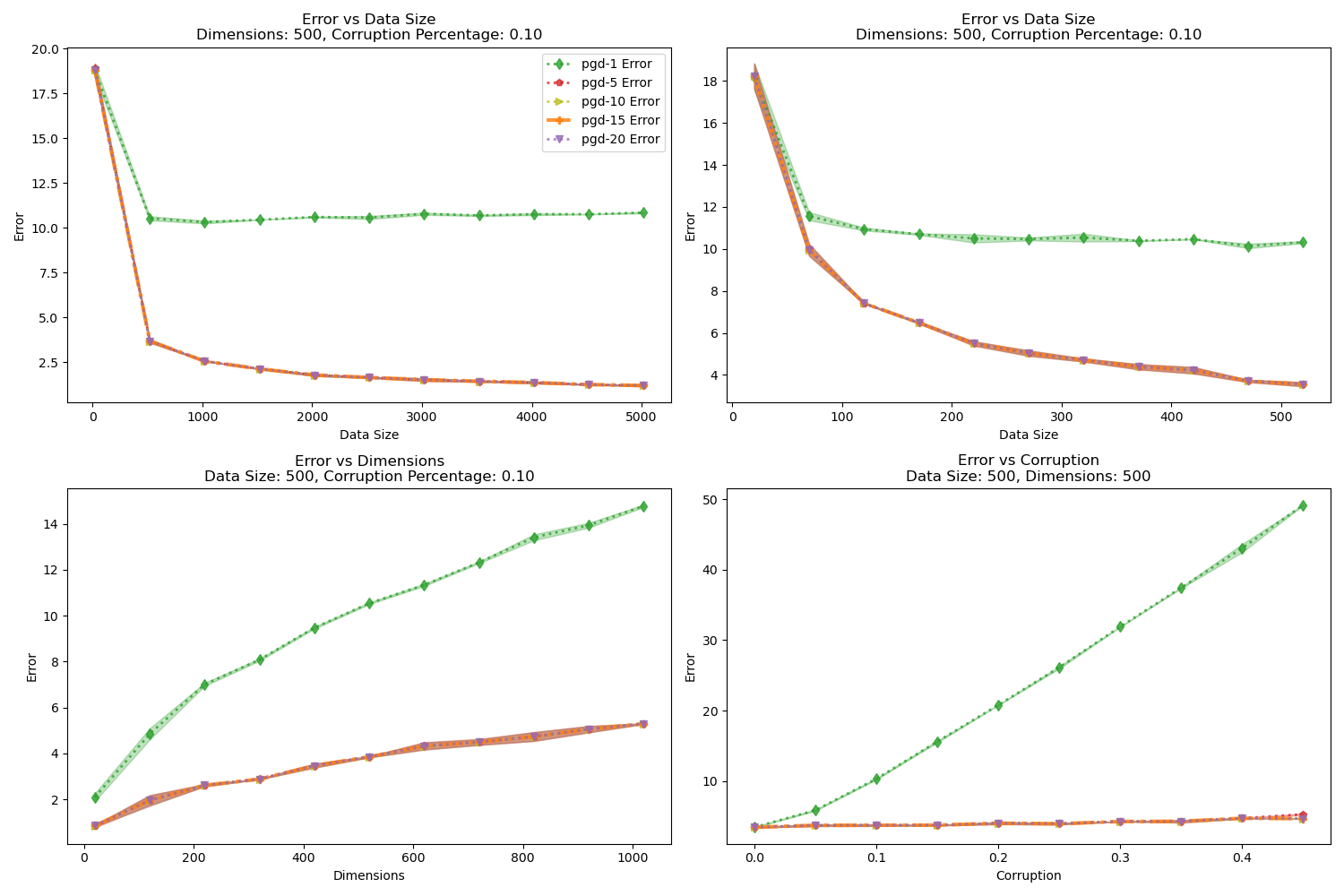}
    \caption{Projected Gradient Descent - Number Of Iterations $\gamma$: Large Spherical Covariance, Additive Variance Shell Noise}
    \label{fig:pgd_sp_gaus_one}
\end{figure}

\begin{figure}[h]
    \centering
    \includegraphics[width=0.85\linewidth]{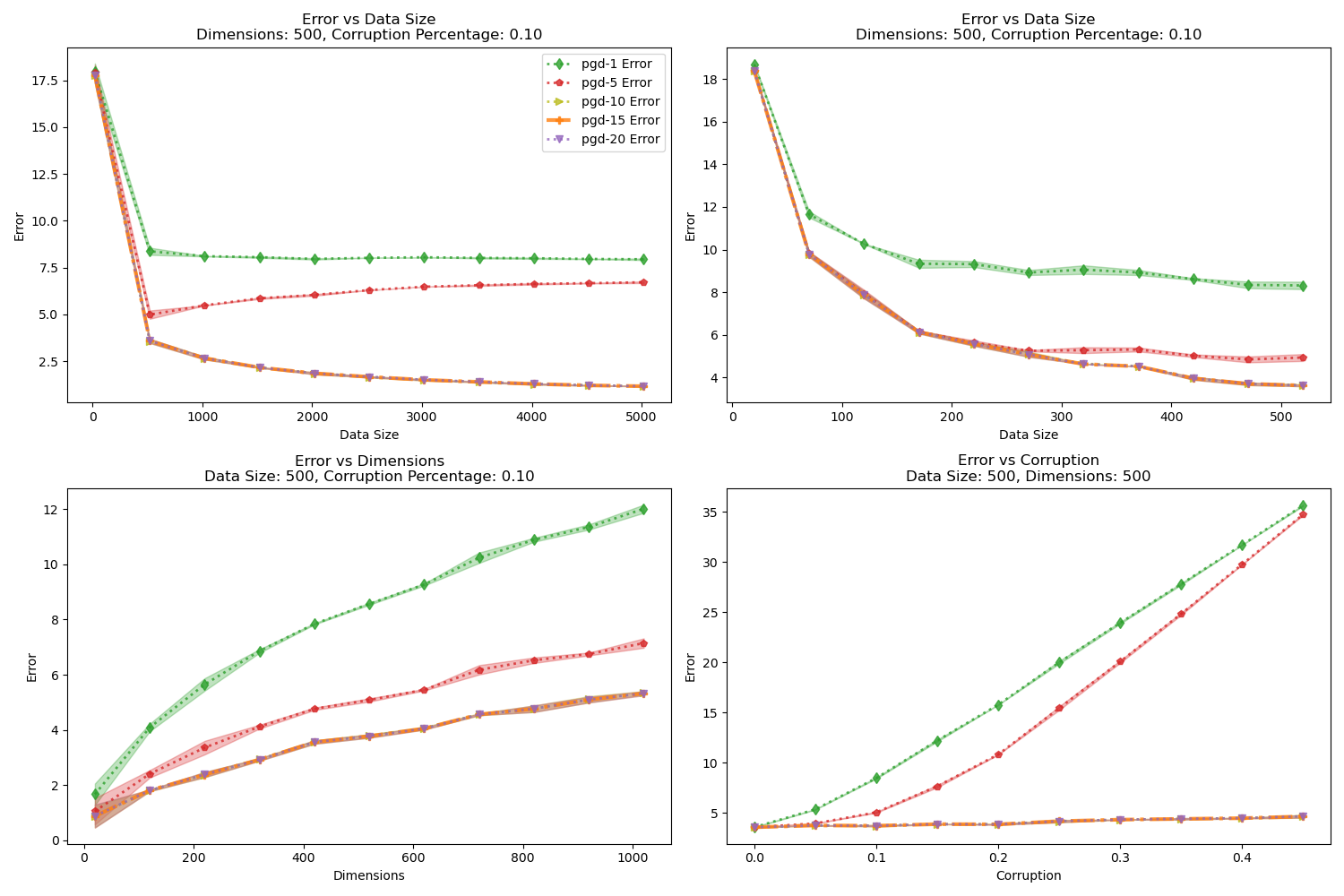}
    \caption{Projected Gradient Descent - Number Of Iterations $\gamma$: Large Diminishing Covariance, Additive Variance Shell Noise}
    \label{fig:pgd_nonsp_gaus_one}
\end{figure}

\clearpage

\subsection{Robustness To Expected Corruption}
\label{app:expected_corruption}

\begin{figure}[h!]
    \centering
    \begin{subfigure}{0.45\linewidth}
        \centering
        \includegraphics[width=\linewidth]{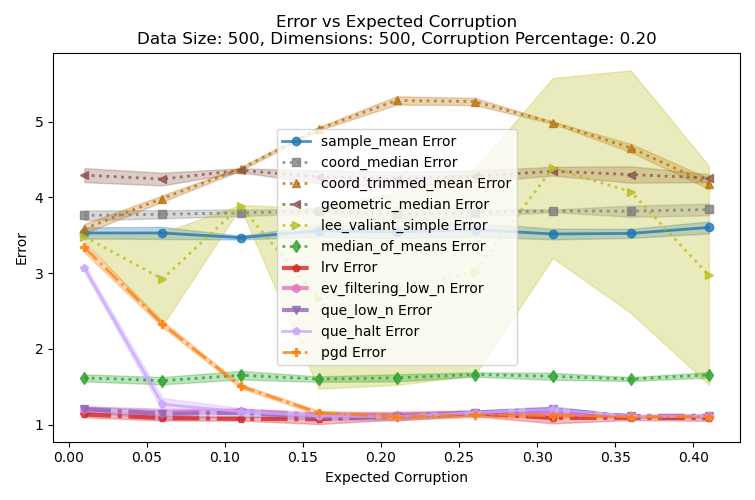}
        \caption{Identity Covariance - DKK Noise}
    \end{subfigure}
    \hfill
    \begin{subfigure}{0.45\linewidth}
        \centering
        \includegraphics[width=\linewidth]{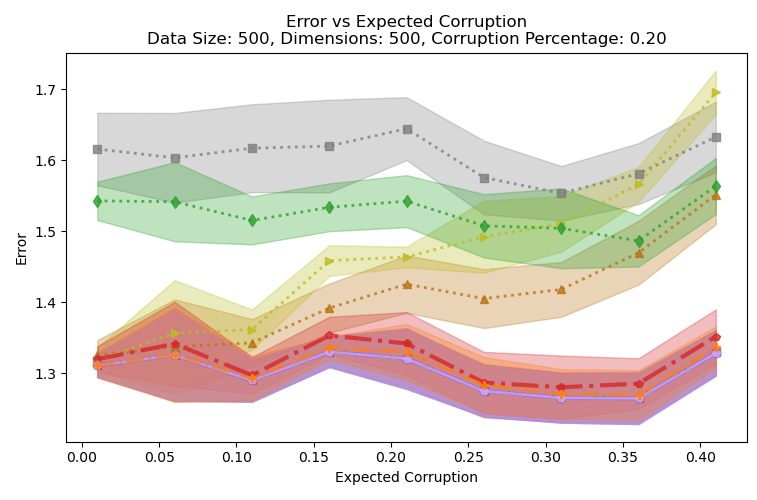}
        \caption{Identity Covariance - Subtractive Noise}
    \end{subfigure}
    \bigskip
    
    \begin{subfigure}{0.45\linewidth}
        \centering
        \includegraphics[width=\linewidth]{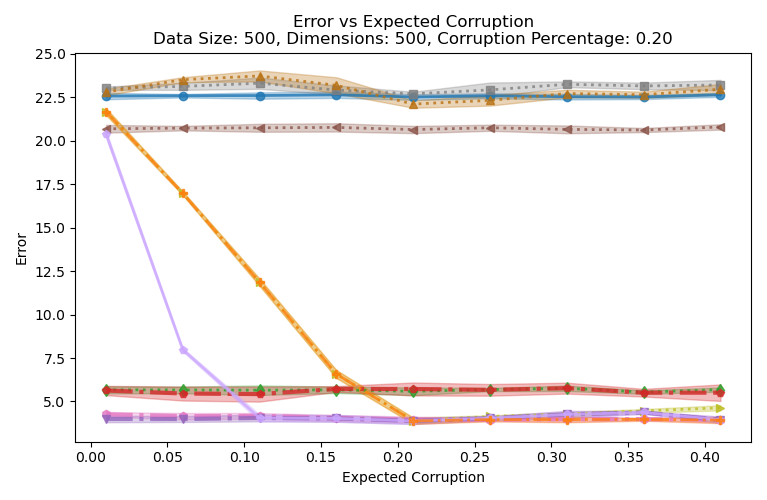}
        \caption{Large Spherical Covariance - Additive Variance Shell Noise}
    \end{subfigure}
    \hfill
    \begin{subfigure}{0.45\linewidth}
        \centering
        \includegraphics[width=\linewidth]{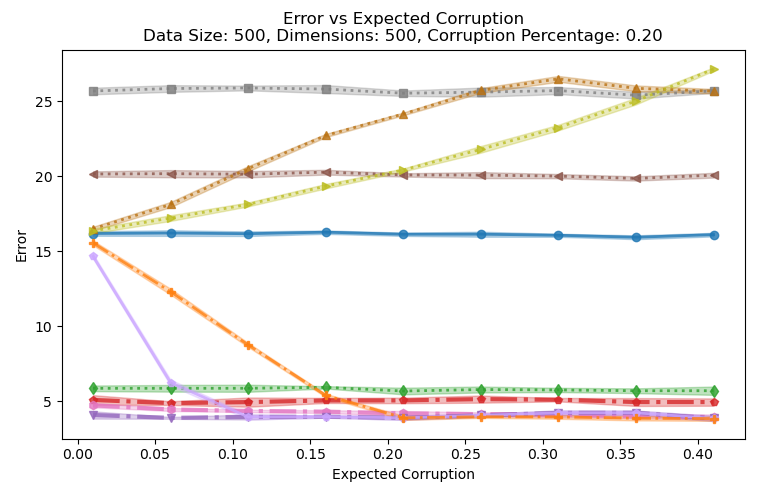}
        \caption{Large Diminishing Covariance - Additive Variance Shell Noise}
    \end{subfigure}
    \caption{Robustness To Expected Corruption: Error vs Expected Corruption $\tau$}
    \label{fig:corr_robust}
\end{figure}

We examine robustness to expected corruption, $\tau$.~This is a hyperparameter for \evln, \queln, \pgd, \lvsim, and \coordprune. In \evln, $\tau$ only plays a soft role as a slack term in the filtering step. In \queln, $\tau$ controls the number of points that are pruned in every iteration of the algorithm, but the number of iterations is unbounded.  In \lvsim, and \coordprune, $\tau$ explicitly controls the amount of data that is pruned in total.  In \pgd, $\tau$ controls the space of feasible outlier weights. 

We evaluate error as expected corruption, $\tau$, varies from $\tau=0.01$ to $\tau=0.46$ with true corruption fixed as $\eta=0.20$. Otherwise, the experiment setup remains the same as seen previously. As in Appendix \ref{app:hp_tuning}, we replicate experiments over Identity Covariance with DKK Noise and with Subtractive Noise; Large Spherical Covariance with Additive Variance Shell Noise; and Large Diminishing Covariance with Additive Variance Shell Noise. These results are shown in Figure \ref{fig:corr_robust}, with all estimators included for reference. We include \queln with and without early halting.

We find that most estimators perform nearly identically regardless of the choice of $\tau$, except for \pgd, which performs nearly identically when $\tau$ is an upper bound on true corruption $\eta$ but degrades with underestimates of $\eta$.    \queln without early halting performs well throughout choices of $\tau$. With smaller choices of $\tau$, it will prune significantly less points at each iteration, but will run for more iterations until the corruption detection threshold is passed, while for larger choices of $\tau$, it will prune more points at each iteration, but will run for less iterations until the corruption detection threshold is passed. \queln with early halting, which is used throughout the real world experiments, sees degradation with underestimates of $\tau$, but identical performance with overestimates. 

\evln also performs nearly identically regardless of the choice of $\tau$, as expected by the soft dependency of the pruning threshold on $\tau$. \lvsim and \coordprune both degrade noticeably as $\tau$ increases over Identity Covariance data with Subtractive Noise and Large Diminishing Covariance data with Additive Variance Shell Noise; in both cases yielding error worse than \sample the more points they prune. Surprisingly, over Large Spherical Covariance with Additive Variance Shell Noise, \lvsim nearly exactly matches the performance of \pgd, except with slightly worse degradation with large overestimates of $\tau$.

\subsection{Image Embedding Experiments}
\label{app:image_experiments}

We evaluate algorithms on the problem of estimating the mean of embeddings of images generated by deep pretrained image models. As in the LLM experiment, we first examine the problem of mean estimation of image embeddings belonging to the same category, reporting LOOCV error. We then examine a corrupted distribution where images belonging to one category are considered inliers and those belonging to another are considered outliers. We utilize a set of images of cats and dogs from the CIFAR10 dataset \citep{Krizhevsky2009LearningML}. We embed these images using 4 deep pretrained image models of varying embedding dimensions: ResNet-18, ResNet-50 \citep{he2015deepresiduallearningimage}, MobileNet V3 \citep{howard2019searchingmobilenetv3}, and EffecientNet B0 \citep{tan2020efficientnetrethinkingmodelscaling}. ResNet-18 has an embedding dimension of 512, MobileNet V3 has one of 960, EfficientNet B0 has one of 1280, and ResNet-50 has one of 2048. 

\paragraph{Common Category Images}

Here we examine LOOCV error vs data size on embeddings of images of cats. We vary data size from $n=10$ to $n=1000$. Otherwise, experiments are run identically to the LLM experiment, fixing expected corruption $\eta=0.1$, employing the trace scaling heuristic on \evln and \queln, the halting heuristic on \queln, and averaging results over 5 runs. We note that, as in the LLM experiments, employing the halting heuristic on \evln does not improve performance. Results are shown in Figure \ref{fig:loocv_cat_images}.

\begin{figure}[t]
    \begin{subfigure}{0.5\linewidth}
        \centering
        \includegraphics[width=\linewidth]{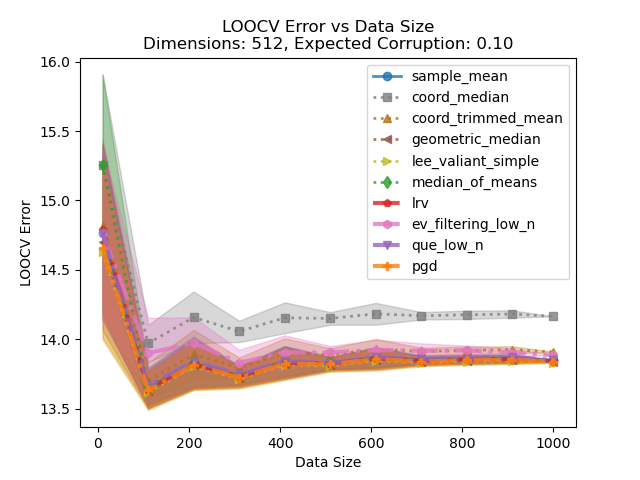}
        \caption{ResNet-18 Embeddings}
    \end{subfigure}
        \begin{subfigure}{0.5\linewidth}
        \centering
        \includegraphics[width=\linewidth]{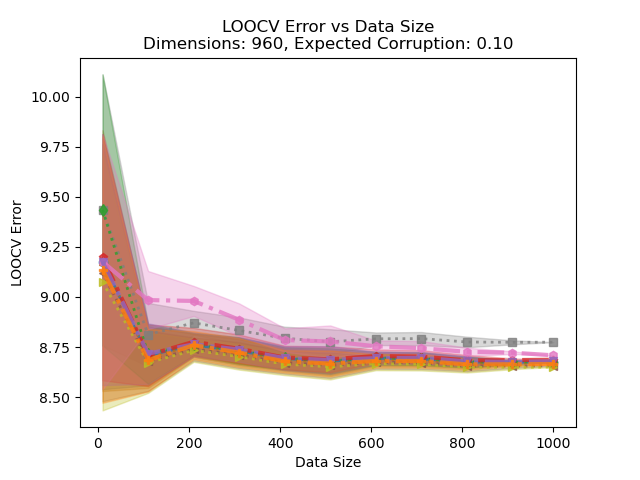}
        \caption{MobileNet V3 Embeddings}
    \end{subfigure}
    \\
    \begin{subfigure}{0.5\linewidth}
        \centering
        \includegraphics[width=\linewidth]{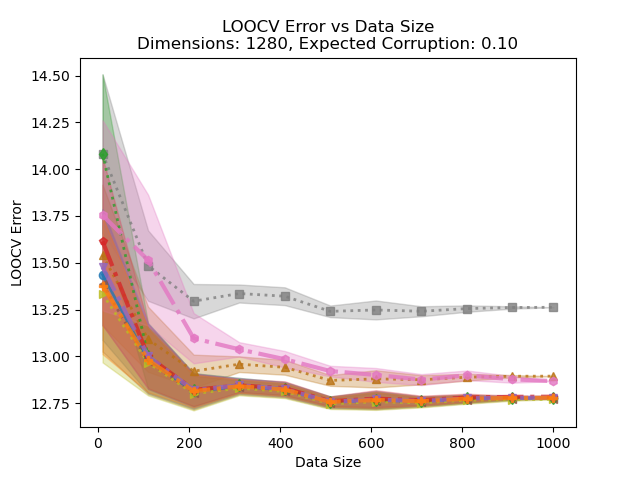}
        \caption{EfficientNet B0 Embeddings}
    \end{subfigure}
    \begin{subfigure}{0.5\linewidth}
        \centering
        \includegraphics[width=\linewidth]{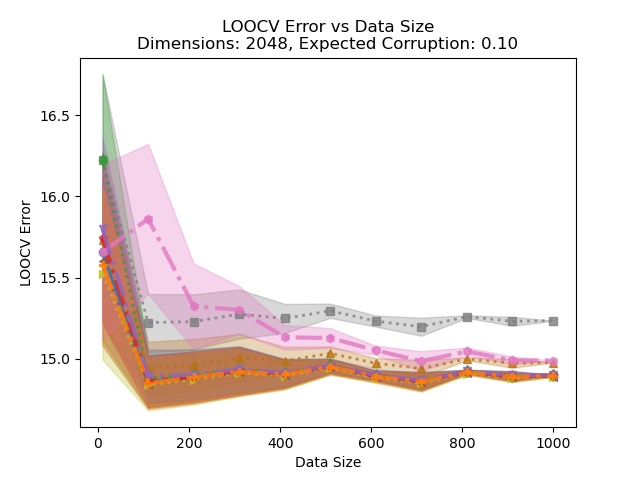}
        \caption{ResNet-50 Embeddings}
    \end{subfigure}
    \caption{LOOCV Error on Cat Image Embeddings}
    \label{fig:loocv_cat_images}
\end{figure}

As in the LLM experiment, we observe that no algorithm significantly outperforms \sample, despite the nontrivial LOOCV error in each setting. As in the LLM experiment, \evln tends to perform worse than other algorithms, which is unsurprising given its sensitivity to knowledge of the true covariance. Other robust mean estimation algorithms, including \queln, perform near identically to \sample.

\paragraph{Corrupted Images}

For the corrupted case, we draw data $X \sim (1 - \eta) P + \eta Q$, where the inlier distribution, $P$, consists of embeddings of images of cats, and the outlier distribution, $Q$, consists of embeddings of images of dogs. We fix data size $n=1000$ to focus on the $n \approx d$ and $n < d$ regime. Otherwise, the experimental setup is identical to in the LLM experiments. Results are shown in Figure \ref{fig:image_corruption}.

\begin{figure}[t]
    \begin{subfigure}{0.5\linewidth}
        \centering
        \includegraphics[width=\linewidth]{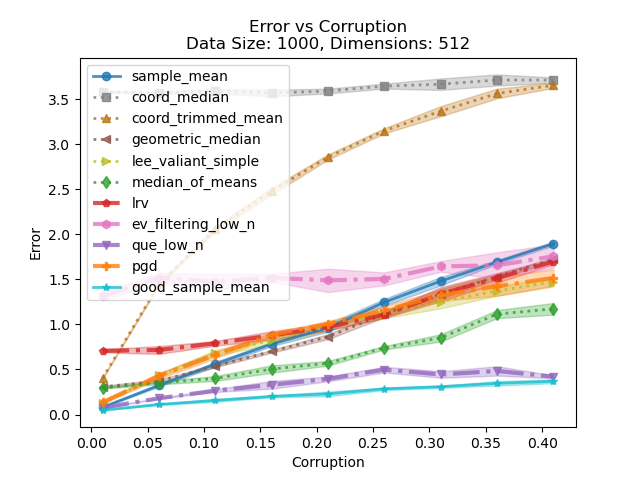}
        \caption{ResNet-18 Embeddings}
    \end{subfigure}
        \begin{subfigure}{0.5\linewidth}
        \centering
        \includegraphics[width=\linewidth]{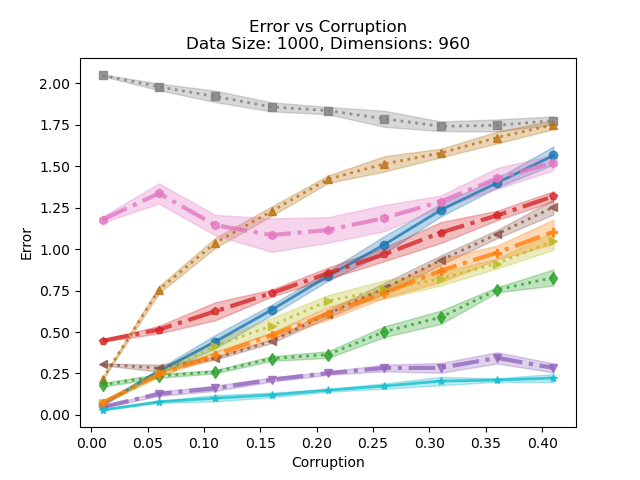}
        \caption{MobileNet V3 Embeddings}
    \end{subfigure}
    \\
    \begin{subfigure}{0.5\linewidth}
        \centering
        \includegraphics[width=\linewidth]{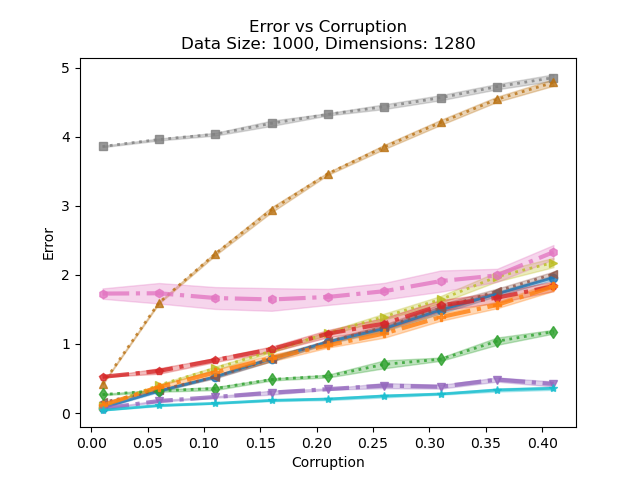}
        \caption{EfficientNet B0 Embeddings}
    \end{subfigure}
    \begin{subfigure}{0.5\linewidth}
        \centering
        \includegraphics[width=\linewidth]{UpdatedFigures/CorruptionImage/ResNet2048Corruption.png}
        \caption{ResNet-50 Embeddings}
    \end{subfigure}
    \caption{Error on Cat Image Embeddings Corrupted with Dog Image Embeddings}
    \label{fig:image_corruption}
\end{figure}

These results demonstrate similar trends to the LLM experiment. One key difference is that \coordprune performs much worse than \sample here, compared to the LLM experiment where it tends to slightly outperform \sample. This suggests that naive pruning does not work well in this setting as outliers are not obvious, reinforcing that this is a difficult setting for robust mean estimation. Nonetheless, several robust mean estimators are able to perform well in this case. In particular, \queln is again the strongest performer, noticeably outperforming all other estimators and nearly matching \gsample error across settings. Notably, this strong performance occurs even with $n$ much less than $d$, with relative performance remaining the same even with $n=1000$ and $d=2048$ in the ResNet-50 embedding case. Among the best estimators in the synthetic data case, \pgd tends to perform similarly to \sample, except with large enough corruption across MobileNet V3 embeddings where it outperforms \sample; \lrv tends to perform similarly to \sample except under low corruption, where it noticably degrades; and \evln fails catastrophically throughout. As in the LLM experiments, \medmean outperforms robust estimators that tend to perform better in the synthetic data cases. However, \lvsim no longer performs near optimally, suggesting its sensitivity to distributional assumptions, as expected due to its general poor performance over synthetic data experiments.

\paragraph{Corruption vs Data Size}

We repeat experiments over the same corrupted data scheme but examine error vs data size. We fix true corruption $\eta=0.1$, set expected corruption $\tau=\eta$, and vary data size $n$. We examine the performance of all estimators with data size ranging from $n=100$ to $n=5000$. We additionally provide a zoomed in plot, examining the performance of estimators excluding \evln, \coordmed, and \coordprune -- which all fail catastrophically -- with data size from $n=100$ to $n=1000$. Results are shown in Figure \ref{fig:image_corr_vs_n}.

We find that the relative performance of algorithms remains similar across data sizes. Particularly, even with very large $n$, such as $n=5000$ and $d=512$ under ResNet-18 Embeddings, only \queln and \medmean consistently outperform \sample. \pgd, \lrv, and \lvsim tend to perform slightly worse than \sample. \evln fails catastrophically regardless of data size, though the error stabilizes with larger $n$. These results suggest that the weakness of robust mean estimators over real world data distributions is not just confined to the low data size regime. Yet again, we find that \queln is the best performer, outperforming all other estimators and achieving near optimal performance throughout settings. Additionally, \medmean does not show this same sensitivity to distributional assumptions as other estimators, and as in the synthetic data experiments, tends to perform near optimally with large enough $n$.

\begin{figure}[h]
    \centering

    \begin{subfigure}[b]{\textwidth}
        \begin{subfigure}[c]{0.45\textwidth}
            \centering
            \includegraphics[width=\textwidth]{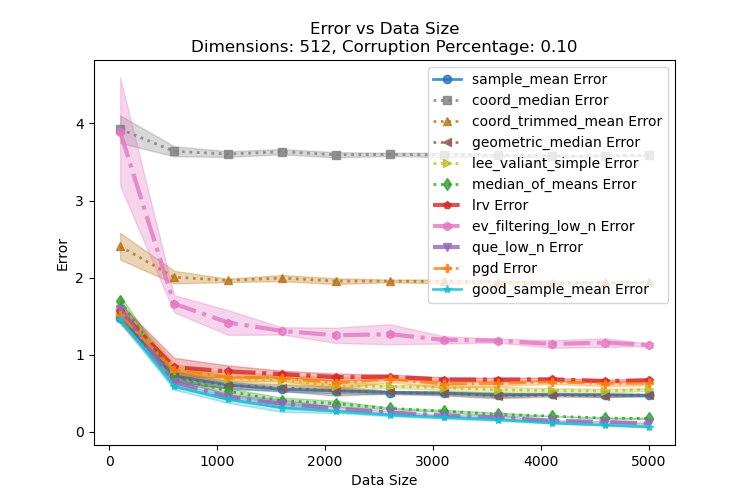}
        \end{subfigure}
        \hfill
        \begin{subfigure}[c]{0.45\textwidth}
            \centering
            \includegraphics[width=\textwidth]{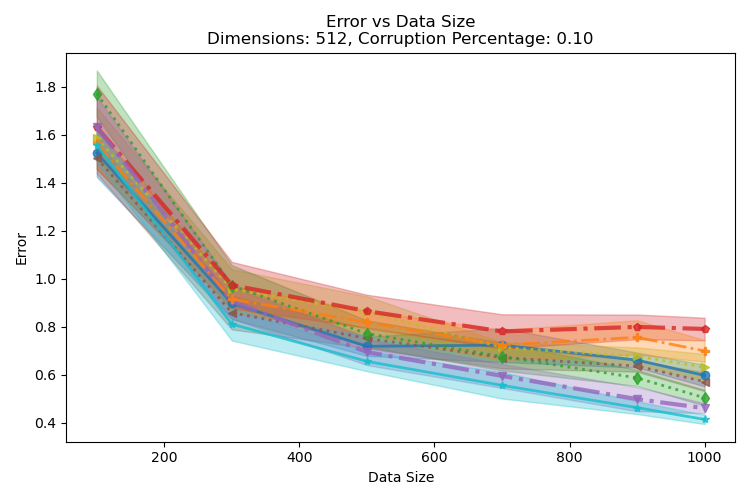}
        \end{subfigure}
        \caption{ResNet-18 Embeddings}
    \end{subfigure}

    \vskip\baselineskip 
\end{figure}

\begin{figure}
\ContinuedFloat
    
    \begin{subfigure}[b]{\textwidth}
        \begin{subfigure}[c]{0.45\textwidth}
            \centering
            \includegraphics[width=\textwidth]{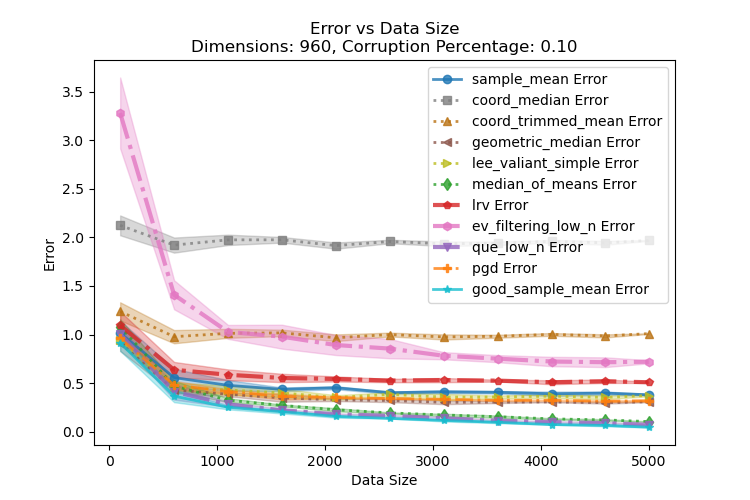}
        \end{subfigure}
        \hfill
        \begin{subfigure}[c]{0.45\textwidth}
            \centering
            \includegraphics[width=\textwidth]{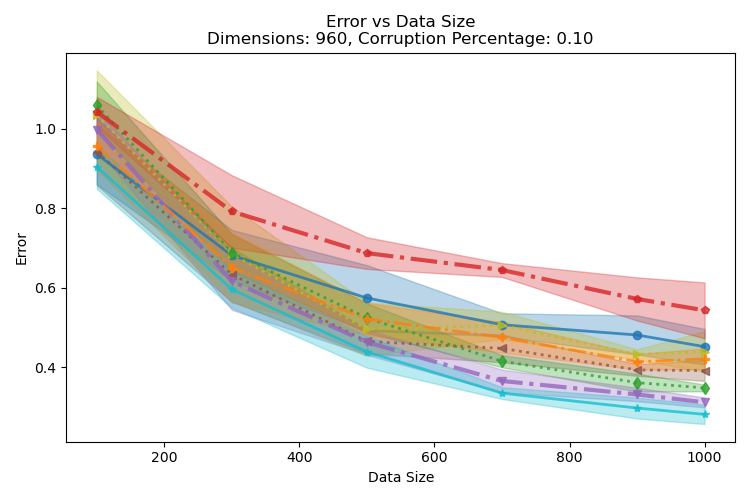}
        \end{subfigure}
        \caption{MobileNet V3 Embeddings}
    \end{subfigure}

    \begin{subfigure}[b]{\textwidth}
        \begin{subfigure}[c]{0.45\textwidth}
            \centering
            \includegraphics[width=\textwidth]{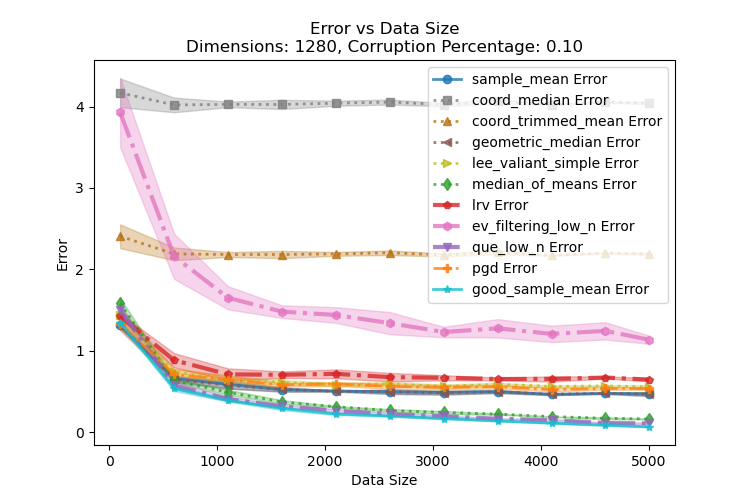}
        \end{subfigure}
        \hfill
        \begin{subfigure}[c]{0.45\textwidth}
            \centering
            \includegraphics[width=\textwidth]{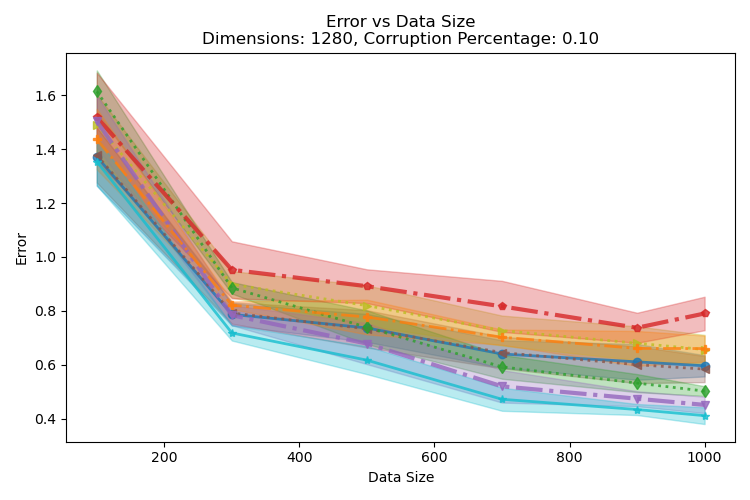}
        \end{subfigure}
         \caption{EfficientNet B0 Embeddings}
    \end{subfigure}
    \vskip\baselineskip

    \begin{subfigure}[b]{\textwidth}
        \begin{subfigure}[c]{0.45\textwidth}
            \centering
            \includegraphics[width=\textwidth]{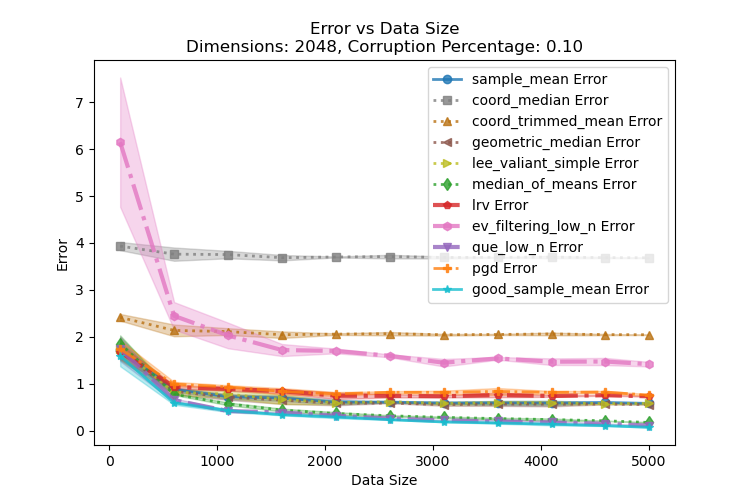}
        \end{subfigure}
        \hfill
        \begin{subfigure}[c]{0.45\textwidth}
            \centering
            \includegraphics[width=\textwidth]{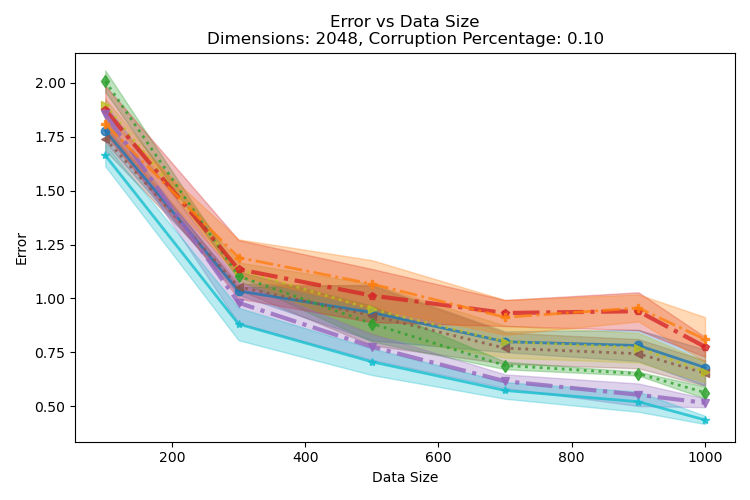}
        \end{subfigure}
        \caption{ResNet-50 Embeddings}
    \end{subfigure}
    
    \caption{Error Vs Data Size on Corrupted Image Data}
    \label{fig:image_corr_vs_n}
\end{figure}

\clearpage

\subsection{Word Embedding Experiments}
\label{app:word_experiments}

We further evaluate algorithms on the problem of estimating the mean of non attention based embeddings of words. As in the LLM experiment, we first examine the problem of mean estimation over words belonging in the same category, reporting LOOCV error. We then examine a corrupted distribution where words belonging to one category are considered inliers and those belonging to another are considered outliers. We examine four different pretrained GloVe \citep{pennington-etal-2014-glove} models from GluonNLP\footnote{\url{https://github.com/dmlc/gluon-nlp/}} generating 50, 100, 200, and 300 dimensional embeddings. We utilize datasets of 100 pleasant words and 100 unpleasant words from \cite{aboagye2023interpretable}. The very limited data size available under this setting provides a valuable real world test for robust estimators under low data size.

\paragraph{Common Category Words}

Here we examine LOOCV error vs data size on embeddings of "pleasant" words. Experiments are run identically to the LLM experiment, employing the trace scaling heuristic on \evln and \queln, the halting heuristic on \queln, and averaging results over 5 runs. Results are shown in Figure \ref{fig:loocv_pleasant}.

\begin{figure}[t]
    \begin{subfigure}{0.5\linewidth}
        \centering
        \includegraphics[width=\linewidth]{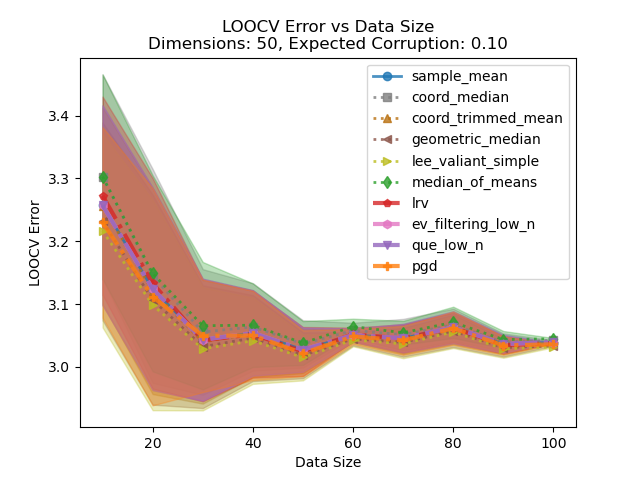}
        \caption{50 Dimensional Embeddings}
    \end{subfigure}
        \begin{subfigure}{0.5\linewidth}
        \centering
        \includegraphics[width=\linewidth]{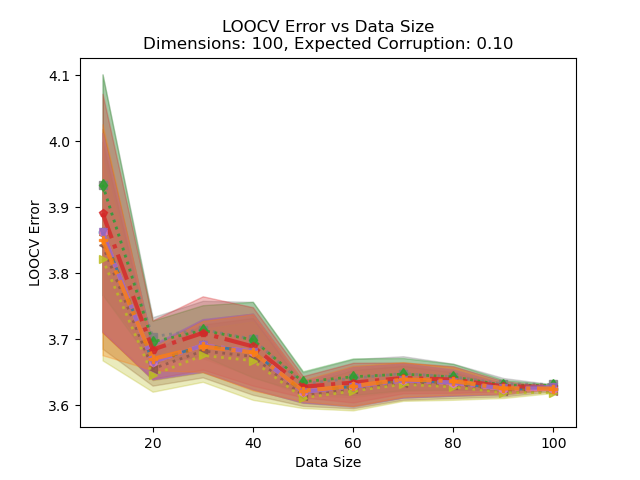}
        \caption{100 Dimensional Embeddings}
    \end{subfigure}
    \\
    \begin{subfigure}{0.5\linewidth}
        \centering
        \includegraphics[width=\linewidth]{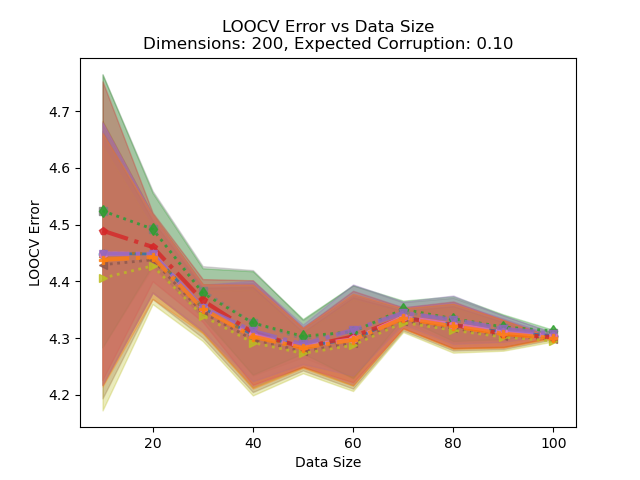}
        \caption{200 Dimensional Embeddings}
    \end{subfigure}
    \begin{subfigure}{0.5\linewidth}
        \centering
        \includegraphics[width=\linewidth]{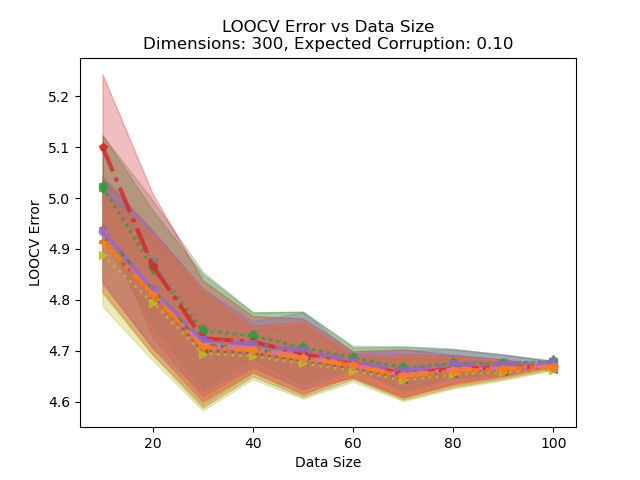}
        \caption{300 Dimensional Embeddings}
    \end{subfigure}
    \caption{LOOCV Error on "Pleasant" GloVe Embeddings}
    \label{fig:loocv_pleasant}
\end{figure}

As in the LLM experiment, we observe that no algorithm significantly outperforms \sample, despite the nontrivial LOOCV error in each setting.  Moreover, we observe that \medmean consistently achieves error slightly worse than \sample, which is not seen in the LLM experiments, suggesting the algorithm's sensitivity to distributional assumptions. However, unlike in the LLM experiment \evln does not fail catastrophically here, instead nearly matching \sample. 
This is not unexpected given \evln, and the trace estimate techniques sensitivity to distributional assumptions will sometime work -- including this case. \lrv also tends to perform worse than other algorithms, though this gap is not as large as in the LLM experiment.

\paragraph{Corrupted Words}

For the corrupted case, we draw data $X \sim (1 - \eta) P + \eta Q$, where the inlier distribution, $P$, consists of embeddings of "pleasant" words, and the outlier distribution, $Q$, consists of embeddings of "unpleasant" words. This models a more extreme version of the case where ill-defined words may be placed in a category, inducing bias. This is a notable problem for word vectors, which do not take context into account~\citep{hu-etal-2016-different}. The experimental setup is identical to in the LLM experiments. Results are shown in Figure \ref{fig:corrupted_wordvecs}.

\begin{figure}[t]
    \begin{subfigure}{0.5\linewidth}
        \centering
        \includegraphics[width=\linewidth]{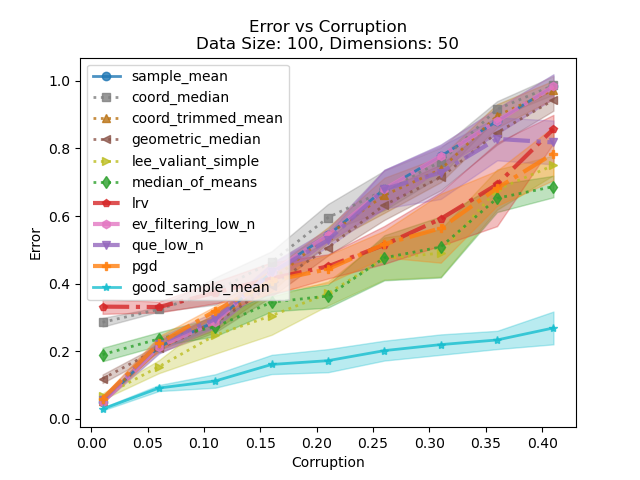}
        \caption{50 Dimensional Embeddings}
    \end{subfigure}
        \begin{subfigure}{0.5\linewidth}
        \centering
        \includegraphics[width=\linewidth]{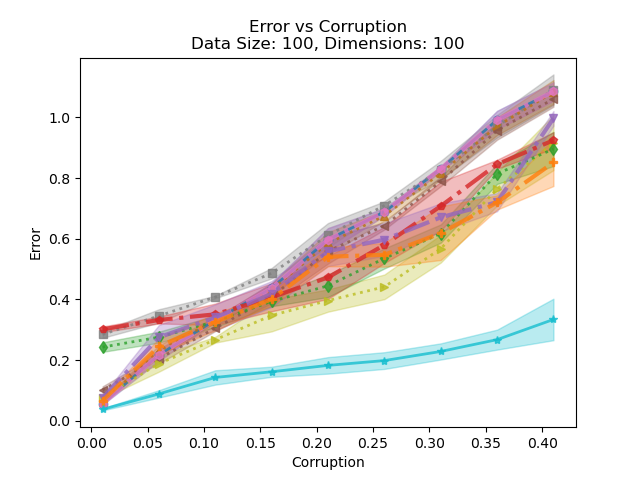}
        \caption{100 Dimensional Embeddings}
    \end{subfigure}
    \\
    \begin{subfigure}{0.5\linewidth}
        \centering
        \includegraphics[width=\linewidth]{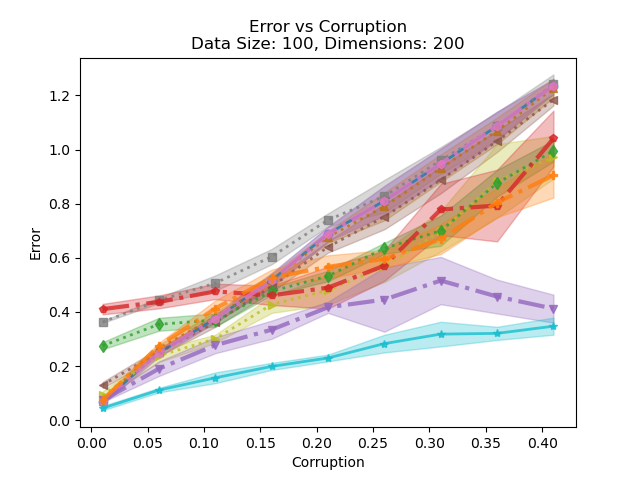}
        \caption{200 Dimensional Embeddings}
    \end{subfigure}
    \begin{subfigure}{0.5\linewidth}
        \centering
        \includegraphics[width=\linewidth]{UpdatedFigures/CorruptionGloVe/300.png}
        \caption{300 Dimensional Embeddings}
    \end{subfigure}
    \caption{Error on "Pleasant" Embeddings Corrupted with "Unpleasant" Embeddings}
    \label{fig:corrupted_wordvecs}
\end{figure}

While these results are different from the LLM experiment, they demonstrate similar trends. In particular, \queln is again the strongest performer, noticeably outperforming all other estimators across the 200 and 300 dimensional cases, and never performing worse than \sample in the 50 and 100 dimensional cases. Notably, this strong performance occurs even with $n$ much less than $d$. Unlike the LLM experiments, here \queln never approaches \gsample, and is beat by other estimators in the 50 and 100 dimensional cases. \lvsim, which tended to perform similarly to \queln and nearly match \gsample in the LLM experiments, does not perform as well in this case. It always beats \sample but does not come close to matching \gsample and performs similarly to other estimators. Likewise, \medmean does not perform as strongly here as in the LLM experiments and even performs worse than \sample over very low corruption. Supported by synthetic data results, this suggests the sensitivity of \medmean and \lvsim to distributional assumptions. As in the LLM experiments, \lrv tends to perform much worse than \sample under low corruption and outperform \sample slightly with higher corruption; \pgd tends to outperform \sample slightly; and \coordmed, \coordprune, and \geomed tend to perform similarly or slightly worse than \sample.  As in the LOOCV experiments, \evln simply matches \sample here. 

\paragraph{Additional Experiments}
We perform additional experiments, swapping the roles of "pleasant" and "unpleasant" embeddings. We report LOOCV error vs data size on embeddings of "unpleasant" words in Figure \ref{fig:loocv_unpleasant}. We report corrupted error vs data size on embeddings of "unpleasant" words corrupted with "pleasant" words in Figure \ref{fig:corrupted_wordvecs_inv}. We observe the same trends as in the previous word embedding experiments. 

\begin{figure}[t]
    \begin{subfigure}{0.5\linewidth}
        \centering
        \includegraphics[width=\linewidth]{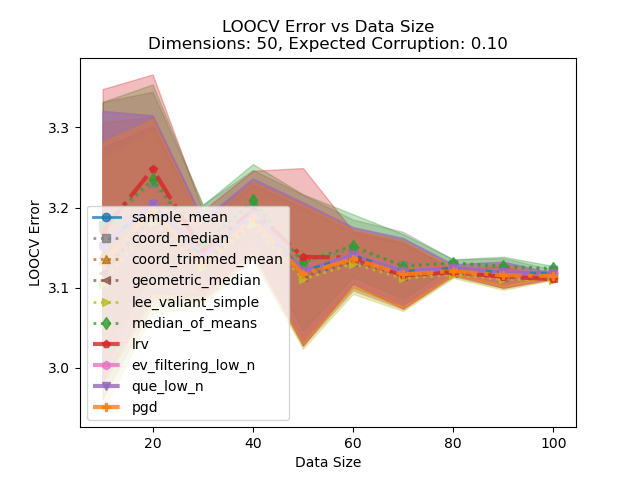}
        \caption{50 Dimensional Embeddings}
    \end{subfigure}
        \begin{subfigure}{0.5\linewidth}
        \centering
        \includegraphics[width=\linewidth]{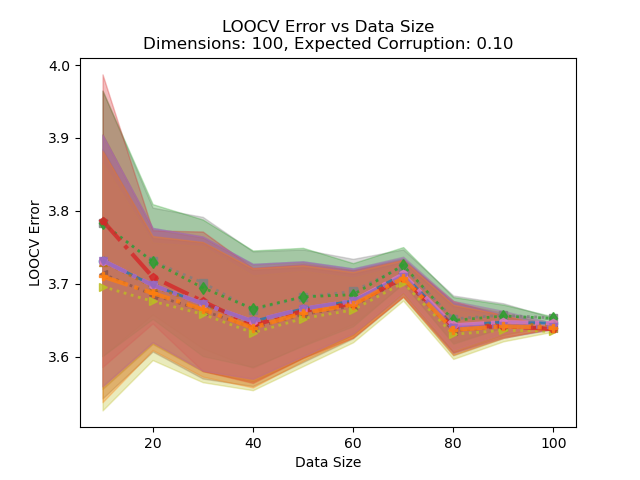}
        \caption{100 Dimensional Embeddings}
    \end{subfigure}
    \\
    \begin{subfigure}{0.5\linewidth}
        \centering
        \includegraphics[width=\linewidth]{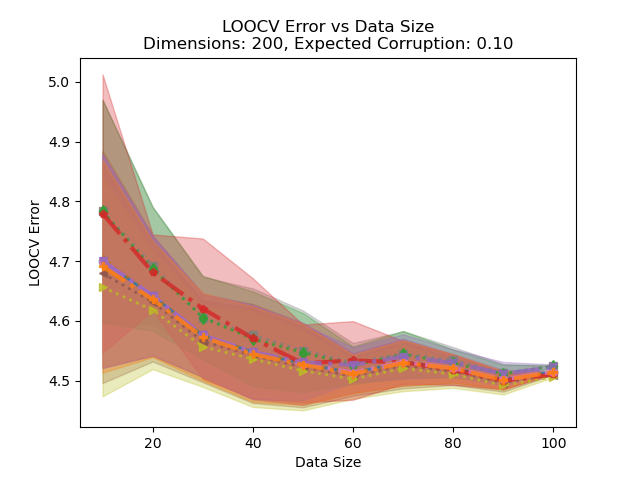}
        \caption{200 Dimensional Embeddings}
    \end{subfigure}
    \begin{subfigure}{0.5\linewidth}
        \centering
        \includegraphics[width=\linewidth]{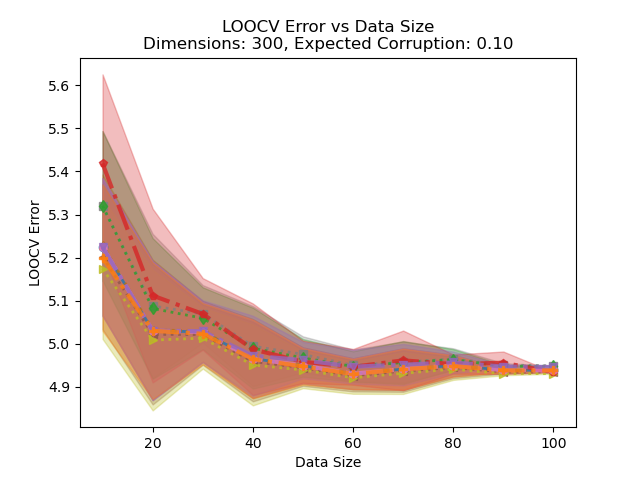}
        \caption{300 Dimensional Embeddings}
    \end{subfigure}
    \caption{LOOCV Error on "Unpleasant" GloVe Embeddings}
    \label{fig:loocv_unpleasant}
\end{figure}

\begin{figure}[t]
    \begin{subfigure}{0.5\linewidth}
        \centering
        \includegraphics[width=\linewidth]{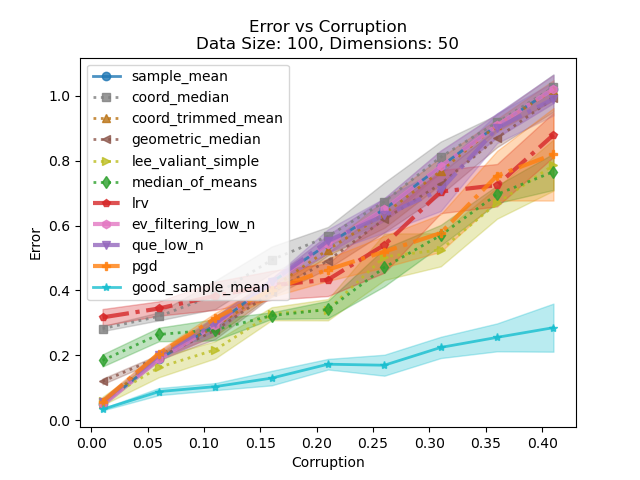}
        \caption{50 Dimensional Embeddings}
    \end{subfigure}
        \begin{subfigure}{0.5\linewidth}
        \centering
        \includegraphics[width=\linewidth]{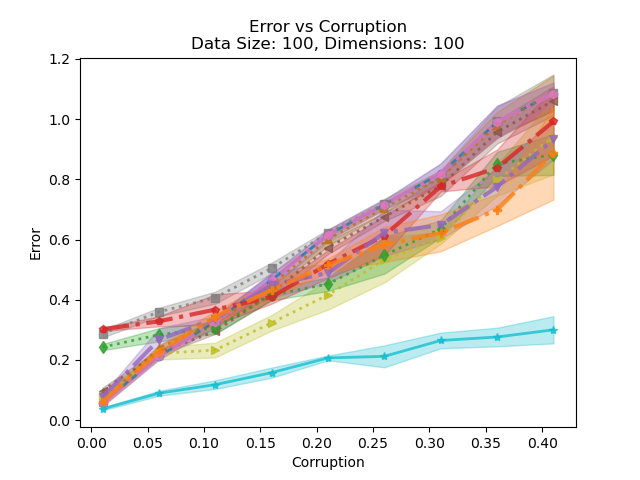}
        \caption{100 Dimensional Embeddings}
    \end{subfigure}
    \\
    \begin{subfigure}{0.5\linewidth}
        \centering
        \includegraphics[width=\linewidth]{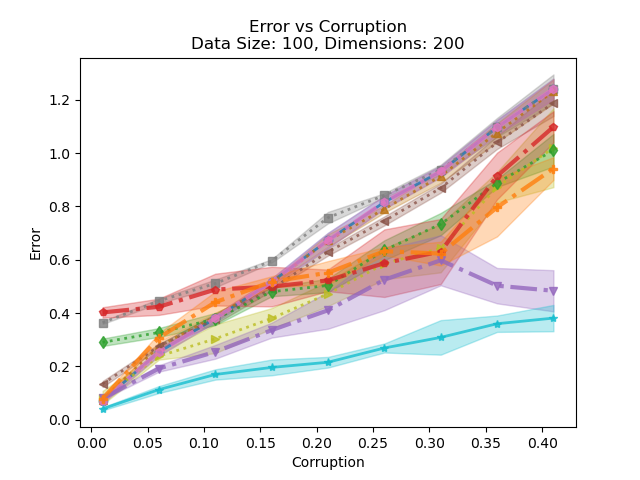}
        \caption{200 Dimensional Embeddings}
    \end{subfigure}
    \begin{subfigure}{0.5\linewidth}
        \centering
        \includegraphics[width=\linewidth]{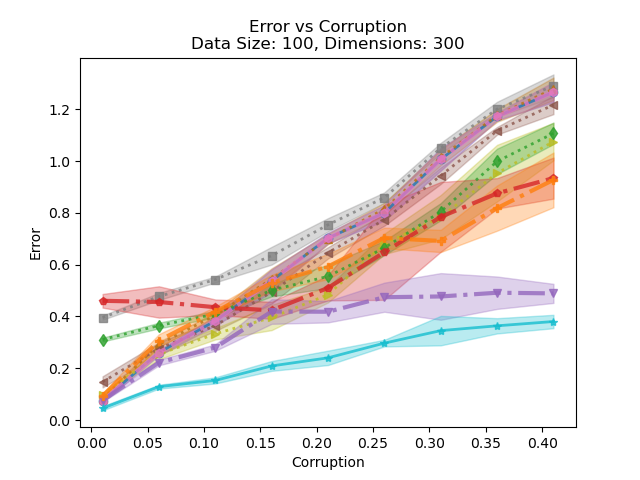}
        \caption{300 Dimensional Embeddings}
    \end{subfigure}
    \caption{Error on "Unpleasant" Embeddings Corrupted with "Pleasant" Embeddings}
    \label{fig:corrupted_wordvecs_inv}
\end{figure}

\clearpage

\subsection{LLM Experiment Ablations}
\label{app:llm_ablations}

\paragraph{Eigenvalue Pruning Method Comparison}

We compare the performance of different pruning subroutines for \evln over a selection of LLM experiments: LOOCV and Corruption Error over MiniLM and BERT embeddings. We evaluate Gaussian pruning, used throughout this paper, along with randomized pruning and fixed pruning, described in Appendix \ref{app:hp_tuning}. We retain the same conditions as in the original experiments, first scaling data utilizing the sample trace. We also include \sample and \queln in our plots for the sake of comparison, noting that \queln and \evln with fixed pruning only differ in their method of scoring outliers. These results are shown in Figure \ref{fig:llm_ev_pruning}. We notice that both randomized and fixed pruning methods do indeed perform better than the Gaussian pruning method. In particular, fixed pruning has the best LOOCV error over MiniLM and matches the error of \sample over BERT, whereas Gaussian pruning fails dramatically. However, this performance does not translate into the corrupted case, where all three pruning routines lead to significant error compared to even \sample, except with large $\eta$ where it \sample's error approaches that of these methods.  Additionally, as discussed in Appendix \ref{app:hp_tuning}, \evln with fixed pruning is not robust to noise distributions that require several runs of the algorithm to prune i.e. cases where noise lays in multiple orthogonal clusters. Notably, \queln outperforms all variations of \evln in the corrupted data case, reinforcing the observation that the outlier detection method of \queln is more robust to distributional assumptions than that of \evln.

\begin{figure}[h!]
    \centering
    \begin{subfigure}{0.45\linewidth}
        \centering
        \includegraphics[width=\linewidth]{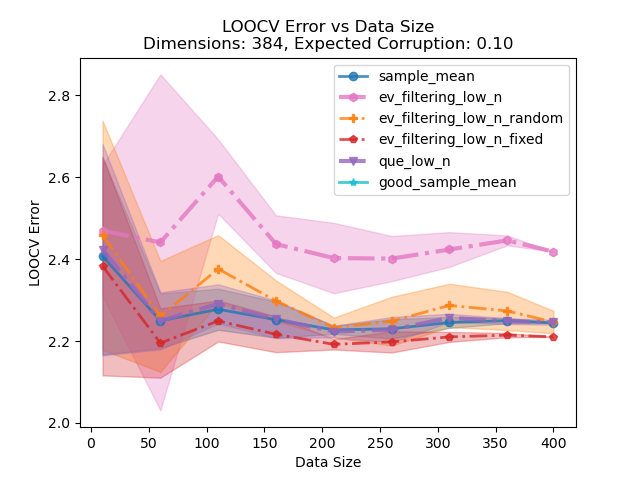}
        \caption{LOOCV Error - MiniLM}
    \end{subfigure}
    \hfill
    \begin{subfigure}{0.45\linewidth}
        \centering
        \includegraphics[width=\linewidth]{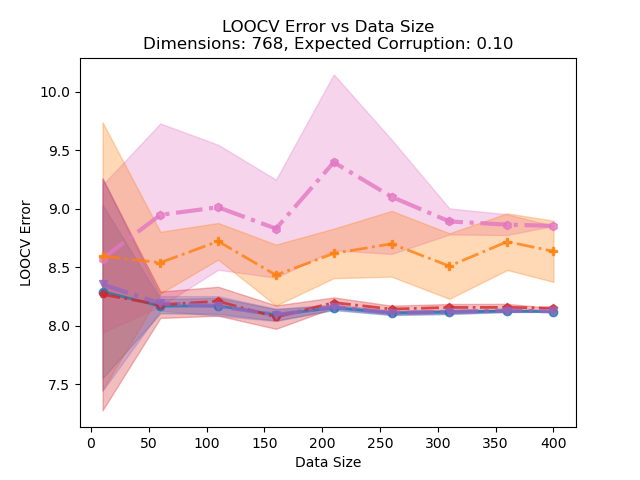}
        \caption{LOOCV Error - BERT}
    \end{subfigure}
    \bigskip
    
    \begin{subfigure}{0.45\linewidth}
        \centering
        \includegraphics[width=\linewidth]{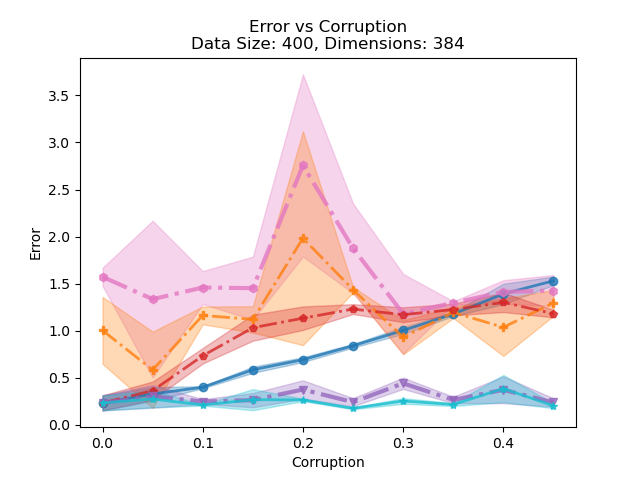}
        \caption{Corrupted Error - MiniLM}
    \end{subfigure}
    \hfill
    \begin{subfigure}{0.45\linewidth}
        \centering
        \includegraphics[width=\linewidth]{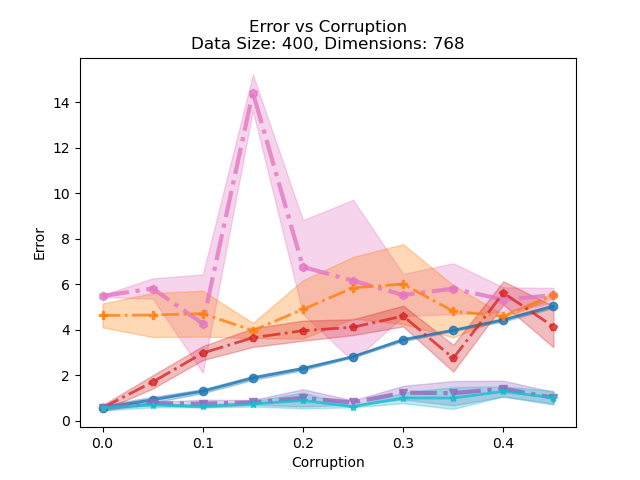}
        \caption{Corrupted Error - BERT}
    \end{subfigure}
    \caption{Eigenvalue Pruning - Pruning Method: LLM Comparison}
    \label{fig:llm_ev_pruning}
\end{figure}

\paragraph{LRV Weighting Procedure}

Here we compare the two different weighting procedures for LRV described in Appendix \ref{app:hp_tuning}: Gaussian weighting, based on downweighting outliers, and general (non-Gaussian) weighting, based on completely pruning outliers. We evaluate these two methods over the same subselection of LLM experiments: LOOCV and Corruption Error over MiniLM and BERT embeddings. These results are shown in Figure \ref{fig:llm_lrv_weighting}. We notice that general weighting outperforms Gaussian weighting in LOOCV error, with this difference being especially noticeable across BERT embeddings. However, this performance increase is not seen in either corrupted case, where Gaussian weighting notably outperforms general weighting, except with small $\eta$. This suggests that, at least under low data size, \lrv is not robust to general distributions, even using a general outlier weighting procedure.

\begin{figure}[h!]
    \centering
    \begin{subfigure}{0.45\linewidth}
        \centering
        \includegraphics[width=\linewidth]{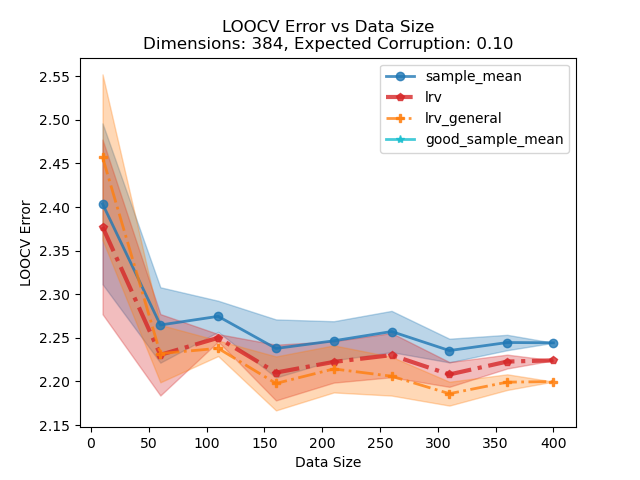}
        \caption{LOOCV Error - MiniLM}
    \end{subfigure}
    \hfill
    \begin{subfigure}{0.45\linewidth}
        \centering
        \includegraphics[width=\linewidth]{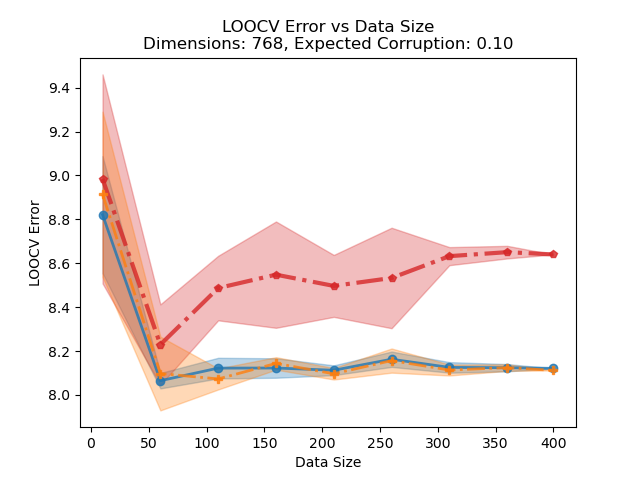}
        \caption{LOOCV Error - BERT}
    \end{subfigure}
    \bigskip
    
    \begin{subfigure}{0.45\linewidth}
        \centering
        \includegraphics[width=\linewidth]{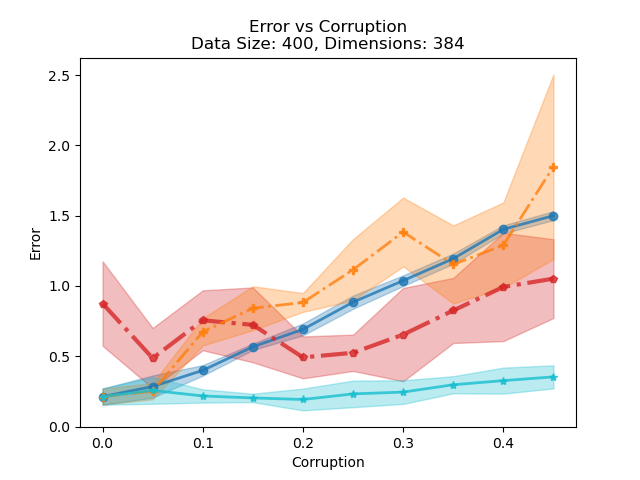}
        \caption{Corrupted Error - MiniLM}
    \end{subfigure}
    \hfill
    \begin{subfigure}{0.45\linewidth}
        \centering
        \includegraphics[width=\linewidth]{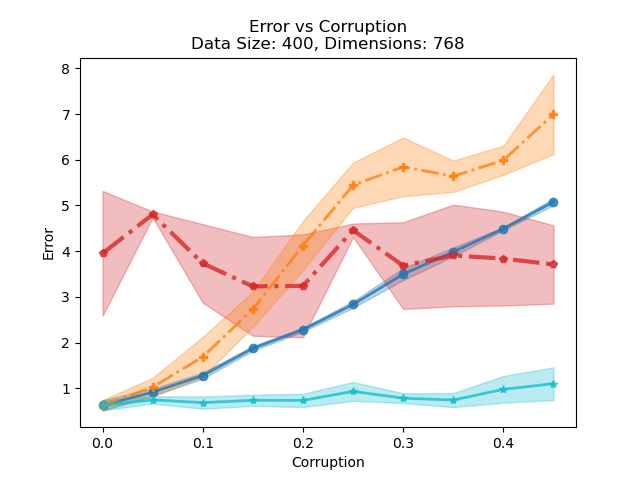}
        \caption{Corrupted Error - BERT}
    \end{subfigure}
    \caption{LRV - Gaussian Vs General Weighting: LLM Comparison}
    \label{fig:llm_lrv_weighting}
\end{figure}

\clearpage

\paragraph{Additional Experiments}

We recreate the experiments in Section \ref{sec:realworld} over two different settings. First, we examine LOOCV Error over embeddings of the word field that correspond to the "field of study" definition rather than to the "field of land" definition. These results are shown in Figure \ref{fig:loocv_field_study}. Second, we examine corrupted embeddings $X \sim (1 - \eta)P + \eta Q$, where inlier data, $P$, consists of embeddings of the word "field" corresponding to the "field of study" definition and outlier data, $Q$, consists of embeddings of the word "field" corresponding to the "field of land" definition; inverting the inlier and outlier data originally examined. These results are shown in Figure \ref{fig:loocv_corrupted_inv}. While the LOOCV error plots are not identical to the original experiment, corresponding to the expected differences in structure between the distributions of $P$ and $Q$, we find the same overall trends across the 4 plots. We additionally observe the same overall trends for corrupted data compared to the original experiment. However, \lvsim, which was consistently the best algorithm alongside \queln for corrupted data originally, breaks down for MiniLM here; always performing notably worse than \gsample. Supported by the general poor performance of \lvsim over synthetic data experiments, this reinforces the unpredictable sensitivity of \lvsim to distributional assumptions. \queln does not see any such degradation, performing near optimally across all cases, as it does in the original LLM experiment.

\begin{figure}[h!]
    \begin{subfigure}{0.5\linewidth}
        \centering
        \includegraphics[width=\linewidth]{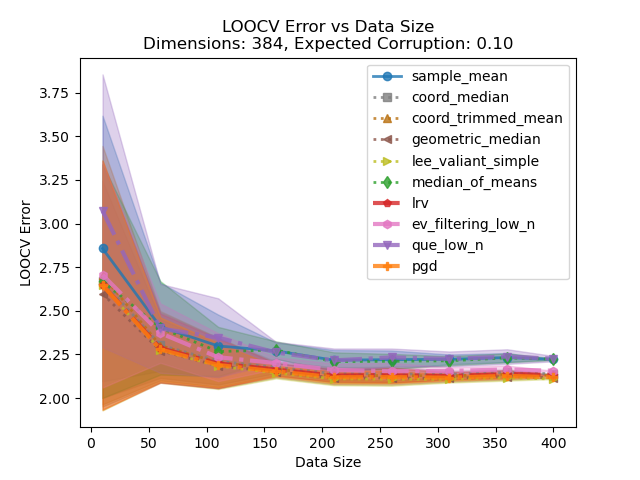}
        \caption{MiniLM}
    \end{subfigure}
        \begin{subfigure}{0.5\linewidth}
        \centering
        \includegraphics[width=\linewidth]{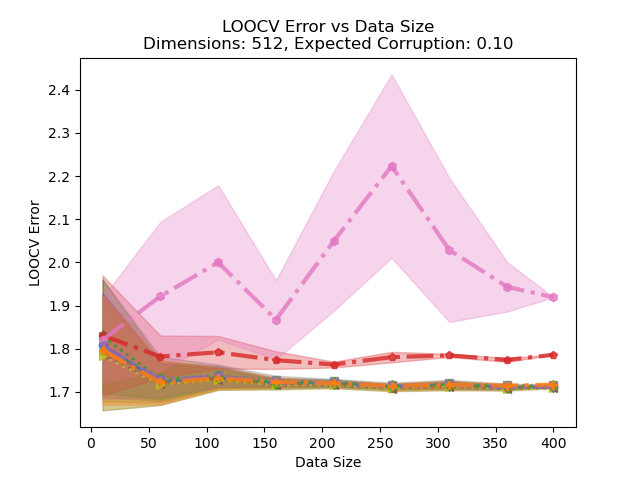}
        \caption{T5}
    \end{subfigure}
    \\
    \begin{subfigure}{0.5\linewidth}
        \centering
        \includegraphics[width=\linewidth]{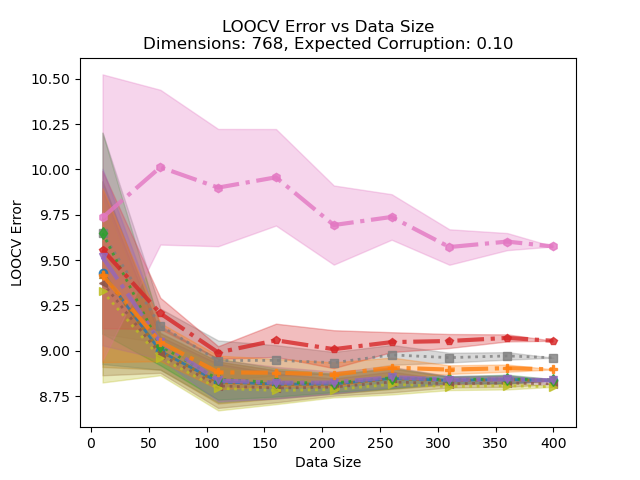}
        \caption{BERT}
    \end{subfigure}
    \begin{subfigure}{0.5\linewidth}
        \centering
        \includegraphics[width=\linewidth]{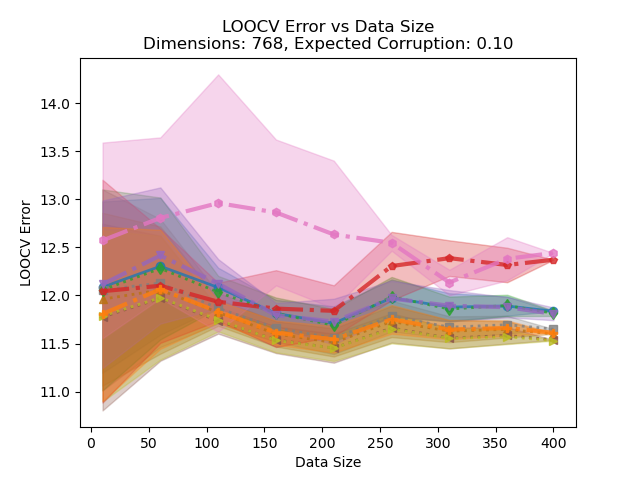}
        \caption{ALBERT}
    \end{subfigure}
    \caption{LOOCV Error on "Field Of Study" Embeddings}
    \label{fig:loocv_field_study}
\end{figure}

\begin{figure}[h!]
    \begin{subfigure}{0.5\linewidth}
        \centering
        \includegraphics[width=\linewidth]{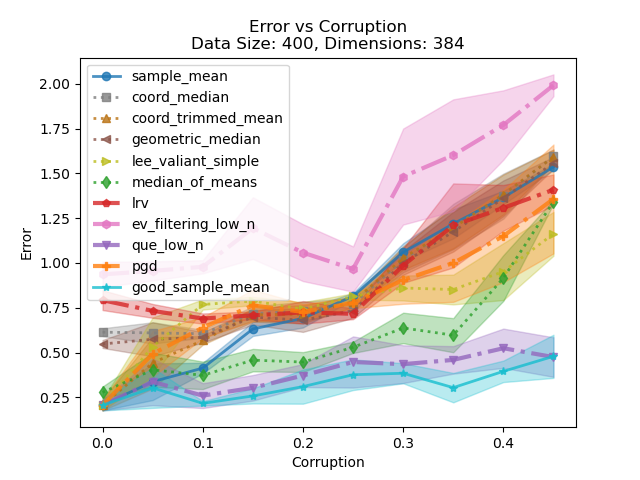}
        \caption{MiniLM}
    \end{subfigure}
        \begin{subfigure}{0.5\linewidth}
        \centering
        \includegraphics[width=\linewidth]{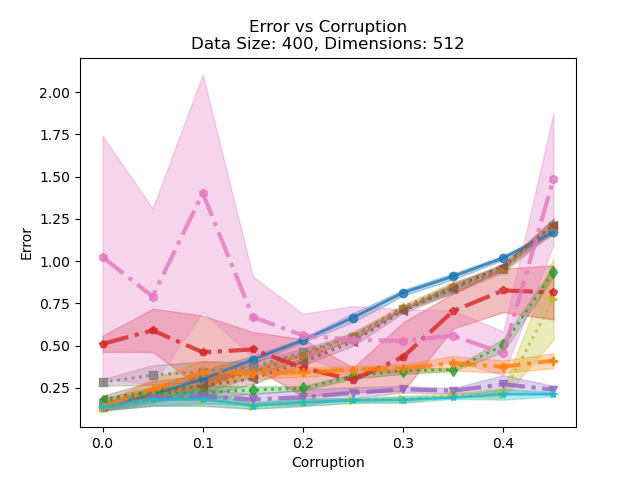}
        \caption{T5}
    \end{subfigure}
    \\
    \begin{subfigure}{0.5\linewidth}
        \centering
        \includegraphics[width=\linewidth]{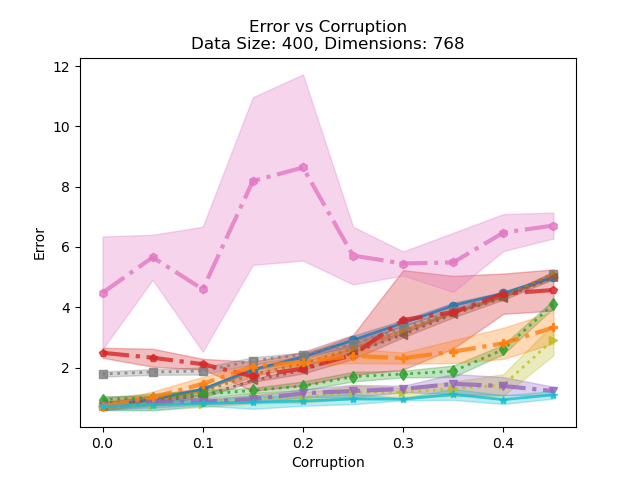}
        \caption{BERT}
    \end{subfigure}
    \begin{subfigure}{0.5\linewidth}
        \centering
        \includegraphics[width=\linewidth]{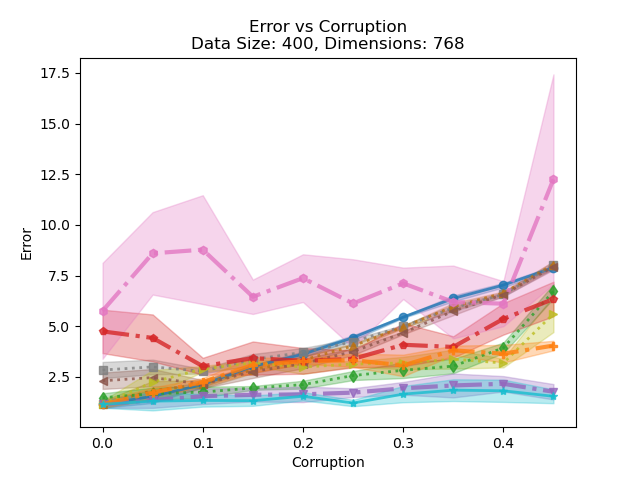}
        \caption{ALBERT}
    \end{subfigure}
    \caption{Error on "Field of Study" Embeddings Corrupted with "Field of Land" Embeddings}
    \label{fig:loocv_corrupted_inv}
\end{figure}

\clearpage

\subsection{Dataset Generation}
\label{app:llm_dataset}

We generate a dataset of 400 sentences for each definition of the word \textit{field} using ChatGPT-4o, accessed in June 2024. Attention based embeddings for the word \textit{field} are extracted from these sentences for use in our LLM experiments. We used the following two prompts to obtain the sentences:

\paragraph{Field of Study}

\begin{quote}
    \textit{I am running an experiment where I examine embeddings of the word "field" with two different contexts. Please generate 400 unique sentences using the word "field" in context with the following definition: "a particular branch of study or sphere of activity or interest." Please return these sentences in the format of a JSON file.}
\end{quote}

\paragraph{Field of Land}

\begin{quote}
    \textit{I am running an experiment where I examine embeddings of the word "field" with two different contexts. Please generate 400 unique sentences using the word "field" in context with the following definition: "an area of open land, especially one planted with crops or pasture, typically bounded by hedges or fences." Please return these sentences in the format of a JSON file.}
\end{quote}

\paragraph{Additional Prompts}

ChatGPT-4o did not produce the full 400 sentences in one go. To address this, we used the following additional prompts until we had generated the required number of sentences, and then manually combined the generated outputs. The prompt for "field of study" sentences is slightly different, as we originally observed that ChatGPT-4o would reuse the same field of study across numerous sentences.

For \textit{Field of Study}:

\begin{quote}
    \textit{Please generate 100 more sentences. Do not repeat similar sentences or use "field" to refer to the same field of study multiple times.}
\end{quote}

For \textit{Field of Land}:

\begin{quote}
    \textit{Please generate 100 more sentences.}
\end{quote}

\paragraph{Tables Of Generated Sentences}

We include the following tables of generated sentences.

\textbf{Field Of Land Sentences}:

\input{UpdatedFigures/SentenceTables/field_land_table}

\textbf{Field Of Study Sentences}:

\input{UpdatedFigures/SentenceTables/field_study_table}

\end{document}

%% file: math_commands.tex
\usepackage{amsmath,amsfonts,bm}
\usepackage{xspace}

\def\eqref#1{equation~\ref{#1}}

\def\1{\bm{1}}

\def\eps{{\epsilon}}

\DeclareMathAlphabet{\mathsfit}{\encodingdefault}{\sfdefault}{m}{sl}
\SetMathAlphabet{\mathsfit}{bold}{\encodingdefault}{\sfdefault}{bx}{n}

\newcommand{\E}{\mathbb{E}}

\newcommand{\R}{\mathbb{R}}

\DeclareMathOperator{\Tr}{Tr}

%% file: UpdatedFigures/SentenceTables/field_land_table.tex
    \begin{longtable}{|c|p{12cm}|}
    \hline
    \textbf{Index} & \textbf{Sentence} \\ \hline
    \endfirsthead
    \multicolumn{2}{c}{\textit{Continued from previous page}} \\ \hline
    \textbf{Index} & \textbf{Sentence} \\ \hline
    \endhead
    \hline \multicolumn{2}{r}{\textit{Continued on next page}} \\ \hline
    \endfoot
    \hline
    \endlastfoot
    1 & The scarecrow stood tall in the middle of the field. \\ \hline
2 & The deer were spotted grazing in the field at dawn. \\ \hline
3 & The field was an ideal spot for stargazing. \\ \hline
4 & He loved to watch the sunset over the field. \\ \hline
5 & The field stretched out as far as the eye could see. \\ \hline
6 & Sunflowers swayed in the field under the clear blue sky. \\ \hline
7 & The field stretched out to the edge of the forest. \\ \hline
8 & The field was alive with the sound of chirping crickets. \\ \hline
9 & They played hide and seek in the field, darting among the tall grasses. \\ \hline
10 & He built a small shed at the edge of the field for storage. \\ \hline
11 & The hot air balloon landed gently in the field. \\ \hline
12 & She enjoyed picnicking in the field of wildflowers near her home. \\ \hline
13 & She found an old, weathered barn at the edge of the field. \\ \hline
14 & The field was fenced off to keep out wild animals. \\ \hline
15 & They walked hand in hand through the field of wildflowers, lost in conversation. \\ \hline
16 & He loved the quiet solitude of the open field. \\ \hline
17 & The field was a perfect spot for birdwatching. \\ \hline
18 & We had a picnic in the wide, open field. \\ \hline
19 & The open field was covered in morning dew. \\ \hline
20 & We watched the meteor shower from the field. \\ \hline
21 & The open field was perfect for stargazing. \\ \hline
22 & The field of strawberries was a patchwork of red and green. \\ \hline
23 & The field was filled with the scent of blooming flowers. \\ \hline
24 & Hikers followed the trail through the field of wild grasses, enjoying the solitude. \\ \hline
25 & They played ultimate frisbee in the field. \\ \hline
26 & The field of herbs was fragrant, each plant releasing its unique scent. \\ \hline
27 & She found a hidden path that led to the field. \\ \hline
28 & The field was a sea of green during the spring. \\ \hline
29 & She found a hidden path that led to the field. \\ \hline
30 & Wildflowers grew abundantly in the field. \\ \hline
31 & A lone tree stood in the middle of the field, providing shade. \\ \hline
32 & The field was a riot of color in the fall. \\ \hline
33 & They harvested wheat from the vast field. \\ \hline
34 & Deer grazed in the field at dusk, their silhouettes blending with the shadows. \\ \hline
35 & The field was a vibrant green after the rain. \\ \hline
36 & The field was a riot of color during the summer. \\ \hline
37 & The field was a sea of gold during the harvest. \\ \hline
38 & She enjoyed painting the landscape of the field. \\ \hline
39 & He loved the feeling of the grass under his feet in the field. \\ \hline
40 & We spotted deer grazing in the distant field. \\ \hline
41 & The field was surrounded by rolling hills. \\ \hline
42 & We walked through the field at sunrise. \\ \hline
43 & The field was blanketed in snow during the winter. \\ \hline
44 & The field was a playground for the neighborhood children. \\ \hline
45 & The field was a sea of purple lavender in full bloom. \\ \hline
46 & The field was blanketed in snow during the winter. \\ \hline
47 & Cows grazed peacefully in the field enclosed by wooden fences. \\ \hline
48 & The field, bordered by ancient oak trees, was a serene spot for a picnic. \\ \hline
49 & She enjoyed walking through the field, picking flowers. \\ \hline
50 & The field was blanketed with snow in winter. \\ \hline
51 & He loved the peace and quiet of the open field. \\ \hline
52 & The field was a patchwork of different crops. \\ \hline
53 & The field was a sea of gold during the wheat harvest. \\ \hline
54 & A scarecrow stood watch over the field. \\ \hline
55 & He spent his afternoons walking through the field, lost in thought. \\ \hline
56 & The field was divided into neat rows for planting. \\ \hline
57 & The field was a perfect spot for a family picnic. \\ \hline
58 & Birds chirped happily in the field, searching for insects among the plants. \\ \hline
59 & The field was dotted with patches of wild grass. \\ \hline
60 & She painted a landscape of the field in her art class. \\ \hline
61 & The field was dotted with hay bales after a long day of harvesting. \\ \hline
62 & He built a small fire pit in the middle of the field. \\ \hline
63 & He loved the smell of fresh-cut grass in the field. \\ \hline
64 & She found a quiet spot in the field to read her book. \\ \hline
65 & The field was a burst of color in the autumn. \\ \hline
66 & The field of rye swayed in the breeze, creating waves of green. \\ \hline
67 & The field was alive with the sound of crickets. \\ \hline
68 & The field was a popular spot for local festivals. \\ \hline
69 & He loved the smell of fresh-cut grass in the field. \\ \hline
70 & The field was a favorite spot for local photographers. \\ \hline
71 & He loved to run through the field with his friends. \\ \hline
72 & The field was alive with the sound of crickets. \\ \hline
73 & We could see the farmhouse from across the field. \\ \hline
74 & He loved the quiet solitude of the open field. \\ \hline
75 & A scarecrow stood at the center of the field, arms outstretched. \\ \hline
76 & We spotted a fox darting through the field. \\ \hline
77 & The children flew paper airplanes in the field. \\ \hline
78 & They set up a makeshift baseball diamond in the field. \\ \hline
79 & A herd of sheep grazed peacefully in the field. \\ \hline
80 & They played a game of tag in the spacious field. \\ \hline
81 & He loved to explore the field with his dog. \\ \hline
82 & A gentle breeze rustled the leaves of the crops in the field. \\ \hline
83 & The field was covered in a blanket of fresh snow. \\ \hline
84 & The horses galloped freely across the open field. \\ \hline
85 & The dogs ran freely in the wide field. \\ \hline
86 & The field of blueberries was a favorite spot for summer picking. \\ \hline
87 & The field was a peaceful place to reflect and relax. \\ \hline
88 & The scarecrow stood guard in the field, its tattered clothes fluttering in the breeze. \\ \hline
89 & The field was a place of peace and serenity. \\ \hline
90 & A gentle fog settled over the field in the morning. \\ \hline
91 & The field was home to several species of birds. \\ \hline
92 & The field was ideal for an outdoor concert. \\ \hline
93 & She enjoyed walking through the field, listening to the birds sing. \\ \hline
94 & The field was alive with the sound of crickets chirping at night. \\ \hline
95 & The field was alive with the sound of crickets. \\ \hline
96 & The field was a perfect spot for a family picnic. \\ \hline
97 & The field of corn rustled in the wind, creating a soothing sound. \\ \hline
98 & The field was a quiet refuge from the busy city. \\ \hline
99 & He built a small bench at the edge of the field. \\ \hline
100 & They played frisbee in the open field. \\ \hline
101 & He loved the smell of fresh-cut grass in the field. \\ \hline
102 & She enjoyed picnicking in the field with her friends. \\ \hline
103 & The farmer worked tirelessly in the field to ensure a good harvest. \\ \hline
104 & A small stream ran along the edge of the field, providing water for the livestock. \\ \hline
105 & She enjoyed walking through the field, picking flowers. \\ \hline
106 & The community garden was set up in the field. \\ \hline
107 & A tractor moved slowly across the field, plowing the earth for new seeds. \\ \hline
108 & They had an Easter egg hunt in the field. \\ \hline
109 & We could see the farmhouse from across the field. \\ \hline
110 & A gentle breeze blew across the field. \\ \hline
111 & The field stretched out to the horizon, seemingly endless. \\ \hline
112 & The field buzzed with the sound of bees collecting nectar. \\ \hline
113 & The field was a riot of color in the fall. \\ \hline
114 & The field was home to a variety of wildlife. \\ \hline
115 & Birds nested in the hedges surrounding the field, singing melodious tunes. \\ \hline
116 & The field was dotted with hay bales, ready for storage. \\ \hline
117 & The field was a sea of gold during the wheat harvest. \\ \hline
118 & The field was lush and green after the rain. \\ \hline
119 & She loved to dance barefoot in the field. \\ \hline
120 & The kids enjoyed a treasure hunt in the field. \\ \hline
121 & The field of cherry blossoms was a sight to behold, petals drifting in the wind. \\ \hline
122 & The farmer walked across the field, inspecting the growing wheat. \\ \hline
123 & After the rain, the field was dotted with puddles reflecting the clouds. \\ \hline
124 & Butterflies flitted about in the field, adding to its charm. \\ \hline
125 & She could see the field from her kitchen window. \\ \hline
126 & We found a quiet spot in the field to relax. \\ \hline
127 & They set up a campfire in the field. \\ \hline
128 & The farmer plowed the field in preparation for planting. \\ \hline
129 & She enjoyed walking through the field, picking flowers. \\ \hline
130 & The field's soil was rich and fertile, ideal for planting. \\ \hline
131 & The field was a vibrant green after the rain. \\ \hline
132 & The field was a peaceful place to escape to. \\ \hline
133 & The field of sunflowers attracted bees with its abundance of pollen. \\ \hline
134 & The field was blanketed in snow during the winter. \\ \hline
135 & A small brook ran alongside the field, providing irrigation. \\ \hline
136 & He watched the sunrise over the field from his porch. \\ \hline
137 & She found a quiet corner of the field to meditate. \\ \hline
138 & She found a hidden path that led to the field. \\ \hline
139 & She picked wild strawberries in the field, their sweetness bursting in her mouth. \\ \hline
140 & The large field was perfect for flying kites. \\ \hline
141 & The field of cotton was ready for picking, the fluffy bolls bursting open. \\ \hline
142 & The field of rapeseed glowed bright yellow against the blue sky. \\ \hline
143 & He loved to run through the field with his friends. \\ \hline
144 & The field was a canvas of colors in the spring. \\ \hline
145 & The field was surrounded by a wooden fence. \\ \hline
146 & Farmers plowed the field, preparing it for the next planting season. \\ \hline
147 & He loved to watch the sunset over the field. \\ \hline
148 & She could hear the distant sound of a tractor working in the field. \\ \hline
149 & Children flew kites in the field, the colorful tails dancing in the wind. \\ \hline
150 & The field was a place of beauty and tranquility. \\ \hline
151 & They had a barbecue in the field last weekend. \\ \hline
152 & The field of lettuce was irrigated regularly, ensuring crisp leaves. \\ \hline
153 & She found a hidden trail that led to the field. \\ \hline
154 & She enjoyed walking through the field, listening to the birds sing. \\ \hline
155 & A narrow path cut through the field, leading to the old barn. \\ \hline
156 & She found a perfect spot in the field for her garden. \\ \hline
157 & The field was a sea of green grass in the spring. \\ \hline
158 & Farm workers toiled from dawn to dusk, tending to the vast field. \\ \hline
159 & The field of oats rustled softly as the wind passed through. \\ \hline
160 & She loved to dance barefoot in the field. \\ \hline
161 & A gentle mist rose from the field in the early morning. \\ \hline
162 & The field of soybeans stretched for miles, a sea of green leaves. \\ \hline
163 & The field was a place of peace and serenity. \\ \hline
164 & The farmer surveyed the field, planning his next move. \\ \hline
165 & The field was a sea of gold during the wheat harvest. \\ \hline
166 & The field was a patchwork of different crops, each thriving in the rich soil. \\ \hline
167 & She spent her afternoons wandering through the field. \\ \hline
168 & The wheat field swayed gently in the wind. \\ \hline
169 & Sunflowers bloomed vibrantly in the field, their faces turning towards the sun. \\ \hline
170 & The field was a perfect place for a family gathering. \\ \hline
171 & The field was alive with the sound of crickets. \\ \hline
172 & The field was lush with green crops swaying in the breeze. \\ \hline
173 & In the distance, a barn overlooked the sprawling field. \\ \hline
174 & The field was a sea of green during the growing season. \\ \hline
175 & The field glistened with morning dew. \\ \hline
176 & They picked strawberries in the field. \\ \hline
177 & The field was a playground for the neighborhood children. \\ \hline
178 & A herd of cows roamed freely in the field. \\ \hline
179 & A flock of birds flew over the field. \\ \hline
180 & The field was blanketed in wildflowers during the spring. \\ \hline
181 & The field was full of life, even in the winter. \\ \hline
182 & The field was a favorite spot for watching the stars at night. \\ \hline
183 & Children played a game of soccer in the field behind the school. \\ \hline
184 & The field was a perfect spot for birdwatching. \\ \hline
185 & The field was a peaceful place to escape to. \\ \hline
186 & The kids enjoyed running freely in the open field. \\ \hline
187 & He set up a telescope in the field to watch the stars. \\ \hline
188 & The children lay down in the field to watch the clouds. \\ \hline
189 & The field was dotted with blooming flowers. \\ \hline
190 & He found a perfect spot in the field for his garden. \\ \hline
191 & The cows roamed freely in the open field. \\ \hline
192 & The field of corn was ready for harvesting. \\ \hline
193 & The field of daisies swayed gently in the wind, a sea of white petals. \\ \hline
194 & The field provided a perfect backdrop for photos. \\ \hline
195 & The field was a sea of gold during the harvest. \\ \hline
196 & They held a yoga class in the field at dawn. \\ \hline
197 & The field lay fallow, resting before the next planting season. \\ \hline
198 & The sun set behind the distant field. \\ \hline
199 & He could see the field from his bedroom window. \\ \hline
200 & The field was a vibrant green after the rain. \\ \hline
201 & She ran through the field of tall grass, her laughter ringing out. \\ \hline
202 & They set up a tent in the field for the event. \\ \hline
203 & The field was a favorite spot for local photographers. \\ \hline
204 & She enjoyed collecting wildflowers from the field. \\ \hline
205 & They flew kites in the large, empty field. \\ \hline
206 & The field was a vibrant green after the rain. \\ \hline
207 & The festival was held in the open field every summer. \\ \hline
208 & The field was a haven for wildflowers. \\ \hline
209 & The field of radishes was ready for picking, their bright red roots peeking out. \\ \hline
210 & The field was lush with green crops swaying in the breeze. \\ \hline
211 & They saw a fox darting across the field. \\ \hline
212 & She found a quiet corner of the field to meditate. \\ \hline
213 & Rows of corn stretched across the field, reaching up towards the sky. \\ \hline
214 & The field was a playground for the neighborhood children. \\ \hline
215 & She took a stroll through the field, enjoying the fresh air. \\ \hline
216 & He loved the smell of fresh-cut hay in the field. \\ \hline
217 & The field was a patchwork of colors during the flower festival. \\ \hline
218 & She spent her afternoons wandering through the field. \\ \hline
219 & The field was covered in a blanket of wildflowers. \\ \hline
220 & The picnic was set up in the field by the lake. \\ \hline
221 & The field was a sea of green during the growing season. \\ \hline
222 & A beautiful field of sunflowers stretched as far as the eye could see. \\ \hline
223 & The scarecrow stood watch over the field, deterring hungry birds. \\ \hline
224 & The horses galloped across the field, kicking up dust behind them. \\ \hline
225 & The horses grazed peacefully in the field. \\ \hline
226 & He loved the peace and quiet of the open field. \\ \hline
227 & The field of tulips was a riot of color, reds, yellows, and pinks blending together. \\ \hline
228 & The field was buzzing with activity during the harvest season. \\ \hline
229 & He loved the quiet solitude of the open field. \\ \hline
230 & She painted a beautiful landscape of the field. \\ \hline
231 & He set up a telescope in the field to watch the stars. \\ \hline
232 & They wandered through the field of pumpkins, searching for the perfect one. \\ \hline
233 & In the winter, the field was covered in a thick layer of snow. \\ \hline
234 & Cows grazed peacefully in the field enclosed by wooden fences. \\ \hline
235 & The farmers planted potatoes in the field closest to the farmhouse. \\ \hline
236 & They set up a tent in the field, ready for a weekend of camping. \\ \hline
237 & The field was a haven for wildflowers. \\ \hline
238 & At sunset, the field glowed with a golden hue, creating a picturesque scene. \\ \hline
239 & The sheep roamed freely in the field, nibbling on fresh grass. \\ \hline
240 & She enjoyed painting the landscape of the field. \\ \hline
241 & He loved to explore the field with his dog. \\ \hline
242 & The field was a patchwork of different crops. \\ \hline
243 & They saw a rainbow stretching over the field. \\ \hline
244 & He found solace in the quiet field, away from the hustle and bustle. \\ \hline
245 & The field was a place of peace and serenity. \\ \hline
246 & The field was an expanse of green grass. \\ \hline
247 & The field was a haven for wildlife, including rabbits and deer. \\ \hline
248 & He built a small bench at the edge of the field. \\ \hline
249 & A gentle breeze rustled the leaves in the field. \\ \hline
250 & She collected wildflowers from the field to make a bouquet. \\ \hline
251 & The field, surrounded by rolling hills, was a picture of serenity. \\ \hline
252 & She found an old, weathered barn at the edge of the field. \\ \hline
253 & The festival was held in the large, open field. \\ \hline
254 & She spent her afternoons wandering through the field. \\ \hline
255 & The field was a playground for the neighborhood kids. \\ \hline
256 & The field was a vibrant green after the rain. \\ \hline
257 & The field of cabbages was neatly arranged in rows, their heads forming a patchwork. \\ \hline
258 & She found a quiet spot in the field to read her book. \\ \hline
259 & Farmers tended to the field of pumpkins, ensuring each one grew plump and round. \\ \hline
260 & The field was a quiet refuge from the busy city. \\ \hline
261 & They set up camp in the field, under a canopy of stars. \\ \hline
262 & The field was a favorite spot for local artists. \\ \hline
263 & The field was plowed into neat, straight rows. \\ \hline
264 & The kids played catch in the large field. \\ \hline
265 & The field was a favorite spot for local photographers. \\ \hline
266 & The farmer rotated his crops to keep the field fertile. \\ \hline
267 & We had a family reunion in the field. \\ \hline
268 & He loved the smell of fresh-cut hay in the field. \\ \hline
269 & The field of onions was harvested in the fall, bulbs dug up and stored for winter. \\ \hline
270 & The field was blanketed in fog early in the morning. \\ \hline
271 & The field provided ample space for the annual county fair. \\ \hline
272 & She enjoyed collecting wildflowers from the field. \\ \hline
273 & She loved to dance barefoot in the field. \\ \hline
274 & The field was a riot of color in the summer. \\ \hline
275 & The field was a perfect spot for birdwatching. \\ \hline
276 & She ran across the field, her laughter echoing in the open space. \\ \hline
277 & Butterflies fluttered over the wildflowers that dotted the field. \\ \hline
278 & The field stretched out to the horizon. \\ \hline
279 & They practiced their golf swings in the open field. \\ \hline
280 & The field was perfect for a game of cricket. \\ \hline
281 & In the summer, the field was a sea of golden barley ready for harvest. \\ \hline
282 & The field was home to a family of rabbits. \\ \hline
283 & The field was alive with the hum of insects. \\ \hline
284 & The field was a vibrant green after the rain. \\ \hline
285 & The field was dotted with patches of wild grass. \\ \hline
286 & We enjoyed a picnic lunch in the sunny field. \\ \hline
287 & We lay on the blanket in the middle of the field. \\ \hline
288 & She found a quiet spot in the field to read her book. \\ \hline
289 & Children played soccer in the open field near the village. \\ \hline
290 & The field was covered in a blanket of snow during the winter months. \\ \hline
291 & The field was a vibrant green after the spring rain. \\ \hline
292 & He loved to explore the field with his dog. \\ \hline
293 & Wildflowers dotted the field, adding splashes of color to the landscape. \\ \hline
294 & They walked through the field of lavender, inhaling its sweet fragrance. \\ \hline
295 & We had a bonfire in the middle of the field. \\ \hline
296 & He found a quiet spot in the field to read his book. \\ \hline
297 & The field was a sea of green during the growing season. \\ \hline
298 & The field of barley was ready for harvest, the grains turning golden. \\ \hline
299 & The field was surrounded by a white picket fence. \\ \hline
300 & He built a small shed at the edge of the field. \\ \hline
301 & The field of sugar cane stretched to the horizon, its stalks swaying gently. \\ \hline
302 & The field was a burst of color in the autumn. \\ \hline
303 & She enjoyed collecting wildflowers from the field. \\ \hline
304 & He loved to watch the sunset over the field. \\ \hline
305 & A beautiful field of sunflowers stretched as far as the eye could see. \\ \hline
306 & She found an old, rusty plow at the edge of the field. \\ \hline
307 & We watched the sunset from the edge of the field. \\ \hline
308 & The dogs loved running around in the open field. \\ \hline
309 & The field of lavender was a haven for bees, buzzing busily among the flowers. \\ \hline
310 & The field was a riot of color in the fall. \\ \hline
311 & Farmers worked diligently in the corn field. \\ \hline
312 & The field was dotted with patches of wild grass. \\ \hline
313 & Sheep roamed freely in the field, their woolly coats glistening in the sun. \\ \hline
314 & The children played soccer in the field. \\ \hline
315 & The field was a favorite spot for local artists. \\ \hline
316 & He built a small bench at the edge of the field. \\ \hline
317 & The field was a tapestry of colors in the spring, with various flowers in bloom. \\ \hline
318 & We could see rabbits hopping in the field. \\ \hline
319 & The field was a sea of gold during the harvest. \\ \hline
320 & She sat under the oak tree in the field, enjoying the shade. \\ \hline
321 & Cattle grazed peacefully in the field, surrounded by rolling hills. \\ \hline
322 & The field looked magical under the light of the full moon. \\ \hline
323 & The field was dotted with bales of hay. \\ \hline
324 & The field was a riot of color during the summer. \\ \hline
325 & He set up a picnic in the middle of the field. \\ \hline
326 & The old oak tree stood alone in the field. \\ \hline
327 & A scarecrow stood tall in the middle of the field, warding off birds. \\ \hline
328 & The field was a favorite spot for local photographers. \\ \hline
329 & She spent hours wandering through the field, collecting herbs. \\ \hline
330 & She enjoyed picnicking in the field with her family. \\ \hline
331 & The field was home to a variety of wildlife. \\ \hline
332 & She lay down in the field, gazing up at the clear blue sky. \\ \hline
333 & She found a hidden trail that led to the field. \\ \hline
334 & The field was full of life, even in the winter. \\ \hline
335 & The field was surrounded by a dense forest. \\ \hline
336 & In the fall, the field turned a golden hue as the crops matured. \\ \hline
337 & A light breeze swept through the open field. \\ \hline
338 & He found a hidden trail that led to the field. \\ \hline
339 & The path led us through a vast field. \\ \hline
340 & He built a small shed at the edge of the field. \\ \hline
341 & He found a perfect spot in the field for his garden. \\ \hline
342 & The field was a haven for wildflowers. \\ \hline
343 & The field was a riot of color during the summer. \\ \hline
344 & The scent of freshly cut grass filled the air as the field was mowed. \\ \hline
345 & They walked through the field of wheat, the stalks brushing against their legs. \\ \hline
346 & Children ran around playing in the field all day. \\ \hline
347 & The field was a quiet refuge from the busy city. \\ \hline
348 & The field was a peaceful place to escape to. \\ \hline
349 & He could see the field from his bedroom window. \\ \hline
350 & The field was a favorite spot for local artists. \\ \hline
351 & The field was full of life, even in the winter. \\ \hline
352 & A fence made of wooden posts and wire encircled the field. \\ \hline
353 & The field was a burst of color in the autumn. \\ \hline
354 & He loved to run through the field with his friends. \\ \hline
355 & A tractor moved slowly across the field, tilling the soil. \\ \hline
356 & She lay in the field, watching the clouds drift by. \\ \hline
357 & The field was a sea of green during the spring. \\ \hline
358 & He loved the peace and quiet of the open field. \\ \hline
359 & The field was a patchwork of different crops. \\ \hline
360 & The field was a favorite spot for flying drones. \\ \hline
361 & Farmers harvested hay from the field, stacking it neatly in bales. \\ \hline
362 & She found a quiet corner of the field to meditate. \\ \hline
363 & The field was a sea of green during the spring. \\ \hline
364 & Children played soccer in the field behind the school, their laughter echoing. \\ \hline
365 & The tractor moved slowly across the plowed field. \\ \hline
366 & The field was dotted with patches of clover. \\ \hline
367 & He built a small shed at the edge of the field. \\ \hline
368 & The field was a place of beauty and tranquility. \\ \hline
369 & He set up a telescope in the field to watch the stars. \\ \hline
370 & The field was used for growing sunflowers. \\ \hline
371 & She enjoyed walking through the field, listening to the birds sing. \\ \hline
372 & The field of wheat stretched across the horizon, golden under the afternoon sun. \\ \hline
373 & They planted a variety of vegetables in the field. \\ \hline
374 & She loved the scent of fresh earth in the field after it rained. \\ \hline
375 & The field was a peaceful place to reflect and relax. \\ \hline
376 & They picked wildflowers from the edge of the field. \\ \hline
377 & A lone tree stood in the middle of the field. \\ \hline
378 & He could see the field from his bedroom window. \\ \hline
379 & The field of hemp grew tall and strong, its fibers used for various products. \\ \hline
380 & He loved the smell of fresh-cut hay in the field. \\ \hline
381 & The field was buzzing with bees collecting nectar. \\ \hline
382 & A rainbow arched over the field after the rain. \\ \hline
383 & They planted rows of vegetables in the fertile field. \\ \hline
384 & The field was a peaceful place to reflect and relax. \\ \hline
385 & The field of vineyards produced grapes for fine wines, rows of vines neatly trellised. \\ \hline
386 & She found an old, weathered barn at the edge of the field. \\ \hline
387 & They picnicked in the field of clover, enjoying sandwiches and lemonade. \\ \hline
388 & She enjoyed painting the landscape of the field. \\ \hline
389 & We took a walk through the field at sunset. \\ \hline
390 & They built a bonfire in the field, its flames lighting up the night sky. \\ \hline
391 & The farmer's dog ran joyfully through the field. \\ \hline
392 & The field was the perfect spot for a family gathering. \\ \hline
393 & She enjoyed picnicking in the field with her family. \\ \hline
394 & The field was covered in a thick layer of frost. \\ \hline
395 & The field of potatoes was ready for digging, the earth yielding its treasures. \\ \hline
396 & He could see deer grazing in the field from his window. \\ \hline
397 & The field was a place of beauty and tranquility. \\ \hline
398 & The field was a perfect spot for a family picnic. \\ \hline
399 & Rain nourished the field, ensuring a bountiful crop for the season. \\ \hline
400 & The field was a haven for birdwatchers. \\ \hline

    \caption{Field Of Land Sentences}
    \label{tab:field_land}
    \end{longtable}

%% file: UpdatedFigures/SentenceTables/field_study_table.tex
    \begin{longtable}{|c|p{12cm}|}
    \hline
    \textbf{Index} & \textbf{Sentence} \\ \hline
    \endfirsthead
    \multicolumn{2}{c}{\textit{Continued from previous page}} \\ \hline
    \textbf{Index} & \textbf{Sentence} \\ \hline
    \endhead
    \hline \multicolumn{2}{r}{\textit{Continued on next page}} \\ \hline
    \endfoot
    \hline
    \endlastfoot
    1 & The field of bioethics addresses the ethical issues in biology and medicine. \\ \hline
2 & He is an innovator in the field of software development. \\ \hline
3 & The conference attracted top professionals specialized in an interesting field of study. \\ \hline
4 & The field of cultural studies examines how culture shapes identity. \\ \hline
5 & The field of immunology studies the immune system. \\ \hline
6 & The field of anthropology studies human cultures and societies. \\ \hline
7 & His expertise in the field of structural engineering is invaluable. \\ \hline
8 & My favorite baseball movie is field of dreams. \\ \hline
9 & The field of socio-cultural anthropology examines human societies and their customs. \\ \hline
10 & The field of cultural anthropology studies human societies. \\ \hline
11 & The field of operations research uses mathematical methods to make decisions. \\ \hline
12 & The field of evolutionary psychology explores the evolutionary origins of human behavior. \\ \hline
13 & He is an authority in the field of health policy and management. \\ \hline
14 & He is a leading expert in the field of forensic anthropology. \\ \hline
15 & Her studies in the field of dance theory are intriguing. \\ \hline
16 & He has made strides in the field of molecular genetics. \\ \hline
17 & His work in the field of plasma physics is highly regarded. \\ \hline
18 & The field of computational sociology uses computational methods to study social phenomena. \\ \hline
19 & His studies in the field of marine biology are fascinating. \\ \hline
20 & The field of telecommunications is rapidly advancing. \\ \hline
21 & The field of supply chain management is vital for global commerce. \\ \hline
22 & He decided to pursue a career in the field of computer science. \\ \hline
23 & The field of biotechnology holds great promise for the future. \\ \hline
24 & She has a background in the field of political science. \\ \hline
25 & He is a leading researcher in the field of evolutionary genetics. \\ \hline
26 & The field of digital forensics investigates cybercrimes. \\ \hline
27 & The field of transpersonal psychology explores spiritual and transcendent aspects of the human experience. \\ \hline
28 & He is a pioneer in the field of machine learning. \\ \hline
29 & Her research in the field of music cognition explores how the brain processes music. \\ \hline
30 & The field of sociology looks at how societies function. \\ \hline
31 & Her research in the field of computational chemistry is groundbreaking. \\ \hline
32 & Her expertise in the field of infectious diseases informs public health policies. \\ \hline
33 & Her expertise in the field of environmental sociology addresses human interactions with the environment. \\ \hline
34 & The field of optics studies light and its interactions. \\ \hline
35 & The field of biogeography studies the distribution of species across geographical areas. \\ \hline
36 & He received an award for his contributions to the field of engineering. \\ \hline
37 & Her passion for the field of public health is evident. \\ \hline
38 & The field of marketing explores consumer behavior and advertising strategies. \\ \hline
39 & She is passionate about her work in the field of social work. \\ \hline
40 & The field of actuarial science assesses financial risks. \\ \hline
41 & The field of hydrology studies the distribution and movement of water on Earth. \\ \hline
42 & She has received accolades for her work in the field of artificial intelligence. \\ \hline
43 & The field of health informatics improves patient care through data. \\ \hline
44 & The field of neuroinformatics combines neuroscience and data analysis. \\ \hline
45 & She is highly respected in the field of architectural history. \\ \hline
46 & He is a pioneer in the field of digital humanities. \\ \hline
47 & The field of sports medicine focuses on athletes' health and performance. \\ \hline
48 & The field of chemical engineering involves the creation of new materials. \\ \hline
49 & He is a leading voice in the field of peace and conflict studies. \\ \hline
50 & He is a prominent figure in the field of environmental sociology. \\ \hline
51 & The field of sports medicine helps athletes recover from injuries. \\ \hline
52 & The field of industrial design merges function with aesthetics. \\ \hline
53 & The field of geriatric medicine focuses on elderly care. \\ \hline
54 & The field of behavioral economics blends psychology and economics. \\ \hline
55 & Her research in the field of climatology addresses global warming. \\ \hline
56 & His research in the field of evolutionary biology is groundbreaking. \\ \hline
57 & Her research in the field of computer graphics enhances visual simulation techniques. \\ \hline
58 & The field of health informatics uses technology to improve healthcare delivery. \\ \hline
59 & He has a deep interest in the field of computational neuroscience. \\ \hline
60 & He is a renowned figure in the field of machine learning. \\ \hline
61 & The field of medical anthropology explores the intersection of culture and health. \\ \hline
62 & The field of acoustics studies sound and its properties. \\ \hline
63 & Her research in the field of gerontology focuses on aging. \\ \hline
64 & The field of molecular gastronomy explores the science behind cooking. \\ \hline
65 & She is advancing knowledge in the field of developmental psychology. \\ \hline
66 & The field of physical therapy helps people recover mobility. \\ \hline
67 & The field of social epidemiology examines health disparities. \\ \hline
68 & The field of artificial intelligence presents many ethical questions. \\ \hline
69 & She is a trailblazer in the field of behavioral neuroscience. \\ \hline
70 & Her expertise in the field of neuroimaging enhances brain research. \\ \hline
71 & She is an innovator in the field of fashion design. \\ \hline
72 & His work in the field of artificial intelligence has been widely recognized. \\ \hline
73 & He has made significant strides in the field of computational neuroscience. \\ \hline
74 & She is making waves in the field of renewable energy. \\ \hline
75 & The field of computational neuroscience models neural systems and behavior. \\ \hline
76 & His studies in the field of immunology are groundbreaking. \\ \hline
77 & He is a leading scholar in the field of information science. \\ \hline
78 & He is a leading figure in the field of aerospace engineering. \\ \hline
79 & The field of digital sociology explores the impact of digital technologies on society. \\ \hline
80 & Her expertise in the field of gerontology addresses aging. \\ \hline
81 & The field of linguistics helps us understand language structure. \\ \hline
82 & The field of library science organizes and manages information resources. \\ \hline
83 & Her expertise in the field of data science is highly sought after. \\ \hline
84 & The field of literary criticism involves analyzing and interpreting texts. \\ \hline
85 & He is a pioneer in the field of disaster risk reduction. \\ \hline
86 & She is an influential figure in the field of gender studies. \\ \hline
87 & The field of cognitive anthropology studies cultural variations in cognition. \\ \hline
88 & The field of agronomy deals with crop production and soil management. \\ \hline
89 & He is a leading figure in the field of quantum information science. \\ \hline
90 & He has published extensively in the field of theoretical physics. \\ \hline
91 & The field of health informatics combines healthcare and IT. \\ \hline
92 & Her research in the field of psychopharmacology examines the effects of drugs on behavior. \\ \hline
93 & He has authored several books in the field of history. \\ \hline
94 & The field of international law governs legal relations between states. \\ \hline
95 & The field of ethology examines animal behavior in natural environments. \\ \hline
96 & The field of social psychology studies how individuals influence each other. \\ \hline
97 & The field of political economy studies the relationship between politics and economics. \\ \hline
98 & The field of anthropology examines human societies and cultures. \\ \hline
99 & Her expertise in the field of marine archaeology uncovers submerged history. \\ \hline
100 & Her work in the field of human-computer interaction designs user-friendly interfaces. \\ \hline
101 & He is well-known in the field of quantum physics. \\ \hline
102 & He has dedicated his career to the field of atmospheric sciences. \\ \hline
103 & She has a distinguished career in the field of comparative literature. \\ \hline
104 & The field of cultural studies examines cultural phenomena. \\ \hline
105 & Her work in the field of artificial intelligence is innovative. \\ \hline
106 & He is a notable expert in the field of space exploration. \\ \hline
107 & Her research in the field of cognitive anthropology explores cultural cognition. \\ \hline
108 & Her work in the field of cognitive psychology studies mental processes. \\ \hline
109 & He is a well-known expert in the field of civil engineering. \\ \hline
110 & The field of behavioral economics explores how psychology impacts economic decisions. \\ \hline
111 & The field of consumer psychology studies consumer behavior and decision-making. \\ \hline
112 & The field of artificial intelligence is growing rapidly. \\ \hline
113 & The field of aerospace science investigates flight and space. \\ \hline
114 & She has published numerous papers in the field of environmental science. \\ \hline
115 & She has a keen interest in the field of film studies. \\ \hline
116 & His innovations in the field of electrical engineering are impressive. \\ \hline
117 & Her research in the field of mobile computing enhances smartphone technology. \\ \hline
118 & The field of computational economics applies computational methods to economic analysis. \\ \hline
119 & The field of marine chemistry studies the chemical composition of oceans. \\ \hline
120 & His work in the field of marine biology is groundbreaking. \\ \hline
121 & Her innovations in the field of textile science are notable. \\ \hline
122 & The field of consumer psychology explores why people buy things. \\ \hline
123 & The field of evolutionary psychology explores human behavior. \\ \hline
124 & The field of computational biology uses data to understand biology. \\ \hline
125 & He has extensive experience in the field of quantum computing. \\ \hline
126 & The field of neuroscience delves into the workings of the brain. \\ \hline
127 & The field of veterinary science focuses on animal health. \\ \hline
128 & She has a strong foundation in the field of theoretical physics. \\ \hline
129 & The field of music therapy uses music to improve mental health. \\ \hline
130 & She has a profound impact on the field of forensic science. \\ \hline
131 & His work in the field of linguistics has redefined language theories. \\ \hline
132 & Her interest in the field of space exploration began in childhood. \\ \hline
133 & The field of industrial design creates functional and aesthetic products. \\ \hline
134 & His studies in the field of artificial intelligence are influential. \\ \hline
135 & The field of educational psychology enhances teaching methods. \\ \hline
136 & He is a notable figure in the field of mechanical engineering. \\ \hline
137 & The field of geophysics examines the physical properties of the Earth. \\ \hline
138 & Her work in the field of urban planning promotes sustainable urban development. \\ \hline
139 & The field of medical imaging develops techniques to visualize internal organs. \\ \hline
140 & She has been working in the field of bioengineering for over a decade. \\ \hline
141 & He has a background in the field of developmental psychology. \\ \hline
142 & The field of musicology delves into the study of music. \\ \hline
143 & He is a renowned scholar in the field of classical studies. \\ \hline
144 & The field of computational archaeology uses computer models to study archaeological data. \\ \hline
145 & He is an authority in the field of cybersecurity. \\ \hline
146 & Her contributions to the field of immunology have been invaluable. \\ \hline
147 & The field of mathematical biology applies mathematical models to biological processes. \\ \hline
148 & Her research in the field of neuroscience is highly respected. \\ \hline
149 & The field of sports psychology helps athletes improve performance. \\ \hline
150 & The field of environmental engineering seeks sustainable solutions. \\ \hline
151 & The field of social neuroscience investigates the neural basis of social behavior. \\ \hline
152 & The field of social geography studies the spatial distribution of social phenomena. \\ \hline
153 & The field of landscape architecture focuses on designing outdoor spaces. \\ \hline
154 & The field of psycholinguistics investigates the relationship between language and the mind. \\ \hline
155 & He is a trailblazer in the field of genetic counseling. \\ \hline
156 & The field of materials science investigates the properties of materials. \\ \hline
157 & The field of computational linguistics develops algorithms for natural language processing. \\ \hline
158 & He has a background in the field of political sociology. \\ \hline
159 & The field of semiotics analyzes signs and symbols in communication. \\ \hline
160 & The field of planetary science explores the formation and evolution of planets. \\ \hline
161 & Her research in the field of robotics engineering advances automation technology. \\ \hline
162 & He is a specialist in the field of aerospace engineering. \\ \hline
163 & Her research in the field of cultural psychology investigates cultural influences on cognition. \\ \hline
164 & She has published numerous papers in the field of quantum mechanics. \\ \hline
165 & He is an expert in the field of agribusiness management. \\ \hline
166 & The field of robotics is seeing remarkable advancements. \\ \hline
167 & The field of entomology studies insects and their behaviors. \\ \hline
168 & The field of literary criticism analyzes literary works. \\ \hline
169 & Her expertise in the field of transportation engineering is crucial for infrastructure projects. \\ \hline
170 & The field of pharmacology investigates how drugs affect the body. \\ \hline
171 & She received an award for her work in the field of environmental science. \\ \hline
172 & She is working on a project in the field of urban planning. \\ \hline
173 & Her research in the field of endocrinology has been transformative. \\ \hline
174 & She is conducting groundbreaking work in the field of bioinformatics. \\ \hline
175 & The field of veterinary medicine cares for animal health. \\ \hline
176 & Her work in the field of evolutionary psychology examines psychological traits. \\ \hline
177 & She has made a name for herself in the field of textile engineering. \\ \hline
178 & The field of forensic science applies scientific methods to criminal investigations. \\ \hline
179 & The field of quantum optics studies the behavior of light and matter at the quantum level. \\ \hline
180 & He is renowned in the field of human-computer interaction. \\ \hline
181 & The field of biotechnology holds great promise for the future. \\ \hline
182 & The field of paleoclimatology reconstructs past climate conditions. \\ \hline
183 & His contributions to the field of economics have been groundbreaking. \\ \hline
184 & He has a deep interest in the field of robotics. \\ \hline
185 & The field of musicology analyzes music history and theory. \\ \hline
186 & Her expertise in the field of forensic anthropology aids in criminal investigations. \\ \hline
187 & The field of computational genetics analyzes genetic data using computational methods. \\ \hline
188 & Her expertise in the field of legal studies is unmatched. \\ \hline
189 & He is a specialist in the field of cardiology. \\ \hline
190 & She has made significant strides in the field of biotechnology. \\ \hline
191 & Her contributions to the field of environmental law are significant. \\ \hline
192 & He is a key player in the field of financial technology. \\ \hline
193 & Her research in the field of computational neuroscience is influential. \\ \hline
194 & She is interested in the field of cognitive science. \\ \hline
195 & He has a background in the field of telecommunications. \\ \hline
196 & He is a respected authority in the field of clinical research. \\ \hline
197 & Her discoveries in the field of biochemistry have been revolutionary. \\ \hline
198 & Her research in the field of artificial life simulates biological processes in computer models. \\ \hline
199 & She has a strong background in the field of political science. \\ \hline
200 & The field of bibliometrics analyzes academic publication patterns. \\ \hline
201 & Her contributions to the field of educational technology are noteworthy. \\ \hline
202 & The field of meteorology studies weather patterns and forecasting. \\ \hline
203 & Her work in the field of financial engineering optimizes investment strategies. \\ \hline
204 & The field of bioethics navigates moral issues in medicine. \\ \hline
205 & The field of environmental economics studies the economic impact of environmental policies. \\ \hline
206 & She has a strong foundation in the field of human resources. \\ \hline
207 & He has a background in the field of peace studies and conflict resolution. \\ \hline
208 & The field of organizational behavior studies how individuals and groups behave within organizations. \\ \hline
209 & He has made significant contributions to the field of renewable energy. \\ \hline
210 & The field of forestry studies the management of forests and natural resources. \\ \hline
211 & The field of agricultural science seeks to improve food production. \\ \hline
212 & The field of photonics involves the study of light generation and manipulation. \\ \hline
213 & He is a recognized expert in the field of computational physics. \\ \hline
214 & He is an innovator in the field of genetic research. \\ \hline
215 & The field of computational biology uses computational methods to analyze biological data. \\ \hline
216 & Her research in the field of human-computer interaction improves user experience. \\ \hline
217 & The field of molecular biology examines the building blocks of life. \\ \hline
218 & His research in the field of telecommunications has advanced the industry. \\ \hline
219 & He is researching climate change within the field of environmental studies. \\ \hline
220 & The field of historical linguistics studies language change over time. \\ \hline
221 & The field of aerospace medicine focuses on the health of pilots and astronauts. \\ \hline
222 & The field of educational sociology examines educational institutions and processes. \\ \hline
223 & She has dedicated her career to the field of education. \\ \hline
224 & Her expertise in the field of public health is widely recognized. \\ \hline
225 & She is advancing the field of clinical psychology. \\ \hline
226 & The field of visual arts encompasses various creative disciplines. \\ \hline
227 & She is a renowned expert in the field of pediatric medicine. \\ \hline
228 & He has a strong background in the field of systems engineering. \\ \hline
229 & The field of astrophysics seeks to understand the universe. \\ \hline
230 & Her work in the field of bioacoustics explores animal communication through sound. \\ \hline
231 & He is a renowned expert in the field of computational linguistics. \\ \hline
232 & He has a profound impact on the field of digital marketing. \\ \hline
233 & The field of marine science explores oceanic systems. \\ \hline
234 & The field of materials science involves studying the properties of materials. \\ \hline
235 & The field of bioinformatics combines biology and computer science. \\ \hline
236 & The field of environmental chemistry studies chemical processes in the environment. \\ \hline
237 & The field of cognitive development explores how thinking processes evolve over time. \\ \hline
238 & The field of toxicology studies the effects of chemicals on living animals. \\ \hline
239 & Her research in the field of climate science is groundbreaking. \\ \hline
240 & Her contributions to the field of computational fluid dynamics are substantial. \\ \hline
241 & The field of fluid dynamics studies the behavior of liquids and gases. \\ \hline
242 & The field of petrochemical engineering deals with petroleum products. \\ \hline
243 & The field of ergonomics designs equipment for efficiency. \\ \hline
244 & She is a key figure in the field of digital humanities. \\ \hline
245 & He has a background in the field of educational leadership. \\ \hline
246 & He is a pioneer in the field of nanotechnology. \\ \hline
247 & The field of genetics explores the inheritance of traits. \\ \hline
248 & He has a deep interest in the field of computational linguistics. \\ \hline
249 & The field of political sociology studies political institutions and behavior. \\ \hline
250 & The field of paleontology uncovers the history of life on Earth. \\ \hline
251 & The field of astrophysics reveals the wonders of the cosmos. \\ \hline
252 & The field of cognitive science examines how we think and learn. \\ \hline
253 & He has a prolific career in the field of synthetic biology. \\ \hline
254 & The field of quantum computing is still in its infancy. \\ \hline
255 & Her career in the field of art history has been illustrious. \\ \hline
256 & He is making significant contributions to the field of microbiology. \\ \hline
257 & He has a deep understanding of the field of financial mathematics. \\ \hline
258 & The field of criminology examines the causes of crime. \\ \hline
259 & The field of computational linguistics combines language and computing. \\ \hline
260 & He is a leader in the field of pharmacology. \\ \hline
261 & The field of actuarial science helps manage financial risks. \\ \hline
262 & The field of environmental economics addresses the impact of economic activity on natural resources. \\ \hline
263 & Her work in the field of cognitive neuroscience investigates brain function. \\ \hline
264 & Her research in the field of behavioral ecology examines animal behavior. \\ \hline
265 & Her research in the field of nutrition has led to healthier eating guidelines. \\ \hline
266 & The field of developmental linguistics studies language acquisition in children. \\ \hline
267 & The field of gerontology explores aging and its effects on individuals and societies. \\ \hline
268 & The field of dialectology studies regional differences in language. \\ \hline
269 & She has dedicated her life to the field of humanitarian aid. \\ \hline
270 & The field of information technology is constantly changing. \\ \hline
271 & He is an expert in the field of acoustical engineering. \\ \hline
272 & The field of developmental psychology studies human growth and development. \\ \hline
273 & She is advancing the field of artificial intelligence. \\ \hline
274 & The field of climate science studies weather and climate change. \\ \hline
275 & He is a thought leader in the field of sustainable development. \\ \hline
276 & The field of evolutionary ecology examines the adaptation of organisms to their environments. \\ \hline
277 & He has made notable contributions to the field of computer engineering. \\ \hline
278 & She is a distinguished researcher in the field of plant biology. \\ \hline
279 & Her research in the field of conservation biology protects biodiversity. \\ \hline
280 & The field of computational chemistry models chemical structures and reactions. \\ \hline
281 & His work in the field of artificial neural networks is groundbreaking. \\ \hline
282 & Her work in the field of atmospheric science predicts weather patterns. \\ \hline
283 & The field of synthetic chemistry creates new compounds. \\ \hline
284 & The field of archaeology uncovers the secrets of ancient civilizations. \\ \hline
285 & The field of chemistry explores the properties of matter. \\ \hline
286 & Advances in the field of medicine have improved patient outcomes significantly. \\ \hline
287 & He has a deep understanding of the field of artificial intelligence. \\ \hline
288 & Her contributions to the field of visual perception are significant. \\ \hline
289 & He has made significant contributions to the field of behavioral economics. \\ \hline
290 & Her expertise in the field of environmental sociology addresses human-environment interactions. \\ \hline
291 & He is a respected voice in the field of climatology. \\ \hline
292 & The field of visual arts encompasses painting, sculpture, and more. \\ \hline
293 & He has made significant advances in the field of cognitive robotics. \\ \hline
294 & The field of cyber security is critical in today's digital world. \\ \hline
295 & The field of geology studies the Earth's physical structure. \\ \hline
296 & The field of educational psychology applies psychology to educational settings. \\ \hline
297 & The field of behavioral genetics investigates the genetic basis of behavior. \\ \hline
298 & He is a pioneer in the field of computational photography. \\ \hline
299 & He has a background in the field of social psychology. \\ \hline
300 & She has a strong interest in the field of sociology. \\ \hline
301 & Her work in the field of psychology has been groundbreaking. \\ \hline
302 & He has a deep understanding of the field of environmental microbiology. \\ \hline
303 & The field of cognitive neuroscience studies the biological basis of cognition. \\ \hline
304 & She is a prominent researcher in the field of computer vision. \\ \hline
305 & He is an expert in the field of social network analysis. \\ \hline
306 & The field of computational linguistics develops algorithms for natural language processing. \\ \hline
307 & The field of ergonomics designs equipment to improve human use. \\ \hline
308 & He is an expert in the field of urban sociology. \\ \hline
309 & She chose to specialize in the field of bioinformatics. \\ \hline
310 & The field of educational psychology applies psychological principles to education. \\ \hline
311 & The field of artificial intelligence is evolving rapidly. \\ \hline
312 & The field of developmental biology examines the growth of organisms. \\ \hline
313 & The field of occupational therapy helps people perform daily activities. \\ \hline
314 & The field of neuropsychology studies the brain-behavior relationship. \\ \hline
315 & He is a leader in the field of sustainable agriculture. \\ \hline
316 & Her expertise in the field of digital anthropology explores online cultures. \\ \hline
317 & He is a leader in the field of renewable energy. \\ \hline
318 & The field of economics encompasses a wide range of topics. \\ \hline
319 & He has a notable career in the field of emergency management. \\ \hline
320 & The field of population genetics investigates genetic variation within populations. \\ \hline
321 & The field of cultural heritage management preserves historical artifacts. \\ \hline
322 & The field of computational genomics analyzes genetic data using computational methods. \\ \hline
323 & The field of cultural sociology examines cultural patterns and practices. \\ \hline
324 & He is conducting research in the field of renewable energy. \\ \hline
325 & The field of biomedical informatics combines healthcare and data science. \\ \hline
326 & The field of astrophysics explores the mysteries of the universe. \\ \hline
327 & The field of cryptography focuses on securing communication. \\ \hline
328 & The field of computational physics uses numerical methods to study physical phenomena. \\ \hline
329 & She has a degree in the field of marine ecology. \\ \hline
330 & Her work in the field of museum studies enhances cultural preservation. \\ \hline
331 & The field of environmental toxicology studies the effects of pollutants. \\ \hline
332 & The field of media studies examines how media affects society. \\ \hline
333 & He is a prominent figure in the field of data analytics. \\ \hline
334 & He has a background in the field of industrial psychology. \\ \hline
335 & Her research in the field of child psychology is pioneering. \\ \hline
336 & His contributions to the field of artificial intelligence are notable. \\ \hline
337 & The field of biomedical engineering innovates healthcare technologies. \\ \hline
338 & He is a leading expert in the field of developmental economics. \\ \hline
339 & He is a recognized authority in the field of bioinformatics. \\ \hline
340 & The field of neuroeconomics combines neuroscience, psychology, and economics. \\ \hline
341 & He is a pioneer in the field of artificial life research. \\ \hline
342 & She is an authority in the field of biomedical engineering. \\ \hline
343 & The field of archaeology uncovers the mysteries of ancient civilizations. \\ \hline
344 & The field of computational social science uses data to study social phenomena. \\ \hline
345 & She has published extensively in the field of medieval literature. \\ \hline
346 & Her work in the field of digital humanities bridges technology and humanities research. \\ \hline
347 & She is an authority in the field of network security. \\ \hline
348 & He has a deep understanding of the field of nanomaterials. \\ \hline
349 & He is an authority in the field of digital anthropology. \\ \hline
350 & The field of game design creates interactive entertainment experiences. \\ \hline
351 & Her research in the field of evolutionary linguistics explores language evolution. \\ \hline
352 & He is a leading researcher in the field of human rights law. \\ \hline
353 & The field of forensic science is crucial for solving crimes. \\ \hline
354 & He is a leading expert in the field of urban ecology. \\ \hline
355 & The field of marine biology studies ocean ecosystems.Her innovative ideas have reshaped the field of urban design. \\ \hline
356 & The field of biomedical sciences advances medical knowledge. \\ \hline
357 & His career in the field of nanotechnology is flourishing. \\ \hline
358 & She is a leading expert in the field of genetics. \\ \hline
359 & The field of data science is transforming industries worldwide. \\ \hline
360 & The field of genetic engineering is a hot topic in scientific circles. \\ \hline
361 & He is well-versed in the field of cultural anthropology. \\ \hline
362 & The field of ethnomusicology studies music within cultural contexts. \\ \hline
363 & The field of psychometrics measures psychological traits and abilities. \\ \hline
364 & Her work in the field of game theory has practical applications in economics. \\ \hline
365 & She is a thought leader in the field of educational technology. \\ \hline
366 & The field of psychology studies the human mind and behavior. \\ \hline
367 & The field of linguistics offers many fascinating areas of study. \\ \hline
368 & The field of public policy shapes governance and society. \\ \hline
369 & Her research in the field of linguistics has garnered international acclaim. \\ \hline
370 & The field of international relations examines global politics. \\ \hline
371 & Her contributions to the field of artificial intelligence are substantial. \\ \hline
372 & He is a key player in the field of international relations. \\ \hline
373 & The field of ecological economics integrates ecology and economics for sustainable development. \\ \hline
374 & The field of sociology examines social behavior and institutions. \\ \hline
375 & He is considered a pioneer in the field of nanotechnology. \\ \hline
376 & He is an authority in the field of risk management. \\ \hline
377 & He is a recognized authority in the field of rehabilitation engineering. \\ \hline
378 & The field of cognitive science integrates psychology, neuroscience, and linguistics. \\ \hline
379 & His expertise in the field of supply chain management is invaluable. \\ \hline
380 & The field of computer vision develops algorithms for interpreting visual data. \\ \hline
381 & She is a thought leader in the field of environmental law. \\ \hline
382 & The field of cybersecurity is essential for protecting information systems. \\ \hline
383 & He is a respected scholar in the field of urban planning. \\ \hline
384 & The field of biophysics combines biology and physics principles. \\ \hline
385 & The field of climate modeling predicts future climate changes. \\ \hline
386 & Her work in the field of neuroethics addresses the moral implications of neuroscience. \\ \hline
387 & Her insights in the field of strategic management are highly valued. \\ \hline
388 & She is exploring new techniques in the field of digital art. \\ \hline
389 & The field of educational psychology helps improve teaching methods. \\ \hline
390 & Her work in the field of social work helps vulnerable populations. \\ \hline
391 & The field of artificial intelligence is continuously evolving. \\ \hline
392 & He is highly respected in the field of electrical engineering. \\ \hline
393 & He has a deep understanding of the field of molecular biology. \\ \hline
394 & Her expertise in the field of computational chemistry aids drug discovery. \\ \hline
395 & The field of urban sociology explores the dynamics of cities. \\ \hline
396 & Her studies in the field of strategic management help businesses thrive. \\ \hline
397 & He has made significant contributions to the field of game development. \\ \hline
398 & Her work in the field of computational neuroscience models neural processes. \\ \hline
399 & The field of medicine requires years of rigorous training. \\ \hline
400 & The field of educational technology enhances teaching and learning through technology. \\ \hline

    \caption{Field Of Study Sentences}
    \label{tab:field_study}
    \end{longtable}